\documentclass[12pt]{article}
\usepackage{amsmath}
\usepackage{amsfonts}
\usepackage{graphicx,psfrag,epsf}
\usepackage{enumerate}
\usepackage[round]{natbib}
\usepackage{url} 

\newcommand{\blind}{1}

\usepackage{geometry}
\usepackage{amsmath}
\usepackage{amsfonts}
\usepackage{algorithm}
\usepackage[algo2e]{algorithm2e}
\usepackage{setspace}
\usepackage[usenames, dvipsnames]{color}
\usepackage[shortlabels]{enumitem} 
\usepackage{caption}
\usepackage[labelformat=simple]{subcaption}
\usepackage{siunitx}
\usepackage{float}
\usepackage[normalem]{ulem}
\usepackage{mathtools}
\usepackage[dvipsnames,table,xcdraw]{xcolor} %
\usepackage{bbm}
\usepackage[utf8]{inputenc}
\usepackage[english]{babel}
\usepackage[noblocks]{authblk}
\usepackage{amsthm}
\newtheorem{theorem}{Theorem}[section]
\newtheorem{corollary}{Corollary}[theorem]
\newtheorem{lemma}[theorem]{Lemma}
\newtheorem{proposition}[theorem]{Proposition}
\newtheorem{remark}{Remark}
\usepackage{xr-hyper}
\usepackage{hyperref}
\hypersetup{
    colorlinks=true,
    linkcolor=blue,
    citecolor=blue,
    filecolor=magenta,      
    urlcolor=blue,
}

\SetKwComment{Comment}{/* }{ */}

\newcommand{\bs}{\boldsymbol{s}}
\newcommand{\bg}{\boldsymbol{g}}
\newcommand{\bh}{\boldsymbol{h}}
\newcommand{\bL}{\boldsymbol{L}}
\newcommand{\bw}{\boldsymbol{w}}
\newcommand{\bm}{\boldsymbol{m}}
\newcommand{\be}{\boldsymbol{e}}
\newcommand{\bE}{\boldsymbol{\epsilon}}
\newcommand{\bW}{\boldsymbol{\Omega}}
\newcommand{\bv}{\boldsymbol{v}}
\newcommand{\bX}{\boldsymbol{X}}
\newcommand{\bZ}{\boldsymbol{Z}}
\newcommand{\bx}{\boldsymbol{x}}
\newcommand{\by}{\boldsymbol{y}}
\newcommand{\bz}{\boldsymbol{z}}
\newcommand{\bU}{\boldsymbol{U}}
\newcommand{\bphi}{\boldsymbol{\phi}}
\newcommand{\bmu}{\boldsymbol{\mu}}

\newcommand{\bzeta}{\boldsymbol{\zeta}}
\newcommand{\sumK}{\sum_{k=1}^K}
\newcommand{\prodK}{\prod_{k=1}^K}
\newcommand{\sumN}{\sum_{j=1}^{n_s}}
\newcommand{\prodN}{\prod_{j=1}^{n_s}}
\newcommand{\btheta}{\boldsymbol{\theta}}
\newcommand{\bgamma}{\boldsymbol{\gamma}}

\definecolor{mycyan}{RGB}{15, 85, 204}
\definecolor{mygreen}{RGB}{56, 118, 28}
\definecolor{myred}{RGB}{204, 0, 0}

\newcommand{\LZadd}[1]{\textcolor{black}{#1}}

\newcommand{\RH}[1]{}

\begin{document}
\def\spacingset#1{\renewcommand{\baselinestretch}%
{#1}\small\normalsize} \spacingset{1}

\hypersetup{citecolor=blue}
\if1\blind
{
  \title{\bf Flexible and efficient emulation of spatial extremes processes via variational autoencoders}
  \author{Likun Zhang, Xiaoyu Ma, Christopher K. Wikle\hspace{.2cm}\\\vspace{-0.4cm}
    Department of Statistics, University of Missouri\\
    and \\
    Rapha\"{e}l Huser \\
    Statistics Program, Computer, Electrical and Mathematical Sciences and Engineering (CEMSE) Division, King Abdullah University of Science and Technology (KAUST)}
  \maketitle
}\fi

\begin{abstract}
Many real-world processes have complex tail dependence structures that cannot be characterized using classical Gaussian processes. More flexible spatial extremes models exhibit appealing extremal dependence properties but are often exceedingly prohibitive to fit and simulate from in high dimensions. In this paper, we aim to push the boundaries on computation and modeling of high-dimensional spatial extremes via integrating a new spatial extremes model that has flexible and non-stationary dependence properties in the encoding-decoding structure of a variational autoencoder called the XVAE. The XVAE can emulate spatial observations and produce outputs that have the same statistical properties as the inputs, especially in the tail. Our approach also provides a novel way of making fast inference with complex extreme-value processes. Through extensive simulation studies, we show that our XVAE is substantially more time-efficient than traditional Bayesian inference while outperforming many spatial extremes models with a stationary dependence structure. Lastly, we analyze a high-resolution satellite-derived dataset of sea surface temperature in the Red Sea, which includes 30 years of daily measurements at 16703 grid cells. We demonstrate how to use XVAE to identify regions susceptible to marine heatwaves under climate change and examine the spatial and temporal variability of the extremal dependence structure.
\end{abstract}

\noindent%
\vfill

\newpage
\spacingset{1.73} 

\section{Introduction}
\LZadd{Statistical emulators, pioneered by \citet{sacks1989designs} and \citet{kennedy2001bayesian}, have been mostly used to accurately approximate patterns and relationships in deterministic model outputs (e.g., from climate models, fluid dynamics or other physical systems), which are computationally prohibitive to obtain at high spatio-temporal resolution. Statistical emulators, used as surrogate models, have thus been beneficial for model calibration, where one estimates unknown parameters of a deterministic model by aligning model outputs with observed data \citep[e.g.,][]{higdon2004combining, bayarri2007computer, chang2016calibrating, gopalan2022higher}.}

\LZadd{Another related key application of statistical emulators is to use them with real or model-output data to quickly generate large ensembles of realistic simulations of complex (random or deterministic) spatio-temporal processes. This is especially advantageous to improve uncertainty quantification (UQ) for various inference targets \citep[see, e.g.,][]{gramacy2020surrogates}, particularly when assessing risks related to rare events, e.g., defined as joint exceedances over high thresholds. For instance, current marine heatwave (MHW) detection methods often involve calculating percentile thresholds empirically from a quite limited number of daily sea surface temperature (SST) observations, averaged spatially over relatively coarse regions \citep{hughes2017global}. Emulating SST data over space and time can thus enhance the estimation of, and UQ for, extreme hotspots defined as regions experiencing high temperatures simultaneously. We come back to such an application in Section~\ref{sec:data_analysis}.}

The efficacy of an emulator hinges greatly on its ability to capture complex spatial variability, which is particularly true when interest lies in the tail dependence structure. However, traditional emulation methods---such as those based on Gaussian processes \citep[e.g.,][]{gu2018robust}, polynomial chaos expansions \citep[e.g.,][]{sargsyan2017surrogate} and more recently, deep neural networks such as generative adversarial networks \citep[][GANs]{goodfellow2014generative} and variational autoencoders \citep[][VAEs]{kingma2013auto}---do not naturally accommodate nor realistically reproduce extreme values, and certainly not dependent extremes. By contrast, classical spatial models justified by extreme-value theory are often overly computationally costly to fit with large datasets \citep{huser2022advances}.

The main methodological contributions of this work are threefold. First, we introduce a novel max-infinitely divisible (max-id) model for spatial extremes with nonstationary dependence structure that varies over both space and time, and formally prove that it flexibly accommodates concurrent and locally dependent extremes. Second, we propose embedding this complex spatial extremes model within a VAE engine, referred to as the XVAE, to facilitate fast inference and simulation in high dimensions. Third, we develop a general validation framework to assess an emulator's quality across low, moderate, and high values. A novel metric with theoretical guarantees is proposed, specifically tailored to evaluate the skill of a spatial model in reproducing dependent extremes. \LZadd{Given that most validation approaches either lack emphasis on the joint tail behavior \citep[e.g.,][]{gneiting2007strictly} or rely on simple bivariate summaries, our proposed framework is a valuable additional tool that complements standard model validation techniques.}

The paper is organized as follows: In Section~\ref{sec:background}, we concisely review classical spatial extremes models and VAEs. In Section~\ref{sec:methodology}, we detail our novel max-id process, derive its flexible extremal dependence properties, and demonstrate how to integrate it within a VAE. We also present our general model validation framework to evaluate spatial process emulators with an emphasis on dependent extremes. In Section~\ref{sec:simulation}, we validate the emulating power of our XVAE through simulations, and compare it to a Gaussian process emulator. In Section~\ref{sec:data_analysis}, we apply the XVAE to high-resolution Red Sea SST, and use it to efficiently enhance UQ of extreme sea temperature hotspot estimates. Finally, in Section~\ref{sec:discussion}, we conclude with some discussion on future research directions.

\section{Background}\label{sec:background}
This section provides background on spatial extremes models and VAEs. Random variables are denoted with capital letters and fixed or observed quantities with lowercase letters.

\subsection{Spatial extremes modeling}\label{sec:comps}
In the spatial extremes literature, extremal dependence is commonly measured by 
\begin{equation}\label{eqn:chi}
  \chi_{ij}(u) = \Pr\{F_j(X_j) > u \mid F_i(X_i) > u\} = \frac{\Pr\{F_j(X_j) > u, F_i(X_i) > u\}}{1-u}\in [0,1],
\end{equation}
for some threshold $u\in(0,1)$ and where $F_i$ and $F_j$ are continuous marginal distribution functions for the random variables $X_i$ and $X_j$, respectively. When $u\approx1$, $\chi_{ij}(u)$ quantifies the probability that one variable is extreme given that another variable is similarly extreme.  If $\chi_{ij}=\lim_{u \rightarrow 1}\chi_{ij}(u) = 0$, $X_i$ and $X_j$ are said to be \emph{asymptotically independent} (AI), and if $\chi_{ij}=\lim_{u \rightarrow 1}\chi_{ij}(u) > 0$, $X_i$ and $X_j$ are \emph{asymptotically dependent} (AD).

Classical asymptotic models such as max-stable \citep{davison2015statistics,davison2012statistical, davison2019spatial} or generalized Pareto \citep{ferreira2014generalized, thibaud2015efficient,de2018high} processes always have $\chi_{ij}>0$, unless $X_i$ and $X_j$ are exactly independent. Conversely, Gaussian processes---or marginal transformations thereof---always have $\chi_{ij}=0$, unless $X_i$ and $X_j$ are perfectly dependent. 

In practice, extremal dependence (i.e., $\chi_{ij}(u)$) estimated from environmental processes is often observed to decay as events get more extreme (i.e., $u\to1$) and to become spatially more localized as their intensity increases \citep{huser2024time}. This phenomenon was observed in numerous studies, e.g., on Dutch wind gust maxima \citep{huser2021max}, threshold exceedances of the daily Fosberg fire index \citep{zhang-2022a}, and winter maximum precipitation data over the Midwest of the U.S. \citep{zhang2022accounting}, just to name a few examples. This implies that the stability property of max-stable and generalized Pareto models is often physically inappropriate. However, a weakening $\chi_{ij}(u)$ as $u$ increases does not necessarily lead to AI, and Gaussian processes have a quite restrictive tail behavior. Therefore, we seek to develop models that exhibit much more flexible tail characteristics and that do not assume an extremal dependence class \textit{a priori}. This is especially important for risk assessment when extrapolating beyond the observed data, as misspecifying the extremal dependence regime can lead to grossly inaccurate joint tail probability estimates.

Recent spatial extremes models have addressed some of these limitations and offer more realistic tail properties. Examples of such models include random scale mixtures \citep[e.g.,][]{Opitz2016, huser2017bridging, Huser2019}, usually applied in the peaks-over-threshold framework, and max-id models \citep[e.g.,][]{reich2012hierarchical, padoan2013extreme, huser2021max, bopp2021hierarchical1,zhong2022modeling}, mostly applied in the block-maxima framework; see \citet{huser2022advances} for an overview. However, these 
models often assume a stationary dependence structure  (in particular, the same dependence class at fixed space-time lag) across space and time, and do not always represent long-range dependence realistically over large geographical domains \citep{hazra2024}. Moreover, the computational demands for fitting such models using standard inference techniques are significant even for moderately-sized datasets \citep[see, e.g.,][who apply such a model on 93 sites]{zhang-2022a}, hampering their applicability to high-resolution climate datasets.

More recently, several attempts have been made to exploit advances in deep learning to facilitate the modeling, inference, and simulation of multivariate and spatial extremes. \citet{richards2022unifying} and \citet{pasche2022neural} use deep extreme quantile regression models to improve the modeling of marginal extremes in spatial and temporal settings, respectively. \LZadd{\citet{boulaguiem2022modeling} use a deep convolutional GAN (called extGAN) to learn the dependence structure of spatial extremes; their approach, however, does not impose any parametric constraint on the extremal dependence structure, which leads to AI. By contrast, \citet{lafon2023vae} develop a VAE tailored to multivariate regularly varying (i.e., jointly heavy-tailed) data; their approach thus only applies to AD data, and it has so far only been validated in small dimensions (specifically, 5 sites in their application).} \citet{lenzi2023neural}, \citet{sainsbury-dale2022, sainsbury2023neural} and \citet{richards2023likelihood} use deep learning methods for fast likelihood-free inference with parametric spatial extreme-value models. These inference methods are amortized \citep{andrew2025amortized} in the sense that they are very fast after an initial upfront computational cost has been incurred to train a neural network. Such methods are simulation-based and cannot, however, easily handle highly-parameterized models such as nonstationary processes (but see \citealp{zammit2020deep}); further, they are meant to provide parameter inferences, not to simultaneously generate realistic data simulations. In the same vein, \citet{majumder2024modeling} use deep learning to speed up updates in a Markov chain Monte Carlo (MCMC) algorithm, in order to fit a complex, but stationary, spatial dependence model. 

In this work, we aim to develop the first VAE able to emulate high-resolution, non-stationary, spatio-temporal extremes data, and that can provide fast parameter inferences and UQ, along with realistic data simulations accounting for the possibility of AI and AD.


\subsection{Variational autoencoder background}\label{sec:VAE_intro}
Bayesian hierarchical models with a lower-dimensional latent process can leverage VAEs for inference and statistical emulation. These models typically assume the joint distribution
\begin{equation*}
   p_{\btheta}(\bx,\bz) = p_{\btheta}(\bx\mid \bz)p_{\btheta}(\bz),
\end{equation*}
where $\bx$ represents observations of a (e.g., physical) process $\bX\in \mathbb{R}^{n_s}$ and $\bz$ denotes realizations of a latent process $\bZ\in \mathbb{R}^{K}$. In the case of spatial data, the vector $\bX$ may be observations of a spatial process $\{X(\bs):\bs\in \mathcal{S}\}$ at $n_s$ locations, and $\bZ$ may be random coefficients from a low-rank basis expansion representation of $\bX$. 

An ideal probabilistic framework for emulating an observed $\bx$ (a number of times, $L$, say) would be to: (1)  estimate parameters $\hat{\btheta}$ given the input $\bx$ and sample latent variables  $\bZ^1,\ldots, \bZ^L$ from the posterior $p_{\hat{\btheta}}(\bz\mid \bx)$; (2) generate $\bX^{l}$ from the posterior predictive distributions $p_{\hat{\btheta}}(\bx\mid \bZ^{l})$, $l=1,\ldots, L$. If the characterization of the distributions is reasonable, the new realizations $\{\bX^{1},\ldots, \bX^{L}\}$ should resemble the input $\bx$, with meaningful variations among the replicates. However, the posterior $p_{\btheta}(\bz\mid \bx)$ is often intractable when the marginal likelihood $p_{\btheta}(\bx)=\int p_{\btheta}(\bx, \bz){\rm d}\bz$ does not have an analytical form, complicating parameter estimation for high-dimensional data using methods like MCMC.

Under the variational Bayes framework, the VAE proposed by \citet{kingma2013auto} approximates the posterior $p_{\btheta}(\bz\mid \bx)$ using a so-called probabilistic \textit{encoder}. Formulated via a multilayer perceptron (MLP) neural network, the encoder maps the input $\bx$ to a variational distribution in the latent space denoted by $q_{\bphi_e}(\bz\mid\bx)$, in which $\bphi_e$ are the weights and biases of the encoder network. Then, a sample $\bZ$ from ${q_{\bphi_e}(\bz\mid\bx)}$ is generated and a \textit{decoder} network acts as an estimator for the model parameters: $\hat{\btheta}_{\text{NN}}=\mathrm{DecoderNeuralNet}_{\bphi_d}(\bZ)$. Finally new realizations of $\bX$ can be generated from $p_{\hat{\btheta}_{\text{NN}}}(\bx\mid \bZ)$. 

We denote through an abuse of notation $p_{\bphi_d}(\bx,\bz) \equiv p_{\hat{\btheta}_{\text{NN}}}(\bx\mid \bz)p_{\hat{\btheta}_{\text{NN}}}(\bz)$,  
$p_{\bphi_d}(\bx) = \int p_{\bphi_d}(\bx,\bz)\text{d}\bz$ and $p_{\bphi_d}(\bz\mid\bx) = p_{\bphi_d}(\bx,\bz)/p_{\bphi_d}(\bx)$. The VAE is typically trained by maximizing the evidence lower bound (ELBO), which balances the log-likelihood and the Kullback--Leibler (KL) divergence between the approximate and true posteriors:
\begin{equation}\label{eqn:ELBO_def}
    \mathcal{L}_{\bphi_e,\bphi_d}(\bx) = \log p_{\bphi_d}(\bx)- D_{KL}\left\{q_{\bphi_e}(\bz\mid\bx)\;||\;p_{\bphi_d}(\bz\mid \bx)\right\}.
\end{equation}
Here, $\log p_{\bphi_d}(\bx)$ is called the \textit{evidence} for $\bx$, and the KL divergence is non-negative. 

In traditional VAEs \citep[e.g.,][]{kingma2019introduction, cartwright2023emulation}, Gaussianity is assumed for both the data model $p_{\bphi_d}(\bx\mid\bz)$ and the encoder $q_{\phi_e}(\bz\mid \bx)$, with the prior $p_{\bphi_d}(\bz)$ often set as a multivariate normal distribution. However, such Gaussian assumptions limit the VAE's ability to capture heavy-tailed distributions \citep{lafon2023vae}.

\section{Methodology}\label{sec:methodology}
To better emulate spatial data with extremes, we define $p_{\btheta}(\bx\mid \bz)$ indirectly through the construction of a novel flexible nonstationary spatial extremes model, introduced in Section~\ref{sec:FlexMaxID}. A detailed description on how the model is integrated into the XVAE is given in Section~\ref{sec:XVAE}. 
In Section~\ref{sec:validation}, we propose a new validation framework that is tailored to assess skill in fitting both the
full range and the joint tail behavior of model outputs. 

\subsection{Flexible nonstationary max-id spatial extremes model}\label{sec:FlexMaxID}

Our model builds upon the max-id process proposed by \citet{reich2012hierarchical} and extended by \citet{bopp2021hierarchical1}. Importantly, a novel extension of our model is its ability to realistically capture the change of asymptotic dependence class as a function of distance, as explained in more detail in Section~\ref{sec:ext_dep}, and to accommodate nonstationarity in space and time. Similar to these earlier works, we define the spatial observation model as
\begin{equation}\label{eqn:model}
    X(\bs)=\epsilon(\bs)Y(\bs),\;\bs\in\mathcal{S},
\end{equation}
where $\mathcal{S}\in \mathbb{R}^2$ is the domain of interest and $\epsilon(\bs)$ is a noise process with independent Fr\'{e}chet$(0,\tau,1/\alpha_0)$ marginal distributions; that is, $\Pr\{\epsilon(\bs)\leq x\}=\LZadd{\exp\{-(x/\tau)^{-1/\alpha_0}\}}$, where $x>0$, $\tau>0$ and $\alpha_0>0$. Then, $Y(\bs)$ is constructed using a low-rank representation: 
\begin{equation}\label{eqn:low_rank_representation}
    Y(\bs)=\left\{\sum_{k=1}^K \omega_k(\bs)^{1/\alpha}Z_{k}\right\}^{\alpha_0},
\end{equation}
where $\alpha\in (0,1)$, $\{\omega_k(\bs): k=1,\ldots,K\}$ are fixed compactly-supported radial basis functions centered at $K$ pre-specified knots such that $\sumK \omega_k(\bs)=1$ for any $\bs\in \mathcal{S}$, and $\{Z_k:k=1,\ldots,K\}$ are independently distributed as exponentially-tilted positive-stable (expPS) random variables, \LZadd{whose densities are of the form
\begin{equation}\label{eqn:expPS_den}
    h(x;\alpha, \gamma_k)=\frac{f_\alpha(x)\exp(-\gamma_k x)}{\exp(-\gamma_k^\alpha)},\; x>0, \; k=1,\ldots,K;
\end{equation}
here, $f_\alpha$ is the density function of a positive-stable variable defined through its Laplace transform $\int_{\mathbb{R}} \exp(-sx) f_\alpha(x){\rm d}x=\exp(-s^\alpha), s\geq 0$ \citep{hougaard1986survival}, $\alpha\in (0,1)$ determines the rate at which the power-law tail of $f_\alpha$ tapers off, and the tilting parameters $\gamma_k\geq 0$ determine the extent of tilting, with larger values of $\gamma_k$ leading to lighter-tailed $Z_k$; see Section~\ref{sec:PS_properties} of the Supplementary Material for details. We write $Z_{k}\stackrel{\text{ind}}{\sim} \mathrm{expPS}(\alpha,\gamma_k)$.} 


Our spatial extremes model, while inspired from \citet{reich2012hierarchical} and \citet{bopp2021hierarchical1}, 
has several key novelties. In both \citet{reich2012hierarchical} and \citet{bopp2021hierarchical1}, the basis functions lack compact support and all tilting parameters are fixed at either $\gamma_k\equiv 0$ or $\gamma_k\equiv\gamma> 0$, resulting in only AD or AI for all pairs of locations, respectively. By contrast, in our model, the use of compactly-supported basis functions and spatially-varying tilting parameters creates a spatial-scale aware extremal dependence model, which enables us to capture local AD or AI for nearby locations while ensuring long-range AI for distant locations---a significant advancement in the spatial extremes literature. Moreover, while both previous works use a noise process with Fr\'{e}chet$(0,1,1/\alpha)$ marginals (i.e., setting $\alpha_0=\alpha$), our approach decouples the noise variance from the tail heaviness, providing better noise control for each time point, while keeping the appealing property of max-infinite divisibility as shown in Section~\ref{sec:ext_dep}. Finally, when temporal replicates are available, we shall allow the concentration parameter $\alpha$ and tilting parameters $\bgamma=\{\gamma_k:k=1,\ldots,K\}$ to take different values across time points (i.e., allowing such parameters, denoted by $\alpha_t$ and $\bgamma_t=\{\gamma_{kt}:k=1,\ldots,K\}$, respectively, to change over time $t$), thus making our spatial extremes model nonstationary over both space and time. To the best of our knowledge, this is the first attempt to capture both spatially and temporally varying extremal dependence structures simultaneously in one model, without sub-domain partitioning, at the spatio-temporal scale that we consider here. \citet{zhong2022modeling} achieved it at a much smaller scale and using quite a rigid covariate-based approach to capture nonstationarity.

\subsubsection{Marginal and dependence properties}\label{sec:ext_dep}
In this section, we 
examine the marginal and joint tail behavior of the spatial model~\eqref{eqn:model}. When temporal replicates are available, one can readily replace the parameters $\alpha$ and $\gamma_k$ with temporally-varying parameters, $\alpha_t$ and $\gamma_{kt}$, respectively. 
For notational simplicity, we write $X_j=X(\bs_j)$, $\omega_{kj}=\omega_k(\bs_j)$, $k=1,\ldots,K$, $j=1,\ldots, n_s$, with $n_s$ the number of observed locations, 
and define $\mathcal{C}_j=\{k:\omega_{kj}\neq 0,k=1,\ldots,K\}$. We require that any location $\bs\in \mathcal{S}$ be covered by at least one basis function, so $\mathcal{C}_j$ cannot be empty for any $j$. 

We first study the marginal distributions of the process~\eqref{eqn:model}.
\begin{proposition}\label{prop:marg_distr}
Let $\mathcal{D} = \{k:\gamma_k=0,\; k=1,\ldots, K\}$ and $\bar{\mathcal{D}}$ be the complement of $\mathcal{D}$. For the process~\eqref{eqn:model}, the marginal distribution function of $X_j=X(\bs_j)$ can be written as
\begin{equation}\label{eqn:marg_cdf}
   F_j(x) = \exp\left\{\sum_{k\in \bar{\mathcal{D}}}\gamma_k^\alpha - \sumK\left(\gamma_k+{\tau}^{{1/\alpha_0}}\omega_{kj}^{{1/\alpha}}x^{-{1/\alpha_0}}\right)^\alpha\right\}.
\end{equation}
As $x\rightarrow\infty$, the survival function $\bar{F}_j(x) = 1-F_j(x)\sim c_j(x/\tau)^{-{1/\alpha_0}}$ if $\mathcal{C}_j\cap \mathcal{D}= \emptyset$, and $\bar{F}_j(x) \sim c_j'(x/\tau)^{-{\alpha/\alpha_0}}$ if $\mathcal{C}_j\cap \mathcal{D}\neq \emptyset$, where $c_j =\alpha\sum_{k\in \bar{\mathcal{D}}} \gamma_k^{\alpha-1}\omega_{kj}^{1/\alpha}$, $c'_j = \sum_{k\in \mathcal{D}}\omega_{kj}$.     
\end{proposition}
The proof of this result can be found in Section~\ref{proof:marg_distr} of the Supplementary Material. It indicates that the process \eqref{eqn:model} has Pareto-like marginal tails at any location in the domain $\mathcal{S}$. If $\mathcal{C}_j\cap \mathcal{D}\neq \emptyset$, that is, if the $j$th location is impacted by an ``un-tilted knot'' (i.e., a knot with $\gamma_k=0$ in the $\mathrm{expPS}(\alpha,\gamma_k)$ distribution of the corresponding latent variable $Z_k$), then $\bar{F}_j(x) =O(x^{-{\alpha/\alpha_0}})$ as $x\rightarrow\infty$ since $\alpha\in (0,1)$. If, however, the location is not within the reach of an un-tilted knot, then the marginal distribution is less heavy-tailed. 

To derive the extremal dependence structure, we first calculate the joint distribution function of a $n_s$-variate random vector $(X_1,\ldots, X_{n_s})^{\rm T}$ drawn from the process \eqref{eqn:model}.
\begin{proposition}\label{prop:joint_distr}
Under the definitions and notation as established in Proposition~\ref{prop:marg_distr}, for locations $\bs_1, \ldots, \bs_{n_s}\in \mathcal{S}$, the exact form of the joint distribution function of the random vector $(X_1,\ldots, X_{n_s})^{\rm T}$ can be written as
\begin{equation}\label{eqn:joint_cdf}
   F(x_1,\ldots, x_{n_s}) = \exp\left\{\sum_{k\in \bar{\mathcal{D}}}\gamma_k^\alpha - \sumK\left(\gamma_k+\tau^{{1/\alpha_0}}\sumN\omega_{kj}^{{1/\alpha}}x_j^{-{1/\alpha_0}}\right)^\alpha\right\}.
\end{equation}
\end{proposition}

The proof of Proposition~\ref{prop:joint_distr} is given in Section~\ref{proof:joint_distr} of the Supplementary Material. Eq.~\eqref{eqn:joint_cdf} ensures that $F^{1/r}(x_1,\ldots, x_{n_s})$ 
is a valid distribution function on $\mathbb{R}^{n_s}$ for any real $r>0$, of the same form as \eqref{eqn:joint_cdf} but with tilting indices $\{\gamma_1r^{-1/\alpha}, \ldots, \gamma_Kr^{-1/\alpha}\}$ and scale parameter $\tau\,r^{-\alpha_0/\alpha}$. By definition, the process $\{X_t(\bs): \bs \in\mathcal{D}\}$ is thus max-infinitely divisible (max-id). It becomes max-stable only when it remains within the same location-scale family, i.e., when $\gamma_1=\cdots=\gamma_K=0$.

We now characterize the tail dependence structure of model~\eqref{eqn:model} using both $\chi_{ij}$ defined in Eq.~\eqref{eqn:chi} and the complementary measure $\eta_{ij}$ defined by $\Pr\{X_i>F_i^{-1}(u),X_j>F_j^{-1}(u)\} = \mathcal{L}\{(1-u)^{-1}\}(1-u)^{1/\eta_{ij}}$, where $\mathcal{L}$ is slowly varying at infinity, i.e., $\mathcal{L}(tx)/\mathcal{L}(t)\rightarrow 1$ as $t\rightarrow\infty$ for all $x>0$. The value of $\eta_{ij}\in (0,1]$ is used to differentiate between the different levels of dependence exhibited by a pair $(X_i,X_j)^{\rm T}$. When $\eta_{ij} = 1$ and $\mathcal{L}(t)\not\rightarrow 0$ as $t\rightarrow \infty$, $(X_i,X_j)^{\rm T}$ is AD ($\chi_{ij}>0$), and the remaining cases are all AI \citep[$\chi_{ij}=0$; see][]{ledford1996statistics}, with stronger tail dependence for larger values of $\eta_{ij}$.

\begin{theorem}\label{thm:dependence_properties}
Under the assumptions of Propositions~\ref{prop:marg_distr} and \ref{prop:joint_distr}, the process $\{X(\bs)\}$ defined in \eqref{eqn:model} has a tail dependence structure characterized as follows: 
\vspace*{-0.4cm}\begin{enumerate}[(a)]
\setlength\itemsep{-0.6em}
    \item\label{item:local_AI} If $\mathcal{C}_i\cap \mathcal{D}= \emptyset$ and $\mathcal{C}_j\cap \mathcal{D}= \emptyset$, we have $\chi_{ij}=0$ with $\eta_{ij}=1/2$. 
    \item\label{prop:case2} If $\mathcal{C}_i\cap \mathcal{D}=\emptyset$ and $\mathcal{C}_j\cap \mathcal{D}\neq \emptyset$, we have $\chi_{ij}=0$ with $\eta_{ij}=\frac{\alpha}{\alpha+1}$ when $\mathcal{C}_i\cap\mathcal{C}_j\neq \emptyset$ and $\eta_{ij}=1/2$ when $\mathcal{C}_i\cap\mathcal{C}_j= \emptyset$.
    
    \item\label{thm:AI_case} If $\mathcal{C}_i\cap \mathcal{D}\neq \emptyset$ and $\mathcal{C}_j\cap \mathcal{D}\neq \emptyset$, we have $\chi_{ij}=2-d_{ij}$ with $\eta_{ij}=1$ when $\mathcal{C}_i\cap \mathcal{C}_j\cap \mathcal{D}\neq \emptyset$, where $d_{ij}=\sum_{k\in \mathcal{D}}\{(\omega_{ki}/c'_i)^{{1/\alpha}}+(\omega_{kj}/c'_j)^{{1/\alpha}}\}^\alpha\in (1,2)$, and $\chi_{ij}=0$ with $\eta_{ij}=1/2$ when $\mathcal{C}_i\cap \mathcal{C}_j\cap \mathcal{D}=\emptyset$.
\end{enumerate}
\end{theorem}
The proof of this result is given in Section~\ref{proof:thm} of the Supplementary Material. The local dependence strength is proportional to the tail-heaviness of the latent variable of the closest knot. There is local AD if $\gamma_k=0$, and local AI if $\gamma_k>0$, as expected. 
\LZadd{The sets $\mathcal{C}_j\cap \mathcal{D}$, $j=1,\ldots, n_s$, are crucial to the behavior of the so-called exponent function in the limiting distribution for normalized maxima. This causes both the asymptotic and sub-asymptotic dependence strength to rely on the tail-heaviness of the local expPS variables and the basis function weights; 
see Remark~\ref{remark:exponent_function} in Section~\ref{proof:thm} of the Supplementary Material for specifics}.

The compactness of the basis functions' support yields long-range exact independence (thus, also AI) for two far-apart sites that are impacted by disjoint sets of basis functions; this is similar in spirit to the Cauchy convolution process of \citet{krupskii2022modeling}, though their model construction is different and less computationally tractable than ours.

\subsection{XVAE: A VAE incorporating the proposed max-id  model}\label{sec:XVAE}
Hereafter, we denote by $\bX_t=\{X_t(\bs_j): j=1,\ldots, n_s\}$ the realizations of process~\eqref{eqn:model} at time $t=1,\ldots, n_t$, and by $\bZ_t=\{Z_{kt}:k=1,\ldots,K\}$ the corresponding latent variables. 

Inference for our flexible extremes model on large spatial datasets poses challenges. A streamlined Metropolis--Hastings MCMC algorithm would be time-consuming and hard to monitor when confronted with the scale of our spatial data in Section~\ref{sec:data_analysis}, where a considerable number of local basis functions $K$ is necessary to capture local extremes. Additionally, when there are many time replicates, inferring time-varying parameters $(\alpha_t,\bgamma_t^\top)^\top$ and latent variables $\bZ_t$ at all time points becomes extremely challenging. 
To overcome these challenges, we modify the encoding-decoding VAE paradigm described in Section~\ref{sec:VAE_intro} to account for our extremes framework. For $t=1,\ldots, n_t$, our encoder $q_{\bphi_e}(\bz_t\mid \bx_t)$ maps each observed replicate $\bx_t$ to the latent space and allows fast random sampling of $\{\bZ_t^1,\ldots, \bZ_t^L\}$ that will be approximately distributed according to the true posterior $p_{\btheta_t}(\cdot\mid\bx_t)$ because of the ELBO regularization, in which $\btheta_t=(\alpha_0,\tau,\alpha_t, \bgamma_t^{\rm T})^{\rm T}$; see Eq.~\eqref{eqn:ELBO_def}. The details of this approach are provided below (see also the illustration in Figure~\ref{fig:VAE-diagram}).

\vspace*{-0.6cm}\paragraph{Approximate Posterior/Encoder ($q_{\bphi_e}(\bz_t\mid \bx_t)$):}
The encoder is defined through
\begin{equation}\label{eqn:encoder_form}
 \begin{split}
 \bz_t&=\bmu_t + \bzeta_t \odot \boldsymbol{\eta}_t, \\[-1ex]
    \eta_{kt}&\stackrel{\text{i.i.d.}}{\sim} \mathrm{Normal}(0,1),\\[-1ex]
    (\bmu_t^{\rm T},\log \bzeta_t^{\rm T})^{\rm T} &= \mathrm{EncoderNeuralNet}_{\bphi_e}(\bx_t),
 \end{split}
\end{equation}
where $\odot$ is the elementwise product, and a standard reparameterization trick with an auxiliary variable $\boldsymbol{\eta}_t=\{\eta_{kt}:k=1,\ldots, K\}$ is used to enable fast computation of Monte Carlo estimates of the gradient $\nabla_{\bphi_e}\mathcal{L}_{\bphi_e, \bphi_d}$. Also, by controlling the mean $\bmu_t$ and variance $\bzeta^2_t$, the distributions $q_{\phi_e}(\bz_t\mid \bx_t)$ are enforced to be close to $p_{\bphi_d}(\bz_t\mid \bx_t)$ for each $t$. This is the primary role of the deep neural network in (\ref{eqn:encoder_form})---to learn the complex relationship between the inputs $\bx_t$ and the latent process $\bz_t$. The specific neural network architecture and implementation details are given in Section~\ref{sec:extVAE_details} of the Supplementary Material.

\vspace*{-0.6cm}\paragraph{Prior on Latent Process ($p_{\bphi_d}({\bz_t})$):}
This is determined by our model construction. Specifically, 
the prior on $\bz_t$ can be written as
\begin{equation}\label{eqn:prior_form}
    p_{\bphi_d}(\bz_t) = \prodK h(z_{kt};\alpha_t,\gamma_{kt}),
\end{equation}
in which $h(\cdot;\alpha_t,\gamma_{kt})$ is the density function  of $\mathrm{expPS}(\alpha_t,\gamma_{kt})$, as defined in \eqref{eqn:expPS_den}. 

\vspace*{-0.6cm}\paragraph{Data Model/Decoder ($p_{\phi_d}(\bx_t\mid \bz_t)$):}
Our decoder is based on the flexible max-id spatial extremes model described in Section~\ref{sec:FlexMaxID}.  Specifically, recall from Eq.~\eqref{eqn:low_rank_representation} that $\Pr(\bX_t\leq \bx_t\mid \bZ_t=\bz_t)= \exp\{-\sumN \left({\tau}/{x_{jt}}\right)^{{1}/{\alpha_0}}\sumK\omega_{kj}^{{1}/{{\alpha}_t}}z_{kt}\}$.
Differentiating this conditional distribution function gives the exact form of the decoder:
\begin{equation}\label{eqn:lik_form}
     p_{\bphi_d}(\bx_t\mid\bz_t) = \left({1/\alpha_0}\right)^{n_s} \left\{\prod_{j=1}^{n_s}  \frac{1}{x_{jt}}\left(\frac{x_{jt}}{\tau y_{jt}}\right)^{-1/\alpha_0}\right\} \exp \left\{ -\sumN \left(\frac{x_{jt}}{\tau y_{jt}}\right)^{-1/\alpha_0}\right\},
\end{equation}
where $y_{jt}=\LZadd{(\sumK\omega_{kj}^{1/{\alpha}_t}z_{kt})^{\alpha_0}}$. This distribution depends on the Fr\'{e}chet parameters $(\alpha_0,\tau)^{\rm T}$ and the dependence parameters $(\alpha_t,\bgamma_t^{\rm T})^{\rm T}$ inherited from the prior distribution of $\bz_t$. The decoder neural network estimates these dependence parameters as
\begin{equation}\label{eqn:decoder_form}
    (\hat{\alpha}_t, \hat{\bgamma}_t^{\rm T})^{\rm T} = \mathrm{DecoderNeuralNet}_{\bphi_{d,0}}(\bZ_t),
\end{equation}
where $\bphi_{d,0}$ are the bias and weight parameters of this neural network (see Eqs.~\eqref{eqn:encoder_weights} and \eqref{eqn:decoder} of the Supplementary Material for more details).  Combining $\bphi_{d,0}$ with the Fr\'{e}chet parameters $(\alpha_0,\tau)^{\rm T}$, we write $\bphi_{d}=(\alpha_0,\tau, \bphi_{d,0}^{\rm T})^{\rm T}$. We use the variational procedure to find estimates of parameters $\bphi_d$ and the encoder neural network parameters $\bphi_e$.

\begin{figure}
    \centering
    \includegraphics[width=0.65\linewidth]{VAE_diagram.png}
    \vskip -0.3cm
    \caption{Diagram of a variational autoencoder (VAE) with the reparameterization trick.}
    \label{fig:VAE-diagram}
\end{figure}

\vspace*{-0.5cm}\paragraph{Encoder/Decoder Estimation:} By drawing $L$ independent samples ${\bZ_t^1,\ldots, \bZ_t^L}$ using Eq.~\eqref{eqn:encoder_form}, we can derive the Monte Carlo estimate of the ELBO,  $\mathcal{L}_{\bphi_e,\bphi_d}(\boldsymbol{x}_t)$, and then find the parameters $\bphi_e$ and $\bphi_d$ that maximize $\sum_{t=1}^{n_t}\mathcal{L}_{\bphi_e,\bphi_d}(\boldsymbol{x}_t)$ via stochastic gradient search, as detailed in Section~\ref{sec:extVAE_details} of the Supplementary Material. We stress again that our XVAE is a ``semi-amortized'' inference approach \citep{andrew2025amortized}: there is a substantial training cost up front, but once the XVAE is trained, posterior simulation of new latent variables $\bZ_t$ can be performed very efficiently following Eq.~\eqref{eqn:encoder_form} and synthetic data can be generated extremely quickly by passing them through the decoder~\eqref{eqn:decoder_form} and sampling from the model $p_{\hat{\btheta}_t}(\bx\mid \bZ_t)$ specified by Eqs.~\eqref{eqn:model} and \eqref{eqn:low_rank_representation}, in which $\hat{\btheta}_t=(\hat{\alpha}_0,\hat{\tau},\hat{\alpha}_t, \hat{\bgamma}_t^{\rm T})^{\rm T}$. 
The XVAE would, however, have to be retrained with new observations $\bX_t$, $t=1,\ldots,n_t$.

The data reconstruction process relies on compactly supported local basis functions at pre-determined knot points, which are not updated with $\bphi_d$ of the decoder.  Although one could choose the knots using a certain space-filling design, we propose a data-driven way to determine the number of knots, their locations, and the radius of basis functions as described in Section~\ref{sec:data_driven_knots} of the Supplementary Material. We show by simulation that this compares favorably to the XVAE initialized with the true knots/radii. Our XVAE implementation in \texttt{R} is publicly accessible on GitHub at \href{https://github.com/likun-stat/XVAE}{https://github.com/likun-stat/XVAE}.

\vspace*{-0.5cm}\paragraph{Uncertainty quantification:}
The decoder~\eqref{eqn:decoder_form} functions as a neural estimator for $(\alpha_t,\bgamma_t^{\rm T})^{\rm T}$. Examining its inferential power is crucial, as accurate emulation heavily relies on precise characterization of spatial inputs. 
Drawing a substantial number of samples from the variational distribution $q_{\bphi_e}(\cdot\mid\bx_t)$ (which is close to ${p_{\bphi_d}(\cdot\mid \bx_t)}$; recall Section~\ref{sec:VAE_intro}) allows us to obtain Monte Carlo estimates of the dependence parameters $(\alpha_t,\bgamma_t^{\rm T})^{\rm T}$ using the decoder~\eqref{eqn:decoder_form}. Combining these estimates yields an approximate sample from the posterior, $(\alpha_t,\bgamma_t^{\rm T})^{\rm T}\mid \{\bx_t:t=1,\ldots, n_t\}$, which enables the calculation of point estimates (posterior mean or maximum \textit{a posteriori}) and approximate confidence regions for UQ.


\subsection{Validation framework for extremes emulation}\label{sec:validation}
We propose a validation framework tailored to assess both the full data range and the joint tail behavior in outputs from any generative spatial extremes model.


First, we predict at held-out locations and calculate the mean squared prediction error (MSPE) and the continuous ranked probability score \citep[CRPS;][]{matheson1976scoring, gneiting2007strictly}. Second, we estimate $\chi_{ij}(u)$, as defined in Eq.~\eqref{eqn:chi}, using two methods: (1) To summarize the average decay of dependence with distance even if the process is non-stationary, we treat $\{X(\bs)\}$ as having a stationary, isotropic dependence structure, where $\chi_{ij}(u)\equiv\chi_h(u)$, with $h=||\bs_i-\bs_j||$ as the distance between locations. For a fixed $h$, we compute empirical conditional probabilities $\widehat{\chi}_h(u)$ across a grid of $u$ values; (2) To avoid the restrictive stationary working assumption, we select a reference point $\bs_0$ and estimate the pairwise $\chi_{0j}(u)$ between $\bs_0$ and other locations $\bs_j$ in the spatial domain $\mathcal{S}$, which can be visualized using raster or heat plots. Third, we examine QQ-plots by pooling spatial data to compare the ranges and quantiles of the input and emulated field. Further details of these diagnostics are provided in Section~\ref{sec:diagnotics} of the Supplementary Material.

Additionally, we propose using a novel joint tail dependence coefficient that formally summarizes the overall dependence strength over the entire spatial domain. This metric characterizes the spatial extent of extreme events conditional on an arbitrary reference point in the domain 
exceeding a particular quantile $u$. \citet{zhang2022accounting} formulated the metric on an empirical basis and named it the averaged radius of exceedances (ARE). 

Given a large number of independent replicates (say $n_r$) from $\{X(\bs)\}$ on a dense regular grid $\mathcal{G}=\{\bg_i\in \mathcal{S}:i=1,\ldots,n_g\}$ over the domain $\mathcal{S}$ with side length $\psi>0$, denote the replicates by $\bX_{r}=\{X_r(\bg_i):i=1,\ldots,n_g\}$, $r=1,\ldots, n_r$. The empirical marginal distribution functions at $\bg_i$ can then be obtained as $\hat{F}_{i}(x)=n_r^{-1}\sum_{r=1}^{n_r}\mathbbm{1}(X_{ir}\leq x)$, where $X_{ir}=X_r(\bg_i)$ and $\mathbbm{1}\{\cdot\}$ is the indicator function. We then transform $(X_{i1},\ldots, X_{in_r})^{\rm T}$ to the uniform scale via $U_{ir} = \hat{F}_{i}(X_{ir})$, $r=1,\ldots, n_r$. Let $\bU_r=\{U_{ir}:i=1,\ldots,n_g\}$ and $U_{0r}=\hat{F}_{0}\{X_r(\bs_0)\}$. The ARE metric at the threshold $u$ is defined by
\begin{equation}\label{eqn:ARE_monte_carlo}
    \widehat{\mathrm{ARE}}_\psi(u) = \left\{\frac{\psi^2\sum_{r=1}^{n_r}\sum_{i=1}^{n_g}\mathbbm{1}(U_{ir}>u, U_{0r}>u)}{\pi\sum_{r=1}^{n_r}\mathbbm{1}(U_{0r}>u)}\right\}^{1/2}.
\end{equation}
The summation $\psi^2\sum_{i=1}^{n_g}\mathbbm{1}(U_{ir}>u, U_{0r}>u)$ in Eq.~\eqref{eqn:ARE_monte_carlo} calculates the area of all grid cells exceeding the extremeness level $u$ jointly with the reference location $\bs_0$, for the same replicate $r$; dividing it by $\pi$ and taking the square root thus yields the ``radius'' of a circular exceedance region that has the same spatial extent. Additionally, Eq.~\eqref{eqn:ARE_monte_carlo} averages over all replicates with the reference location exceeding the extremeness level $u$. Therefore, $\widehat{\mathrm{ARE}}_\psi(u)$ has the same units as $\psi$, or the distance metric used on the domain $\mathcal{S}$, which makes it an interpretable metric for domain scientists because it reflects the average length scale of extreme events (e.g., warm pool size in SST data).

The following result ensures that $\widehat{\mathrm{ARE}}_\psi(u)$, which does not require stationarity or isotropy, converges to the square root of the spatial average of $\chi_{0i}(u)$ as $n_r\rightarrow \infty$.
\begin{theorem}\label{thm:consistency}
    For a fixed regular grid $\mathcal{G}$ with side length $\psi$, a reference location $\bs_0$ and $u\in(0,1)$, we have that, almost surely,  
    \begin{equation}\label{eqn:ARE_psi}
    \widehat{\mathrm{ARE}}_\psi(u)\rightarrow\mathrm{ARE}_{\psi}(u)=\left(\psi^2\sum_{i=1}^{n_g}\chi_{0i}(u)/\pi\right)^{1/2},
\end{equation}
    as $n_r\rightarrow\infty$, where $\chi_{0i}(u)$ is the $\chi$-measure between locations $\bs_0$ and $\bg_i$ defined in Eq.~\eqref{eqn:chi}.
\end{theorem}

Due to the presence of the white noise $\{\epsilon(\bs)\}$, there is no version of the process $\{X(\bs)\}$ that has measurable paths, which means that $X(\bs)\not\rightarrow X(\bs_0)$ (in probability) as $\bs\rightarrow\bs_0$. Nevertheless, we know from Theorem~\ref{thm:dependence_properties} that there is continuity in the dependence measure $\chi_{0i}$ because $\{\epsilon(\bs)\}$ barely impacts the dependence structure of $\{Y(\bs)\}$. That is, $\chi_{\bs_0,\bs}$, denoting the $\chi$-measure between location $\bs_0$ and $\bs$, is a continuous function of $\bs\in\mathcal{S}$ when fixing the reference location $\bs_0$; we define this property as \textit{tail-continuity}. The following result further confirms that under the tail-continuity, $\widehat{\mathrm{ARE}}_\psi(u)$ also converges to the square root of the spatial integral of $\chi_{\bs_0,\bs}$ as $u\rightarrow 1$ and as $\mathcal{G}$ becomes infinitely dense.
\begin{theorem}\label{prop:ARE_psi}
Let the domain $\mathcal{S}$ be bounded (i.e., its area $|\mathcal{S}|<\infty$) and process $\{X(\bs):\bs\in \mathcal{S}\}$ be tail-continuous for $\bs_0$ (i.e., $\chi_{\bs_0,\bs}$ is a continuous function of $\bs$ in $\mathcal{S}$). Then,
\begin{equation}\label{eqn:are_res_limit}
    \lim_{\psi\rightarrow 0, u\rightarrow 1}\psi\left(\sum_{i=1}^{n_g} \chi_{0i}(u)\right)^{1/2}=\left\{\int_{\mathcal{S}} \chi_{\bs_0, \bs} \mathrm{d}\bs\right\}^{1/2}.
\end{equation}
\end{theorem}
\begin{remark}
    Tail-continuity is met by many spatial extremes models, like max-stable, inverted-max-stable, and others \citep[e.g.,][]{Opitz2016,Huser2019, krupskii2022modeling}. Our model~\eqref{eqn:model} also adheres to tail-continuity, as indicated by Theorem~\ref{thm:dependence_properties}.
\end{remark}
\begin{remark}
    Together, Theorems~\ref{thm:consistency} and \ref{prop:ARE_psi} ensure that $\widehat{\mathrm{ARE}}_\psi(u)\approx \left\{\int_{\mathcal{S}} \chi_{\bs_0, \bs} \mathrm{d}\bs\right\}^{1/2}/\pi^{1/2}$ if there are a large number of replicates from the process $\{X(\bs)\}$ on a very dense grid $\mathcal{G}$. 
\end{remark}



Similarly, we can estimate $\mathrm{ARE}_\psi(u)$ for the emulator by running the decoder repeatedly to obtain emulated replicates of $\{X(\bs)\}$ on the same grid. By comparing the $\mathrm{ARE}_\psi(u)$ estimates at a series of $u$ levels, we can evaluate whether spatially-aggregated exceedances are consistent between the spatial data inputs and their XVAE emulation counterparts.

\section{Simulation study}\label{sec:simulation}
In this section, we simulate data from five different parametric models that have varying levels of extremal dependence across space. By examining the diagnostics introduced in Section~\ref{sec:validation}, we validate the efficacy of our XVAE to analyze and emulate data from both model \eqref{eqn:model} and misspecified models.

\subsection{General setting}\label{sec:sim_setting}
We conduct a simulation study in which data are generated at $n_s=2,000$ random locations uniformly sampled over the square $[0,10]\times [0,10]$. We simulate $n_t=100$ replicates of the process from each of the following different models:
\begin{enumerate}[label=\Roman*., ref=\Roman*]
\setlength\itemsep{-0.1em}
    \item \label{modelGP} Gaussian process with zero mean, unit variance, and Mat\'{e}rn correlation $C(\bs_j,\bs_j;\phi,\nu)$, in which $\phi=3$ and $\nu=5/2$ are range and smoothness parameters;
    \item\label{modelAI} Max-id process \eqref{eqn:model} with $K=25$ basis functions and $|\mathcal{D}|=0$ un-tilted knots;
    \item\label{modelFlex} Max-id process \eqref{eqn:model} with $K=25$ basis functions and $0<|\mathcal{D}|<K$ un-tilted knots;
    \item\label{modelAD} Max-id process \eqref{eqn:model} with $K=25$ basis functions and $|\mathcal{D}|=K$ un-tilted knots;
    \item\label{modelMaxStable} Max-stable \citet{reich2012hierarchical} model with $K=25$ basis functions.
\end{enumerate}
When simulating from Models~\ref{modelAI}--\ref{modelAD}, we first consider time-invariant dependence parameters $\alpha_t\equiv 1/2$ and $\bgamma_t\equiv \bgamma$, and attempt to recover the spatial dependence structure; see Figure~\ref{fig:sim_knots} for the knot locations and $\bgamma$ values. Recall that $K$ is the number of basis functions and $\mathcal{D} = \{k:\gamma_k=0\}$. We sample the latent variables $\bZ_{t}$ from the expPS distribution independently for each time replicate. The white noise process $\{\epsilon_t(\bs)\}$ follows the same independent Fr\'{e}chet($0,\tau,1/\alpha_0$) distribution with $\tau=1$ and $\alpha_0=1/4$.
\begin{figure}[!t]
    \centering
    \includegraphics[height=0.3\linewidth]{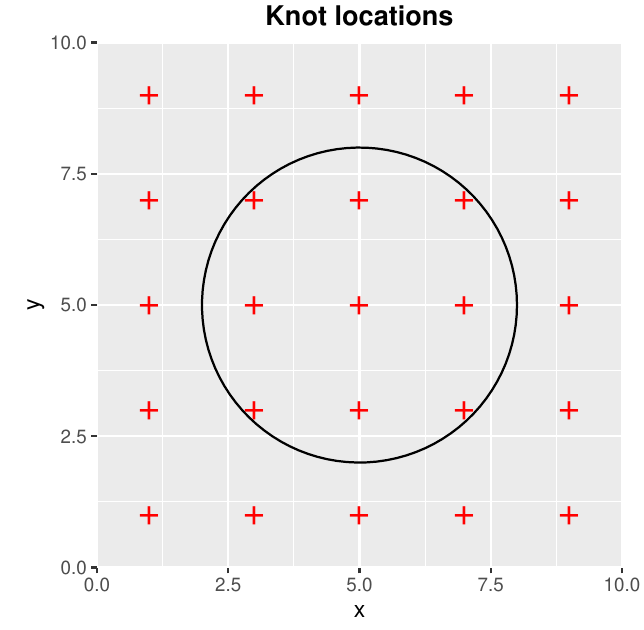}
    \includegraphics[height=0.3\linewidth]{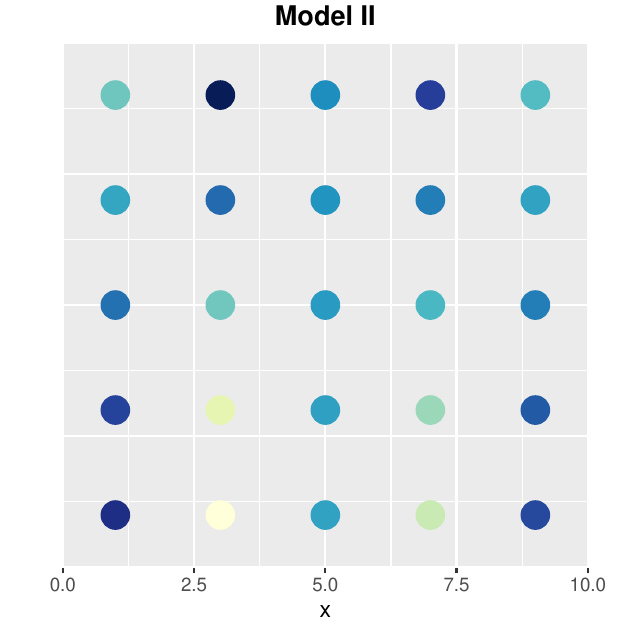}
    \includegraphics[height=0.3\linewidth]{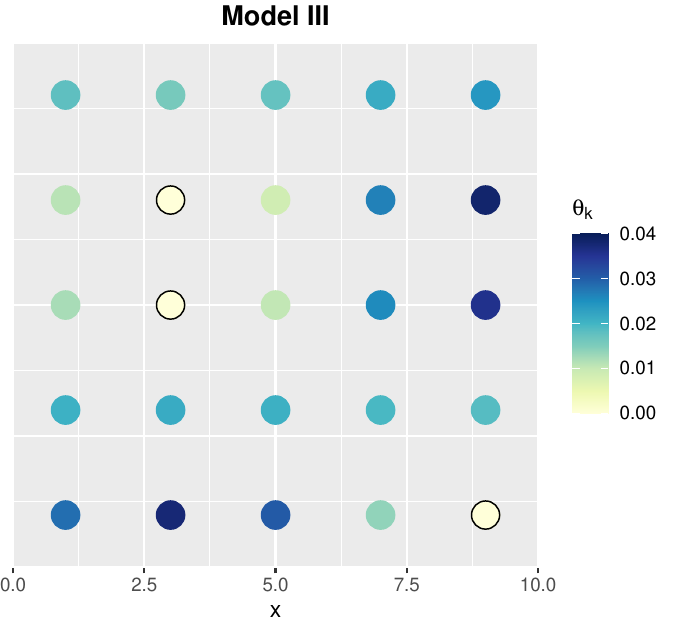}
    \vskip -0.3cm
    \caption{The left panel presents knot locations used for Models~\ref{modelAI}--\ref{modelAD}, and we only show the support of the one Wendland basis function centered at knot in the middle of the domain. Model~\ref{modelMaxStable} uses the same set of knots but the basis functions are not compactly supported. The middle and right panels display the $\gamma_k$ values, $k=1,\ldots, K$, used in the expPS variables for Models~\ref{modelAI} and \ref{modelFlex} respectively. The circled knots signify $\gamma_k=0$, which induces local AD.}
    \label{fig:sim_knots}
\end{figure}

Model~\ref{modelGP} is a stationary and isotropic Gaussian process with a Mat\'{e}rn covariance function. It is known that the joint distribution of the Gaussian process at any two locations $\bs_i$ and $\bs_j$ is light-tailed and leads to AI unless the correlation equals one. 
For Models~\ref{modelAI}--\ref{modelAD}, we simulate data from the max-id model \eqref{eqn:model} with $K=25$ evenly-spread knots across the grid, denoted by $\{\Tilde{\bs}_1, \ldots, \Tilde{\bs}_K\}$. Setting the range parameter to $r=3$, we use compactly supported Wendland basis functions $\omega_k(\bs, r)\propto\{1-d(\bs, \Tilde{\bs}_k)/r\}^2_+$ centered at each knot \citep{wendland1995piecewise}, $k=1,\ldots, K$; see Figure~\ref{fig:sim_knots}. The basis function values are standardized so that for each $\bs$, $\sumK\omega_k(\bs, r)=1$. The main difference between Models~\ref{modelAI}, \ref{modelFlex} and \ref{modelAD} lies in the $\gamma_k$ values: Model~\ref{modelAI} has no zero $\gamma_k$'s (i.e., $|\mathcal{D}|=0$), whereas Model~\ref{modelFlex} has a mix of positive and zero $\gamma_k$'s, and Model~\ref{modelAD} has only zero $\gamma_k$'s (i.e., $|\mathcal{D}|=K$). By Theorem~\ref{thm:dependence_properties}, we know Model~\ref{modelAI} gives only local AI and Model~\ref{modelAD} gives only local AD. In contrast, Model~\ref{modelFlex} gives both local AD and local AI. By contrast, Model~\ref{modelMaxStable} adopts the same set of knots but it uses Gaussian radial basis functions which are not compactly supported. Therefore, Model~\ref{modelMaxStable} is the \citet{reich2012hierarchical} max-stable model. 

Models~\ref{modelGP}--\ref{modelMaxStable} gradually exhibit increasingly stronger extremal dependence, and they can help us test whether the XVAE can capture spatially-varying dependence structures that exhibit local AD and/or local AI. Since the proposed process \eqref{eqn:model} allows $\gamma_k$ to change across the different knots ($k=1,\ldots, K$), a well-trained XVAE should be able to differentiate between local AD ($\gamma_k=0$) and local AI ($\gamma_k>0$). 

Additionally, for each space-time simulated dataset, we randomly set aside 100 locations as a validation set. Subsequently, we analyze the dependence structure of the remaining 1,900 locations using both the proposed XVAE (initialized with data-driven knots unless specified otherwise) and a Gaussian process regression with heteroskedastic noise implemented in the \texttt{R} package \texttt{hetGP} \citep{hetGP}. We then perform predictions at the 100 holdout locations (see Section~\ref{sec:predict} of the Supplementary Material).
In the following, we show that both emulators perform well when emulating datasets from Models~\ref{modelGP} and \ref{modelAI}, but only XVAE 
appropriately captures heavy tails and AD in Models~\ref{modelFlex}--\ref{modelMaxStable}. 

\LZadd{In Section~\ref{sec:nonstat_sim} of the Supplementary Material, we further examine the XVAE's ability to capture $\bgamma_t$ when there is both spatial and temporal nonstationarity. Moreover, in Section~\ref{sec:extGAN_sim} of the Supplementary Material, we simulate data on a regular grid and compare the emulation performance between XVAE and extGAN proposed by \citet{boulaguiem2022modeling}; we see that extGAN has limitations in capturing the extremal dependence appropriately.}

\subsection{Emulation results}
Figure~\ref{fig:comps_across_models} and Figure~\ref{fig:comps_across_models1} of the Supplementary Material compare emulated replicates from XVAE and \texttt{hetGP} with data replicates from Models~\ref{modelGP}--\ref{modelMaxStable}, while Figure~\ref{fig:qqplots_across_models} displays QQ-plots that align well with the 1-1 line in all cases for XVAE but not for \texttt{hetGP}. 
Since the Gaussian process has much weaker extremal dependence, the resulting $\bgamma_t$ estimated in \eqref{eqn:decoder_form} after convergence is consistently far greater than $0.1$, indicating light tails in the expPS variables and thus, local AI at all knots. In contrast, Model~\ref{modelAI} exhibits AI across the domain with much smaller $\bgamma_t$ values (see Figure~\ref{fig:sim_knots}), producing heavier-tailed expPS variables than Model~\ref{modelGP}. As a result, \texttt{hetGP} struggles to capture extremal dependence and the QQ-plot shows its underestimation of large tail values, though Model~\ref{modelAI} still shows only AI.
\begin{figure}[!t]
    \centering
    \includegraphics[height=0.33\linewidth]{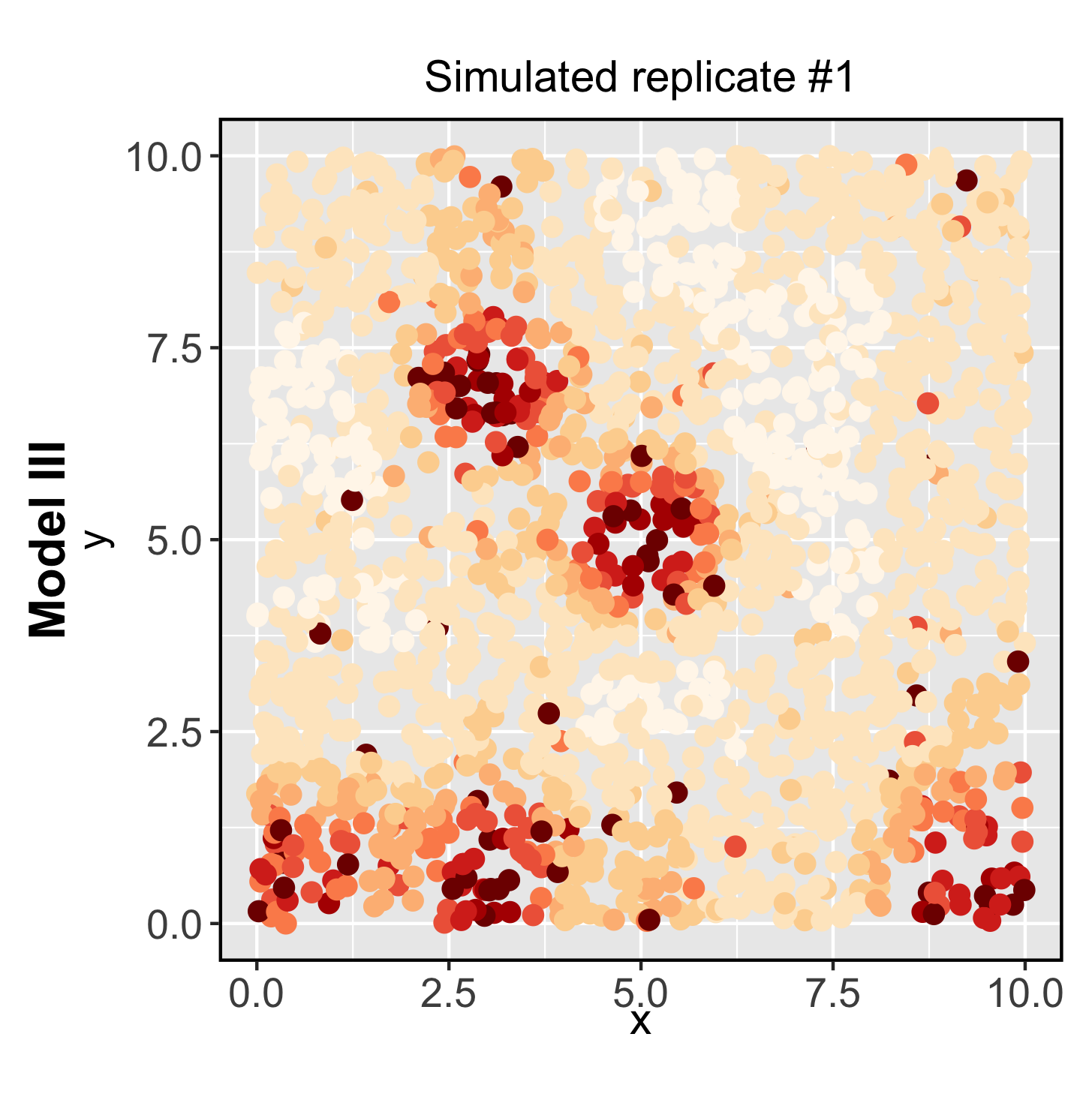}
    \includegraphics[height=0.33\linewidth]{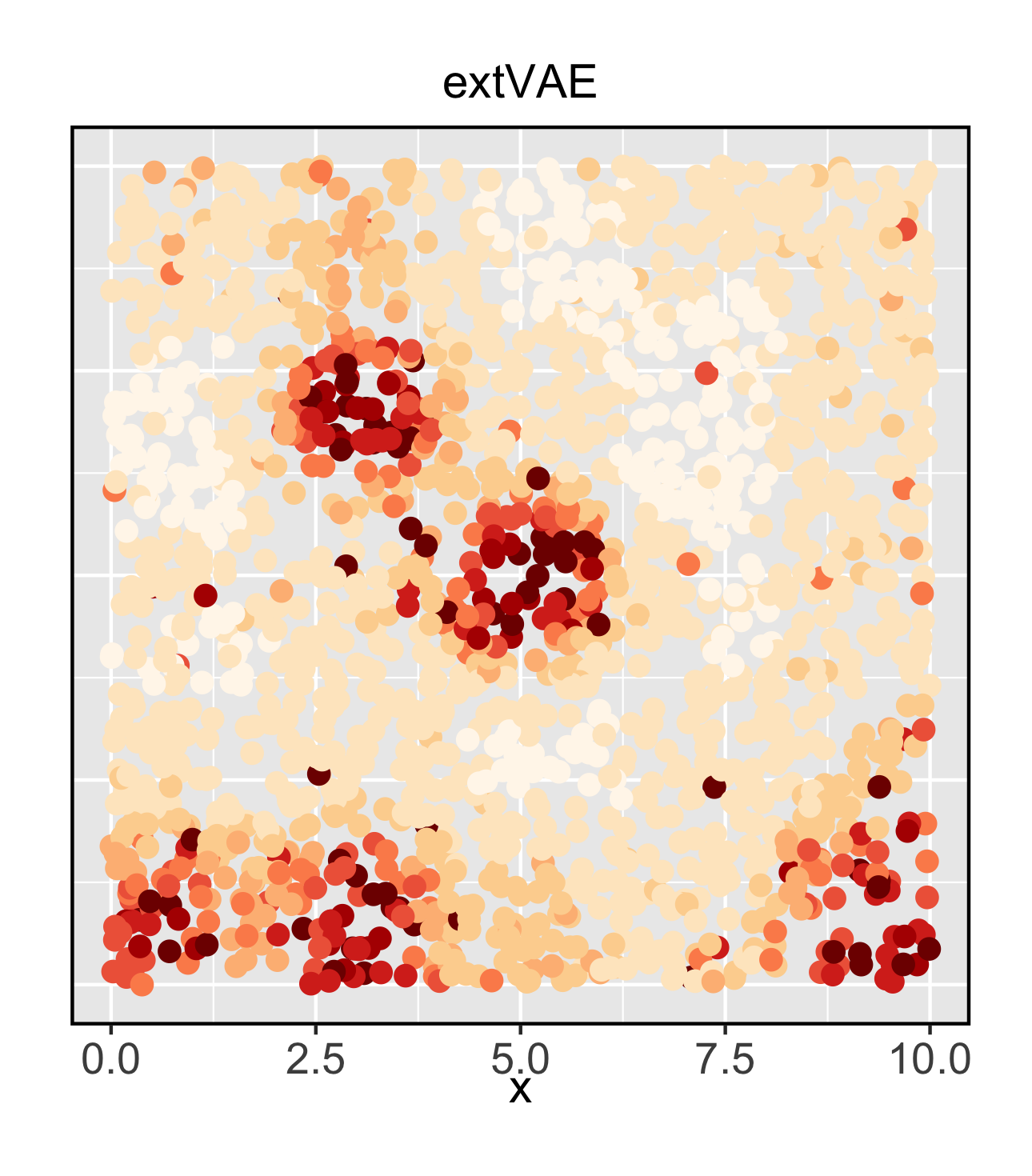}
    \includegraphics[height=0.33\linewidth]{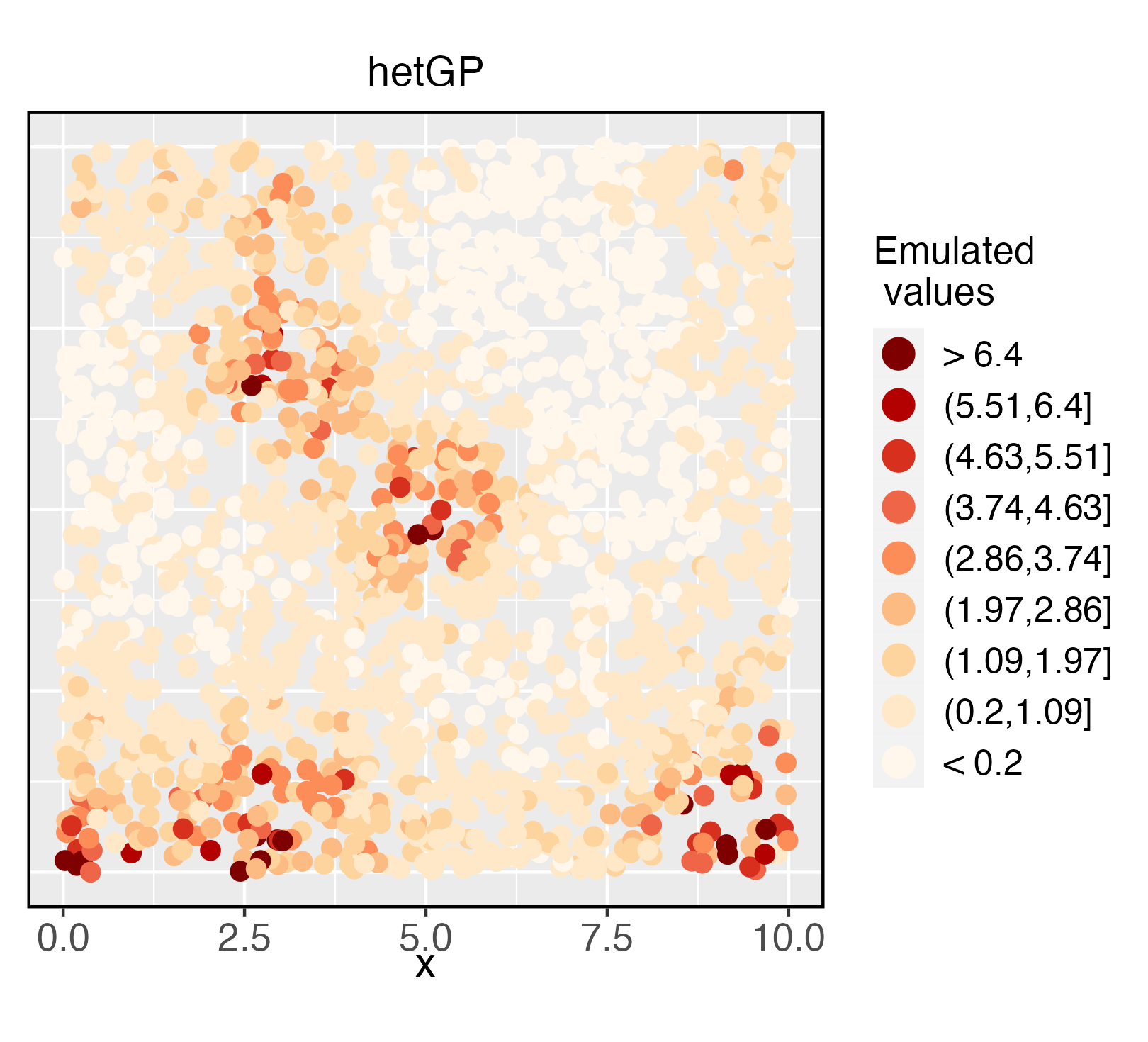}
    \vskip -0.4cm
    \caption{Data replicate (left) and its corresponding emulated fields (XVAE, middle; \texttt{hetGP}, right) from Model~\ref{modelFlex}. See Figure~\ref{fig:comps_across_models1} of the Supplementary Material for comparisons for the other models. In all cases, we use data-driven knots for emulation using XVAE.}
    \label{fig:comps_across_models}
\end{figure}

For Models~\ref{modelFlex}--\ref{modelMaxStable}, there is local AD, and we see that \texttt{hetGP} completely fails at emulating the co-occurrence of extreme values. Because \texttt{hetGP} focuses on the bulk of the distribution, it ignores spatial extremal dependence. This validates the need to incorporate a flexible spatial model in the emulator to capture tail dependence accurately.    

\begin{figure}[!t]
    \centering
    \includegraphics[height=0.257\linewidth]{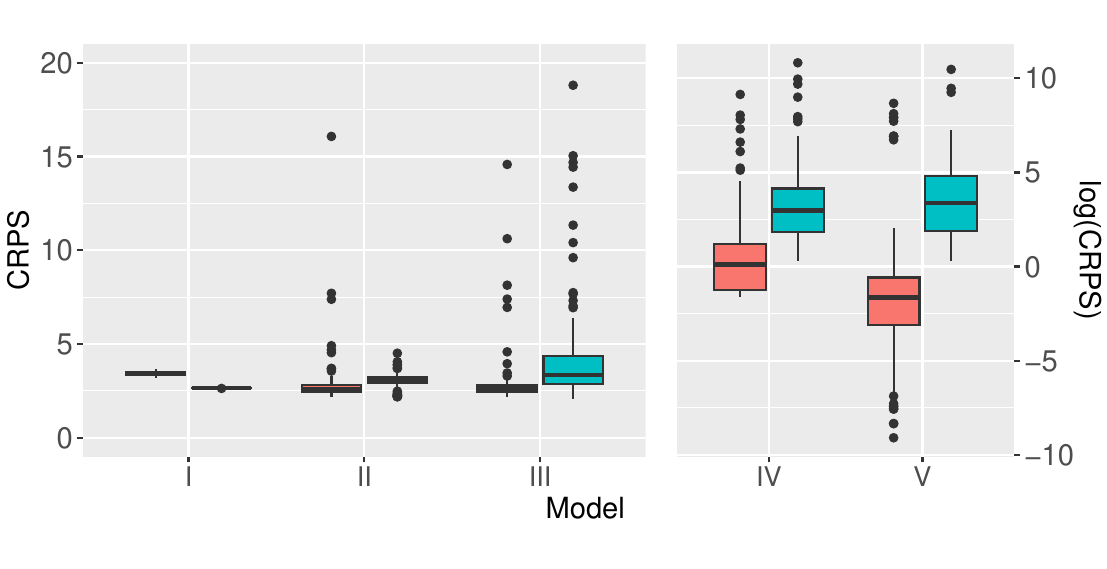}
    \hspace*{1em}\includegraphics[height=0.2458\linewidth]{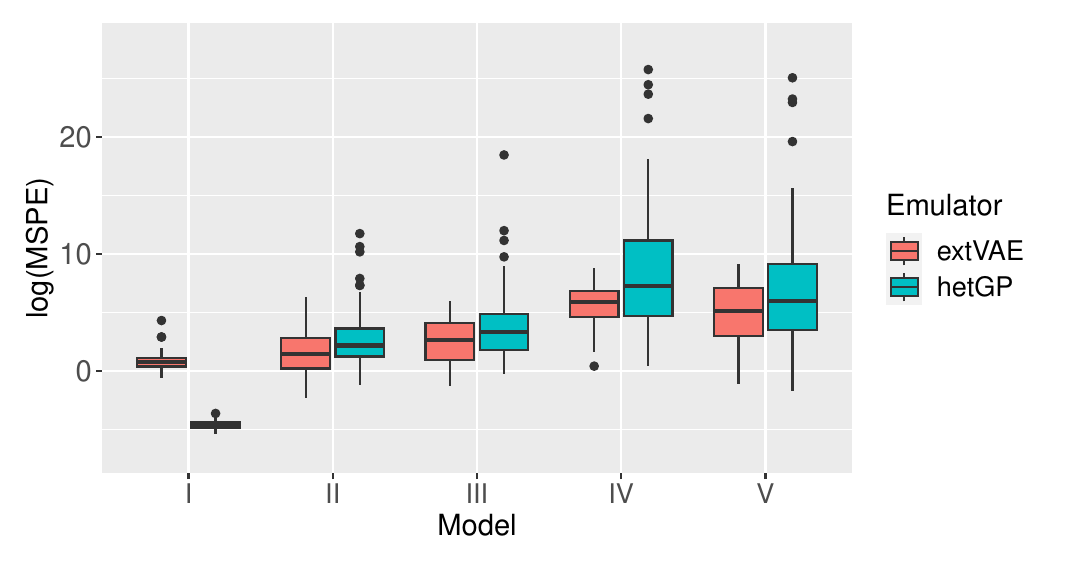}
    \vskip -0.4cm
    \caption{The CRPS (left) and MSPE (right) values from two emulation approaches on the datasets simulated from Models~\ref{modelGP}--\ref{modelMaxStable}. For both metrics, lower values indicate better emulation results. Also, for Models~\ref{modelAD} and \ref{modelMaxStable}, we plot the CRPS values on the log scale since the AD in the data generating process causes the margins to be very heavy-tailed.}
    \label{fig:res_summary}
\end{figure}
Figure~\ref{fig:res_summary} compares the performance of spatial predictions at the 100 holdout locations. For Model~\ref{modelGP}, \texttt{hetGP} has lower CRPS and MSPE scores, indicating higher predictive power, as expected since the true process is Gaussian. However, the XVAE model still performs quite well in this (misspecified) case. For Models~\ref{modelAI}--\ref{modelMaxStable}, XVAE uniformly outperforms \texttt{hetGP}. 
Also, the CRPS and MSPE for \texttt{hetGP} are significantly higher for time replicates with extreme events.

The first three panels of Figure~\ref{fig:chi_ests} and Figure~\ref{fig:chi_ests2} of the Supplementary Material compare nonparametric estimates of the upper tail dependence $\chi_h(u)$ from the data replicates and emulations at three different distances $h\in\{0.5,2,5\}$  under the working assumption of stationarity. In general, we see that the dependence strength decays as $h$ and $u$ increase, with varying levels of positive limits as $u\rightarrow 1$ for Models~\ref{modelFlex}--\ref{modelMaxStable}. The results in Figure~\ref{fig:chi_ests} demonstrate that our XVAE manages to accurately emulate the dependence behavior at both low and high quantiles and the empirical confidence envelopes of $\chi_h(u)$ are essentially indistinguishable between the simulated and emulated data.
\begin{figure}[!t]
    \centering
    \includegraphics[height=0.26\linewidth]{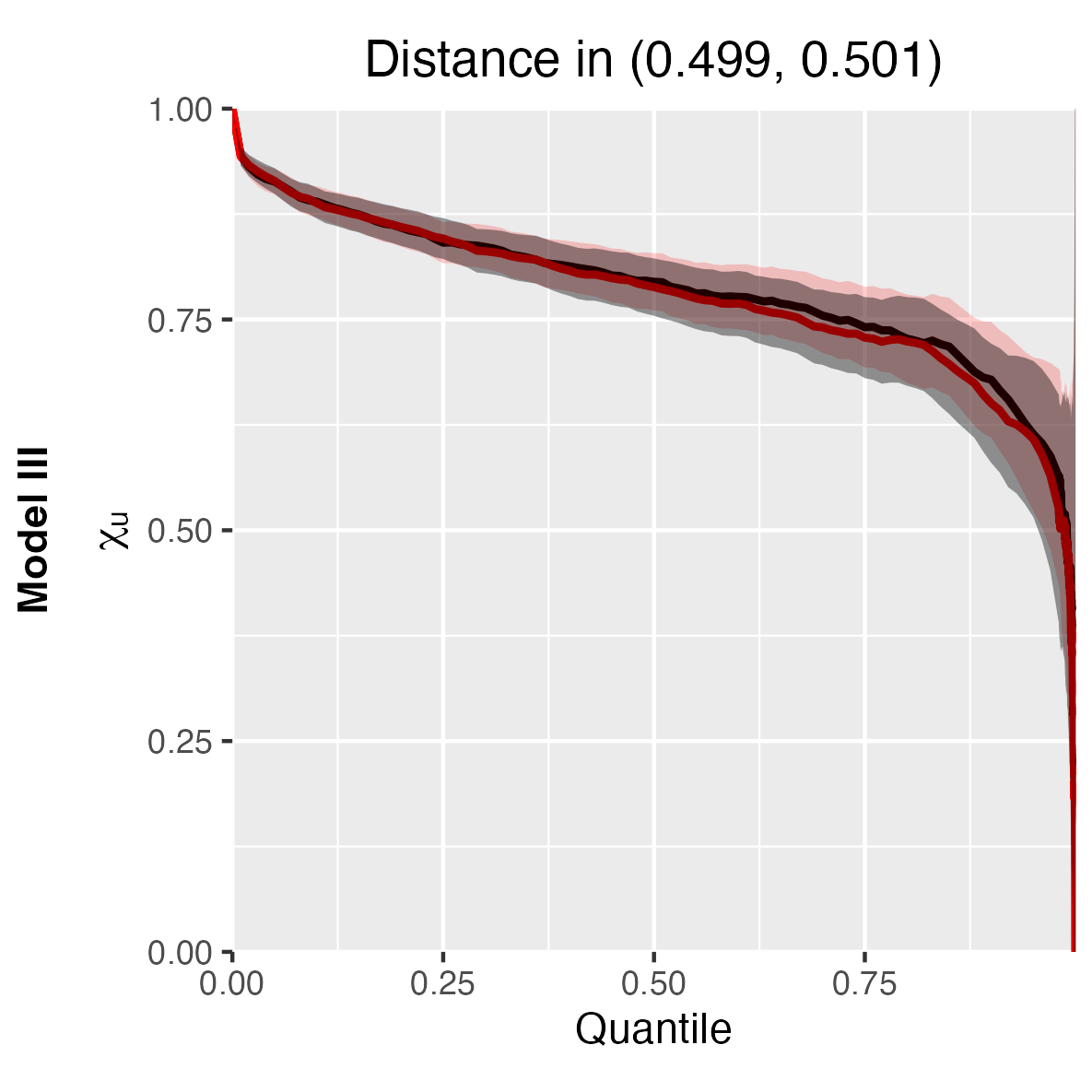}
    \hspace*{-0.2cm}\includegraphics[height=0.26\linewidth, trim={0.25cm 0 0 0}, clip]{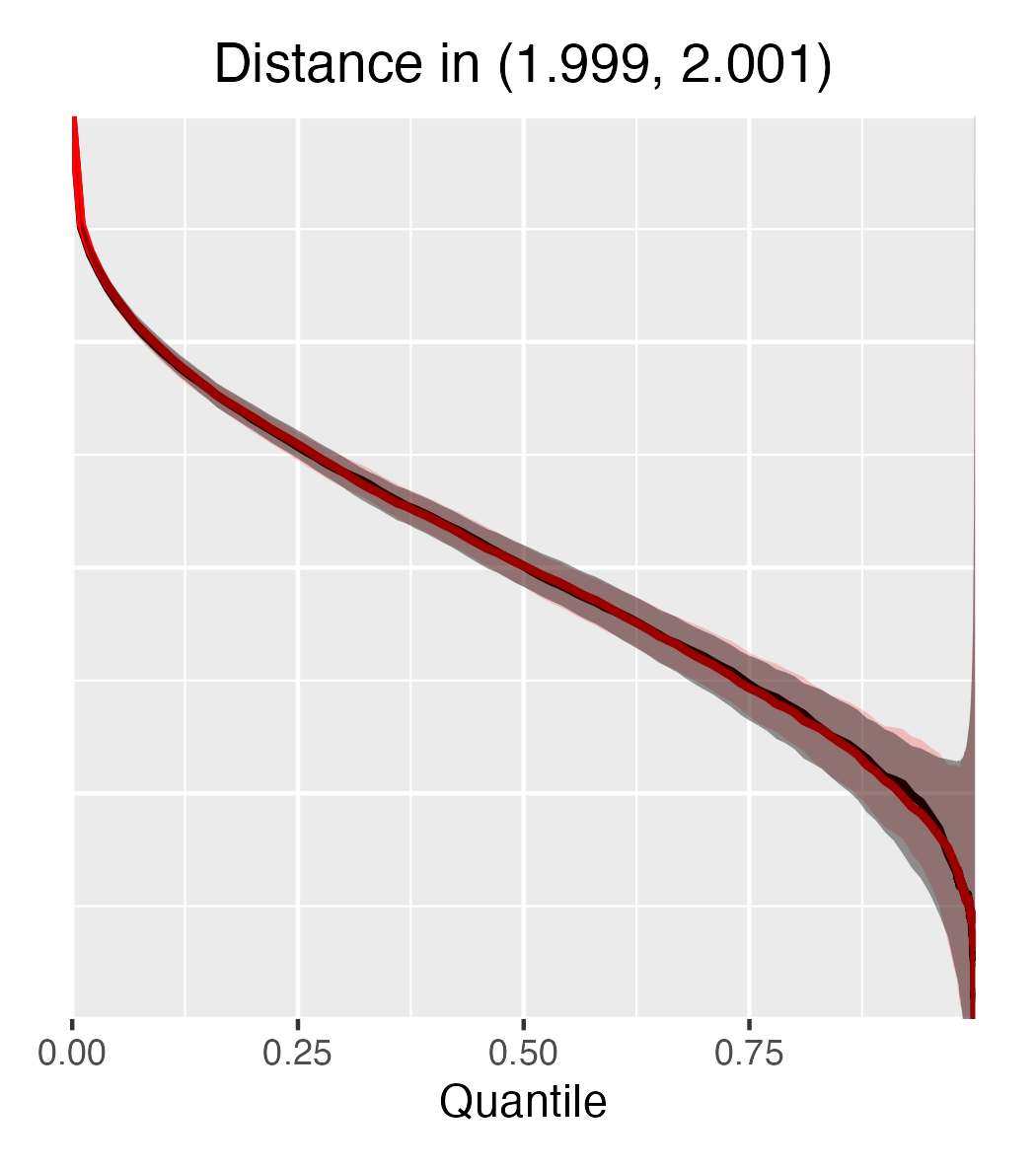}
    \hspace*{-0.3cm}\includegraphics[height=0.26\linewidth, trim={0 0 3cm 0}, clip]{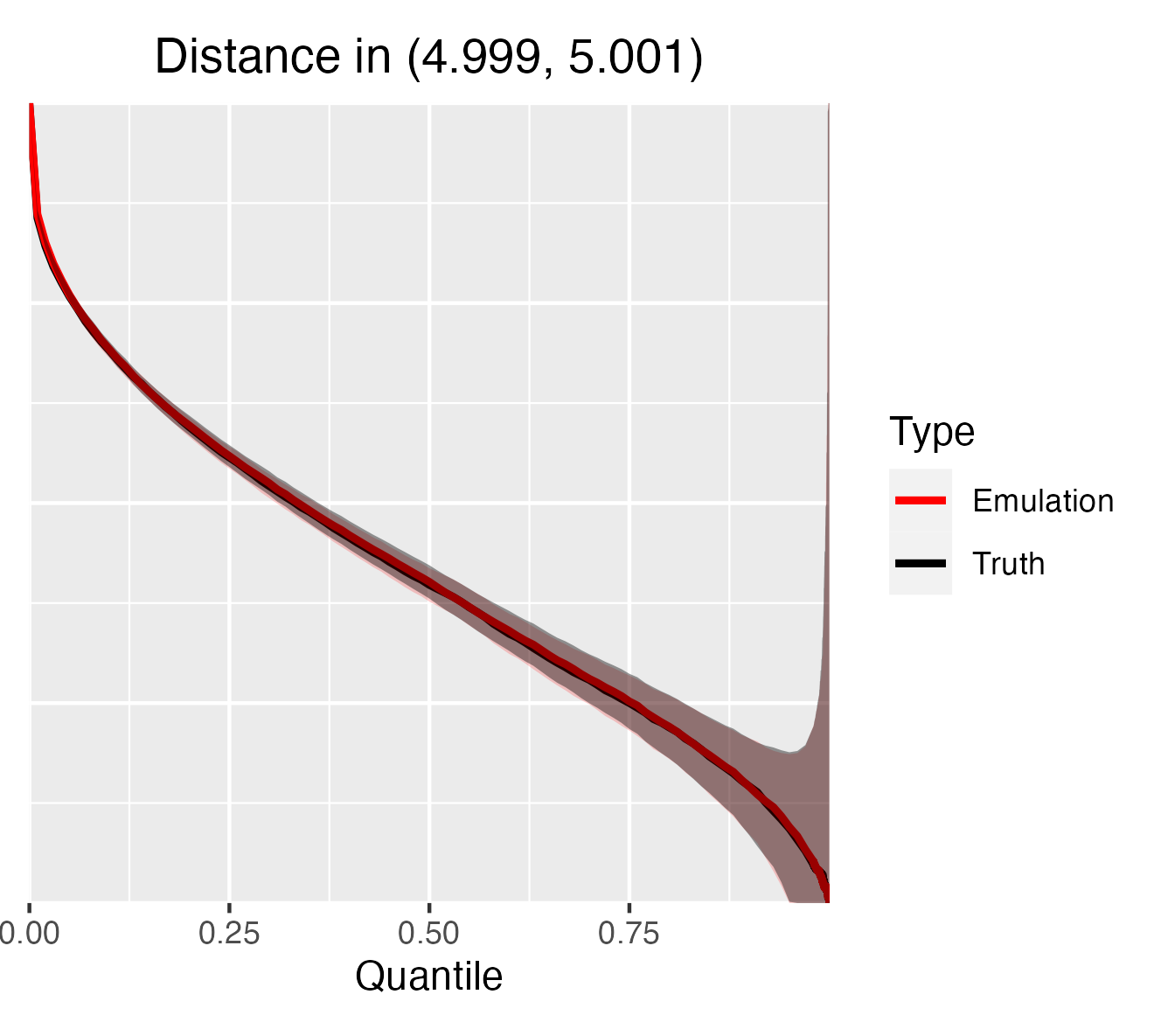}
    \hspace*{0.1cm}\includegraphics[height=0.26\linewidth, trim={0 0 2cm 0}, clip]{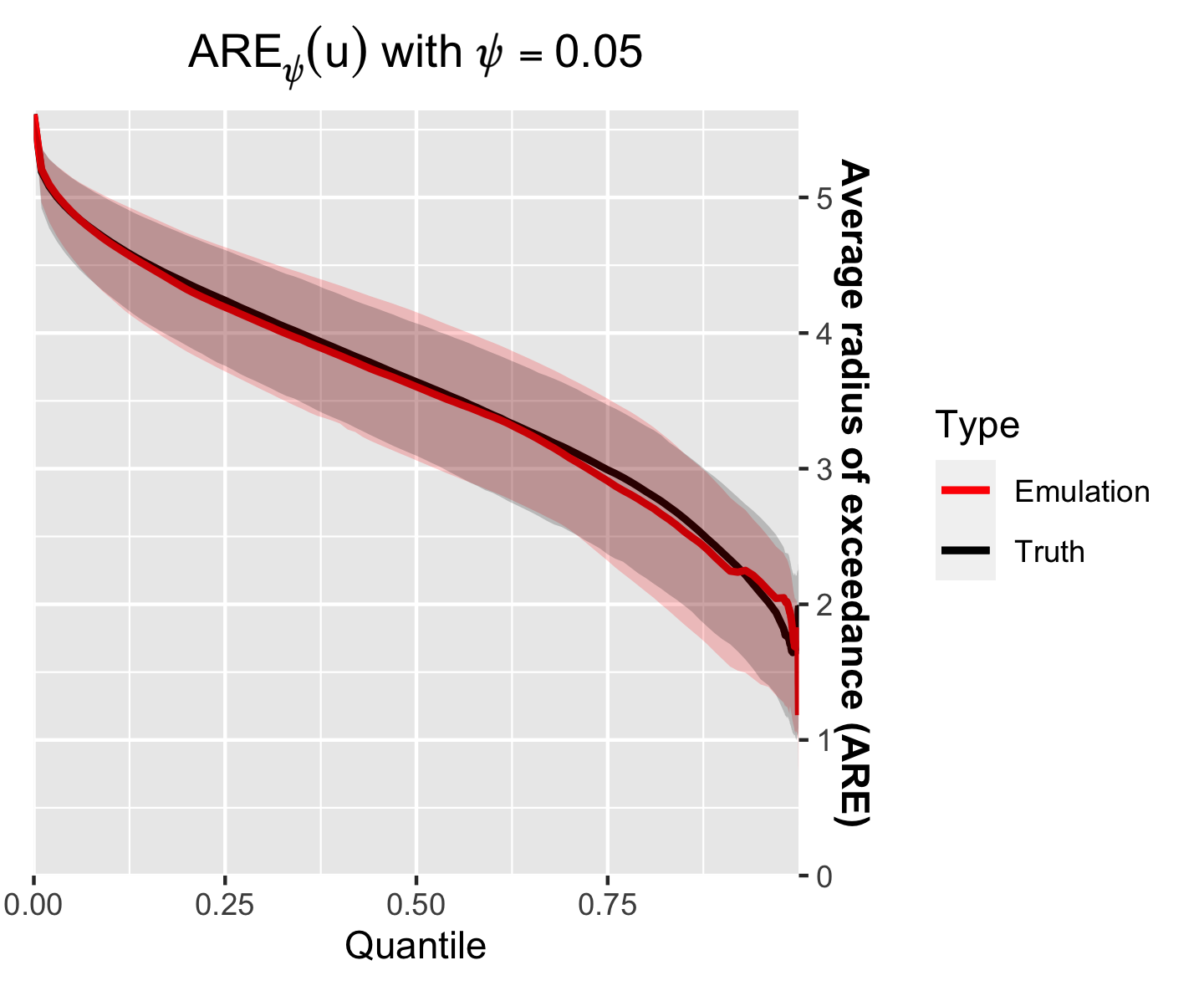}
    \vskip -0.4cm
    \caption{From left to right, we show the empirically-estimated $\chi_h(u)$ at $h=0.5,2,5$, and $\mathrm{ARE}_\psi(u)$ with $\psi=0.05$ for Model~\ref{modelFlex} based on data replicates (black) and XVAE emulated data (red). The $\chi_h(u)$ and $\mathrm{ARE}_\psi(u)$ estimates for the other models are shown in Figures~\ref{fig:chi_ests2} and \ref{fig:ARE_comps2} of the Supplementary Material, respectively.}
    \label{fig:chi_ests}
\end{figure}

Choosing 
$(5,5)$ as the reference point, the rightmost panel of Figure~\ref{fig:chi_ests} displays estimates of ARE$_\psi(u)$, $\psi=0.05$, for both data replicates and emulated data under Model~\ref{modelFlex}; see the Supplementary Material for the other models. 
We see that the empirical AREs from the XVAE are consistent with the ones estimated from the data except for (misspecified) Model~\ref{modelMaxStable}, where ARE$(u)$ is slightly underestimated at low thresholds $u$ but overestimated at high $u$. As expected, the limit of ARE$(u)$ as $u\rightarrow 1$ is non-negative for Models~\ref{modelFlex}--\ref{modelMaxStable} when there is local AD, and the limit increases from Model~\ref{modelFlex} to \ref{modelMaxStable}.

\begin{figure}[!t]
    \centering
    \includegraphics[height = 0.27\linewidth]{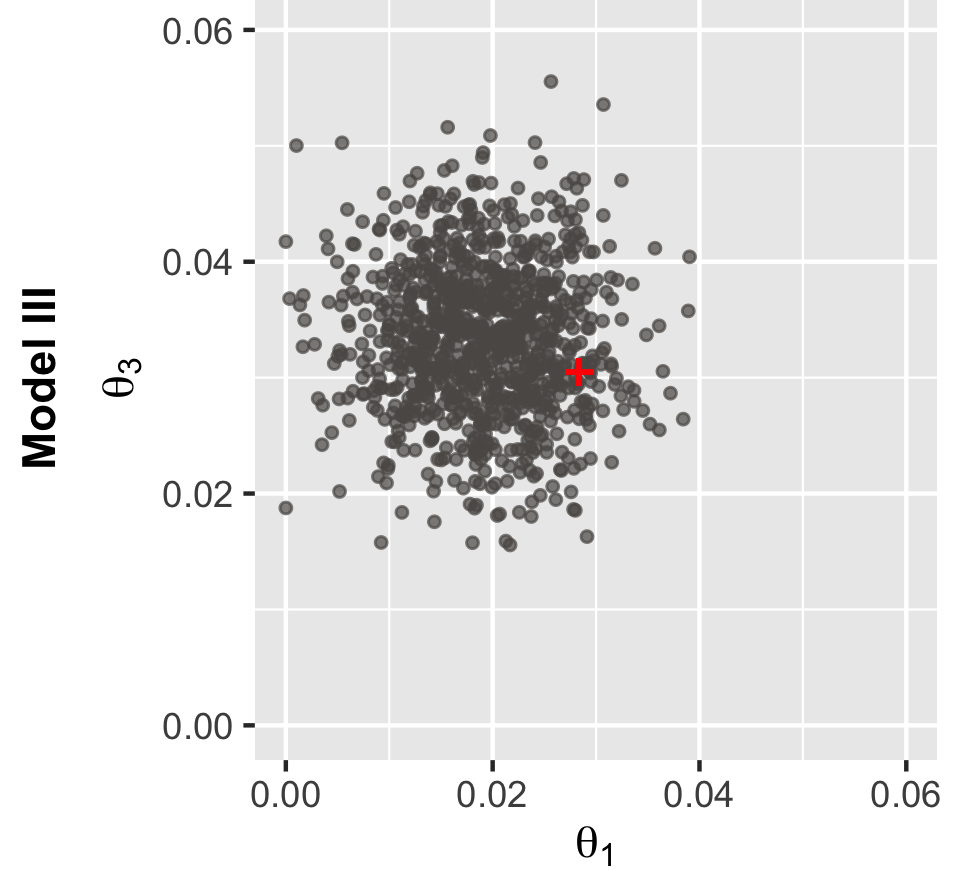}
    \includegraphics[height = 0.27\linewidth]{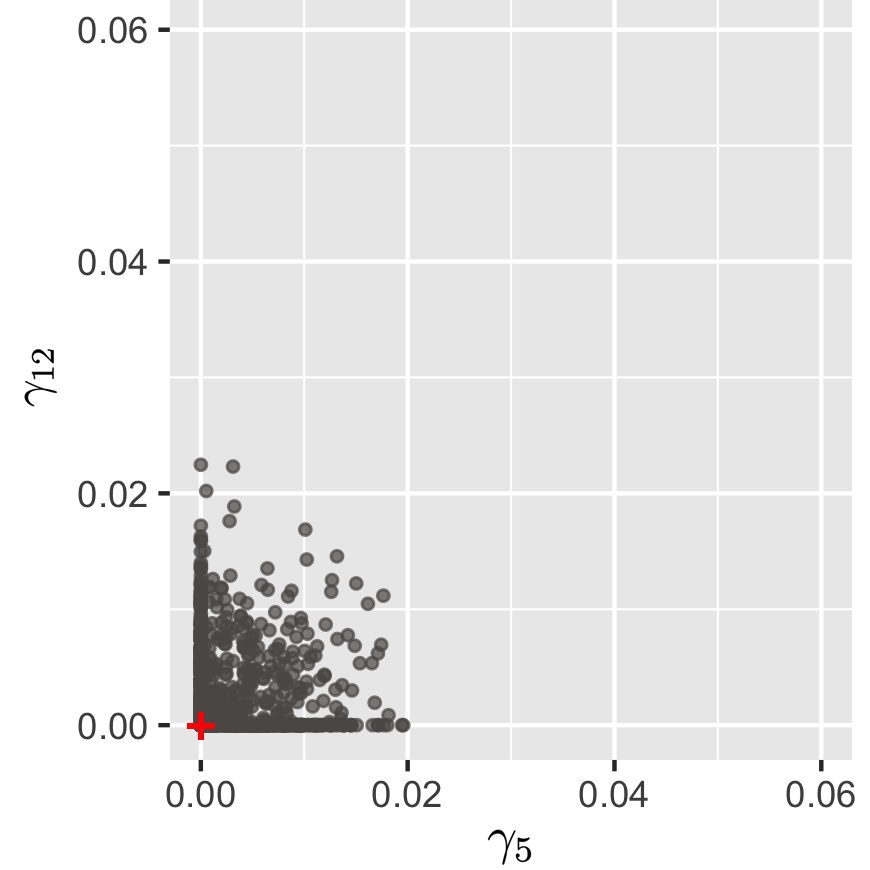}
    \includegraphics[height = 0.27\linewidth]{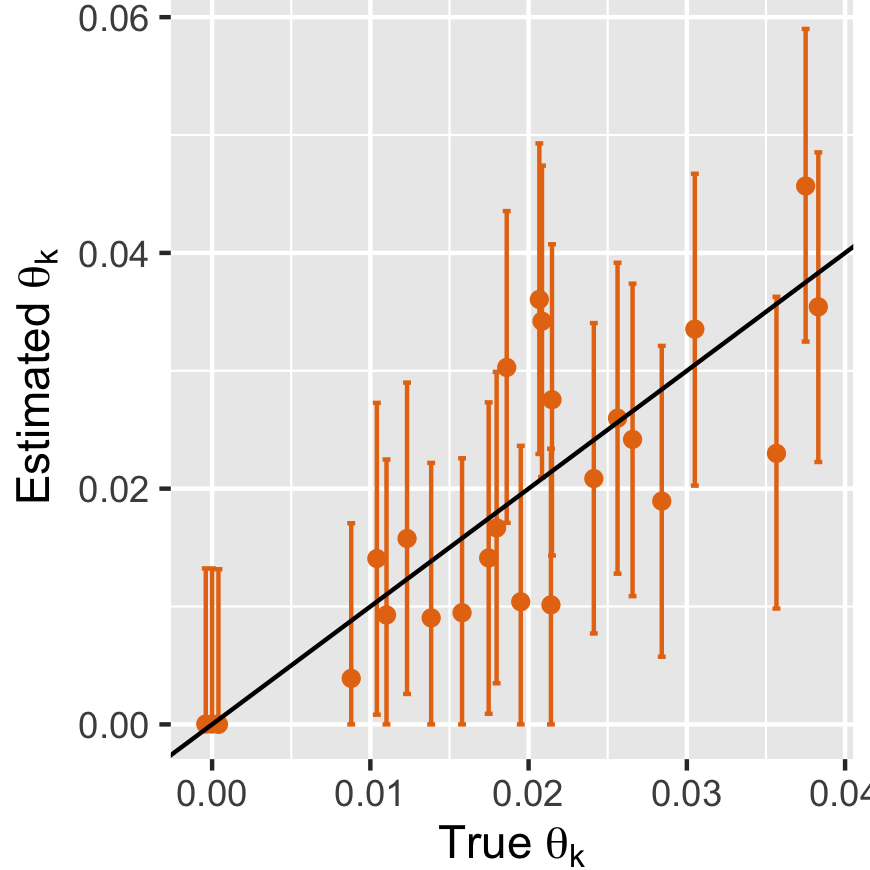}
    \vskip -0.3cm
    \caption{Initializing the XVAE using the true knots from Model~\ref{modelFlex}, 
    we show the estimates of $(\gamma_{1t},\gamma_{3t})^{\rm T}$ (left) and $(\gamma_{5t},\gamma_{12t})^{\rm T}$ (middle) from 1,000 samples generated with the trained decoder $(t=1)$. On the right, we also show the medians, 2.5\% and 97.5\% quantiles of the $n_t$ estimates of $\{\gamma_{kt}: k=1,\ldots, K\}$ for $t=1$, from the decoder \eqref{eqn:decoder_form}, in which the 1-1 line is displayed in black for reference.}
    \label{fig:theta_t_comp}
\end{figure}
To showcase the inferential capabilities of our approach, we initialize the XVAE with true knots and rerun it on datasets simulated from Model~\ref{modelFlex}. 
Figure~\ref{fig:theta_t_comp} displays $\bgamma_t$ estimates obtained by running the decoder (i.e., Eq.~\eqref{eqn:decoder_form}) $1000$ times at $t=1$. The results highlight the XVAE's ability to produce accurate estimates of $\bgamma_t=\{\gamma_{kt}:k=1,\ldots,K\}$, and correctly identify the extremal dependence class, with satisfactory agreement between true and estimated values, accounting for uncertainty. Additionally, we perform a coverage analysis by simulating $99$ more datasets with $n_s=2000$ and $n_t=100$ from Model~\ref{modelFlex}, running the XVAE on each to generate empirical credible intervals for $\bgamma_t$. Figure~\ref{fig:coverage} of the Supplementary Material shows the coverage probabilities of $\bgamma_t$ for $t=1$. Most estimated probabilities align closely with the nominal $95$\% level, except when $\gamma_{kt}=0$, where the coverage is poorer due to the true value residing on the parameter space boundary. Nevertheless, these promising results endorse the XVAE as a fast and robust inference tool for estimating parameters in the max-id process \eqref{eqn:model} and for Bayesian UQ. 

\section{Application to Red Sea surface temperature data}\label{sec:data_analysis}
The Red Sea, a biodiversity hotspot, is susceptible to coral bleaching due to climate change and rising SST anomalies \citep{furby2013susceptibility}. Corals are unlikely to survive once the temperature exceeds a bleaching threshold annually, which in turn causes disruptions in fish migration and slow decline in fish abundance. Here, we analyze and emulate a Red Sea surface temperature dataset, which consists of satellite-derived daily SST estimates at 16,703 locations on a $1/20^\circ$ grid from 1985/01/01 to 2015/12/31 (11,315 days in total); see \citet{donlon2012operational}. This yields about 189 million correlated spatio-temporal data points. 

We extract monthly maxima from renormalized data to ensure temporal independence and modeling accuracy of sitewise marginal distributions. Sections~\ref{appendix:remove_season}--\ref{appendix:margin_trans} of the Supplementary Material detail how we remove the seasonal trends and how we transform the sitewise records to the Pareto scale on which we then apply the XVAE. The third panel of Figure~\ref{fig:chi_raster} displays the data-driven knot locations chosen by our algorithm ($K=243$), and the initial radius shared by the Wendland basis functions is 1.2$^\circ$.

Similar to Figure~\ref{fig:chi_ests}, we estimate $\chi_h(u)$ empirically for the original monthly maxima and emulated fields, under the working assumption of stationarity and isotropy. Figure~\ref{fig:chi_SST} of the Supplementary Material attests once more that our XVAE characterizes the extremal dependence structure accurately from low to high quantiles. 
Furthermore, we examine the pairwise $\chi$-measures between the center of the Red Sea $(38.104^\circ\mathrm{E}$, $21.427^\circ\mathrm{N})$ (denoted by $\bs_0$) and all observed locations $\bs_j\in\mathcal{S}$ in the Red Sea. The left two panels of Figure~\ref{fig:chi_raster} include raster plots of these measures evaluated at the level $u=0.85$, in which the $\chi_{0j}(u)$ values estimated from the observed and emulated data are very similar to each other.
\begin{figure}[!t]
    \centering
    \hspace*{-0.4cm}\includegraphics[height=0.29\linewidth]{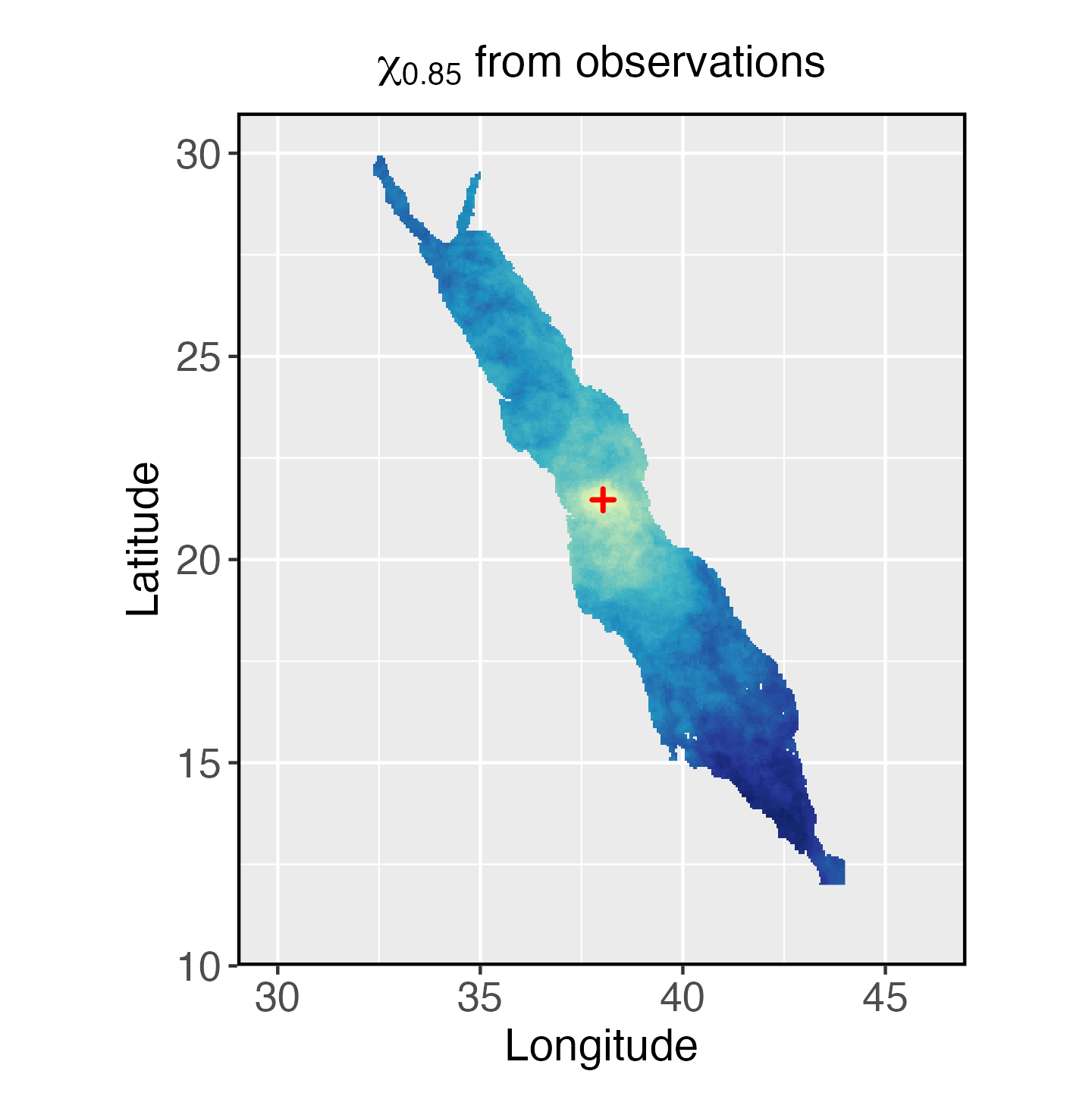}
    \hspace*{-0.05cm}\includegraphics[height=0.29\linewidth]{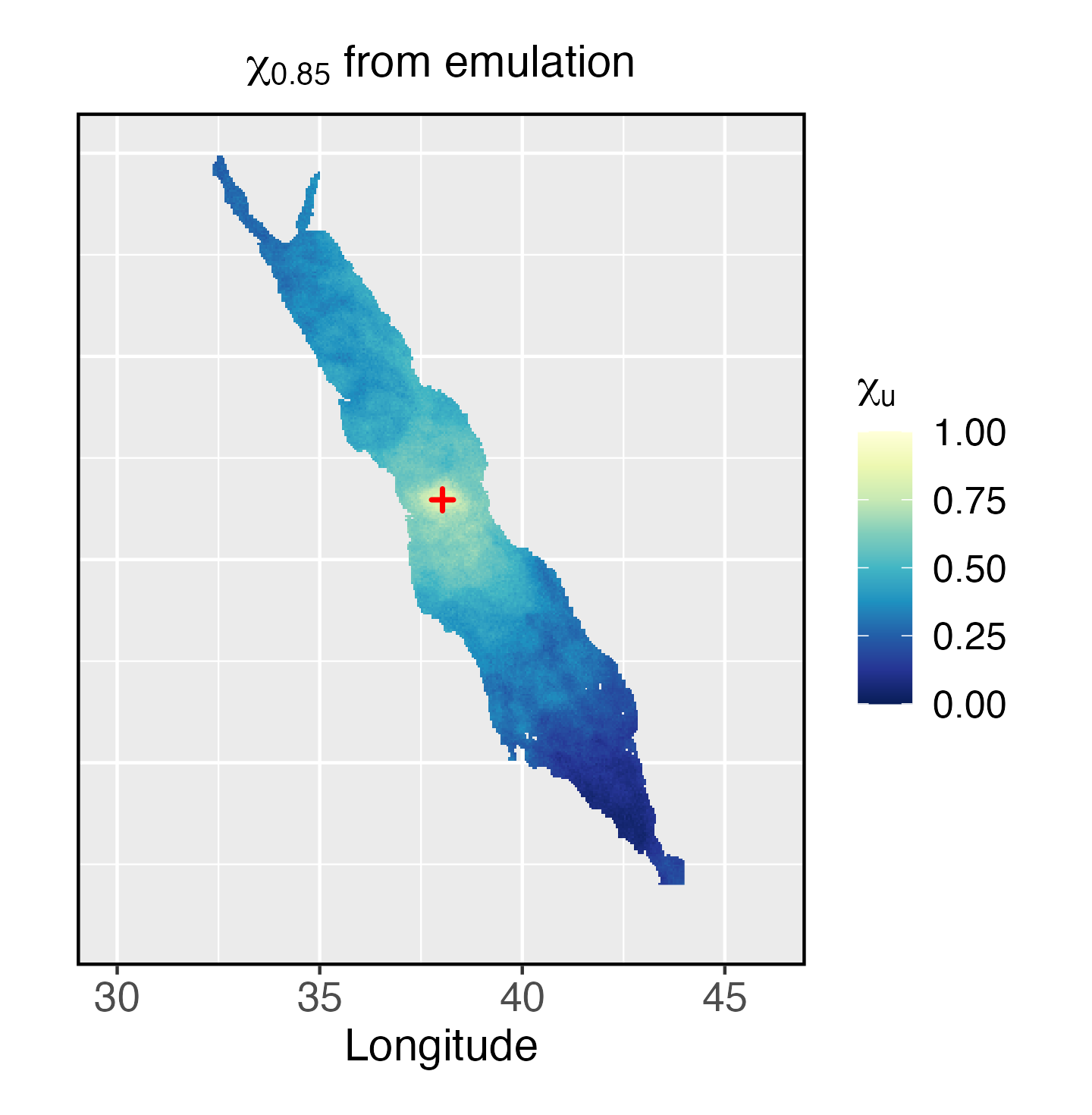} 
    \hspace*{-0.1cm}\includegraphics[height=0.27\linewidth]{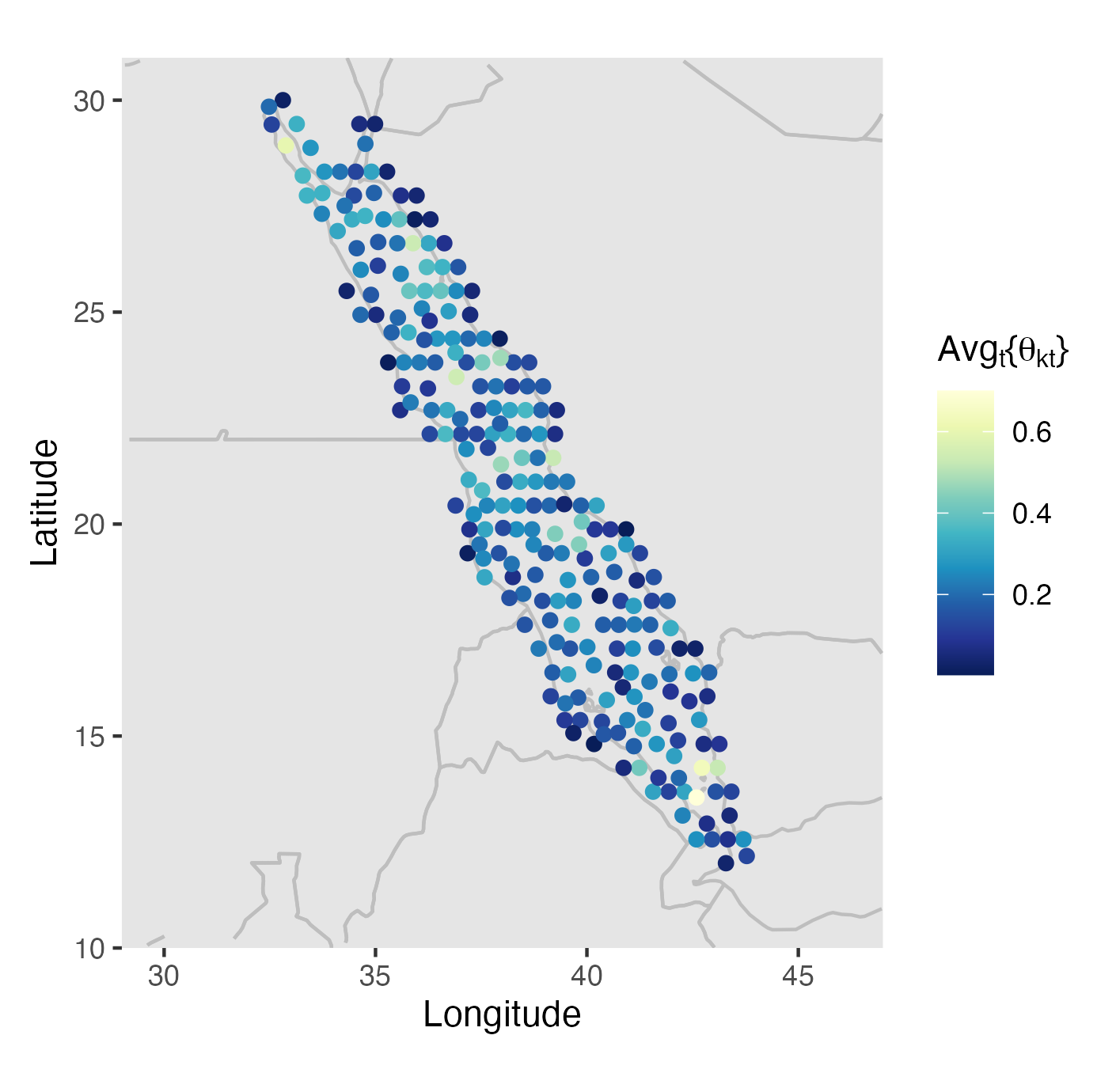} 
    \hspace*{-0.1cm}\includegraphics[height=0.27\linewidth]{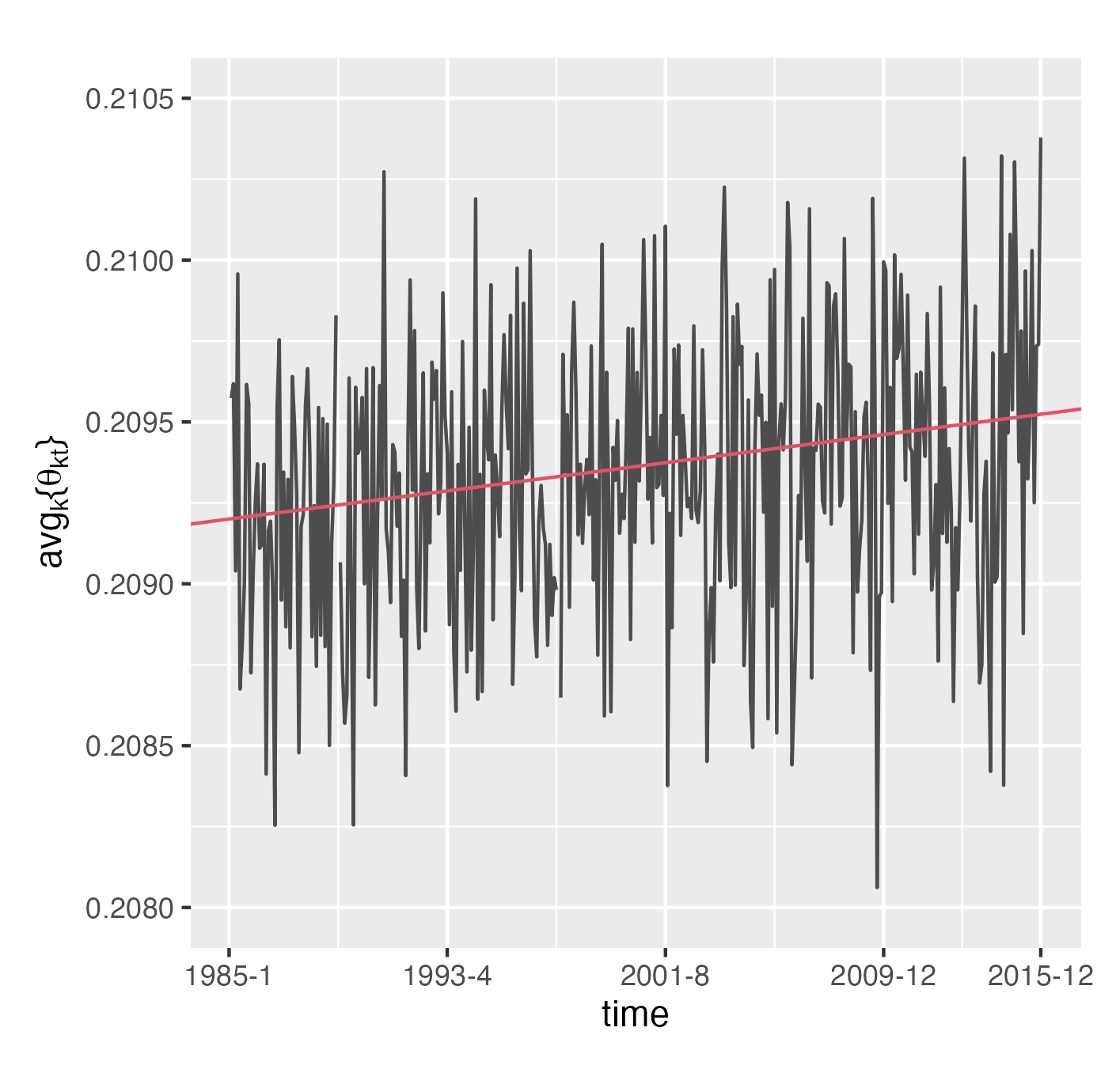}
    \vskip -0.3cm
    \caption{In the left two panels, we show empirically-estimated pairwise tail dependence measure $\chi_{0j}(u)$, $u=0.85$, between $\bs_0 = (38.104,21.427)$, marked using a red cross, and all $\bs_j\in\mathcal{S}$, from observations and emulated data. In the right two panels, we show the estimated tilting parameters at $K=243$ data-driven knots averaged over time (i.e., $n_t^{-1}\sum_{t=1}^{n_t}\hat{\gamma}_{kt}$, $k=1,\ldots, K$), and the estimated tilting parameters averaged over space (i.e., $K^{-1}\sumK\hat{\gamma}_{kt}$, $t=1,\ldots, n_t$) with the best linear regression fit (red line).}
    \label{fig:chi_raster}
\end{figure}

The right two panels of Figure~\ref{fig:chi_raster} show estimates of $\{\gamma_{kt}:k=1,\ldots, K, \;t=1,\ldots, n_t\}$ averaged over time/space. We see that the $\gamma_{kt}$ values are generally lower near the coast compared to the interior of the Red Sea, indicating SST tends to be more heavy-tailed on the coast. We also observe an upward trend in $\gamma_{kt}$ over time, indicating that extreme events are becoming more localized. This is consistent with the findings in \citet{genevier2019marine}. It is worth noting that this probabilistic approach to examining the spatio-temporal variation in the dependence structure would not be possible using either \texttt{hetGP} or extGAN.


To focus on extreme values, we transform emulated $\{X_t(\bs)\}$ fields to the original SST scale using the fitted marginal distributions and censor the simulations with a fixed thermal threshold of 31$^\circ$C. The left panels of Figure~\ref{fig:POT_analysis} display realizations at two time points, 1985/9 and 2015/9, in which threshold exceedances primarily occur in the southern region. This is 
expected: the southern Red Sea experiences higher SSTs compared to the northern area. However, coral reefs in different parts of the Red Sea have developed varying levels of thermal tolerance \citep{hazra2021estimating}. To explore regional variation in marine heatwave (MHW) and coral bleaching risk, we divide the Red Sea into four regions based on \citet{raitsos2013remote} and \citet{genevier2019marine}: North (25.5--30$^\circ$N), North Central (22--25.5$^\circ$N), South Central (17.5--22$^\circ$N) and South (12.5--17.5$^\circ$N).

\begin{figure}[!t]
    \centering
    {\includegraphics[width=0.285\linewidth]{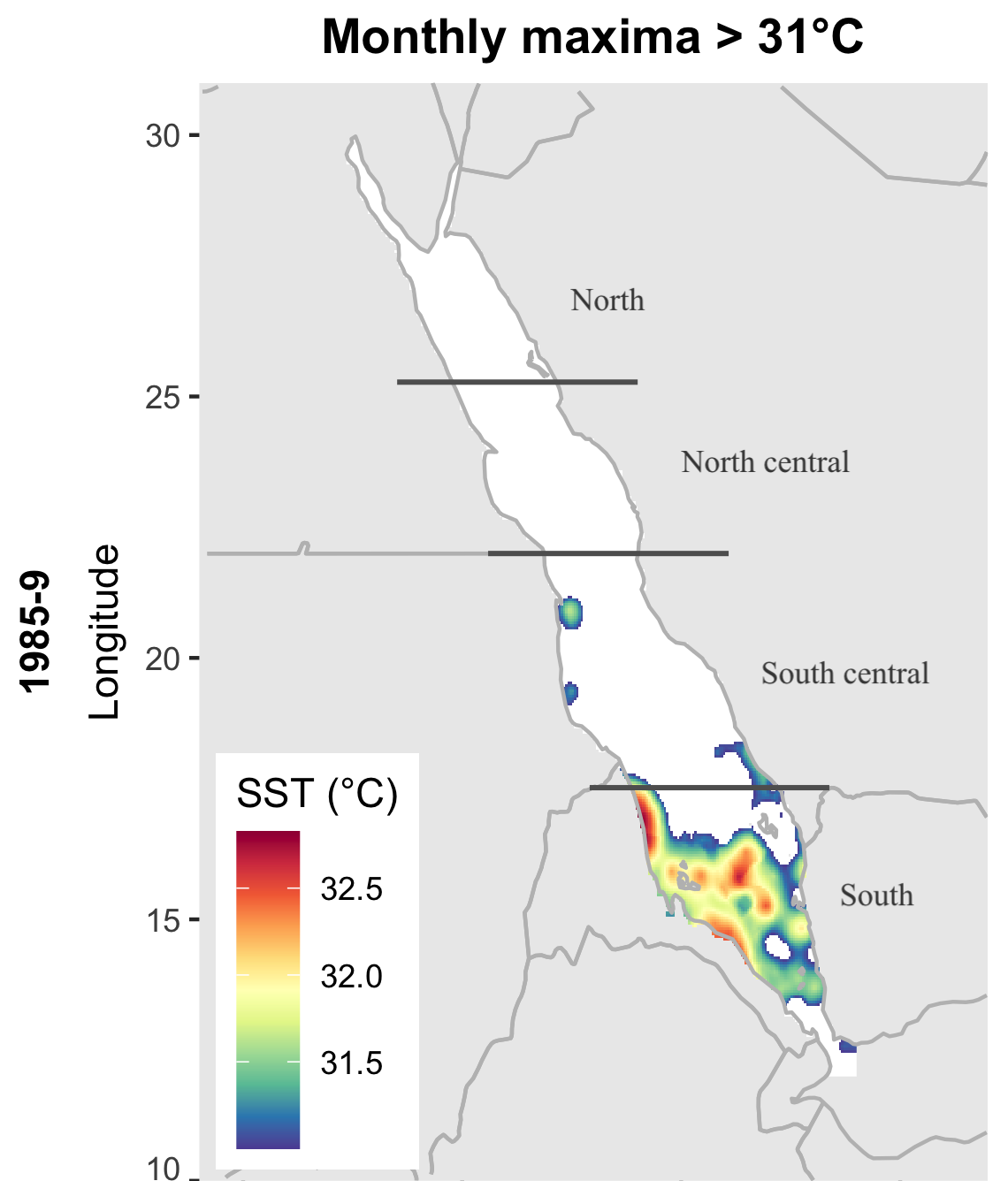}
    \includegraphics[width=0.29\linewidth]{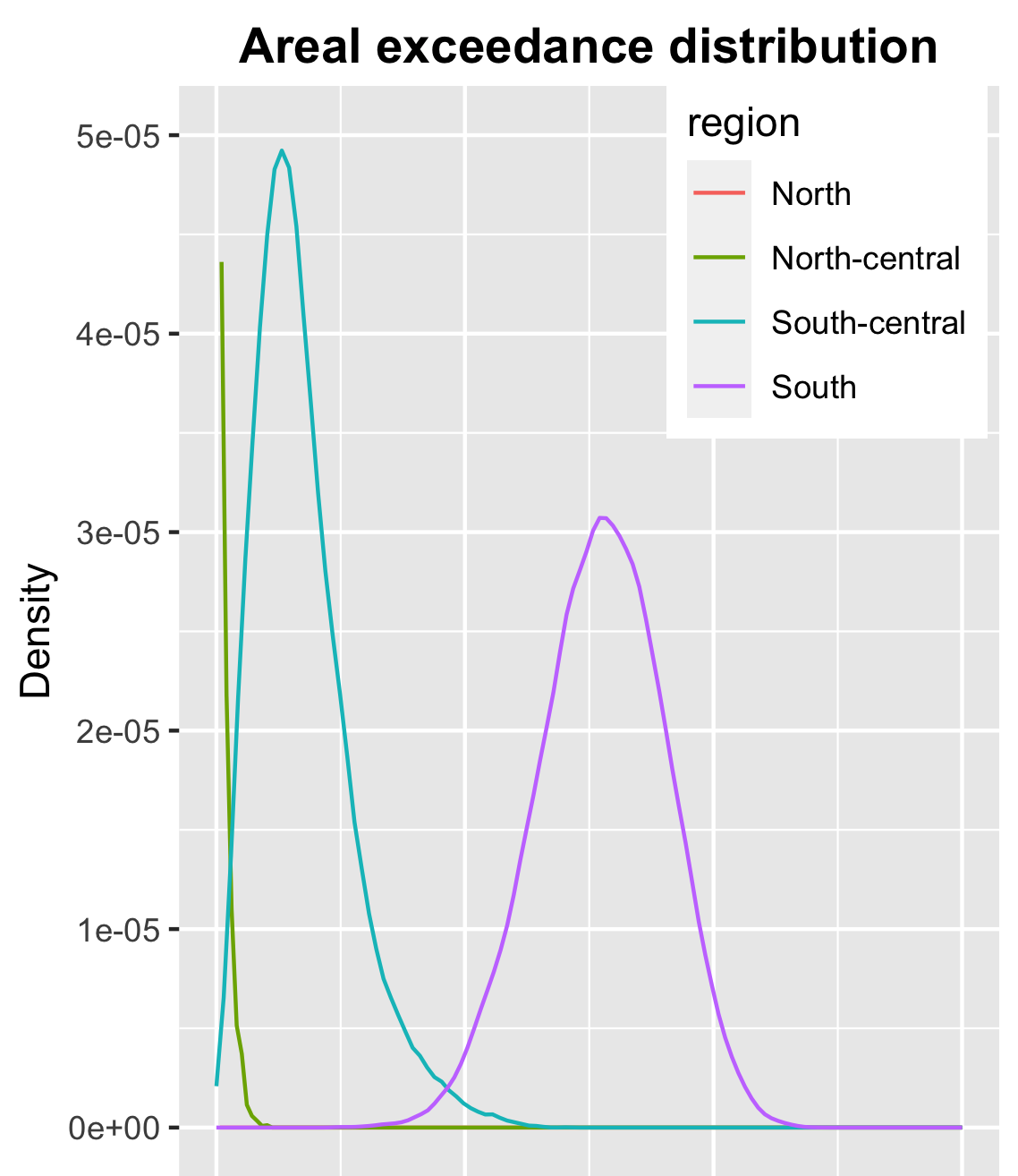}
    \includegraphics[width=0.29\linewidth]{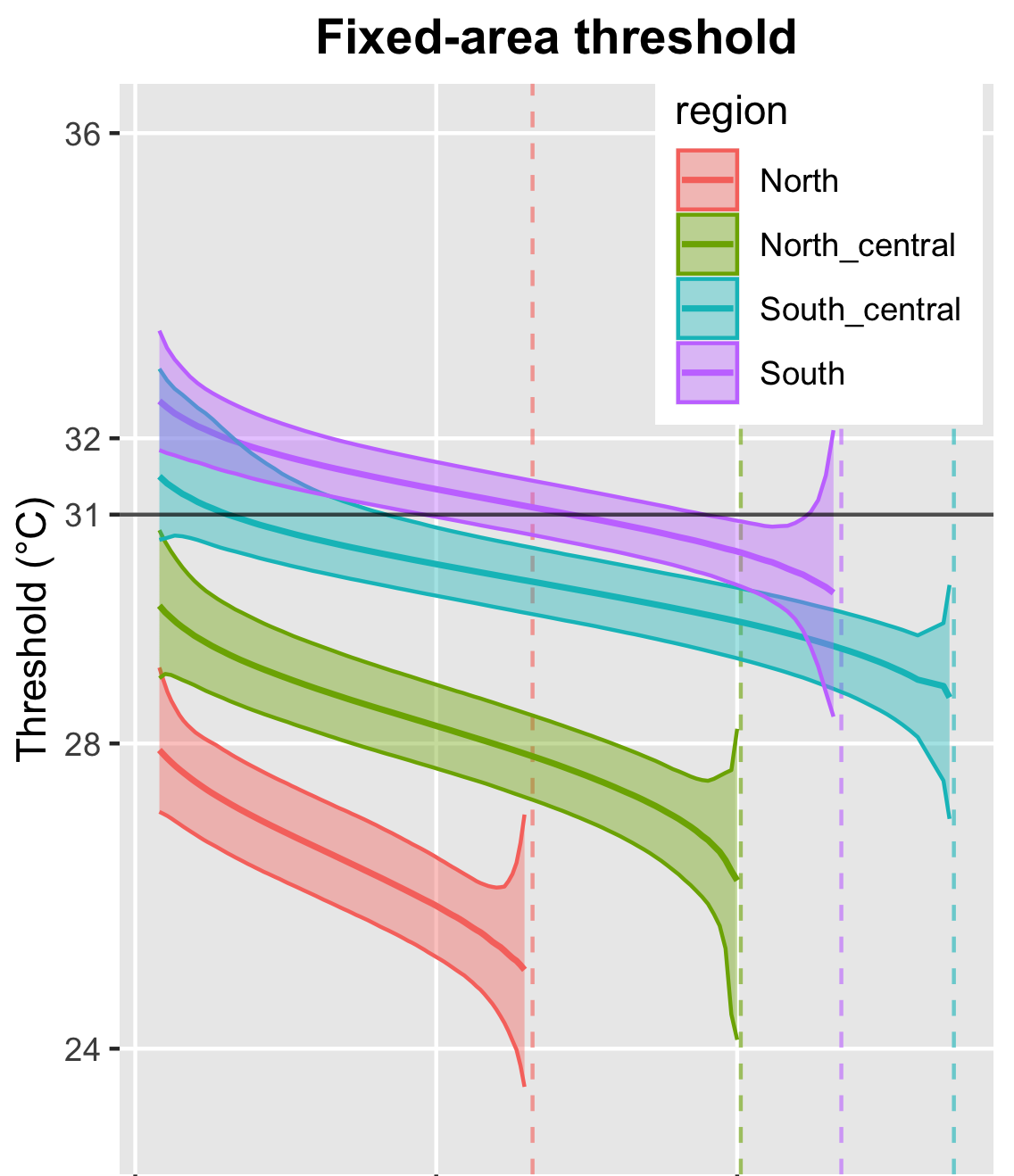}}

    \raisebox{0.3\linewidth}{\includegraphics[width=0.285\linewidth, trim={0 0 0 2.65cm}, clip]{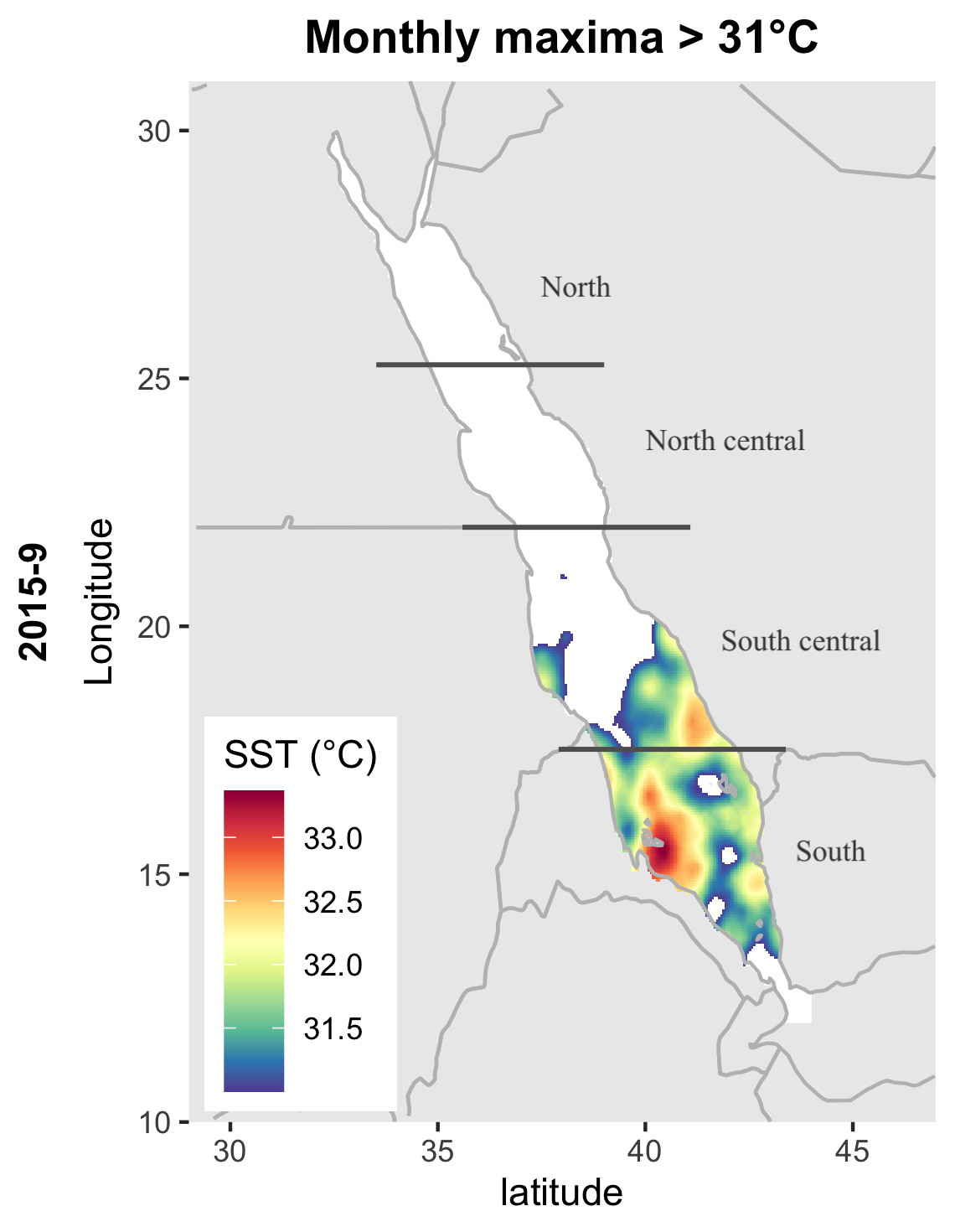}}
    \raisebox{0.294\linewidth}{\includegraphics[width=0.29\linewidth, trim={0 0 0 2.7cm}, clip]{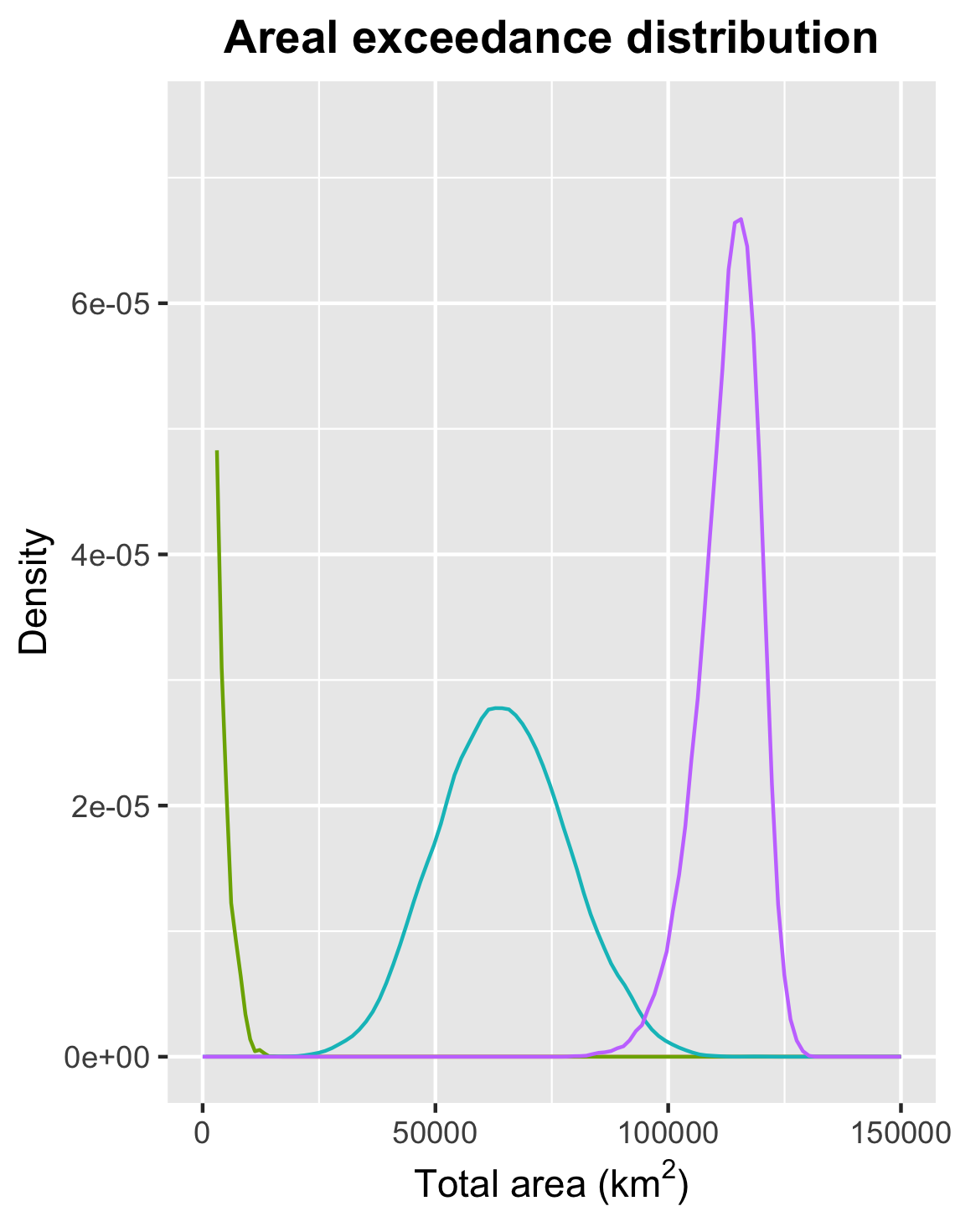}}
    \raisebox{0.294\linewidth}{\includegraphics[width=0.29\linewidth, trim={0 0 0 2.7cm}, clip]{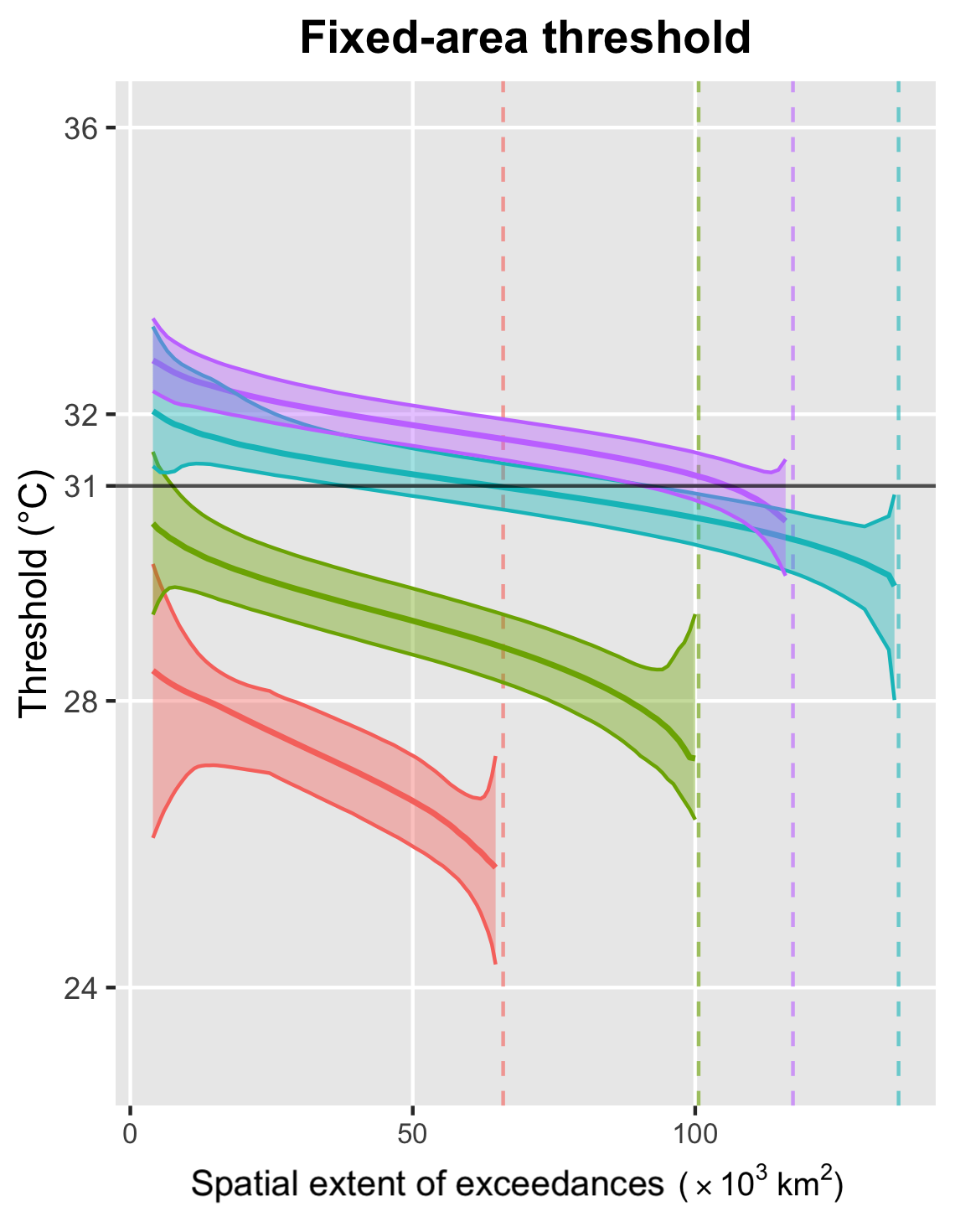}}
    \vskip -0.315\linewidth
    \caption{The left panels show realizations of Red Sea SST monthly maxima emulated with fitted parameters from the XVAE for 1985/9 (top) and 2015/9 (bottom) months. The emulations are censored with a threshold of 31$^\circ$C. 
    From 30,000 such emulations, we estimate the distribution of total area exceeding 31$^\circ$C within each region. On the right, we estimate the threshold it takes to have a fixed area of exceedance. The 95\% confidence intervals are also shown. The vertical dashed lines are total areas of each subregion, and the horizontal slices at 31$^\circ$C yield results that align with the middle panels. }
    \label{fig:POT_analysis}
\end{figure}

A useful metric is the areal exceedance probability, representing the spatial extent  of a region simultaneously at extreme MHW risk. To estimate these joint probabilities and uncertainties, we generate 30,000 independent SST emulations using XVAE for each time point and compute the total area exceeding 31$^\circ$C. Different, potentially spatially-varying thermal thresholds could also be used. The middle panels of Figure~\ref{fig:POT_analysis} show the density of the total area at risk of MHW within each region. The South Central and South regions show larger affected areas, while the North Central and North regions have little to no exceedance, reflecting cooler temperatures with increasing latitude. The results also suggest that under rising SSTs, larger simultaneous exceedances may become more likely over time, except in the North region where 31$^\circ$C remains above the highest possible temperature in 2015.


To further analyze joint exceedances at varying extreme levels, we estimate the SST thresholds required for different fixed spatial extents of exceedances. For each fixed spatial extent, we calculate the minimal threshold needed to reach that area of joint exceedances from each emulated replicate, and we then group all 30,000 estimated thresholds together to derive $95\%$ empirical confidence intervals. We repeat this process for all spatial extents of exceedances in between 100 km$^2$ and $1.4\times 10^5$ km$^2$. Note that this can be computed rapidly thanks to the semi-amortized nature of our XVAE. The right panels of Figure~\ref{fig:POT_analysis} reveal a consistent rise in SST thresholds from 1985 (top panel) to 2015 (bottom panel), confirming the warming trend. Slicing the confidence bands at the 31$^\circ$C threshold aligns with the middle panel results. This also shows that within each subregion, the spatial extent of exceedances decreases as the threshold increases and extreme events becomes more localized as they get more extreme. Furthermore, as the spatial extent approaches zero, the threshold estimates represent the highest possible SST for a specific month in a subregion, a valuable metric for studying phytoplankton bloom.

To directly assess the impact of climate change, Figure~\ref{fig:thresh_vs_time} shows results for specific spatial extent of exceedances (i.e., $5\times 10^4$ km$^2$) across all September months from 1985 to 2015. We see that the fixed-area SST threshold has increased steadily by about 0.7$^\circ$ in all four subregions on average over the studied time period, corroborating the warming trend in the Red Sea and the localized nature of extremes shown in Figures~\ref{fig:chi_raster} and \ref{fig:POT_analysis}.

\begin{figure}
    \centering
    \includegraphics[width=0.54\linewidth]{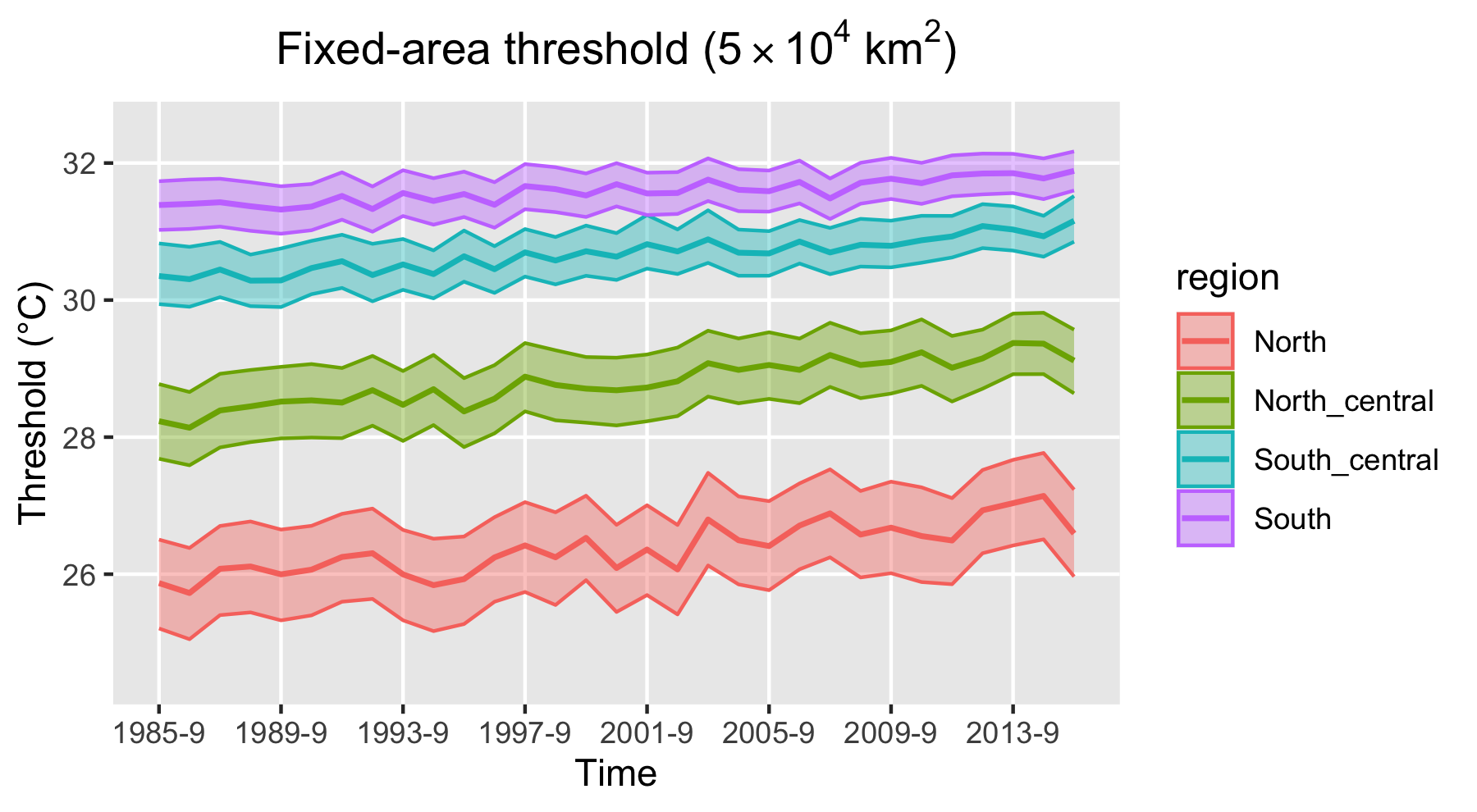}
    \vskip -0.4cm
    \caption{Similar to Figure~\ref{fig:POT_analysis}, we emulate 30,000 independent SST fields for each September from 1985 to 2015. For a fixed areal exceedance of $5\times 10^4$ km$^2$, we estimate its associated required threshold along with the 95\% Monte Carlo confidence intervals.}
    \label{fig:thresh_vs_time}
\end{figure}

\LZadd{Through our analysis, we have demonstrated the ability of our XVAE to quantify risk associated with SSTs exceeding critical bleaching thresholds. Our approach not only captures spatial dependence of extreme temperatures but also provides robust UQ. Our findings could support conservation efforts by identifying regions most susceptible to coral bleaching and predicting the potential impact of rising SST on the broader marine environment.}

\section{Concluding remarks}\label{sec:discussion}
In this paper, we propose XVAE, a new variational autoencoder, which integrates a novel max-id model for spatial extremes that exhibits flexible extremal dependence properties. It greatly advances the ability to model extremes in high-dimensional spatial problems and expands the frontier on computation and modeling of complex extremal processes. The encoder and decoder construction and the trained distributions of the latent variables allow for parameter estimation and uncertainty quantification within a variational Bayesian framework. We also develop a validation framework for evaluating emulator performance 
when applied to spatial data with dependent extremes. 

We note that our emulator extends beyond emulating large datasets for UQ. As highlighted in the introduction, the XVAE can serve as a surrogate model for mechanistic-based computer models. It can also be applied to areas other than climate-related problems. For example, turbulent buoyant plume can be simulated from a system of compressible Euler conservative equations in flux formulation, but the computational cost is prohibitively expensive with increasing Reynolds number \citep{bhaganagar2021MWR}. Our XVAE can provide a promising avenue for efficiently emulating the chaotic and irregular turbulence observations at high resolutions.

One major drawback of XVAE is that the latent expPS variables are independent over space and time, which is unrealistic for physical processes that exhibit dynamics at short-time scales. As a result, it cannot capture temporal dependence appropriately. In future work, we are planning to include a time component with data-driven dynamic learning based on a stochastic dynamic spatio-temporal model. Hence, the latent variables in the encoded space will evolve smoothly over time while retaining heavy tails and thus simultaneously ensuring local extremal dependence. Furthermore, it is possible to improve the XVAE by allowing the basis functions' radii $r_k$, $k=1,\ldots, K$, to be spatially-varying, and estimating them by optimizing the ELBO together with the other parameters.

Another promising direction for future work is to implement a \emph{conditional} VAE \citep[CVAE;][]{sohn2015learning} with a similar underlying max-id model; in such a model, we can allow the parameters of both the encoder and decoder to change conditional on different climate scenarios (e.g., radiative forcings, seasons, soil conditions, etc.). This will allow us to simulate new data under different conditions. We will need to ensure that the CVAE emulates $\bx_t$ differently according to different input states (e.g., tuning parameters and/or forcing variables). In doing so, we will allow changes to the parameters for both the encoder and decoder conditioning on different scenarios (e.g., different climate states).







{
\spacingset{0.9}
\setlength{\bibsep}{5pt}
\addcontentsline{toc}{section}{References}
\bibliography{main}

@article{nolan2020univariate,
  title={Univariate Stable Distributions},
  author={Nolan, John P},
  journal={Springer Series in Operations Research and Financial Engineering, DOI},
  volume={10},
  pages={978--3},
  year={2020},
  publisher={Springer}
}

@book{resnick2008extreme,
  title={Extreme Values, Regular Variation, and Point Processes},
  author={Resnick, Sidney I},
  volume={4},
  year={2008},
  publisher={Springer Science \& Business Media},
  doi = {https://doi.org/10.1007/978-0-387-75953-1}
}

@article{sen1967theory,
  title={On the theory of rank order tests for location in the multivariate one sample problem},
  author={Sen, Pranab Kumar and Puri, Madan Lal},
  journal={The Annals of Mathematical Statistics},
  volume={38},
  number={4},
  pages={1216--1228},
  year={1967},
  publisher={Institute of Mathematical Statistics}
}

@article{ruymgaart1974asymptotic,
  title={Asymptotic normality of nonparametric tests for independence},
  author={Ruymgaart, Frederik Hendrik},
  journal={The Annals of Statistics},
  pages={892--910},
  year={1974},
  publisher={JSTOR}
}

@article{ruymgaart1978asymptotic,
  title={Asymptotic normality of multivariate linear rank statistics in the non-iid case},
  author={Ruymgaart, Frits H and van Zuijlen, MCA},
  journal={The Annals of Statistics},
  volume={6},
  number={3},
  pages={588--602},
  year={1978},
  publisher={Institute of Mathematical Statistics}
}

@article{huser2017bridging,
  title={Bridging asymptotic independence and dependence in spatial extremes using {G}aussian scale mixtures},
  author={Huser, Rapha{\"e}l and Opitz, Thomas and Thibaud, Emeric},
  journal={Spatial Statistics},
  volume={21},
  pages={166--186},
  year={2017},
  publisher={Elsevier}
}

@article{hartigan1979algorithm,
  title={{Algorithm AS 136}: A k-means clustering algorithm},
  author={Hartigan, John A and Wong, Manchek A},
  journal={Journal of the Royal Statistical Society: Series C},
  volume={28},
  number={1},
  pages={100--108},
  year={1979},
  publisher={JSTOR}
}

@article{huser2021eva,
  title={{EVA} 2019 data competition on spatio-temporal prediction of {Red Sea} surface temperature extremes},
  author={Huser, Rapha{\"e}l},
  journal={Extremes},
  volume={24},
  pages={91--104},
  year={2021},
  publisher={Springer}
}

@article{Opitz2016,
  title={Modeling asymptotically independent spatial extremes based on {L}aplace random fields},
  author={Opitz, Thomas},
  journal={Spatial Statistics},
  volume={16},
  pages={1--18},
  year={2016},
  publisher={Elsevier}
}

@article{Huser2019,
  title={Modeling spatial processes with unknown extremal dependence class},
  author={Huser, Rapha{\"e}l and Wadsworth, Jennifer L},
  journal={Journal of the American Statistical Association},
  volume={114},
  number={525},
  pages={434--444},
  year={2019},
  publisher={Taylor \& Francis}
}

@article{padoan2013extreme,
  title={Extreme dependence models based on event magnitude},
  author={Padoan, Simone A},
  journal={Journal of Multivariate Analysis},
  volume={122},
  pages={1--19},
  year={2013},
  publisher={Elsevier}
}

@article{huser2021max,
  title={Max-infinitely divisible models and inference for spatial extremes},
  author={Huser, Rapha{\"e}l and Opitz, Thomas and Thibaud, Emeric},
  journal={Scandinavian Journal of Statistics},
  volume={48},
  number={1},
  pages={321--348},
  year={2021},
  publisher={Wiley Online Library}
}

@article{reich2012hierarchical,
  title={A hierarchical max-stable spatial model for extreme precipitation},
  author={Reich, Brian J and Shaby, Benjamin A},
  journal={The Annals of Applied Statistics},
  volume={6},
  number={4},
  pages={1430--1451},
  year={2012},
  publisher={NIH Public Access}
}

@article{bopp2021hierarchical1,
  title={A hierarchical max-infinitely divisible spatial model for extreme precipitation},
  author={Bopp, Gregory P and Shaby, Benjamin A and Huser, Rapha{\"e}l},
  journal={Journal of the American Statistical Association},
  volume={116},
  number={533},
  pages={93--106},
  year={2021},
  publisher={Taylor \& Francis}
}

@article{bhaganagar2021MWR,
  title={{Implementing a new formulation in WRF-LES for Buoyant Plume Simulations: bPlume-WRF-LES model}},
  author={Bhimireddy, Sudheer R and Bhaganagar, Kiran},
  journal={Monthly Weather Review},
  volume={149},
  number={7},
  pages={2299-2319},
  year={2021},
}

@article{lafon2023vae,
  title={{A VAE approach to sample multivariate extremes}},
  author={Lafon, Nicolas and Naveau, Philippe and Fablet, Ronan},
  journal={arXiv preprint arXiv:2306.10987},
  year={2023}
}

@article{zhang2022accounting,
  title={Accounting for the spatial structure of weather systems in detected changes in precipitation extremes},
  author={Zhang, Likun and Risser, Mark D and Molter, Edward M and Wehner, Michael F and O'Brien, Travis A},
  journal={Weather and Climate Extremes},
  volume={38},
  pages={100499},
  year={2023},
  publisher={Elsevier}
}

@article{cotsakis2022perimeter,
  title={On the perimeter estimation of pixelated excursion sets of {2D} anisotropic random fields},
  author={Cotsakis, Ryan and Di Bernardino, Elena and Opitz, Thomas},
  year={2022},
  journal = {Hal Science preprint: hal-03582844v2}
}

@article{matheson1976scoring,
  title={Scoring rules for continuous probability distributions},
  author={Matheson, James E and Winkler, Robert L},
  journal={Management Science},
  volume={22},
  number={10},
  pages={1087--1096},
  year={1976},
  publisher={INFORMS}
}

@article{gneiting2007strictly,
  title={Strictly proper scoring rules, prediction, and estimation},
  author={Gneiting, Tilmann and Raftery, Adrian E},
  journal={Journal of the American statistical Association},
  volume={102},
  number={477},
  pages={359--378},
  year={2007},
  publisher={Taylor \& Francis}
}

@article{kingma2019introduction,
  title={An introduction to variational autoencoders},
  author={Kingma, Diederik P and Welling, Max and others},
  journal={Foundations and Trends{\textregistered} in Machine Learning},
  volume={12},
  number={4},
  pages={307--392},
  year={2019},
  publisher={Now Publishers, Inc.}
}

@article{cartwright2023emulation,
  title={Emulation of greenhouse-gas sensitivities using variational autoencoders},
  author={Cartwright, Laura and Zammit-Mangion, Andrew and Deutscher, Nicholas M},
  journal={Environmetrics},
  volume={34},
  number={2},
  year={2023},
  publisher={Wiley Online Library}
}

@article{hougaard1986survival,
  title={Survival models for heterogeneous populations derived from stable distributions},
  author={Hougaard, Philip},
  journal={Biometrika},
  volume={73},
  number={2},
  pages={387--396},
  year={1986},
  publisher={Oxford University Press}
}

@article{kingma2013auto,
  title={Auto-encoding variational {B}ayes},
  author={Kingma, Diederik P and Welling, Max},
  journal={arXiv preprint arXiv:1312.6114},
  year={2013}
}

@Manual{rtorch,
  title = {torch: Tensors and Neural Networks with `GPU' Acceleration},
  author = {Daniel Falbel and Javier Luraschi},
  year = {2023},
  note = {\textbf{URL:} https://github.com/mlverse/torch},
}

@book{keydana2023deep,
  title={Deep Learning and Scientific Computing with R torch},
  author={Keydana, Sigrid},
  year={2023},
  publisher={CRC Press}
}

@article{polyak1964momentum,
title = {Some methods of speeding up the convergence of iteration methods},
journal = {USSR Computational Mathematics and Mathematical Physics},
volume = {4},
number = {5},
pages = {1--17},
year = {1964},
issn = {0041-5553},
doi = {https://doi.org/10.1016/0041-5553(64)90137-5},
OPTurl = {https://www.sciencedirect.com/science/article/pii/0041555364901375},
author = {Polyak, Boris T},
publisher={Elsevier}
}

@article{simpson2020high,
  title={High-dimensional modeling of spatial and spatio-temporal conditional extremes using {INLA} and {Gaussian Markov} random fields},
  author={Simpson, Emma S and Opitz, Thomas and Wadsworth, Jennifer L},
  journal={Extremes},
  pages={1--45},
  year={2023},
  publisher={Springer}
}

@article{sainsbury-dale2022,
    author = {Matthew Sainsbury-Dale and Andrew Zammit-Mangion and Raphaël Huser},
    title = {Likelihood-Free Parameter Estimation with Neural {B}ayes Estimators},
    journal = {The American Statistician},
    pages = {1--14},
    year  = {2024},
    volume  = {78},
    publisher = {Taylor & Francis},
    doi = {10.1080/00031305.2023.2249522}
}

@article{zammit2020deep,
    author = {Andrew Zammit-Mangion and Christopher K. Wikle},
    title = {Deep integro-difference equation models for spatio-temporal forecasting},
    journal = {Spatial Statistics},
    pages = {100408},
    year  = {2020},
    volume  = {37}
}

@article{richards2023likelihood,
  title={Likelihood-free neural {B}ayes estimators for censored peaks-over-threshold models},
  author={Richards, Jordan and Sainsbury-Dale, Matthew and Zammit-Mangion, Andrew and Huser, Rapha{\"e}l},
  journal={arXiv preprint arXiv:2306.15642},
  year={2023}
}

@article{sainsbury2023neural,
  title={Neural Bayes estimators for irregular spatial data using graph neural networks},
  author={Sainsbury-Dale, Matthew and Zammit-Mangion, Andrew and Richards, Jordan and Huser, Rapha{\"e}l},
  journal={arXiv preprint arXiv:2310.02600},
  year={2023}
}

@article{andrew2025amortized,
  title={Neural Methods for Amortised Parameter Inference},
  author={Zammit-Mangion, Andrew and Sainsbury-Dale, Matthew and Huser, Rapha{\"e}l},
  journal={Annual Reviews of Statistics and Its Application},
  year={2025},
  note={to appear}
}

@article{gopalan2022higher,
  title={A higher-order singular value decomposition tensor emulator for spatiotemporal simulators},
  author={Gopalan, Giri and Wikle, Christopher K},
  journal={Journal of Agricultural, Biological and Environmental Statistics},
  volume={27},
  number={1},
  pages={22--45},
  year={2022},
  publisher={Springer}
}

@article{krupskii2022modeling,
  title={Modeling spatial tail dependence with {C}auchy convolution processes},
  author={Krupskii, Pavel and Huser, Rapha{\"e}l},
  journal={Electronic Journal of Statistics},
  volume={16},
  number={2},
  pages={6135--6174},
  year={2022},
  publisher={The Institute of Mathematical Statistics and the Bernoulli Society}
}

@article{ledford1996statistics,
  title={Statistics for near independence in multivariate extreme values},
  author={Ledford, Anthony W and Tawn, Jonathan A},
  journal={Biometrika},
  volume={83},
  number={1},
  pages={169--187},
  year={1996},
  publisher={Oxford University Press}
}

@book{gramacy2020surrogates,
  title={Surrogates: Gaussian Process Modeling, Design, and Optimization for the Applied Sciences},
  author={Gramacy, Robert B},
  year={2020},
  publisher={Chapman and Hall/CRC}
}

@article{gu2018robust,
  title={Robust {Gaussian} stochastic process emulation},
  author={Gu, Mengyang and Wang, Xiaojing and Berger, James O},
  journal={The Annals of Statistics},
  volume={46},
  number={6A},
  pages={3038--3066},
  year={2018},
  publisher={JSTOR}
}

@incollection{sargsyan2017surrogate,
  title={Surrogate models for uncertainty propagation and sensitivity analysis},
  author={Sargsyan, Khachik},
  booktitle={Handbook of uncertainty quantification},
  pages={673--698},
  year={2017},
  publisher={Springer}
}

@article{goodfellow2014generative,
  title={Generative adversarial nets},
  author={Goodfellow, Ian and Pouget-Abadie, Jean and Mirza, Mehdi and Xu, Bing and Warde-Farley, David and Ozair, Sherjil and Courville, Aaron and Bengio, Yoshua},
  journal={Advances in Neural Information Processing Systems},
  volume={27},
  year={2014}
}

@article{davison2015statistics,
  title={Statistics of extremes},
  author={Davison, Anthony C. and Huser, Rapha\"el},
  journal={Annual Review of Statistics and its Application},
  volume={2},
  year={2015},
  pages = {203--235}
}

@article{donlon2012operational,
  title={The operational sea surface temperature and sea ice analysis ({OSTIA}) system},
  author={Donlon, Craig J and Martin, Matthew and Stark, John and Roberts-Jones, Jonah and Fiedler, Emma and Wimmer, Werenfrid},
  journal={Remote Sensing of Environment},
  volume={116},
  pages={140--158},
  year={2012},
  publisher={Elsevier}
}

@article{hazra2021estimating,
  title={Estimating high-resolution {Red Sea} surface temperature hotspots, using a low-rank semiparametric spatial model},
  author={Hazra, Arnab and Huser, Rapha{\"e}l},
  journal={The Annals of Applied Statistics},
  volume={15},
  number={2},
  pages={572--596},
  year={2021},
  publisher={Institute of Mathematical Statistics}
}

@article{simpson2021conditional,
  title={Conditional modelling of spatio-temporal extremes for {Red Sea} surface temperatures},
  author={Simpson, Emma S and Wadsworth, Jennifer L},
  journal={Spatial Statistics},
  volume={41},
  pages={100482},
  year={2021},
  publisher={Elsevier}
}

@article{oesting2022patterns,
  title={Patterns in spatio-temporal extremes},
  author={Oesting, Marco and Huser, Rapha\"el},
  journal={arXiv preprint arXiv:2212.11001},
  year={2022}
}

@article{wendland1995piecewise,
  title={Piecewise polynomial, positive definite and compactly supported radial functions of minimal degree},
  author={Wendland, Holger},
  journal={Advances in Computational Mathematics},
  volume={4},
  pages={389},
  year={1995},
  publisher={Springer}
}

@Article{hetGP,
    title = {{hetGP}: Heteroskedastic {G}aussian Process Modeling and Sequential Design in {R}},
    author = {Micka\"el Binois and Robert B. Gramacy},
    journal = {Journal of Statistical Software},
    year = {2021},
    volume = {98},
    number = {13},
    pages = {1--44},
    doi = {10.18637/jss.v098.i13},
}

@article{richards2022unifying,
  title={A unifying partially-interpretable framework for neural network-based extreme quantile regression},
  author={Richards, Jordan and Huser, Rapha{\"e}l},
  journal={arXiv preprint arXiv:2208.07581},
  year={2022}
}

@article{sohn2015learning,
  title={Learning structured output representation using deep conditional generative models},
  author={Sohn, Kihyuk and Lee, Honglak and Yan, Xinchen},
  journal={Advances in Neural Information Processing Systems},
  volume={28},
  year={2015}
}

@article{davison2012statistical,
author = {A. C. Davison and S. A. Padoan and M. Ribatet},
title = {{Statistical Modeling of Spatial Extremes}},
volume = {27},
journal = {Statistical Science},
number = {2},
publisher = {Institute of Mathematical Statistics},
pages = {161--186},
year = {2012}
}

@incollection{davison2019spatial,
  title={Spatial extremes},
  author={Davison, Anthony C and Huser, Rapha{\"e}l and Thibaud, Emeric},
  year={2019},
  booktitle = {Handbook of Environmental and Ecological Statistics},
  publisher={editors A. E. Gelfand, M. Fuentes, J. A. Hoeting and R. L. Smith, CRC Press}, 
  pages={711--744},
}

@article{ferreira2014generalized,
author = {Ana Ferreira and Laurens de Haan},
title = {{The generalized Pareto process; with a view towards application and simulation}},
volume = {20},
journal = {Bernoulli},
number = {4},
publisher = {Bernoulli Society for Mathematical Statistics and Probability},
pages = {1717--1737},
year = {2014},
doi = {10.3150/13-BEJ538},
OPTurl = {https://doi.org/10.3150/13-BEJ538}
}

@article{thibaud2015efficient,
  title={Efficient inference and simulation for elliptical {P}areto processes},
  author={Thibaud, Emeric and Opitz, Thomas},
  journal={Biometrika},
  volume={102},
  number={4},
  pages={855--870},
  year={2015},
  publisher={Oxford University Press}
}

@article{huser2022advances,
  title={Advances in statistical modeling of spatial extremes},
  author={Huser, Rapha{\"e}l and Wadsworth, Jennifer L},
  journal={Wiley Interdisciplinary Reviews: Computational Statistics},
  volume={14},
  number={1},
  pages={e1537},
  year={2022},
  publisher={Wiley Online Library}
}

@article{huser2024time,
  title={Modeling of spatial extremes in environmental data science: Time to move away from max-stable processes},
  author={Huser, Rapha{\"e}l and Opitz, Thomas and Wadsworth, Jennifer L},
  journal={arXiv preprint arXiv:2401.17430},
  year={2024},
}

@article{zhang-2022a,
    author = {Zhang, Likun and Shaby, Benjamin A and Wadsworth, Jennifer L},
    title = {Hierarchical transformed scale mixtures for flexible modeling of spatial extremes on datasets with many locations},
    journal = {Journal of the American Statistical Association},
    volume = {117},
    number = {539},
    pages = {1357-1369},
    year  = {2022},
    publisher = {Taylor & Francis},
    doi = {10.1080/01621459.2020.1858838},
    OPTurl = {https://doi.org/10.1080/01621459.2020.1858838},
    eprint = {https://doi.org/10.1080/01621459.2020.1858838}
}

@article{hughes2017global,
  title={Global warming and recurrent mass bleaching of corals},
  author={Hughes, Terry P and Kerry, James T and {\'A}lvarez-Noriega, Mariana and {\'A}lvarez-Romero, Jorge G and Anderson, Kristen D and Baird, Andrew H and Babcock, Russell C and Beger, Maria and Bellwood, David R and Berkelmans, Ray and others},
  journal={Nature},
  volume={543},
  number={7645},
  pages={373--377},
  year={2017},
  publisher={Nature Publishing Group UK London}
}

@article{genevier2019marine,
  title={Marine heatwaves reveal coral reef zones susceptible to bleaching in the {Red Sea}},
  author={Genevier, Lily GC and Jamil, Tahira and Raitsos, Dionysios E and Krokos, George and Hoteit, Ibrahim},
  journal={Global Change Biology},
  volume={25},
  number={7},
  pages={2338--2351},
  year={2019},
  publisher={Wiley Online Library}
}

@article{furby2013susceptibility,
  title={Susceptibility of central {Red Sea} corals during a major bleaching event},
  author={Furby, Kathryn A and Bouwmeester, Jessica and Berumen, Michael L},
  journal={Coral Reefs},
  volume={32},
  pages={505--513},
  year={2013},
  publisher={Springer}
}

@article{raitsos2013remote,
  title={Remote sensing the phytoplankton seasonal succession of the {Red Sea}},
  author={Raitsos, Dionysios E and Pradhan, Yaswant and Brewin, Robert JW and Stenchikov, Georgiy and Hoteit, Ibrahim},
  journal={PLoS One},
  volume={8},
  number={6},
  pages={e64909},
  year={2013},
  publisher={Public Library of Science San Francisco, USA}
}

@article{de2018high,
  title={High-dimensional peaks-over-threshold inference},
  author={de Fondeville, Rapha{\"e}l and Davison, Anthony C},
  journal={Biometrika},
  volume={105},
  number={3},
  pages={575--592},
  year={2018},
  publisher={Oxford University Press}
}

@article{zhong2022modeling,
  title={Modeling nonstationary temperature maxima based on extremal dependence changing with event magnitude},
  author={Zhong, Peng and Huser, Rapha\"el and Opitz, Thomas},
  journal={Annals of Applied Statistics},
  pages={272--299},
  year={2022},
  volume={16}
}

@article{boulaguiem2022modeling,
  title={Modeling and simulating spatial extremes by combining extreme value theory with generative adversarial networks},
  author={Boulaguiem, Younes and Zscheischler, Jakob and Vignotto, Edoardo and van der Wiel, Karin and Engelke, Sebastian},
  journal={Environmental Data Science},
  volume={1},
  pages={e5},
  year={2022},
  publisher={Cambridge University Press}
}

@article{kennedy2001bayesian,
  title={Bayesian calibration of computer models},
  author={Kennedy, Marc C and O'Hagan, Anthony},
  journal={Journal of the Royal Statistical Society: Series B (Statistical Methodology)},
  volume={63},
  number={3},
  pages={425--464},
  year={2001},
  publisher={Wiley Online Library}
}

@article{higdon2004combining,
  title={Combining field data and computer simulations for calibration and prediction},
  author={Higdon, Dave and Kennedy, Marc and Cavendish, James C and Cafeo, John A and Ryne, Robert D},
  journal={SIAM Journal on Scientific Computing},
  volume={26},
  number={2},
  pages={448--466},
  year={2004},
  publisher={SIAM}
}

@article{chang2016calibrating,
  title={Calibrating an ice sheet model using high-dimensional binary spatial data},
  author={Chang, Won and Haran, Murali and Applegate, Patrick and Pollard, David},
  journal={Journal of the American Statistical Association},
  volume={111},
  number={513},
  pages={57--72},
  year={2016},
  publisher={Taylor \& Francis}
}

@article{bayarri2007computer,
  title={Computer model validation with functional output},
  author={Bayarri, MJ and Berger, JO and Cafeo, John and Garcia-Donato, G and Liu, F and Palomo, J and Parthasarathy, RJ and Paulo, R and Sacks, Jerry and Walsh, D23639561144},
  journal={The Annals of Statistics},
  pages={1874--1906},
  year={2007},
  publisher={JSTOR}
}

@article{sacks1989designs,
  title={Designs for computer experiments},
  author={Sacks, Jerome and Schiller, Susannah B and Welch, William J},
  journal={Technometrics},
  volume={31},
  number={1},
  pages={41--47},
  year={1989},
  publisher={Taylor \& Francis}
}

@article{hazra2024,
author = {Arnab Hazra and Raphaël Huser and David Bolin},
title = {Efficient Modeling of Spatial Extremes over Large Geographical Domains},
journal = {Journal of Computational and Graphical Statistics},
year = {2024},
note={to appear},
doi = {10.1080/10618600.2024.2409784}
}

@article{pasche2022neural,
  title={Neural networks for extreme quantile regression with an application to forecasting of flood risk},
  author={Pasche, Olivier C and Engelke, Sebastian},
  journal={The Annals of Applied Statistics},
  year={2024},
  note={to appear}
}

@article{lenzi2023neural,
  title={Neural networks for parameter estimation in intractable models},
  author={Lenzi, Amanda and Bessac, Julie and Rudi, Johann and Stein, Michael L},
  journal={Computational Statistics \& Data Analysis},
  volume={185},
  pages={107762},
  year={2023},
  publisher={Elsevier}
}

@article{majumder2024modeling,
  title={Modeling extremal streamflow using deep learning approximations and a flexible spatial process},
  author={Majumder, Reetam and Reich, Brian J and Shaby, Benjamin A},
  journal={The Annals of Applied Statistics},
  volume={18},
  number={2},
  pages={1519--1542},
  year={2024},
  publisher={Institute of Mathematical Statistics}
}

@article{maceda2024variational,
  title={A variational neural Bayes framework for inference on intractable posterior distributions},
  author={Maceda, Elliot and Hector, Emily C and Lenzi, Amanda and Reich, Brian J},
  journal={arXiv preprint arXiv:2404.10899},
  year={2024}
}

@article{bradley2005basic,
  title={Basic properties of strong mixing conditions. A survey and some open questions},
  author={Bradley, Richard C},
  year={2005},
  journal = {Probability Surveys},
  number={2},
  pages={107--144}
}
}
\newpage
\appendix
\numberwithin{equation}{section}
\numberwithin{figure}{section}

\section{Technical details} \label{sec:proofs}
\subsection{Properties of exponentially-tilted positive-stable variables}\label{sec:PS_properties}
Before we proceed to prove Proposition~\ref{prop:marg_distr} from the main paper, we first recall some useful results in \citet{hougaard1986survival} about positive-stable (PS) distributions and their exponentially-tilted variation. If $Z\sim \mathrm{expPS}(\alpha,0)$, we denote the density function by $f_\alpha(z)$, $z>0$. Then for $\alpha\in (0,1]$, it has Laplace transform
\begin{equation*}
    L(s) = \mathbb{E}{e^{-sZ}}=\exp(-s^\alpha),\; s\geq 0.
\end{equation*}
For an exponentially-tilted variable $Z\sim \mathrm{expPS}(\alpha,\gamma)$, the Laplace transform becomes
\begin{equation}\label{eqn:laplace_trans}
    L(s) = \mathbb{E}{e^{-sZ}}=\exp\left[-\{(\gamma+s)^\alpha-\gamma^\alpha\}\right],\; s\geq 0,\;\gamma\geq 0
\end{equation}
and its density is 
\begin{equation*}
    h(x;\alpha, \gamma)=\frac{f_\alpha(x)\exp(-\gamma x)}{\exp(-\gamma^\alpha)},\; x>0.
\end{equation*}

\begin{lemma}
If $Z\sim \mathrm{expPS}(\alpha,0)$ and $\alpha\in (0,1)$, then $Z\sim \text{Stable}\left\{\alpha, 1,\cos^{{1}/{\alpha}} ({\pi\alpha}/{2}),0\right\}$ in the 1-parameterization \citep{nolan2020univariate}.
\end{lemma}

\begin{proof}
From Proposition 3.2 of \citet{nolan2020univariate}, we know that the Laplace transform of the random variable $Z\sim\text{Stable}(\alpha, 1, \xi,0;1)$, $\alpha\in (0,2]$, is
\begin{equation*}
\mathbb{E}{e^{-sZ}} = 
    \begin{cases}
        \exp\{-\xi^\alpha(\sec \frac{\pi\alpha}{2})s^\alpha\}, &\alpha\in (0,1)\cup (1,2],\\
        \exp\{-\xi \frac{2}{\pi}s\log s\}, &\alpha=1.
    \end{cases}
\end{equation*}
When $\xi=|\cos \frac{\pi\alpha}{2}|^{1/\alpha}$, the Laplace transform becomes
\begin{equation*}
\mathbb{E}{e^{-sZ}} = 
    \begin{cases}
        \exp(-s^\alpha), &\alpha\in (0,1),\\
        \exp(s^\alpha), &\alpha\in (1,2].
    \end{cases}
\end{equation*}
That is, $Z\sim \mathrm{expPS}(\alpha,0)$ when $\alpha\in (0,1)$.
\end{proof}
\begin{remark}
If $\alpha=1/2$, then $|\cos \frac{\pi\alpha}{2}|^{1/\alpha}=1/2$ and $Z\sim \text{Stable}(1/2, 1, 1/2,0;1)$, which is equivalent to $Z\sim \text{L\'{e}vy}(0, 1/2)$ or $Z\sim \text{InvGamma}(1/2, 1/4)$.
\end{remark}
\begin{remark}
    To facilitate the computation of the prior in Eq.~\eqref{eqn:prior_form} of the main paper, we follow the Monte Carlo integration steps in Section 4 of the Supplementary Material of \citet{bopp2021hierarchical1} to calculate the density $h(\cdot; \alpha,\gamma)$.
\end{remark}

\subsection{Proof of Proposition~\ref{prop:marg_distr} of the main paper}\label{proof:marg_distr}
\begin{proof}[Proof of Proposition~\ref{prop:marg_distr}]
Since at the location $\bs_j$,
\begin{small}
    \begin{align*}
    \Pr(X(\bs_j)\leq x)&=\mathbb{E}\left\{\Pr\left(\left.\epsilon(\bs_j)\leq \frac{x}{Y(\bs_j)} \right\vert Z_{1},\ldots, Z_{K}\right)\right\}=\mathbb{E} \left[\exp\left\{-\left(\left.\frac{\tau Y(\bs_j)}{x}\right)^{\frac{1}{\alpha_0}}\right\}\right\vert Z_{1},\ldots, Z_{K}\right] \\
    &= \mathbb{E} \exp\left\{-\left(\frac{\tau}{x}\right)^{\frac{1}{\alpha_0}}\sumK \omega_k(\bs_j, r_k)^{\frac{1}{\alpha}}Z_{k}\right\}
    =  \exp\left[\sum_{k\in \bar{\mathcal{D}}}\gamma_k^\alpha - \sumK\left\{\gamma_k+\left(\frac{\tau}{x}\right)^{\frac{1}{\alpha_0}}\omega_{kj}^{\frac{1}{\alpha}}\right\}^\alpha\right].
\end{align*}
\end{small}
We now show that as $x\rightarrow\infty$, the survival function $\bar{F}_j(x) = 1-F_j(x)$ satisfies
\begin{small}
    \begin{equation}\label{eqn:survival_fn}
    \bar{F}_j(x) =c'_j\left(\frac{x}{\tau}\right)^{-\frac{\alpha}{\alpha_0}}+c_j\left(\frac{x}{\tau}\right)^{-\frac{1}{\alpha_0}}+\left(d_j-\frac{c_j^2}{2}\right)\left(\frac{x}{\tau}\right)^{-\frac{2}{\alpha_0}}-\frac{c'_j{}^2}{2}\left(\frac{x}{\tau}\right)^{-\frac{2\alpha}{\alpha_0}}-c'_jc_j\left(\frac{x}{\tau}\right)^{-\frac{\alpha+1}{\alpha_0}} + o\left(x^{-\frac{2}{\alpha_0}}\right),
\end{equation}
\end{small}
\noindent where $c_j =\alpha\sum_{k\in \bar{\mathcal{D}}} \gamma_k^{\alpha-1}\omega_{kj}^{1/\alpha}$, $c'_j = \sum_{k\in \mathcal{D}}\omega_{kj}$, and $d_j = \frac{\alpha(\alpha-1)}{2}\sum_{k\in \bar{\mathcal{D}}} \gamma_k^{\alpha-2}\omega_{kj}^{2/\alpha}$.     

First, we apply Taylor's expansion with the Peano remainder:
\begin{align}\label{eqn:taylor_1_plus_t}
    (1+t)^{\alpha}=1+\alpha t +\frac{\alpha(\alpha-1)}{2}t^2 + o(t^2),\;\text{as }t\rightarrow 0.
\end{align}
Then, as $x\rightarrow\infty$, we have
\begin{equation*}
    \begin{split}
        \sum_{k\in \bar{\mathcal{D}}}&\left\{\gamma_k+\left(\frac{\tau}{x}\right)^{\frac{1}{\alpha_0}}\omega_{kj}^{\frac{1}{\alpha}}\right\}^\alpha = \sum_{k\in \bar{\mathcal{D}}}\gamma_k^\alpha\left\{1+\left(\frac{\tau}{x}\right)^{\frac{1}{\alpha_0}}\frac{\omega_{kj}^{1/\alpha}}{\gamma_k}\right\}^\alpha\\
    &=\sum_{k\in \bar{\mathcal{D}}}\gamma_k^\alpha+\alpha\left(\frac{\tau}{x}\right)^{\frac{1}{\alpha_0}}\sum_{k\in \bar{\mathcal{D}}}\frac{\omega_{kj}^{1/\alpha}}{\gamma_k^{1-\alpha}}+\frac{\alpha(\alpha-1)}{2}\left(\frac{\tau}{x}\right)^{\frac{2}{\alpha_0}}\sum_{k\in \bar{\mathcal{D}}}\frac{\omega_{kj}^{2/\alpha}}{\gamma_k^{2-\alpha}} + o\left(x^{-\frac{2}{\alpha_0}}\right),
    \end{split}
\end{equation*}
which leads to
\begin{small}
    \begin{equation*}
    \begin{split}
        \sum_{k\in \bar{\mathcal{D}}}&\gamma_k^\alpha - \sumK\left\{\gamma_k+\left(\frac{\tau}{x}\right)^{\frac{1}{\alpha_0}}\omega_{kj}^{\frac{1}{\alpha}}\right\}^\alpha=-\left(\frac{\tau}{x}\right)^{\frac{\alpha}{\alpha_0}}\sum_{k\in \mathcal{D}}\omega_{kj}-\alpha\left(\frac{\tau}{x}\right)^{\frac{1}{\alpha_0}}\sum_{k\in \bar{\mathcal{D}}}\frac{\omega_{kj}^{1/\alpha}}{\gamma_k^{1-\alpha}}\\
    &-\frac{\alpha(\alpha-1)}{2}\left(\frac{\tau}{x}\right)^{\frac{2}{\alpha_0}}\sum_{k\in \bar{\mathcal{D}}}\frac{\omega_{kj}^{2/\alpha}}{\gamma_k^{2-\alpha}} + o\left(x^{-\frac{2}{\alpha_0}}\right)=-c_j'\left(\frac{x}{\tau}\right)^{-\frac{\alpha}{\alpha_0}}-c_j \left(\frac{x}{\tau}\right)^{-\frac{1}{\alpha_0}}-d_j\left(\frac{x}{\tau}\right)^{-\frac{2}{\alpha_0}}+o\left(x^{-\frac{2}{\alpha_0}}\right),
    \end{split}
\end{equation*}
\end{small}
where the constants $c'_j$, $c_j$ and $d_j$ are defined in Proposition~\ref{prop:marg_distr} from the main paper.

Next we apply the following Taylor expansion
\begin{equation}\label{eqn:taylor_exp_t}
    1-\exp(-t) = t - \frac{t^2}{2}+o(t^2),\;\text{as }t\rightarrow 0.
\end{equation}
to get
\begin{align*}
    \bar{F}_j&(x) =c_j'\left(\frac{x}{\tau}\right)^{-\frac{\alpha}{\alpha_0}}+c_j \left(\frac{x}{\tau}\right)^{-\frac{1}{\alpha_0}}+d_j\left(\frac{x}{\tau}\right)^{-\frac{2}{\alpha_0}}\\
    &-\frac{1}{2}\left\{c_j'\left(\frac{x}{\tau}\right)^{-\frac{\alpha}{\alpha_0}}+c_j \left(\frac{x}{\tau}\right)^{-\frac{1}{\alpha_0}}+d_j\left(\frac{x}{\tau}\right)^{-\frac{2}{\alpha_0}}\right\}^2+o\left(x^{-\frac{2}{\alpha_0}}\right),
\end{align*}
\vskip -0.3cm
\noindent from which we can expand the squared term and discard the terms with higher decaying rates than $o(x^{-2/\alpha_0})$ to establish \eqref{eqn:survival_fn}.

Lastly, from \eqref{eqn:survival_fn}, it is clear that as $x\rightarrow\infty$, $\bar{F}_j(x) \sim c_j(x/\tau)^{-{1/\alpha_0}}$ if $\mathcal{C}_j\cap \mathcal{D}= \emptyset$, and $\bar{F}_j(x) \sim c_j'(x/\tau)^{-{\alpha/\alpha_0}}$ if $\mathcal{C}_j\cap \mathcal{D}\neq \emptyset$.  



\end{proof}

The following result directly delineates how the quantile level changes as $u\rightarrow 1$. It will be used to derive the tail dependence structure for two arbitrary spatial locations.
\begin{corollary}\label{cor:quantile_func}
As $t\rightarrow\infty$, the marginal quantile function $q_j(t)=F_j^{-1}(1-1/t)$ can be approximated as follows under the assumptions of Proposition~\ref{prop:marg_distr} from the main paper:
\vskip -0.6cm
\begin{equation*}
    q_j(t)=
    \begin{cases}
    \tau c'_j{}^{\alpha_0/\alpha}t^{\alpha_0/\alpha}\left\{1+\frac{\alpha_0 c_jt^{1-1/\alpha}}{\alpha c'_j{}^{1/\alpha}}-\frac{\alpha_0 t^{-1}}{2\alpha}+O\left(t^{-1/\alpha}\right)\right\},& \text{if } \mathcal{C}_j\cap \mathcal{D}\neq \emptyset,\\
    \tau c_j^{\alpha_0}t^{\alpha_0}\{1+\alpha_0(\frac{d_j}{c_j^2}-\frac{1}{2})t^{-1}+o(t^{-1})\},& \text{if } \mathcal{C}_j\cap \mathcal{D}= \emptyset.
    \end{cases}
\end{equation*}
\end{corollary}

\begin{proof}
By definition, $t^{-1}=\bar{F}_j\{q_j(t)\}$. When $\mathcal{C}_j\cap \mathcal{D}\neq \emptyset$, \eqref{eqn:survival_fn} leads to
\begin{equation}\label{eqn:quantile_q_t}
    \begin{split}
         t^{-1}=c'_j\tau^{\frac{\alpha}{\alpha_0}}q_j^{-\frac{\alpha}{\alpha_0}}&(t)
   \left[1+\frac{c_j\tau^{\frac{1-\alpha}{\alpha_0}}}{c'_j}q_j^{-\frac{1-\alpha}{\alpha_0}}(t)+\frac{c_j\tau^{\frac{2-\alpha}{\alpha_0}}}{c'_j}\left(d_j-\frac{c_j^2}{2}\right)q_j^{-\frac{2-\alpha}{\alpha_0}}(t)-\right.\\
   &\left.\frac{c'_j\tau^{\frac{\alpha}{\alpha_0}}}{2}q_j^{-\frac{\alpha}{\alpha_0}}(t)-c_j\tau^{\frac{1}{\alpha_0}}q_j^{-\frac{1}{\alpha_0}}(t) + o\left\{q_j^{-\frac{2-\alpha}{\alpha_0}}(t)\right\}\right]\;
   \text{as }t\rightarrow\infty.
    \end{split}
\end{equation}
Since $q_j(t)\rightarrow\infty$ as $t\rightarrow\infty$, the term in the square bracket of the previous display can simply be approximated by $1+o(1)$. Thus, we have
\begin{equation}\label{eqn:quantile_q_t_1}
    q_j(t) = \tau c'_j{}^{\frac{\alpha_0}{\alpha}}t^{\frac{\alpha_0}{\alpha}}\{1+o(1)\}.
\end{equation}
Since $\alpha\in (0,1)$, we can also re-organize \eqref{eqn:quantile_q_t} to obtain
\begin{small}
    \begin{equation}\label{eqn:quantile_q_t_2}
    \begin{split}
        q_j(t)-\tau c'_j{}^{\frac{\alpha_0}{\alpha}}t^{\frac{\alpha_0}{\alpha}}&=q_j(t)\left(1-\left[1+\frac{c_j\tau^{\frac{1-\alpha}{\alpha_0}}}{c'_j}q_j^{-\frac{1-\alpha}{\alpha_0}}(t)-\frac{c'_j\tau^{\frac{\alpha}{\alpha_0}}}{2}q_j^{-\frac{\alpha}{\alpha_0}}(t)+O\left\{q_j^{-\frac{1}{\alpha_0}}(t)\right\}\right]^{-\frac{\alpha_0}{\alpha}}\right)\\
    &=q_j(t)\left[\frac{\alpha_0 c_j\tau^\frac{1-\alpha}{\alpha_0}}{\alpha c'_j}q_j^{-\frac{1-\alpha}{\alpha_0}}(t)-\frac{\alpha_0 c'_j\tau^\frac{\alpha}{\alpha_0}}{2\alpha}q_j^{-\frac{\alpha}{\alpha_0}}(t)+O\left\{q_j^{-\frac{1}{\alpha_0}}(t)\right\}\right].
    \end{split}
\end{equation}
\end{small}
On the last line, we applied the Taylor expansion in \eqref{eqn:taylor_1_plus_t} again. Then we combine \eqref{eqn:quantile_q_t_1} and \eqref{eqn:quantile_q_t_2} to get
\begin{align*}
    q_j(t)-\tau c'_j{}^{\frac{\alpha_0}{\alpha}}t^{\frac{\alpha_0}{\alpha}}&=\tau c'_j{}^{\frac{\alpha_0}{\alpha}}t^{\frac{\alpha_0}{\alpha}}\{1+o(1)\}\left\{\frac{\alpha_0 c_j}{\alpha c'_j{}^{1/\alpha}}t^{1-\frac{1}{\alpha}}-\frac{\alpha_0}{2\alpha}t^{-1}+O\left(t^{-\frac{1}{\alpha}}\right)\right\}\\
    &=\tau c'_j{}^{\frac{\alpha_0}{\alpha}}t^{\frac{\alpha_0}{\alpha}}\left\{\frac{\alpha_0 c_j}{\alpha c'_j{}^{1/\alpha}}t^{1-\frac{1}{\alpha}}-\frac{\alpha_0}{2\alpha}t^{-1}+O\left(t^{-\frac{1}{\alpha}}\right)\right\},
\end{align*}
which concludes the proof for the first case.

Similarly, when $\mathcal{C}_j\cap \mathcal{D}= \emptyset$, we have
\begin{align*}
    \tau c_j^{\alpha_0}t^{\alpha_0} = q_j(t)\left[1+\left(\frac{d_j}{c_j}-\frac{c_j}{2}\right)\tau^{\frac{1}{\alpha_0}}q_j^{-\frac{1}{\alpha_0}}(t) + o\left\{q_j^{-\frac{1}{\alpha_0}}(t)\right\}\right]^{-\alpha_0}\;
   \text{as }t\rightarrow\infty,
\end{align*}
which ensures $q_j(t) = c_j^{\alpha_0}t^{\alpha_0}\{1+o(1)\}$,
and
\begin{align*}
    q_j(t)-\tau c_j^{\alpha_0}t^{\alpha_0}&=q_j(t)\left(1-\left[1+\left(\frac{d_j}{c_j}-\frac{c_j}{2}\right)\tau^{\frac{1}{\alpha_0}}q_j^{-\frac{1}{\alpha_0}}(t) + o\left\{q_j^{-\frac{1}{\alpha_0}}(t)\right\}\right]^{-\alpha_0}\right)\\
    &=\tau c_j^{\alpha_0}t^{\alpha_0}\{1+o(1)\}\left[\alpha_0\left(\frac{d_j}{c_j}-\frac{c_j}{2}\right)\tau^{\frac{1}{\alpha_0}}q_j^{-\frac{1}{\alpha_0}}(t)+o\left\{q_j^{-\frac{1}{\alpha_0}}(t)\right\}\right]\\
    &=\tau c_j^{\alpha_0}t^{\alpha_0}\left\{\alpha_0\left(\frac{d_j}{c_j^2}-\frac{1}{2}\right)t^{-1}+o(t^{-1})\right\}.
\end{align*}
\end{proof}

\subsection{Proof of Proposition \ref{prop:joint_distr} of the main paper}\label{proof:joint_distr}
\begin{proof}[Proof of Proposition \ref{prop:joint_distr}]
The joint distribution for the discretization of $\{X(\bs)\}$ is
\begin{align*}
    F(x_1,\ldots, x_n) &= \Pr(X(\bs_1)\leq x_1,\ldots, X(\bs_{n_s})\leq x_n)\\
    &=\mathbb{E}\left\{\Pr\left(\left.\epsilon(\bs_1)\leq \frac{x_1}{Y(\bs_1)},\ldots, \epsilon(\bs_{n_s})\leq \frac{x_n}{Y(\bs_{n_s})} \right\vert Z_{1},\ldots, Z_{K}\right)\right\}\\
    &=\mathbb{E} \left[\prodN\exp\left\{-\left(\left.\frac{\tau Y(\bs_j)}{x_j}\right)^{\frac{1}{\alpha_0}}\right\}\right\vert Z_{1},\ldots, Z_{K}\right] \\
    &=\prodK \mathbb{E} \exp\left\{-\sumN \omega_{kj}^{\frac{1}{\alpha}}\left(\frac{\tau}{x_j}\right)^{\frac{1}{\alpha_0}}Z_{k}\right\}
    =  \exp\left[\sum_{k\in \bar{\mathcal{D}}}\gamma_k^\alpha - \sumK\left\{\gamma_k+\tau^{\frac{1}{\alpha_0}}\sumN\frac{\omega_{kj}^{1/\alpha}}{x_j^{1/\alpha_0}}\right\}^\alpha\right],
\end{align*}
in which we utilized the Laplace transform of the exponentially-tilted PS variables displayed in Eq.~\eqref{eqn:laplace_trans}.
\end{proof}

\subsection{Proof of Theorem \ref{thm:dependence_properties} of the main paper}\label{proof:thm}

\begin{proof}[Proof of Theorem \ref{thm:dependence_properties}]
By definitions of the tail dependence measures $\chi_{ij}$ and $\eta_{ij}$, 
\begin{equation}\label{eqn:chi_split}
    \begin{split}
        \chi_{ij}&=\lim_{u\rightarrow 1}\frac{\Pr\{X(\bs_i)>F_i^{-1}(u),X(\bs_j)>F_j^{-1}(u)\}}{1-u}\\
        &=\lim_{t\rightarrow \infty}t\Pr\{X(\bs_i)>q_i(t),X(\bs_j)>q_j(t)\}\\
        &=\lim_{t\rightarrow\infty}t\left[1-2\left(1-\frac{1}{t}\right)+\Pr\{X(\bs_i)\leq q_i(t),X(\bs_j)\leq q_j(t))\}\right]\\
    &=\lim_{t\rightarrow\infty} 2-t\left[1-F_{ij}\{q_i(t),q_j(t)\}\right],
    \end{split}
\end{equation}
and 
\begin{equation*}
    \begin{split}
        \Pr\{X(\bs_i)>q_i(t),X(\bs_j)>q_j(t)\} = \mathcal{L}(t)t^{-1/\eta_{ij}},\;t\rightarrow\infty.
    \end{split}
\end{equation*}
Further, 
\begin{equation}\label{eqn:eta_split}
    \begin{split}
        \lim_{t\rightarrow\infty}\frac{\log\Pr\{X(\bs_i)>q_i(t),X(\bs_j)>q_j(t)\}}{\log t}= -\frac{1}{\eta_{ij}},
    \end{split}
\end{equation}
provided that
\begin{equation*}
    \lim_{t\rightarrow\infty} \frac{\log\mathcal{L}(t)}{\log t} =0
\end{equation*}
for the slowly varying function $\mathcal{L}$. This can be easily shown using the Karamata Representation theorem \citep{resnick2008extreme}.

To facilitate the proofs of each case listed in Theorem \ref{thm:dependence_properties}, we first introduce some constants for simplicity:
\begin{equation}\label{eqn:constants}
    \begin{split}
        c_{ij}&=\alpha(\alpha-1)\sum_{k\in \mathcal{C}_i\cap \mathcal{C}_j}\gamma_k^{\alpha-2}\omega_{ki}^{1/\alpha}\omega_{kj}^{1/\alpha}, \text{ and }d_{ij}=\sum_{k\in \mathcal{D}}\left(\frac{\omega_{ki}^{1/\alpha}}{c'_i{}^{1/\alpha}}+\frac{\omega_{kj}^{1/\alpha}}{c'_j{}^{1/\alpha}}\right)^\alpha.
    \end{split}
\end{equation}
In addition, constants $c_j$, $c_j'$ and $d_j$ are defined in Eq.~\eqref{eqn:survival_fn}.

\begin{enumerate}[(a)] 
\item If $\mathcal{C}_i\cap \mathcal{D}= \emptyset$ and $\mathcal{C}_j\cap \mathcal{D}= \emptyset$, we know from Corollary \ref{cor:quantile_func} that 
\begin{equation*}
    \begin{split}
        q_i(t) &=\tau c_i^{\alpha_0}t^{\alpha_0}\{1+R_i(t)+o(t^{-1})\},\\
        q_j(t) &=\tau c_j^{\alpha_0}t^{\alpha_0}\{1+R_j(t)+o(t^{-1})\},
    \end{split}
\end{equation*}
in which $R_i(t) = \alpha_0({d_i}/{c_i^2}-1/2)t^{-1}$ and $R_j(t) = \alpha_0({d_j}/{c_j^2}-1/2)t^{-1}$. Using the joint distribution in Proposition~\ref{prop:joint_distr} in the main paper, we first deduce
\begin{small}
    \begin{align*}
    \log F_{ij}&\{q_i(t),q_j(t)\}=\sum_{k\in \bar{\mathcal{D}}}\gamma_k^\alpha -  \sumK\left[\gamma_k+\frac{\omega_{ki}^{1/\alpha}}{c_it\{1+R_i(t)+o(t^{-1})\}}+\frac{\omega_{kj}^{1/\alpha}}{c_jt\{1+R_j(t)+o(t^{-1})\}}\right]^\alpha\\
    &=\sum_{k\in \bar{\mathcal{D}}}\gamma_k^\alpha -\sum_{k\in \bar{\mathcal{D}}}\left[\gamma_k+\frac{\omega_{ki}^{1/\alpha}}{c_it}\{1-R_i(t)\}+\frac{\omega_{kj}^{1/\alpha}}{c_jt}\{1-R_j(t)\}+o\left(\frac{1}{t^2}\right)\right]^\alpha\\
     &=\sum_{k\in \bar{\mathcal{D}}}\gamma_k^\alpha -  \sum_{k\in \bar{\mathcal{D}}}\gamma_k^\alpha\left[1+\frac{\alpha\omega_{ki}^{1/\alpha}/\gamma_k}{c_it}\{1-R_i(t)\}+\frac{\alpha\omega_{kj}^{1/\alpha}/\gamma_k}{c_jt}\{1-R_j(t)\}+o\left(\frac{1}{t^2}\right)\right],
\end{align*}
\end{small}
in which the penultimate equality uses the negative binomial expansion and the last euqaltiy is derived from the Taylor expansion in Eq.~\eqref{eqn:taylor_1_plus_t}.  Recall the definitions of $c_i$ and $c_j$ in Proposition~\ref{prop:marg_distr} from the main paper, and we find
\begin{align*}
    \log F_{ij}\{q_i(t),q_j(t)\}&=-\frac{2}{t}+\frac{R_i(t)+R_j(t)}{t}-o\left(\frac{1}{t^2}\right)\;\text{as }t\rightarrow\infty.
\end{align*}
Then it follows from Eq.~\eqref{eqn:taylor_exp_t} that
\begin{align*}
    1- F_{ij}\{q_i(t),q_j(t)\}&=1-\exp\left\{-\frac{2}{t}+\frac{R_i(t)+R_j(t)}{t}-o\left(\frac{1}{t^2}\right)\right\} \\
    &= \frac{2}{t}-\frac{R_i(t)+R_j(t)}{t}+o\left(\frac{1}{t^2}\right).
\end{align*}
Plugging this result into \eqref{eqn:chi_split}, we have $\chi_{ij} = \lim_{t\rightarrow\infty}\{R_i(t)+R_j(t)+o(t^{-1})\}=0$.

In the meantime,
\begin{small}
    \begin{align*}
    \log\Pr\{X(\bs_i)>q_i(t),X(\bs_j)>q_j(t)\}&\sim\log\frac{R_i(t)+R_j(t)}{t}\\
    &=\log\alpha_0+\log\left(\frac{d_i}{c_i^2}+\frac{d_j}{c_j^2}-1\right)-2\log t
\end{align*}
\end{small}
as $t \rightarrow\infty$. By Eq.~\eqref{eqn:eta_split}, $\eta_{ij}=1/2$.

\item If $\mathcal{C}_i\cap \mathcal{D}= \emptyset$ and $\mathcal{C}_j\cap \mathcal{D}\neq \emptyset$, we deduce $c_i\neq 0$, $c_i'=0$, $c_j'\neq 0$ and $\mathcal{C}_i=\bar{\mathcal{D}}$. From Corollary \ref{cor:quantile_func}, 
\begin{equation}\label{eqn:asymp_approx_b}
    \begin{split}
         q_i(t) &\sim  \tau c_i{}^{\alpha_0}t^{\alpha_0}\{1+R_i(t)+o(t^{-1})\},\\
    q_j(t) &\sim  \tau c'_j{}^{\alpha_0/\alpha}t^{\alpha_0/\alpha}\{1+R^*_j(t)+O(t^{-1/\alpha})\},
    \end{split}
\end{equation}
as $t\rightarrow\infty$, where $R_i(t) = \alpha_0({d_i}/{c_i^2}-1/2)t^{-1}$ and $R^*_j(t) =\alpha_0 c_jt^{1-1/\alpha}/(\alpha c'_j{}^{1/\alpha})-\alpha_0 t^{-1}/(2\alpha)$. Again by the joint distribution in Proposition~\ref{prop:joint_distr} of the main paper,
\begin{small}
    \begin{equation*}
       \begin{split}
           \log& F_{ij}\{q_i(t),q_j(t)\}=\sum_{k\in \bar{\mathcal{D}}}\gamma_k^\alpha -  \sum_{k=1}^K\left\{\gamma_k+\frac{\tau^{1/\alpha_0}\omega_{ki}^{1/\alpha}}{q_i^{1/\alpha_0}(t)}+\frac{\tau^{1/\alpha_0}\omega_{kj}^{1/\alpha}}{q_j^{1/\alpha_0}(t)}\right\}^\alpha\\
    =&\sum_{k\in \bar{\mathcal{D}}}\gamma_k^\alpha -  \sum_{k\in \bar{\mathcal{D}}}\left\{\gamma_k+\frac{\tau^{1/\alpha_0}\omega_{ki}^{1/\alpha}}{q_i^{1/\alpha_0}(t)}+\frac{\tau^{1/\alpha_0}\omega_{kj}^{1/\alpha}}{q_j^{1/\alpha_0}(t)}\right\}^\alpha - \sum_{k\in \mathcal{D}} \frac{\tau^{\alpha/\alpha_0}\omega_{kj}}{q_j^{\alpha/\alpha_0}(t)}.
       \end{split}
   \end{equation*}
\end{small}
Here, we split the sum over $k\in {1,\ldots, K}$ to $k\in \mathcal{D}$ and $k\in \bar{\mathcal{D}}$. The third summation can be re-written as $\sum_{k\in \mathcal{D}} \frac{\tau^{\alpha/\alpha_0}\omega_{kj}}{q_j^{\alpha/\alpha_0}(t)}=c_j'\left\{\frac{q_j(t)}{\tau}\right\}^{-\alpha/\alpha_0}$. For the second summation, we apply Eq.~\eqref{eqn:taylor_1_plus_t} again to get
\begin{footnotesize}
    \begin{align}
        \sum_{k\in \bar{\mathcal{D}}}&\left\{\gamma_k+\frac{\tau^{1/\alpha_0}\omega_{ki}^{1/\alpha}}{q_i^{1/\alpha_0}(t)}+\frac{\tau^{1/\alpha_0}\omega_{kj}^{1/\alpha}}{q_j^{1/\alpha_0}(t)}\right\}^\alpha \nonumber\\
        =& \sum_{k\in \bar{\mathcal{D}}}\gamma_k^\alpha\left[1+\frac{\alpha\tau^{1/\alpha_0}\omega_{ki}^{1/\alpha}}{\gamma_k q_i^{1/\alpha_0}(t)}+\frac{\alpha\tau^{1/\alpha_0}\omega_{kj}^{1/\alpha}}{\gamma_k q_j^{1/\alpha_0}(t)}+ \frac{\alpha(\alpha-1)}{2\gamma_k^2}\left\{\frac{\tau^{1/\alpha_0}\omega_{ki}^{1/\alpha}}{q_i^{1/\alpha_0}(t)}+\frac{\tau^{1/\alpha_0}\omega_{kj}^{1/\alpha}}{q_j^{1/\alpha_0}(t)}\right\}^2+o(t^{-1-\frac{1}{\alpha}})\right]\nonumber\\
        =&\sum_{k\in \bar{\mathcal{D}}}\gamma_k^\alpha+c_i\left\{\frac{q_i(t)}{\tau}\right\}^{-\frac{1}{\alpha_0}}+c_j\left\{\frac{q_j(t)}{\tau}\right\}^{-\frac{1}{\alpha_0}}+d_i\left\{\frac{q_i(t)}{\tau}\right\}^{-\frac{2}{\alpha_0}}+d_j\left\{\frac{q_j(t)}{\tau}\right\}^{-\frac{2}{\alpha_0}} \nonumber\\
        &+c_{ij}\left\{\frac{q_i(t)q_j(t)}{\tau^2}\right\}^{-\frac{1}{\alpha_0}}+o(t^{-1-\frac{1}{\alpha}}),\label{eqn:case_c_2}
    \end{align}
\end{footnotesize}
in which the constants $c_i$, $c_j$, $d_i$, $d_j$ and $c_{ij}$ are defined previously in Eq.~\eqref{eqn:constants}, and the residual term $o\left(t^{-\frac{2}{\alpha}}\right)$ is derived using the asymptotic approximation of $q_i(t)$ and $q_j(t)$ in Eq.~\eqref{eqn:asymp_approx_b}. 

Combining this result with Eq.~\eqref{eqn:case_c_2} and feeding them into Eq.~\eqref{eqn:case_c_1}, we have
\begin{footnotesize}
    \begin{align*}
    1&- F_{ij}\{q_i(t),q_j(t)\}=\\
    &=1-\exp\left[c_i\left\{\frac{q_i(t)}{\tau}\right\}^{-\frac{1}{\alpha_0}}+c_j\left\{\frac{q_j(t)}{\tau}\right\}^{-\frac{1}{\alpha_0}}+c_j'\left\{\frac{q_j(t)}{\tau}\right\}^{-\frac{\alpha}{\alpha_0}}\right.\\
    &\left. + d_i\left\{\frac{q_i(t)}{\tau}\right\}^{-\frac{2}{\alpha_0}}+d_j\left\{\frac{q_j(t)}{\tau}\right\}^{-\frac{2}{\alpha_0}}+c_{ij}\left\{\frac{q_i(t)q_j(t)}{\tau^2}\right\}^{-\frac{1}{\alpha_0}}+o\left(t^{-\frac{2}{\alpha}}\right)\right]\\
    =&c_i\left\{\frac{q_i(t)}{\tau}\right\}^{-\frac{1}{\alpha_0}}+c_j\left\{\frac{q_j(t)}{\tau}\right\}^{-\frac{1}{\alpha_0}}+c_j'\left\{\frac{q_j(t)}{\tau}\right\}^{-\frac{\alpha}{\alpha_0}}+d_i\left\{\frac{q_i(t)}{\tau}\right\}^{-\frac{2}{\alpha_0}}+d_j\left\{\frac{q_j(t)}{\tau}\right\}^{-\frac{2}{\alpha_0}}\\
    &+c_{ij}\left\{\frac{q_i(t)q_j(t)}{\tau^2}\right\}^{-\frac{1}{\alpha_0}}-\frac{c_i^2}{2}\left\{\frac{q_i(t)}{\tau}\right\}^{-\frac{2}{\alpha_0}}-\frac{c_j^2}{2}\left\{\frac{q_j(t)}{\tau}\right\}^{-\frac{2}{\alpha_0}}-\frac{c_j'^2}{2}\left\{\frac{q_j(t)}{\tau}\right\}^{-\frac{2\alpha}{\alpha_0}}\\
    &-c_ic_j\left\{\frac{q_i(t)q_j(t)}{\tau^2}\right\}^{-\frac{1}{\alpha_0}}-c_ic_j'\frac{q_i^{-\frac{1}{\alpha_0}}(t)q^{-\frac{\alpha}{\alpha_0}}_j(t)}{\tau^{-\frac{1+\alpha}{\alpha_0}}}-c_jc_j'\frac{q_j^{-\frac{1+\alpha}{\alpha_0}}(t)}{\tau^{-\frac{1+\alpha}{\alpha_0}}}+o(t^{-1-\frac{1}{\alpha}}).
\end{align*}
\end{footnotesize}
Then we utilize the asymptotic approximation of the marginal distribution in Eq.~\eqref{eqn:survival_fn} to get
\begin{equation*}
    \begin{split}
        1- F_{ij}\{q_i(t),&q_j(t)\}=\bar{F}_i\{q_i(t)\}+\bar{F}_j\{q_j(t)\}+c_{ij}\left\{\frac{q_i(t)q_j(t)}{\tau^2}\right\}^{-\frac{1}{\alpha_0}}\\
        &-c_ic_j\left\{\frac{q_i(t)q_j(t)}{\tau^2}\right\}^{-\frac{1}{\alpha_0}}-c_ic_j'\frac{q_i^{-\frac{1}{\alpha_0}}(t)q^{-\frac{\alpha}{\alpha_0}}_j(t)}{\tau^{-\frac{1+\alpha}{\alpha_0}}}-c_jc_j'\frac{q_j^{-\frac{1+\alpha}{\alpha_0}}(t)}{\tau^{-\frac{1+\alpha}{\alpha_0}}}+o(t^{-1-\frac{1}{\alpha}}).
    \end{split}
\end{equation*}
By definition, $t^{-1}=\bar{F}_i\{q_i(t)\}=\bar{F}_j\{q_j(t)\}$. Therefore, $$\chi_{ij}=\lim_{t\rightarrow \infty}2-t\left[1-F_{ij}\{q_i(t),q_j(t)\}\right]=0.$$

\LZadd{If $\mathcal{C}_i\cap\mathcal{C}_j=\emptyset$, $c_{ij}=0$ and $c_j=0$. By Eq.~\eqref{eqn:eta_split} and \eqref{eqn:asymp_approx_b}, 
\begin{equation*}
    \begin{split}
            -\frac{1}{\eta_{ij}} = \lim_{t\rightarrow\infty}\frac{\log
    \left[c_ic_j'\frac{q_i^{-{1}/{\alpha_0}}(t)q^{-{\alpha}/{\alpha_0}}_j(t)}{\tau^{-{1+\alpha}/{\alpha_0}}}+o\left(t^{-1-\frac{1}{\alpha}}\right)\right]}{\log t}=-2,
    \end{split}
\end{equation*}
and $\eta_{ij}=1/2$.}

If $\mathcal{C}_i\cap\mathcal{C}_j\neq\emptyset$, then $\bar{\mathcal{D}}\cap\mathcal{C}_j\neq\emptyset$ and $c_j> 0$, $c_{ij}<0$. Thus, $2c_ic_j-c_{ij}>0$ and by Eq.~\eqref{eqn:asymp_approx_b},
\begin{equation*}
    c_ic_j\left\{\frac{q_i(t)q_j(t)}{\tau^2}\right\}^{-\frac{1}{\alpha_0}} +c_jc_j'\frac{q_j^{-\frac{1+\alpha}{\alpha_0}}(t)}{\tau^{-\frac{1+\alpha}{\alpha_0}}}-c_{ij}\left\{\frac{q_i(t)q_j(t)}{\tau^2}\right\}^{-\frac{1}{\alpha_0}}=\frac{2c_ic_j-c_{ij}}{c_ic_j'{}^{1/\alpha}}t^{-1-\frac{1}{\alpha}} + o(t^{-1-\frac{1}{\alpha}}).
\end{equation*}
Consequently, 
\begin{equation*}
    \begin{split}
            -\frac{1}{\eta_{ij}} = \lim_{t\rightarrow\infty}\frac{\log
    \left[\frac{2c_ic_j-c_{ij}}{c_ic_j'{}^{1/\alpha}}t^{-1-\frac{1}{\alpha}} + o(t^{-1-\frac{1}{\alpha}})\right]}{\log t}=-1-\frac{1}{\alpha},
    \end{split}
\end{equation*}
and $\eta_{ij}=\alpha/(\alpha+1)$.

\item When $\mathcal{C}_i\cap \mathcal{D}\neq \emptyset$ and $\mathcal{C}_j\cap \mathcal{D}\neq \emptyset$, we have $c'_i\neq 0$, $c'_j\neq 0$ and
\begin{equation}\label{eqn:asymp_approx_c}
    \begin{split}
        q_i(t) &\sim  \tau c'_i{}^{\alpha_0/\alpha}t^{\alpha_0/\alpha}\left\{1+R^*_i(t)+O(t^{-1/\alpha})\right\}\\
    q_j(t) &\sim  \tau c'_j{}^{\alpha_0/\alpha}t^{\alpha_0/\alpha}\left\{1+R^*_j(t)+O(t^{-1/\alpha})\right\}
    \end{split}
\end{equation}
as $t\rightarrow\infty$, in which $R^*_i(t)$ and $R^*_j(t)$ have the forms given in Corollary~\ref{cor:quantile_func}. Consequently, we obtain:
\begin{footnotesize}
    \begin{equation}\label{eqn:case_c_1}
        \begin{split}
            \log &F_{ij}\{q_i(t),q_j(t)\}=\sum_{k\in \bar{\mathcal{D}}}\gamma_k^\alpha -  \sum_{k=1}^K\left\{\gamma_k+\frac{\tau^{1/\alpha_0}\omega_{ki}^{1/\alpha}}{q_i^{1/\alpha_0}(t)}+\frac{\tau^{1/\alpha_0}\omega_{kj}^{1/\alpha}}{q_j^{1/\alpha_0}(t)}\right\}^\alpha\\
    &=\sum_{k\in \bar{\mathcal{D}}}\gamma_k^\alpha -  \sum_{k\in \bar{\mathcal{D}}}\left\{\gamma_k+\frac{\tau^{1/\alpha_0}\omega_{ki}^{1/\alpha}}{q_i^{1/\alpha_0}(t)}+\frac{\tau^{1/\alpha_0}\omega_{kj}^{1/\alpha}}{q_j^{1/\alpha_0}(t)}\right\}^\alpha-\sum_{k\in \mathcal{D}}\left\{\frac{\tau^{1/\alpha_0}\omega_{ki}^{1/\alpha}}{q_i^{1/\alpha_0}(t)}+\frac{\tau^{1/\alpha_0}\omega_{kj}^{1/\alpha}}{q_j^{1/\alpha_0}(t)}\right\}^\alpha.
        \end{split}
    \end{equation}
\end{footnotesize}
For the second summation, the approximation in Eq.~\eqref{eqn:case_c_2} still holds, except that the residual term becomes $o\left(t^{-\frac{2}{\alpha}}\right)$ due to the asymptotic approximations in Eq.~\eqref{eqn:asymp_approx_c}. \LZadd{In the following, we examine the third summation by conditioning on whether $\mathcal{C}_i \cap \mathcal{C}_j \cap \mathcal{D} = \emptyset$ or $\mathcal{C}_i \cap \mathcal{C}_j \cap \mathcal{D} \neq \emptyset$.}

\LZadd{If $\mathcal{C}_i \cap \mathcal{C}_j \cap \mathcal{D} = \emptyset$, then locations $i$ and $j$ are not covered by the same compact basis function with $\gamma_k = 0$ even though $\mathcal{C}_i \cap \mathcal{D} \neq \emptyset$ and $\mathcal{C}_j \cap \mathcal{D} \neq \emptyset$ (i.e., they are individually impacted by different heavy-tailed expPS variables). In this case:
\begin{footnotesize}
    \begin{equation*}
    \begin{split}
        \sum_{k\in \mathcal{D}}\left\{\frac{\tau^{1/\alpha_0}\omega_{ki}^{1/\alpha}}{q_i^{1/\alpha_0}(t)}+\frac{\tau^{1/\alpha_0}\omega_{kj}^{1/\alpha}}{q_j^{1/\alpha_0}(t)}\right\}^\alpha &= \sum_{k\in \mathcal{D}}\frac{\tau^{\alpha/\alpha_0}\omega_{ki}}{q_i^{\alpha/\alpha_0}(t)}+\sum_{k\in \mathcal{D}}\frac{\tau^{\alpha/\alpha_0}\omega_{kj}}{q_j^{\alpha/\alpha_0}(t)}= c_i'\left\{\frac{q_i(t)}{\tau}\right\}^{-\frac{\alpha}{\alpha_0}}+c_j'\left\{\frac{q_j(t)}{\tau}\right\}^{-\frac{\alpha}{\alpha_0}}.\\
    \end{split}
\end{equation*}
\end{footnotesize}
Combine this result with Eq.~\eqref{eqn:case_c_2} and feed them in Eq.~\eqref{eqn:case_c_1}, we have
\begin{footnotesize}
    \begin{align*}
    1&- F_{ij}\{q_i(t),q_j(t)\}=\\
    &=1-\exp\left[c_i\left\{\frac{q_i(t)}{\tau}\right\}^{-\frac{1}{\alpha_0}}+c_j\left\{\frac{q_j(t)}{\tau}\right\}^{-\frac{1}{\alpha_0}}+c_i'\left\{\frac{q_i(t)}{\tau}\right\}^{-\frac{\alpha}{\alpha_0}}+c_j'\left\{\frac{q_j(t)}{\tau}\right\}^{-\frac{\alpha}{\alpha_0}}\right.\\
    &\left. + d_i\left\{\frac{q_i(t)}{\tau}\right\}^{-\frac{2}{\alpha_0}}+d_j\left\{\frac{q_j(t)}{\tau}\right\}^{-\frac{2}{\alpha_0}}+c_{ij}\left\{\frac{q_i(t)q_j(t)}{\tau^2}\right\}^{-\frac{1}{\alpha_0}}+o\left(t^{-\frac{2}{\alpha}}\right)\right]\\
    =&c_i\left\{\frac{q_i(t)}{\tau}\right\}^{-\frac{1}{\alpha_0}}+c_j\left\{\frac{q_j(t)}{\tau}\right\}^{-\frac{1}{\alpha_0}}+c_i'\left\{\frac{q_i(t)}{\tau}\right\}^{-\frac{\alpha}{\alpha_0}}+c_j'\left\{\frac{q_j(t)}{\tau}\right\}^{-\frac{\alpha}{\alpha_0}}+d_i\left\{\frac{q_i(t)}{\tau}\right\}^{-\frac{2}{\alpha_0}}+d_j\left\{\frac{q_j(t)}{\tau}\right\}^{-\frac{2}{\alpha_0}}\\
    &+c_{ij}\left\{\frac{q_i(t)q_j(t)}{\tau^2}\right\}^{-\frac{1}{\alpha_0}}-\frac{c_i^2}{2}\left\{\frac{q_i(t)}{\tau}\right\}^{-\frac{2}{\alpha_0}}-\frac{c_j^2}{2}\left\{\frac{q_j(t)}{\tau}\right\}^{-\frac{2}{\alpha_0}}-\frac{c_i'^2}{2}\left\{\frac{q_i(t)}{\tau}\right\}^{-\frac{2\alpha}{\alpha_0}}-\frac{c_j'^2}{2}\left\{\frac{q_j(t)}{\tau}\right\}^{-\frac{2\alpha}{\alpha_0}}\\
    &-c_ic_j\left\{\frac{q_i(t)q_j(t)}{\tau^2}\right\}^{-\frac{1}{\alpha_0}}-c_ic_i'\frac{q_i^{-\frac{1+\alpha}{\alpha_0}}(t)}{\tau^{-\frac{1+\alpha}{\alpha_0}}}-c_ic_j'\frac{q_i^{-\frac{1}{\alpha_0}}(t)q^{-\frac{\alpha}{\alpha_0}}_j(t)}{\tau^{-\frac{1+\alpha}{\alpha_0}}}-c_i'c_j\frac{q_i^{-\frac{\alpha}{\alpha_0}}(t)q^{-\frac{1}{\alpha_0}}_j(t)}{\tau^{-\frac{1+\alpha}{\alpha_0}}}-c_jc_j'\frac{q_j^{-\frac{1+\alpha}{\alpha_0}}(t)}{\tau^{-\frac{1+\alpha}{\alpha_0}}}\\
    &-c_i'c_j'\left\{\frac{q_i(t)q_j(t)}{\tau^2}\right\}^{-\frac{\alpha}{\alpha_0}}+o\left(t^{-\frac{2}{\alpha}}\right).
\end{align*}
\end{footnotesize}
Then we utilize the asymptotic approximation of the marginal distribution in Eq.~\eqref{eqn:survival_fn} to get
\begin{equation*}
    \begin{split}
        1- F_{ij}\{q_i(t),&q_j(t)\}=\bar{F}_i\{q_i(t)\}+\bar{F}_j\{q_j(t)\}+c_{ij}\left\{\frac{q_i(t)q_j(t)}{\tau^2}\right\}^{-\frac{1}{\alpha_0}}\\
        &-c_ic_j\left\{\frac{q_i(t)q_j(t)}{\tau^2}\right\}^{-\frac{1}{\alpha_0}}-c_ic_j'\frac{q_i^{-\frac{1}{\alpha_0}}(t)q^{-\frac{\alpha}{\alpha_0}}_j(t)}{\tau^{-\frac{1+\alpha}{\alpha_0}}}-c_i'c_j\frac{q_i^{-\frac{\alpha}{\alpha_0}}(t)q^{-\frac{1}{\alpha_0}}_j(t)}{\tau^{-\frac{1+\alpha}{\alpha_0}}}\\
    &-c_i'c_j'\left\{\frac{q_i(t)q_j(t)}{\tau^2}\right\}^{-\frac{\alpha}{\alpha_0}}+o\left(t^{-\frac{2}{\alpha}}\right).
    \end{split}
\end{equation*}
By definition, $t^{-1}=\bar{F}_i\{q_i(t)\}=\bar{F}_j\{q_j(t)\}$. Therefore, it straightforwardly follows that $\chi_{ij}=\lim_{t\rightarrow \infty}2-t\left[1-F_{ij}\{q_i(t),q_j(t)\}\right]=0$. By Eq.~\eqref{eqn:eta_split} and \eqref{eqn:asymp_approx_c}, 
\begin{equation*}
    \begin{split}
            -\frac{1}{\eta_{ij}} = \lim_{t\rightarrow\infty}\frac{\log
    \left[c_i'c_j'\left\{\frac{q_i(t)q_j(t)}{\tau^2}\right\}^{-\frac{2\alpha}{\alpha_0}}+o\left(t^{-2}\right)\right]}{\log t}=-2,
    \end{split}
\end{equation*}
and $\eta_{ij}=1/2$.}

If $\mathcal{C}_i\cap \mathcal{C}_j\cap\mathcal{D}\neq \emptyset$,
\begin{footnotesize}
    \begin{align*}
   \sum_{k\in \mathcal{D}}&\left\{\frac{\tau^{1/\alpha_0}\omega_{ki}^{1/\alpha}}{q_i^{1/\alpha_0}(t)}+\frac{\tau^{1/\alpha_0}\omega_{kj}^{1/\alpha}}{q_j^{1/\alpha_0}(t)}\right\}^\alpha\\
   =&  \sum_{k\in \mathcal{D}}\left[\frac{\omega_{ki}^{1/\alpha}t^{-1/\alpha}}{c'_i{}^{1/\alpha}\left\{1+R^*_i(t)+O(t^{-1/\alpha})\right\}^{1/\alpha_0}}+\frac{\omega_{kj}^{1/\alpha}t^{-1/\alpha}}{c'_j{}^{1/\alpha}\left\{1+R^*_j(t)+O(t^{-1/\alpha})\right\}^{1/\alpha_0}}\right]^\alpha\\
    =&\sum_{k\in \mathcal{D}}\left(\frac{\omega_{ki}^{1/\alpha}}{c'_i{}^{1/\alpha}}+\frac{\omega_{kj}^{1/\alpha}}{c'_j{}^{1/\alpha}}\right)^\alpha t^{-1} + O(t^{1-\frac{2}{\alpha}})=d_{ij} t^{-1} + O(t^{1-\frac{2}{\alpha}}),\;\text{as }t\rightarrow\infty,
\end{align*}
\end{footnotesize}
in which we use the asymptotic approximation in Eq.~\eqref{eqn:asymp_approx_c} again and by the subadditivity of power function with $\alpha\in(0,1)$, 
\begin{equation*}
    d_{ij}=\sum_{k\in \mathcal{D}}\left(\frac{\omega_{ki}^{1/\alpha}}{c'_i{}^{1/\alpha}}+\frac{\omega_{kj}^{1/\alpha}}{c'_j{}^{1/\alpha}}\right)^\alpha < \sum_{k\in \mathcal{D}}\frac{\omega_{ki}}{c'_i}+\sum_{k\in \mathcal{D}}\frac{\omega_{kj}}{c'_i} = 2.
\end{equation*}
Here the inequality is strict because $\mathcal{C}_i\cap \mathcal{C}_j\cap\mathcal{D}\neq \emptyset$. Meanwhile $d_{ij}>\sum_{k\in \mathcal{D}}{\omega_{ki}}/{c'_i}=1$. On the other hand, we note that  in Eq.~\eqref{eqn:case_c_2},
\begin{equation*}
    c_i\left\{\frac{q_i(t)}{\tau}\right\}^{-\frac{1}{\alpha_0}}+c_j\left\{\frac{q_j(t)}{\tau}\right\}^{-\frac{1}{\alpha_0}}=\left(\frac{c_i}{c'_i{}^{1/\alpha}}+\frac{c_j}{c'_j{}^{1/\alpha}}\right)t^{-\frac{1}{\alpha}}+O(t^{1-\frac{2}{\alpha}}),
\end{equation*}
which results in
\begin{align*}
    1- F_{ij}\{q_i(t),q_j(t)\}=&1-\exp\left\{-d_{ij} t^{-1}-\left(\frac{c_i}{c'_i{}^{1/\alpha}}+\frac{c_j}{c'_j{}^{1/\alpha}}\right)t^{-\frac{1}{\alpha}} - O(t^{1-\frac{2}{\alpha}})\right\}\\
    =&d_{ij} t^{-1}+\left(\frac{c_i}{c'_i{}^{1/\alpha}}+\frac{c_j}{c'_j{}^{1/\alpha}}\right)t^{-\frac{1}{\alpha}}
    + O(t^{1-\frac{2}{\alpha}}),
\end{align*}
and
\begin{align*}
    t\Pr\{X(\bs_i)>q_i(t),X(\bs_j)>q_j(t)\}=2-d_{ij}-\left(\frac{c_i}{c'_i{}^{1/\alpha}}+\frac{c_j}{c'_j{}^{1/\alpha}}\right)t^{1-\frac{1}{\alpha}}-O(t^{2-\frac{2}{\alpha}}),
\end{align*}
as $t\rightarrow\infty$. Since $d_{ij}\in (1,2]$, we know from \eqref{eqn:chi_split} that $\chi_{ij} = 2-d_{ij}\in (0,1)$ and
\begin{align*}
    \chi_{ij}(u)-\chi_{ij}  = \left(\frac{c_i}{c'_i{}^{1/\alpha}}+\frac{c_j}{c'_j{}^{1/\alpha}}\right)(1-u)^{\frac{1}{\alpha}-1}+O\left\{(1-u)^{\frac{2}{\alpha}-2}\right\}.
\end{align*}
\end{enumerate}
\end{proof}

\begin{remark}\label{remark:exponent_function}
The exponent function, defined by
\begin{equation*}
    V(x_1, \ldots, x_{n_s})=\lim_{t\rightarrow\infty}t(1-F[F_1^{-1}\{1-(tx_1)^{-1}\},\ldots, F_{n_s}^{-1}\{1-(tx_{n_s})^{-1}\}]),
\end{equation*}
is a limiting measure that occurs in the limiting distribution for normalized maxima \citep{Huser2019}. It is used to describe the multivariate extremal dependence of a spatial process, and the $n_s$-dimensional extremal coefficient $V(1, \ldots, 1)$ is of particular interest. This extremal coefficient has a range of $[1,n_s]$, with the lower and upper ends indicating, respectively, perfect dependence and independence. As a polarized case, if $\gamma_k>0$, for all $k=1,\ldots, K$, then $\mathcal{C}_j\cap \mathcal{D}=\emptyset$ for all $j$'s, and thus we have  
\begin{equation*}
    \gamma_k^\alpha - \left\{\gamma_k+\tau^{\frac{1}{\alpha_0}}\sumN\frac{\omega_{kj}^{1/\alpha}}{q_j^{1/\alpha_0}(t)}\right\}^\alpha \sim \alpha \tau^{\frac{1}{\alpha_0}}\gamma_k^{\alpha-1}\sumN\frac{\omega_{kj}^{1/\alpha}}{q_j^{1/\alpha_0}(t)}, \;t\rightarrow \infty.
\end{equation*}
Here, we can approximate $q_j(t)$ using the results from Corollary \ref{cor:quantile_func}. From Proposition~\ref{prop:marg_distr} of the main paper, we can deduce that $V(1,\ldots, 1) = n_s$, which corresponds to joint extremal independence. By contrast, if all $\gamma_k=0$ and one knot covers the entire spatial domain, we have $V(1,\ldots, 1)\in [1, n_s)$, which corresponds to joint extremal dependence.
\end{remark}

\section{Validation framework details}\label{sec:diagnotics}
\subsection{Full range evaluation}\label{sec:predict}
To examine the quality of the emulation from the XVAE, we will predict at $n_h$ locations $\{\bh_i: i=1,\ldots, n_h\}$ held out from the analyses. To perform these predictions, we calculate the basis function values at these locations, with which we can mix the encoded variables from Eq.~\eqref{eqn:encoder_form} in the main paper to get predicted values. For each time $t$ and holdout location $\bh_i$, denote the true observation of $X_t(\bh_i)$ by $x_{it}$ and the emulated prediction by $x^*_{it}$. Then the mean squared prediction error (MSPE) for time $t$ is
\begin{equation*}
    \mathrm{MSPE}_t = \frac{1}{n_h}\sum_{i=1}^{n_h} (x_{it}-x^*_{it})^2,
\end{equation*}
where $t=1,\ldots, n_t$. All MSPEs from different time replicates can be summarized in a boxplot; see Section \ref{sec:simulation} in the main paper for example. Similarly, we can calculate the continuously ranked probability
score \citep[CRPS;][]{matheson1976scoring, gneiting2007strictly} across time for each location, i.e.,
\begin{equation*}
    \mathrm{CRPS}_i=\frac{1}{n_t}\sum_{t=1}^{n_t}\int_{-\infty}^\infty (F_{i}(z)-\mathbbm{1}(x^*_{it}\leq z))^2 \mathrm {d}z,
\end{equation*}
where $F_{i}$ is the marginal distribution estimated using parameters at the holdout location $\bh_i$, $i=1,\ldots,n_h$, and again $x^*_{it}$ is the emulated value. Smaller CRPS indicates that the distribution $F_i$ is concentrated around $x^*_{it}$, and thus can be used to measure how well the distribution fits all emulated values. Section \ref{sec:simulation} in the main paper also shows how we present the CRPS values from all holdout locations for each emulation. In addition, we will examine the quantile-quantile (QQ)-plots obtained by pooling the spatial data into the same plot to check if the spatial input and the emulation have similar ranges and quantiles. 

\subsection{Empirical tail dependence measures}\label{subsec:dependence_res}
To assess the tail dependence structure of the emulated fields, we will estimate $\chi_{ij}(u)$ defined in Eq.~\eqref{eqn:chi} empirically in two ways. First, to examine the overall dependence strength, we treat $\{X(\bs)\}$ as if it had a stationary and isotropic dependence structure so that $\chi_{ij}(u)\equiv\chi_{h}(u)$, with $h=||\bs_i-\bs_j||$ being the distance between locations. Then for a fixed $h$, we find all pairs of locations with similar distances (within a small tolerance, say $\epsilon=0.001$), and compute the empirical conditional probabilities $\widehat{\chi}_h(u)$ at a grid of $u$ values. Confidence envelopes can be calculated by regarding the outcome (i.e., simultaneously exceed $u$ or not) of each pair as a Bernoulli variable and computing pointwise binomial confidence intervals, assuming that all pairs of points are independent from each other. Examples in Section~\ref{sec:simulation} of the main paper demonstrate how this empirical measure can be used to compare the extremal dependence structures between the spatial data input and realizations from the emulator. While this metric does not completely characterize the non-stationarity in the process, it is still well-defined as a summary statistic and carries important information about the average decay of dependence with distance irrespective of the direction. 

Second, to avoid the stationary assumption, we can choose a reference point denoted by $\bs_0$ and estimate the pairwise $\chi_{0j}(u)$ empirically between $\bs_0$ and all observed locations $\bs_j$ in the spatial domain $\mathcal{S}$. These pairwise estimates can then be presented using a raster plot (if gridded) or a heat plot. Section~\ref{sec:data_analysis} in the main paper shows examples of the empirical $\chi_{0j}(u)$, $u=0.85$, estimated from the real and emulated datasets, where $\bs_0$ is the center of $\mathcal{S}$.

\section{Areal radius of exceedance}\label{appendix:ARE_SEC}
\subsection{Monte Carlo estimates of \texorpdfstring{$\mathrm{ARE}_\psi(u)$}{Lg}}\label{appendix:ARE_MC}
\begin{proof}[Proof of Theorem~\ref{thm:consistency} of the main paper]
It suffices to prove that 
    \begin{equation}\label{eqn:as}
        \lim_{n_r\rightarrow\infty}\frac{\sum_{r=1}^{n_r}\mathbbm{1}(U_{ir}>u, U_{0r}>u)}{\sum_{r=1}^{n_r}\mathbbm{1}(U_{0r}>u)}=\chi_{\bs_0, \bg_i}(u), \qquad \text{a.s.}
    \end{equation}
    for all $i=1,\ldots, n_g$.
    
    First, since $U_{0r'} = \hat{F}_{0}(X_{0r'})$, it is clear that
    \begin{equation*}
        n_rU_{0r'} =\sum_{r=1}^{n_r}\mathbbm{1}\{X_{0r}\leq X_{0r'}\}
    \end{equation*}
    is the rank of $X_{0r'}$ in $\bX_0$, $r'=1,\ldots, n_r$. Thus,
    \begin{equation}\label{eqn:denom}
        \frac{1}{n_r}\sum_{r=1}^{n_r}\mathbbm{1}(U_{0r}>u)=\frac{\lfloor n_r (1-u)\rfloor }{n_r}\rightarrow 1-u, \text{ as }n_r\rightarrow\infty,
    \end{equation}
    in which $\lfloor \cdot \rfloor$ is the floor function.

    Second, denote the rank of $X_{ir'}$ in $\bX_i$ by $R_{ir'}$, $r'=1,\ldots, n_r$, $i=1,\ldots, n_g$. Then we know $R_{ir'}=n_rU_{ir'}$ and
    \begin{equation*}
        S_{i0}:=\frac{1}{n_r}\sum_{r=1}^{n_r}\mathbbm{1}(U_{ir}>u, U_{0r}>u)=\frac{1}{n_r}\sum_{r=1}^{n_r}\mathbbm{1}\left\{\frac{R_{ir}}{n_r}>u\right\}\mathbbm{1}\left\{\frac{R_{0r}}{n_r}>u\right\},
    \end{equation*}
    This is thus a bivariate linear rank statistics of $\bX_i$ and $\bX_0$, for which the regression constants as defined in \citet{sen1967theory} all have a value of $1$ and the scores have a product structure with each term being generated by $\phi(x)=\mathbbm{1}\{x>u\}$, $x\in (0,1)$. \citet{sen1967theory} and \citet{ruymgaart1974asymptotic} established the asymptotic normality of the multivariate linear rank statistics under weak restrictions that asymptotically no individual regression constant is much larger than the other constants and that $\phi$ is square integrable on $(0,1)^2$; that is,
    \begin{equation*}
        0<\int_{(0,1)^2} \{\phi(u_1, u_2)-\bar{\phi}\}^2\mathrm{d}u_1\mathrm{d}u_2 <\infty \text{ with }\bar{\phi}=\int_0^1\phi(u)\mathrm{d}u,
    \end{equation*}
    in which $\phi(u_1, u_2)=\phi(u_1)\phi(u_2)$. Since our regression constants are all $1$'s, the restriction on the regression constants is easily satisfied. Also, for $\phi(u_1, u_2)=\mathbbm{1}\{u_1>u, u_2>u\}$, $\int_0^1\int_0^1 \{\phi(u_1, u_2)-\bar{\phi}\}^2\mathrm{d}u_1\mathrm{d}u_2=\bar{\phi}-\bar{\phi}^2$ with $\bar{\phi}=(1-u)^2$. Therefore,
    \begin{equation}\label{eqn:clt}
        n^{1/2}\{S_{i0}-\mu_{i0}\}\rightarrow_d N(0, \sigma_{i0}^2)
    \end{equation}
    as $n_r\rightarrow \infty$, in which $\mu_{i0}$ and $\sigma_{i0}^2$ can be derived using Eq. (1.3) and (3.5) in \citet{ruymgaart1974asymptotic} as
    \begin{small}
        \begin{equation}\label{eqn:mean_and_var}
        \begin{split}
            \mu_{i0} &= \int\int \phi(F_i(x))\phi(F_0(y)) {\rm d}F_{i0}(x,y)=\Pr\{F_i(X_i)>u, F_0(X_0)>u\},\\
            \sigma_{i0}^2&=\mathrm{Var}\big(\mathbbm{1}\{F_i(X_i)>u, F_0(X_0)>u\}+[\mathbbm{1}\{F_i(X_i)\leq u\}-u]\Pr\{F_0(X_0)>u\mid F_i(X_i)=u\}\\
            &\qquad\qquad\qquad\qquad +[\mathbbm{1}\{F_0(X_0)\leq u\}-u]\Pr\{F_i(X_i)>u\mid F_0(X_0)=u\}\big).
        \end{split}
    \end{equation}
    \end{small}
    Since $\mu_{i0}/(1-u)=\chi_{0i}(u)$, we know from Expressions~\eqref{eqn:denom} and \eqref{eqn:clt} that as $n_r\rightarrow\infty$,
    \begin{equation}\label{eqn:CLT2}
        n^{\frac{1}{2}}\left\{\frac{\sum_{r=1}^{n_r}\mathbbm{1}(U_{ir}>u, U_{0r}>u)}{\sum_{r=1}^{n_r}\mathbbm{1}(U_{0r}>u)}-\chi_{\bs_0, \bg_i}(u)\right\}\rightarrow_d N\left\{0, \frac{\sigma_{i0}^2}{(1-u)^2}\right\},
    \end{equation}
    which ensures Expression~\eqref{eqn:as}.
\end{proof}
\begin{remark}
     The asymptotic normality of $n^{1/2}\{\widehat{\mathrm{ARE}}_\psi(u)- \mathrm{ARE}_\psi(u)\}$ is also ensured by Expression~\eqref{eqn:CLT2}. However, the exact expression of its asymptotic variance requires a much more careful examination of the correlations among the ranks of $\bX_i$, $i=0,1,\ldots, n_g$; that is, we need to device a multivariate linear rank statistics of $\bX_i$, $i=0,1,\ldots, n_g$; see \citet{ruymgaart1978asymptotic}.
\end{remark}

\subsection{Convergence of \texorpdfstring{$\mathrm{ARE}_\psi(u)$}{Lg}}\label{appendix:ARE}
\begin{proof}[Proof of Theorem~\ref{prop:ARE_psi} of the main paper]
By the definition of the tail dependence measure in Eq.~\eqref{eqn:chi} of the main paper,
\begin{equation*}
   \lim_{u\rightarrow 1} \sum_{i=1}^{n_g} \chi_{0i}(u)= \sum_{i=1}^{n_g} \chi_{0i}.
\end{equation*}
It is clear that the right-hand side is the Riemann sum of $\chi_{\bs_0, \bs}$ as a function of $\bs$ with respect to the grid. Since $\chi_{\bs_0, \bs}$ is a continuous function of $\bs$ (i.e., Riemann-integrable), we have
\begin{equation*}
    \lim_{\psi\rightarrow 0}\psi^2\sum_{i=1}^{n_g}\chi_{0i}= \int_{\mathcal{S}} \chi_{\bs_0, \bs}\mathrm{d}\bs.
\end{equation*}

Therefore, we have 
\begin{equation*}
    \lim_{\psi\rightarrow 0, u\rightarrow 1}\psi\left(\sum_{i=1}^{n_g} \chi_{0i}(u)\right)^{1/2}=\left\{\int_{\mathcal{S}} \chi_{\bs_0, \bs} \mathrm{d}\bs\right\}^{1/2}.
\end{equation*}
\end{proof}

\begin{remark}
In the spatial extremes literature, many models that have a spatially-invariant set of dependence parameter $\bphi_d$ and they satisfy
\begin{equation*}
    \chi_{\bs_0, \bs}(u)-\chi_{\bs_0, \bs}=c(\bs_0, \bs, \bphi_d)(1-u)^{d(\bphi_d)}\{1+o(1)\},
\end{equation*}
where $c(\bs_0, \bs, \bphi_d)$ is multiplicative constant defined by $\bs$, $\bs_0$ and $\bphi_d$. Also, the rate of decay $d(\bphi_d)$ is independent of $\bs$ and $\bs_0$. Such examples include the models proposed by \citet{huser2017bridging}, \citet{Huser2019} and \citet{bopp2021hierarchical1}. In this case,
\begin{equation*}
    \pi\widehat{\mathrm{ARE}}^2_\psi(u)-\psi^2\sum_{i=1}^{n_g}\chi_{\bs_0, \bg_i} \approx \left\{\psi^2\sum_{i=1}^{n_g}c(\bs_0, \bg_i, \bphi_d)\right\}(1-u)^{d(\bphi_d)}\{1+o(1)\}.
\end{equation*}
That is, $\widehat{\mathrm{ARE}}_\psi(u)$ has similar decaying behaviors as $\chi_{\bs_0, \bs}(u)$, which was observed empirically in Figure 3(b) and 4(b) in \citet{zhang2022accounting}.
\end{remark}

\begin{remark}
We note that \citet{cotsakis2022perimeter} proposed a similar metric which measures the length of the perimeter of excursion sets of anisotropic random fields on $\mathbb{R}^2$ under some smoothness assumptions. This estimator acts on the empirically accessible binary digital images of the excursion regions and computes the length of a piecewise linear approximation of the excursion boundary. In their work, the main focus is to prove strong consistency of the perimeter estimator as the image pixel size tends to zero. In comparison, we show that our estimator of $\mathrm{ARE}_\psi(u)$ is strongly consistent as the number of replicates drawn from the process $\{X(\bs)\}$ approaches infinity. Furthermore, the length scale $\mathrm{ARE}_\psi(u)$ is, in our view, more interpretable than the perimeter of excursion sets. Also, $\mathrm{ARE}_\psi(u)$ is closely tied to the bivariate $\chi$ measure, which further bridges spatial extremes to applications in other fields. 
\end{remark}

\section{XVAE details}\label{sec:extVAE_details}
\subsection{General framework}
In this section, we will illustrate the details of Eqs.~\eqref{eqn:encoder_form} and \eqref{eqn:decoder_form} in the main paper.  Recall the encoder in the XVAE encodes the information in $\boldsymbol{x}_t$, $t=1,\ldots,n_t$, using a three-layer perceptron neural network. The three-layer perceptron neural network has the form of:
\begin{equation}\label{eqn:encoder_weights}
    \begin{split}
         \boldsymbol{h}_{1,t} &= \mathrm{relu} (\boldsymbol{W}_1\boldsymbol{x}_t +\boldsymbol{b}_1),\\
    \boldsymbol{h}_{2,t} &= \mathrm{relu} (\boldsymbol{W}_2\boldsymbol{h}_{1,t} +\boldsymbol{b}_2),\\
    \log\bzeta^2_t &= \boldsymbol{W}_3\boldsymbol{h}_{2,t} +\boldsymbol{b}_3,\\
    \boldsymbol{\mu}_t &= \mathrm{relu}(\boldsymbol{W}_4\boldsymbol{h}_{2,t} +\boldsymbol{b}_4).
    \end{split}
\end{equation}
The weights $\{\boldsymbol{W}_1,\ldots,\boldsymbol{W}_4\}$ and biases $\{\boldsymbol{b}_1,\ldots,\boldsymbol{b}_4\}$ combined are denoted by $\bphi_e$ and are shared across time replicates. Here, $\boldsymbol{W}_1$ is a $K \times n_s$ weight matrix and $\boldsymbol{W}_2,\ldots,\boldsymbol{W}_4$ are all $K \times K$ matrices, and $\boldsymbol{b}_1,\ldots,\boldsymbol{b}_4$ are all $K \times 1$ vectors. Then we use a Gaussian encoder $\bz_t\sim N\{\boldsymbol{\mu}_t, \mathrm{diag}(\bzeta^2_t)\}$ and we have
\begin{align}\label{eqn:q_phi_e_expression}
    q_{\bphi_e}(\boldsymbol{z}_t\mid \boldsymbol{x}_t) =\frac{1}{(2\pi)^{n/2}\prodK \zeta_{kt}}\exp\left\{-\sumK\frac{(z_{kt}-\mu_{kt})^2}{2\zeta^2_{kt}}\right\}.
\end{align}
\LZadd{Here, we opted to not use a heavy-tailed distribution for the variational distribution $q_{\bphi_e}(\bz_t\mid \bx_t)$ for a few reasons: (1) the variational distribution serves to ``approximate" the true posterior distribution $p_{\btheta}(\bz_t\mid\bx_t)$. Many variational Bayesian methods \citep[see Eq. (4) of][for example]{maceda2024variational} use a heteroskedastic Gaussian model $q_{\bphi_e}(\bz_t\mid \bx_t)$ to approximate $p_{\btheta}(\bz\mid \bx)$. This choice is theoretically supported by an adaptation of Bernstein--von Mises Theorem to the spatial context under a form of spatial mixing condition, which basically requires that observations become ``effectively independent" as the distance between them grows---a condition met by our compactly supported Wendland basis functions. Using the terminology outlined in \citet{bradley2005basic}, we can verify that our max-id model satisfies the more stringent $\phi$-mixing conditions. (2) We experimented with Pareto-tailed variational distributions and they underperformed compared to Gaussian distributions. Intuitively, the mean vector $\bmu_t$ in $q_{\bphi_e}(\bz_t\mid \bx_t)$ anchors the encoding's center, while the standard deviation $\bzeta_t$ determines the range of variation around this center. Since the prior distribution $p_{\btheta}(\bz)$ is already heavy-tailed, allowing the variational distribution to diverge too widely from the mean proved counterproductive. (3) The variational distribution is regularized by the evidence lower bound (ELBO), in which we try to minimize the KL distance between $q_{\bphi_e}(\bz_t\mid \bx_t)$ and $p_{\btheta}(\bz_t\mid\bx_t)$. As long as the ELBO converges, we believe using the heteroskedastic Gaussian model as the variational distribution is sufficient, as evidenced by the extensive simulation results presented in the paper.}

For the decoder, we also use a three-layer perceptron neural network:
\begin{equation}\label{eqn:decoder}
    \begin{split}
    \boldsymbol{l}_{1,t} &= \mathrm{relu} (\boldsymbol{W}_5\boldsymbol{z}_t +\boldsymbol{b}_5),\\
    \boldsymbol{l}_{2,t} &= \mathrm{relu} (\boldsymbol{W}_6 \boldsymbol{l}_{1,t} +\boldsymbol{b}_6),\\
    (\alpha_t,\bgamma_t^\top)^\top &=\mathrm{relu}(\boldsymbol{W}_7\boldsymbol{l}_{2,t} +\boldsymbol{b}_7),\\
    \by_t&=(\bW^{1/\alpha_t}\bz_t)^{\alpha_0},
    \end{split}
\end{equation}
in which $\bW=(\bw_1, \cdots, \bw_{n_s})^\top$ is a $n_s\times K$ matrix with its $j$th row being $\bw_j^\top=(\omega_{1j},\ldots, \omega_{Kj})$. The weights $\{\boldsymbol{W}_5,\ldots,\boldsymbol{W}_7\}$ and biases $\{\boldsymbol{b}_5,\ldots,\boldsymbol{b}_7\}$ combined are denoted by $\bphi_d$, in which $\boldsymbol{W}_5$ and $\boldsymbol{W}_6$ are both $K\times K$ matrices while $\boldsymbol{W}_7$ is a $(K+1)\times K$ matrix, and $\boldsymbol{b}_5$ and $\boldsymbol{b}_6$ are $K\times 1$ vectors while $\boldsymbol{b}_7$ is a $(K+1)\times 1$ vector.

\subsection{ELBO empirical estimates}\label{sec:elbo_MC}
Since $p_{\bphi_d}(\bz\mid \bx)$ is unknown, we rewrite the marginal likelihood $p_{\bphi_d}(\bx)$ as follows
\begin{align*}
     \log p_{\bphi_d}(\bx) 
    = \mathbb{E}_{\bZ\sim q_{\bphi_e}(\bz\;\mid\;\bx)}\left\{\log\frac{p_{\bphi_d}(\bx,\bZ)}{q_{\bphi_e}(\bZ\mid\bx)}\right\}+D_{KL}\left\{q_{\bphi_e}(\bz\mid\bx)\;||\;p_{\bphi_d}(\bz\mid \bx)\right\}.
\end{align*}
Therefore, the ELBO can be approximated by Monte Carlo as
\begin{equation}\label{eqn:ELBO_approx}
    \mathcal{L}_{\bphi_e,\bphi_d}(\boldsymbol{x})\approx \frac{1}{L}\sum_{l=1}^L\log\frac{p_{\bphi_d}(\bx,\bZ^{l})}{q_{\bphi_e}(\bZ^{l}\mid\bx)},
\end{equation}
where $\bZ^1,\ldots, \bZ^L$ are independent draws from $q_{\bphi_e}(\cdot\mid\bx)$. 
If there are replicates of the process, $\bx_1,\ldots,\bx_{n_t}$,  then $\sum_{t=1}^{n_t}\mathcal{L}_{\bphi_e,\bphi_d}(\boldsymbol{x}_t)$ 
is considered.

\subsection{Reparameterization trick}\label{sec:reparam_trick}
Recall that the ELBO is defined as
\begin{equation*}
    \mathcal{L}_{\bphi_e,\bphi_d}(\boldsymbol{x}_t)=\mathbb{E}_{q_{\bphi_e}(\bz_t\mid\bx_t)}\left\{\log\frac{p_{\bphi_d}(\bx_t,\bZ_t)}{q_{\bphi_e}(\bZ_t\mid\bx_t)}\right\},
\end{equation*}
which can be approximated using Monte Carlo as shown in Eq.~\eqref{eqn:ELBO_approx}. However, it is not straightforward to approximate the partial derivative of the ELBO with respect to $\bphi_e$ (denoted by $\nabla_{\bphi_e}\mathcal{L}_{\bphi_e,\bphi_d}$), which is needed in the stochastic gradient descent algorithm. Since the expectation in ELBO is taken under the distribution $q_{\bphi_e}(\bz_t\mid\bx_t)$.
\begin{equation*}
    \nabla_{\bphi_e}\mathcal{L}_{\bphi_e,\bphi_d}(\boldsymbol{x}_t)\neq \mathbb{E}_{q_{\bphi_e}(\bZ_t\mid\bx_t)}\left\{\nabla_{\bphi_e}\log\frac{p_{\bphi_d}(\bx_t,\bZ_t)}{q_{\bphi_e}(\bZ_t\mid\bx_t)}\right\},
\end{equation*}

To simplify the gradient of the ELBO with respect to $\bphi_e$, we express $\bZ_t$ in terms of a random vector $\boldsymbol{\eta}_t$ that is independent of $\bx_t$ and $\bphi_e$:
\begin{align*}
     \bZ_t&=\bmu_t + \bzeta_t \odot \boldsymbol{\eta}_t,
\end{align*}
in which $\boldsymbol{\eta}_t=(\eta_{1t},\eta_{2t},\cdots, \eta_{Kt})^\top$ and $\eta_{kt}\stackrel{\text{i.i.d.}}{\sim} N(0,1)$. As a consequence, the Jacobian of the transformation from $\bZ_t$ to $\boldsymbol{\eta}_t$ is
\begin{equation*}
    J(\boldsymbol{\eta}_t)=\frac{\partial \bz_t}{\partial \boldsymbol{\eta}_t}=\mathrm{diag}(\bzeta_t),
\end{equation*}
and we can apply a change-of-variable formula to the multiple integral in the ELBO:
\begin{small}
    \begin{align*}
    \mathcal{L}_{\bphi_e,\bphi_d}(\boldsymbol{x}_t) &= \int \log\frac{p_{\bphi_d}(\bx_t,\bz_t)}{q_{\bphi_e}(\bz_t\mid\bx_t)}q_{\bphi_e}(\bz_t\mid\bx_t)\mathrm{d}\bz_t\\
    &=\int \log\frac{p_{\bphi_d}(\bx_t,\bmu_t + \bzeta_t \odot \boldsymbol{\eta}_t)}{q_{\bphi_e}(\bmu_t + \bzeta_t \odot \boldsymbol{\eta}_t\mid\bx_t)}q_{\bphi_e}(\bmu_t + \bzeta_t \odot \boldsymbol{\eta}_t\mid\bx_t)\left|\mathrm{det}\{J(\boldsymbol{\eta}_t)\}\right|\mathrm{d}\boldsymbol{\eta}_t\\
    &=\int \log\frac{p_{\bphi_d}(\bx_t,\bmu_t + \bzeta_t \odot \boldsymbol{\eta}_t)}{q_{\bphi_e}(\bmu_t + \bzeta_t \odot \boldsymbol{\eta}_t\mid\bx_t)}\prodK\frac{\exp(-\eta_{kt}^2/2)}{2\pi}\mathrm{d}\boldsymbol{\eta}_t=\mathbb{E}_{p(\boldsymbol{\eta}_t)}\left\{\log\frac{p_{\bphi_d}(\bx_t,\bmu_t + \bzeta_t \odot \boldsymbol{\eta}_t)}{q_{\bphi_e}(\bmu_t + \bzeta_t \odot \boldsymbol{\eta}_t\mid\bx_t)}\right\}.
\end{align*}
\end{small}
On the last line, we plugged $\bZ_t=\bmu_t + \bzeta_t \odot \boldsymbol{\eta}_t$ in Eq.~\eqref{eqn:q_phi_e_expression} to obtain the clean form, and $p(\boldsymbol{\eta}_t)$ denotes the joint density of $K$ independent standard normal variables. Therefore, we can now form simple Monte Carlo estimators of $\mathcal{L}_{\bphi_e,\bphi_d}$, $\nabla_{\bphi_e}\mathcal{L}_{\bphi_e,\bphi_d}$, and $\nabla_{\bphi_d}\mathcal{L}_{\bphi_e,\bphi_d}$. More specifically,
\begin{align*}
    \mathcal{L}_{\bphi_e,\bphi_d}(\boldsymbol{x}_t)&\approx \frac{1}{L}\sum_{l=1}^L \log\frac{p_{\bphi_d}(\bx_t,\bmu_t + \bzeta_t \odot \boldsymbol{\eta}^l)}{q_{\bphi_e}(\bmu_t + \bzeta_t \odot \boldsymbol{\eta}^l\mid\bx_t)}\\
    &=\frac{1}{L}\sum_{l=1}^L \log p_{\bphi_d}(\bx_t\mid\bZ^l) + \frac{1}{L}\sum_{l=1}^L \log p_{\bphi_d}(\bZ^l)-\frac{1}{L}\sum_{l=1}^L \log p(\boldsymbol{\eta}^l) +\sumK \log \zeta_{kt},
\end{align*}
where $\boldsymbol{\eta}^l$, $l=1,\ldots, L$, are independent draws from $N(\boldsymbol{0}_K,\boldsymbol{I}_{K\times K})$ and $\bZ^l=\bmu_t + \bzeta_t \odot \boldsymbol{\eta}^l$. Also, $p_{\bphi_d}(\bx_t\mid\bz^l)$ and $p_{\bphi_d}(\bz^l)$ are defined in Eqs.~\eqref{eqn:lik_form} and \eqref{eqn:prior_form} of the main paper. Furthermore,
\begin{equation*}
    \nabla_{\bphi_e}\mathcal{L}_{\bphi_e,\bphi_d}(\boldsymbol{x}_t)\approx \frac{1}{L}\sum_{l=1}^L \nabla_{\bphi_e}\log p_{\bphi_d}(\bx_t\mid\bZ^l) + \frac{1}{L}\sum_{l=1}^L \nabla_{\bphi_e}\log p_{\bphi_d}(\bZ^l) +\sumK \nabla_{\bphi_e}\log \zeta_{kt}
\end{equation*}
and
\begin{equation*}
    \nabla_{\bphi_d}\mathcal{L}_{\bphi_e,\bphi_d}(\boldsymbol{x}_t)\approx \frac{1}{L}\sum_{l=1}^L \nabla_{\bphi_d}\log p_{\bphi_d}(\bx_t\mid\bZ^l) + \frac{1}{L}\sum_{l=1}^L \nabla_{\bphi_d}\log p_{\bphi_d}(\bZ^l).
\end{equation*}


\subsection{Effect of knot locations}\label{sec:data_driven_knots}
\begin{algorithm}
\setstretch{1.21}
\SetKwInOut{Input}{Input}
\caption{Derive data-driven knots}\label{alg:kmeans}
\Input{$\kappa$: number of possible clusters from each time replicate\\
\hspace*{-0.5em} $\{\bx_t:t=1,\ldots, n_t\}$: observed $n_t$ spatial replicates\\
\hspace*{-0.5em} $\{\bs_j:j=1,\ldots, n_s\}$: coordinates of the observed sites in the domain $\mathcal{S8z}$\\
\hspace*{-0.5em} $u$: a high quantile level between $0$ and $1$\\
\hspace*{-0.5em} $\lambda$: minimum distance between knots} 
\KwResult{\\
\Indp $K$: number of data-driven knots\\
\hspace*{-0.5em} $\{\Tilde{\bs}_1, \ldots, \Tilde{\bs}_{K}\}$: the coordinates of data-driven knots\\
\hspace*{-0.5em} $r$: basis function radius shared by all knots}

\vskip 0.2cm
{$x^*\gets$  $u$th quantile of the concatenated vector $(\bx_1^\top,\cdots,\bx_{n_t}^\top)^\top$\tcp*{A high threshold}
$Knots \gets $ list()\tcp*{Empty list for the chosen knot locations}
\For{$t\gets1,n_t$}{
    $\mathcal{E}_t\gets$  where($\bx_t>x^*$)\tcp*{Indices of the locations exceeding the threshold}
    $wss\_vec\gets $ repeat(\texttt{NA}, $\kappa$)\tcp*{\small Vector for the total within-cluster sums of squares}
    \For{nclust $\gets1,\kappa$}{
       $init\_centers \gets$ sample($\{\bs_j:j\in \mathcal{E}_t\},\; nclust$) \tcp*{$nclust$ initial centers}
        $res\_tmp \gets$ kmeans($\{\bs_j:j\in \mathcal{E}_t\}, \; init\_centers$)\tcp*{\small \citet{hartigan1979algorithm}}
        $wss\_vec\;[nclust] \gets res\_tmp$ [\texttt{"tot.withinss"}];
    }
    $best\_nclust\gets $ which.max($wss\_vec$)\tcp*{Determine the best number of clusters}
    $init\_centers \gets$ sample($\{\bs_j:j\in \mathcal{E}_t\},\; best\_nclust$);\\
    \hspace*{-0.5em}$res \gets$ kmeans($\{\bs_j:j\in \mathcal{E}_t\}, \; init\_centers$);\\ 
    \hspace*{-0.5em}$Knots\gets $ append($Knots$, $res$ [\texttt{"centers"}])\tcp*{Cluster centers as knots}
}
$Knots \gets$ remove points from $Knots$ so that all knots are no closer than $\lambda$;\\
$K \gets$ length($Knots$);\\
$\{\Tilde{\bs}_1, \ldots, \Tilde{\bs}_{K}\}\gets Knots$;\\
$r\gets $ the minimum radius such that any $\bs\in \mathcal{S}$ is covered by at least one basis function.
}
\label{algo:knots}
\end{algorithm}

Algorithm \ref{alg:kmeans} outlines how we derive the data-driven knots. First, we perform $k$-means clustering on each time replicate of the data input 
to determine how many clusters of high values ($u > 0.95$) there are, and we then train XVAE with $K$ being the number of clusters combined for all time replicates. Second, the cluster centroids are used as knot locations $\{\Tilde{\bs}_1, \ldots, \Tilde{\bs}_{K}\}$. To initialize $\bW$ (defined in Eq.~\eqref{eqn:decoder}) using the Wendland basis functions $\omega_k(\bs, r)=\{1-d(\bs, \Tilde{\bs}_k)/r\}^2_+$, $k=1,\ldots, K$, 
we pick $r$ by looping over clusters and calculating the Euclidean distance of each point within one cluster from its centroid, and we set the maximum of all distances as the initial $r$. If $r$ is not large enough for all $\omega_k(\bs, r)$ to cover the entire spatial domain, we gradually increase $r$ until the full coverage is met. 
\begin{figure}[!t]
    \centering
    \captionsetup[subfigure]{justification=centering}
    \begin{subfigure}[t]{.328\linewidth}
  \hspace*{-0.15cm}\includegraphics[height=0.9\linewidth]{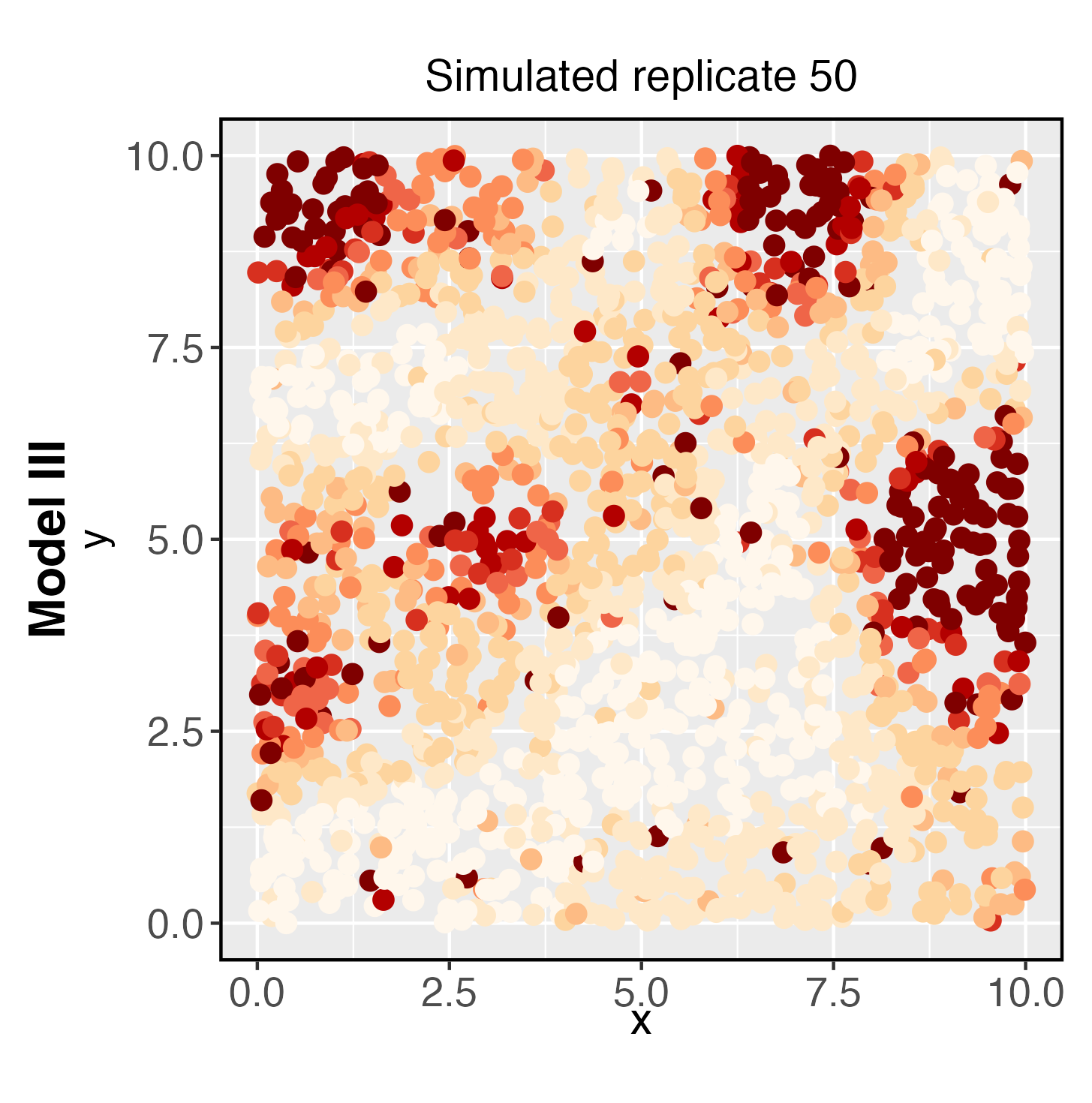}
    \caption{Input replicate at time 50}
    \label{subfig:sim1}
    \end{subfigure}
    \begin{subfigure}[t]{.328\linewidth}
\centering\includegraphics[height=0.9\linewidth]{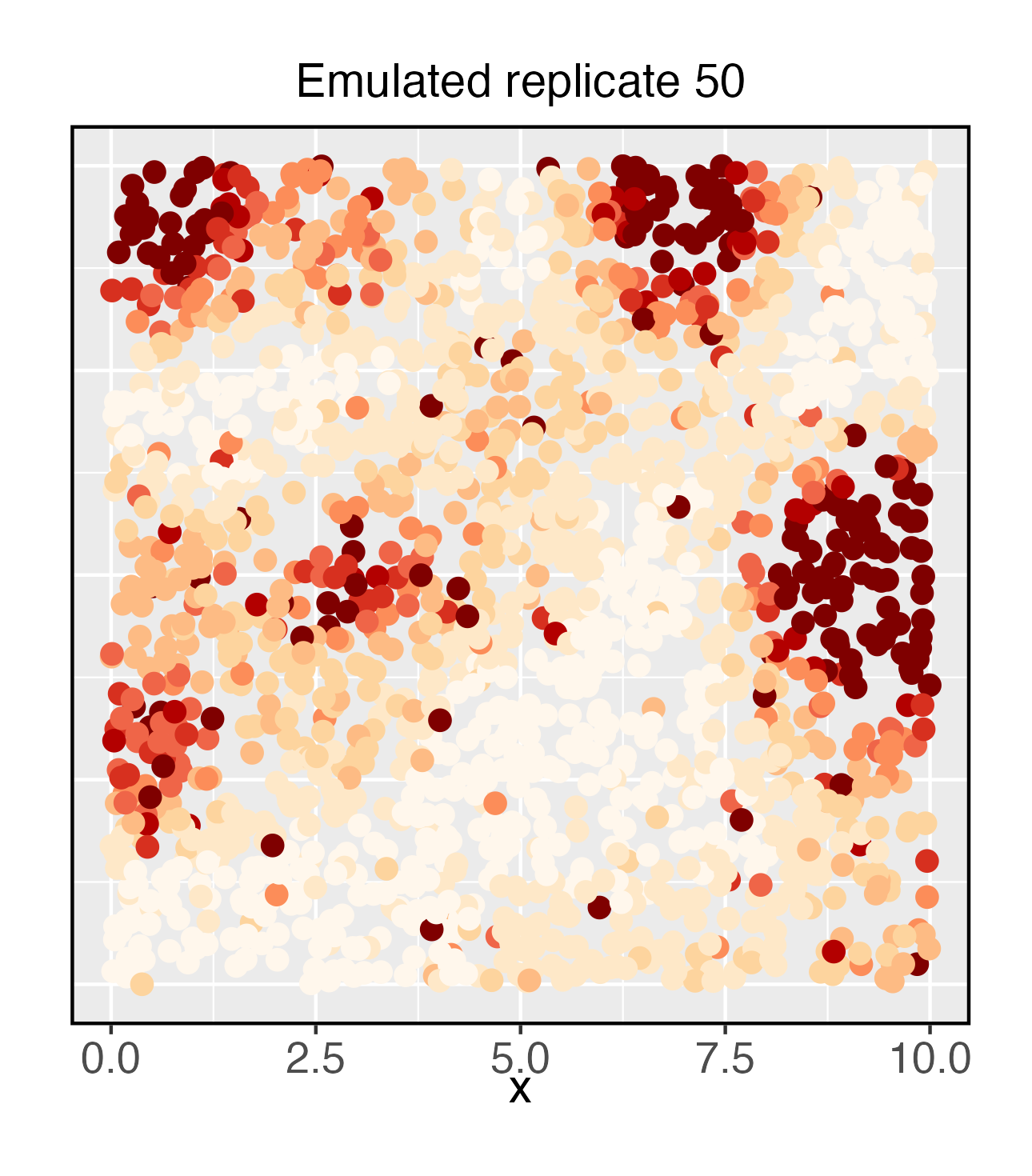}
\caption{Emulation with the true knots}
    \label{subfig:emu_true_knots}
    \end{subfigure}
    \begin{subfigure}[t]{.328\linewidth}
    \centering\includegraphics[height=0.9\linewidth]{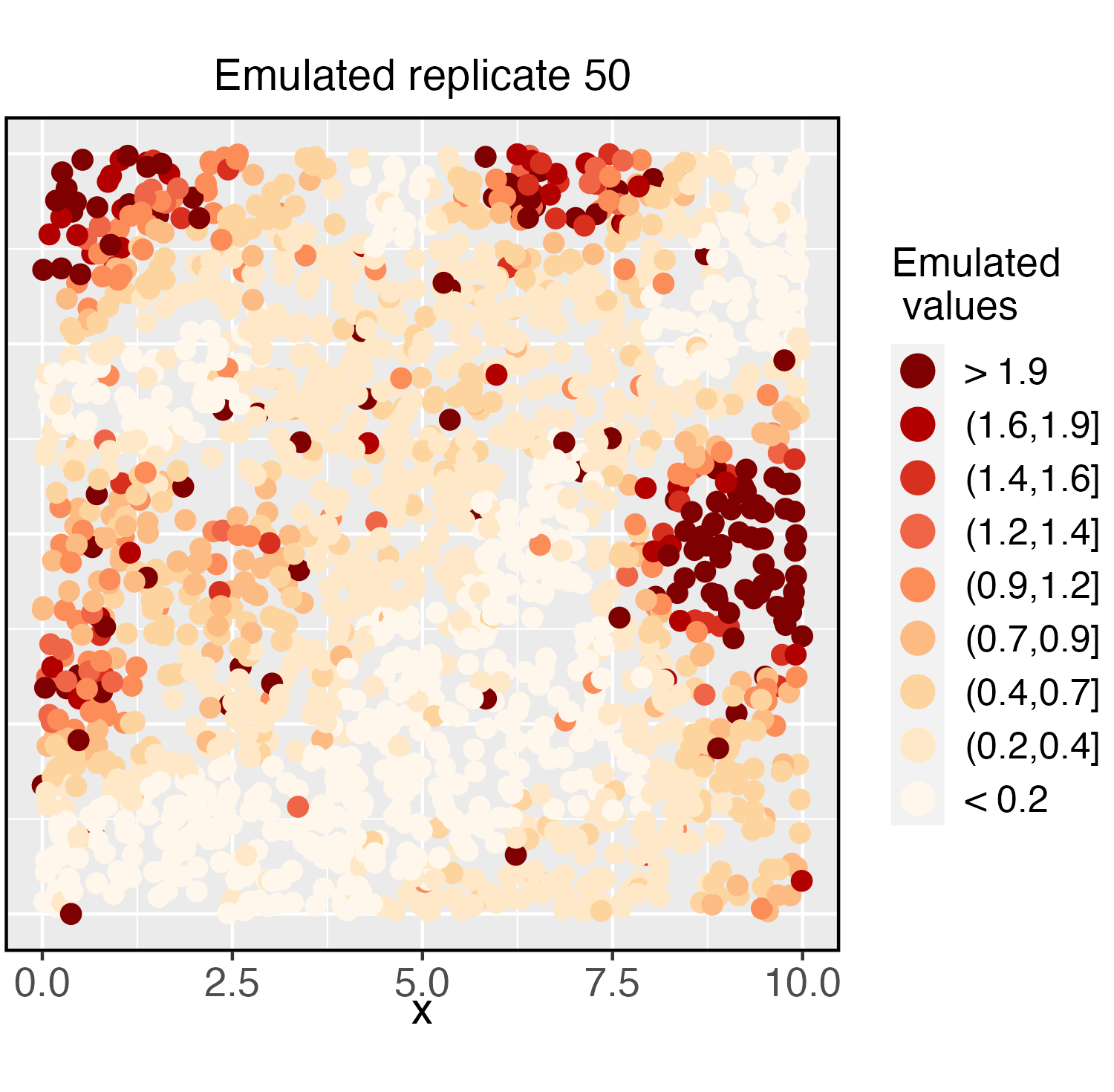}
    \caption{Emulation with the data-driven knots}
    \label{subfig:emu_data_driven}
    \end{subfigure}

    \begin{subfigure}[t]{.328\linewidth}
  \hspace*{0.4cm}\includegraphics[height=0.85\linewidth]{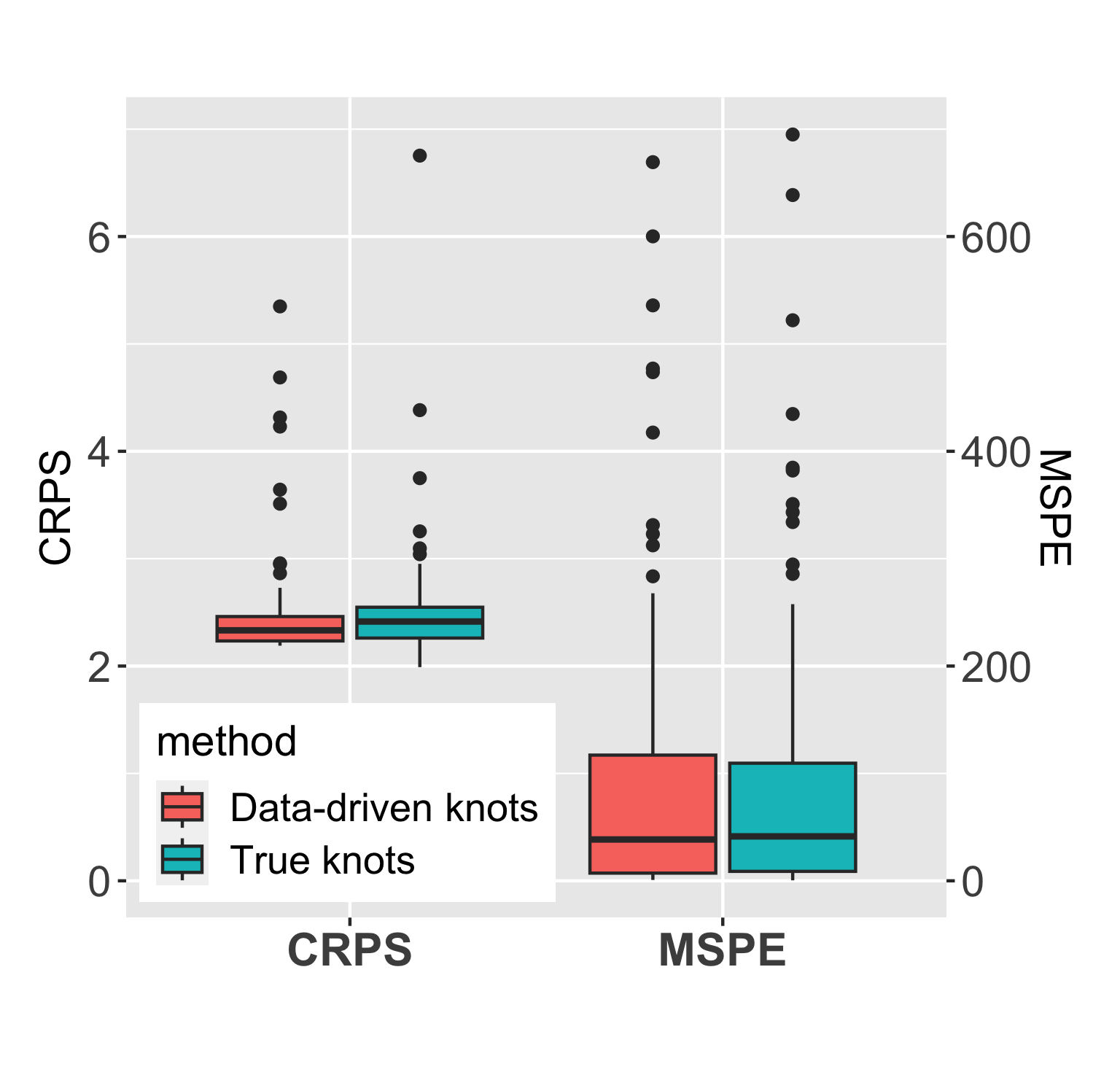}
    \caption{Spatial predictions}
    \label{subfig:pred1}
    \end{subfigure}
    \begin{subfigure}[t]{.328\linewidth}
\hspace*{-0.2cm}\includegraphics[height=0.9\linewidth]{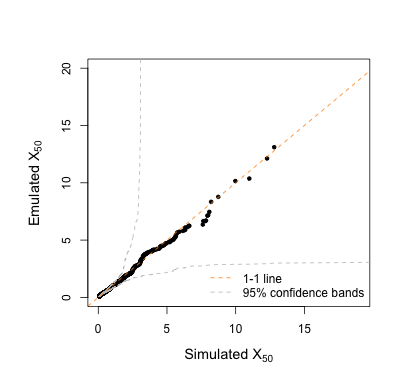}
\caption{QQ-plot with the true knots}
    \label{subfig:qq_true_knots}
    \end{subfigure}
    \begin{subfigure}[t]{.328\linewidth}
\hspace*{-0.6cm}\includegraphics[height=0.9\linewidth]{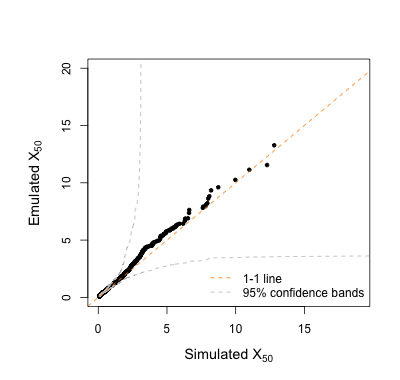}
\caption{QQ-plot with the data-driven knots}
    \label{subfig:qq_data_driven}
    \end{subfigure}
\caption{Comparing the emulation results from initializing the XVAE with the true knots and data-driven knots for data simulated from Model \ref{modelFlex}.}
\label{fig:knots_comp}
\end{figure}

Figure~\ref{fig:knots_comp} displays the results from emulating the data set simulated from Model \ref{modelFlex} while initializing the weights differently using the true knots and the data-driven knots. 
Figures~\ref{subfig:emu_true_knots} and \ref{subfig:emu_data_driven} show one emulation replicate from the decoder for the 50th time replicate. We see that both figures exhibit a striking resemblance to the original simulation, and from visual examination, we can see little difference in the quality of the emulations. Figure~\ref{subfig:pred1} compares the spatial predictions on the 100 holdout locations from the two emulations. The CRPS and MSPE values are again very similar for emulations based on the true knots and data-driven knots.

Figures~\ref{subfig:qq_true_knots} and \ref{subfig:qq_data_driven} compare the simulated and emulated spatial fields of the 50th replicate by plotting their quantiles against each other (when pooling the spatial data into the same plot). We see that both emulations align very well with the simulated data set. Although this might not be the most appropriate way of evaluating the quality of the emulations because there is spatial dependence and non-stationarity within each spatial replicate, QQ-plots still provide value in determining whether the spatial distribution is similar at all quantile levels, which is complementary to the empirical $\chi_{ij}(u)$ described in Section~\ref{subsec:dependence_res}.

Overall, Figure~\ref{fig:knots_comp} demonstrates that emulation based on data-driven knots performs similarly to using the true knots. This justifies applying the XVAE on a data set stemming from a misspecified model (i.e., Models \ref{modelGP} or \ref{modelMaxStable}, for which the data-generating process does not involve any Wendland basis functions). Thus, we will use the XVAE with data-driven knots in all remaining simulation experiments and the real data application.

\subsection{Stochastic gradient descent optimization} \label{sec:SGD}
A major advantage of approximating the ELBO as presented in Eq.~\eqref{eqn:ELBO_approx} lies in the ability to perform joint optimization over all parameters ($\bphi_e$ and $\bphi_d$) using stochastic gradient descent (SGD). This optimization is efficiently implemented using a tape-based automatic differentiation module called \texttt{autograd} within the \texttt{R} package \texttt{torch} \citep{rtorch}. Built on \texttt{PyTorch}, this package offers rapid array computation, leveraging robust GPU acceleration for enhanced computational efficiency. It stores all the data inputs and VAE parameters in the form of \texttt{torch} tensors, which are similar to \texttt{R} multi-dimensional arrays but are designated for fast and scalable matrix calculations and differentiation. 

Algorithm \ref{algo:SGD} outlines the pseudo-code for the ELBO optimization of our XVAE. As the ELBO is constructed within each iteration of the SGD algorithm, the \texttt{autograd} module of \texttt{torch} tracks the computations (i.e., linear operations and ReLU activation on the tensors) in all layers of the encoding/decoding neural networks, and then performs the reverse-mode automatic differentiation via a backward pass through the graph of tensor operations to obtain the partial derivatives or the gradients with respect to each weight and bias parameter \citep{keydana2023deep}. 

The iterative steps of Algorithm \ref{algo:SGD} involve advancing in the direction of the gradients on the ELBO $\sum_{t=1}^{n_t}\mathcal{L}_{\bphi_e,\bphi_d}(\boldsymbol{x}_t)$ (or a minibatch version $\sum_{t\in \mathcal{M}}\mathcal{L}_{\bphi_e,\bphi_d}(\boldsymbol{x}_t)$, $\mathcal{M}\subset \{1,\ldots,n_t\}$). This is guided by a user-defined learning rate $\nu>0$. To enhance stability, a convex combination of the prior update and the current gradient incorporates a momentum parameter $\zeta_m$ into the optimization process \citep{polyak1964momentum}. Notably, our experiments indicate that setting the number of Monte Carlo samples $L$ to 1 suffices, provided the minibatch size $|\mathcal{M}|$ is adequately large, aligning with the recommendation by \citet{kingma2013auto}. Upon successful training of $\bphi_e$ and $\bphi_d$, the encoder and decoder can be efficiently executed as needed. Leveraging the amortized nature of our estimation approach, these processes generate an ensemble of numerous samples, all originating from the same (approximate) distribution as the spatial inputs.

Importantly, our XVAE algorithm can scale efficiently to massive spatial data sets. The existing max-stable, inverted-max-stable, and other spatial extremes models are limited to applications with less than approximately $1,000$ locations using a full likelihood or Bayesian approach; see Section~\ref{sec:comps} of the main paper for more details on these alternative approaches. By contrast, our approach can fit a globally non-stationary spatial extremes process, with parameters evolving over time, to a data set of unprecedented spatial dimension of more than $16,000$ locations, and also facilitates data emulation in such dimensions. See Section~\ref{sec:data_analysis} of the main paper for details.


\begin{algorithm}
\setstretch{1.21}
\SetKwInOut{Input}{Input}
\caption{Stochastic Gradient Descent with momentum to maximize the ELBO defined in Eq.~\eqref{eqn:ELBO_approx}. We set $|\mathcal{M}| = n_t$ and $L = 1 $ in our experiments.}\label{alg:sim}
\Input{Learning rate $\nu>0$, momentum parameter $\zeta_m\in (0,1)$, convergence tolerance $\delta$\\
\hspace*{-0.5em} $\{\bx_t:t=1,\ldots, n_t\}$: observed $n_t$ spatial replicates\\
 \hspace*{-0.5em} $q_{\bphi_e}(\bz_t\mid \bx_t)$: inference model\\
 \hspace*{-0.5em} $p_{\bphi_d}(\bx_t, \bz_t)$: generative data model}
 
\KwResult{Optimized parameters $\bphi_e,\;\bphi_d$}
\vskip 0.2cm
{$j \gets $ 0;\\
\hspace*{-0.5em}$K \gets $ Number of data-driven knots;\\
\hspace*{-0.5em}$\{\Tilde{\bs}_1, \ldots, \Tilde{\bs}_{K}\} \gets $ Specify knot locations\tcp*{See Section~\ref{sec:data_driven_knots} for details}
$r \gets $ Basis function radius shared by all knots;\\
\hspace*{-0.5em}$(\bphi_e^{(j)},\bphi_d^{(j)})^\top \gets$ Initialized parameters\tcp*{See Section \ref{sec:find_ini_vals} for details}
$\bv\gets\boldsymbol{0}$\tcp*{Velocity}
$\bL\gets $ repeat(\texttt{-Inf}, 200)\tcp*{A vector of 200 negative infinite values}
}
\While{$|\mathrm{mean}\{\bL[(j-200):(j-101)]\}-\mathrm{mean}\{\bL[(j-100):j]\}|> \delta$}{
  $\mathcal{M}\sim \{1, \ldots, n_t\}$\tcp*{Indices for the random minibatch}
  $\eta_{kt}\stackrel{\mathrm{i.i.d.}}{\sim} \mathrm{Normal}(0,1)$, $k=1,\ldots,K$, $t\in \mathcal{M}$\tcp*{Reparameterization trick}
  \For{$t\in \mathcal{M}$}{
    $(\bmu_t^\top,\log \bzeta_t^\top)^\top \gets \mathrm{EncoderNeuralNet}_{\bphi_e^{(j)}}(\bx_t)$;\\
  \hspace*{-0.5em}$\bz_t\gets\bmu_t + \bzeta_t \odot \boldsymbol{\eta}_t$;\\
  \hspace*{-0.5em}$(\alpha_t, \bgamma_t^\top)^\top \gets \mathrm{DecoderNeuralNet}_{\bphi_d^{(j)}}(\bz_t)$;\\
  \hspace*{-0.5em}Calculate $q_{\bphi_e^{(j)}}(\bz_t\mid\bx_t)$, $p_{\bphi_d^{(j)}}(\bx_t\mid\bz_t)$ and $p_{\bphi_d^{(j)}}(\bz_t)$\tcp*{See Eq.~\eqref{eqn:encoder_form}-\eqref{eqn:prior_form}}
    }
  Obtain the ELBO $\mathcal{L}_{\bphi_e^{(j)},\bphi_d^{(j)}}(\mathcal{M})=\sum_{t\in\mathcal{M}}\mathcal{L}_{\bphi_e^{(j)},\bphi_d^{(j)}}(\boldsymbol{x}_t)$ and its gradients $\boldsymbol{J}_\mathcal{L} = \{\nabla_{\bphi_e,\bphi_d}\mathcal{L}_{\bphi_e,\bphi_d}(\mathcal{M})\}(\bphi_e^{(j)},\bphi_d^{(j)})$;\\
  \hspace*{-0.5em}Compute velocity update: $\bv\gets\zeta_m \bv +\nu \boldsymbol{J}_\mathcal{L}$;\\
  \hspace*{-0.5em}Apply update: $(\bphi_e^{(j+1)},\bphi_d^{(j+1)})^\top\gets (\bphi_e^{(j)},\bphi_d^{(j)})^\top + \bv$;\\
  \hspace*{-0.5em} $\bL\gets (\bL^\top,\mathcal{L}_{\bphi_e^{(j)},\bphi_d^{(j)}}(\mathcal{M}))^\top$ \tcp*{Add the latest ELBO value to the vector $\bL$}
  $j\gets j+1$;
}
\label{algo:SGD}
\end{algorithm}

\subsection{Finding starting values}\label{sec:find_ini_vals}
In finding a reasonable starting values of parameters in XVAE, we choose $\alpha_0=1/4$ and $\tau=1$ for the white noise process, and $\alpha=1/2$ for the latent exponentially-tilted PS variables. From Eq.~\eqref{eqn:low_rank_representation}, $(\by^{1/\alpha_0}_1, \cdots, \by^{1/\alpha_0}_{n_t})=\bW^{1/\alpha} (\bz_1, \cdots, \bz_{n_t})$, in which $\bW$ is defined in Eq.~\eqref{eqn:decoder}.
Since $\{\epsilon_t(\bs):t=1,\ldots, n_t\}$ are treated as error processes, we have $\bx_t\approx\by_t$ and thus a good approximation for $\boldsymbol{z}_t$ can be obtained via projection:
\begin{align*}
    \hat{\bz}_t \approx \{(\bW^{\frac{1}{\alpha}})^\top \bW^{\frac{1}{\alpha}}\}^{-1} (\bW^{\frac{1}{\alpha}})^\top \bx_t^{\frac{1}{\alpha_0}}, \; t=1,\ldots,n_t.
\end{align*}
We use QR decomposition to solve the following linear system to get the initial value $\boldsymbol{W}_1^{(0)}$:
    $(\hat{\bz}_1,\cdots,\hat{\bz}_{n_t})^\top=(\bx_1,\cdots,\bx_{n_t})^\top\boldsymbol{W}_1^\top$. 
Also, set $\boldsymbol{b}_1^{(0)}=(0,\ldots,0)^\top$. The initial values of $\boldsymbol{h}_{1,t}$ in Eq.~\eqref{eqn:encoder_weights} satisfy $\boldsymbol{h}_{1,t}\approx \hat{\bz}_t$, $t=1,\ldots, n_t$.

Furthermore, we set $\boldsymbol{W}_2^{(0)}$ and $\boldsymbol{W}_4^{(0)}$ to be identity matrices. All remaining parameters, both variational and generative, were initialized by random sampling from $N(0, 0.01)$. 

To optimize the ELBO following the steps outlined in Algorithm \ref{algo:SGD}, we monitor the convergence of the ELBO via calculating the difference in the average ELBO values in the latest 100 iterations (or epochs) and the 100 iterations before that. Once the difference is less than $\delta=10^{-6}$, we stop the stochastic gradient search.

\section{Additional results from the simulation study}\label{appendix:additional_results}
We show additional figures that are complementary to those included in Section~\ref{sec:simulation} of the main paper. Figure~\ref{fig:comps_across_models1} displays the simulated data sets from Models~\ref{modelGP}, \ref{modelAI}, \ref{modelAD} and \ref{modelMaxStable} and their emulated fields using both XVAE and \texttt{hetGP}. See Figure~\ref{fig:comps_across_models} of the main paper for comparison for Model \ref{modelFlex}. Figure~\ref{fig:qqplots_across_models} displays QQ-plots from the spatial data to compare the overall distributions of the simulated and emulated data sets. Figure~\ref{fig:chi_ests2} compares the empirically estimated $\chi_h(u)$ as described in Section~\ref{subsec:dependence_res} from the data replicates simulated from Models~\ref{modelGP}, \ref{modelAI}, \ref{modelAD} and \ref{modelMaxStable} and their emulations at three different distances $h=0.5,2,5$ under the working assumption of stationarity. Figure~\ref{fig:ARE_comps2} shows the estimates of $\mathrm{ARE}_\psi(u)$ defined in Eq.~\eqref{eqn:ARE_monte_carlo} of the main paper, $\psi=0.05$, for both simulations and XVAE emulations under Models~\ref{modelGP}, \ref{modelAI}, \ref{modelAD} and \ref{modelMaxStable}. See Figure~\ref{fig:chi_ests} of the main paper for $\chi_h(u)$ and $\mathrm{ARE}_\psi(u)$ estimates for Model \ref{modelFlex}. Lastly, Figure~\ref{fig:coverage} shows coverage probabilities of $\{\gamma_{kt}: k=1,\ldots, K\}$ for $t=1$ from fitting Model \ref{modelFlex}. Coverage probabilities when $\gamma_k=0$ are poor, though upper bounds of credible intervals are consistently less than $10^{-6}$.
\begin{figure}
    \centering
    \includegraphics[height=0.245\linewidth, trim={0 1.6cm 0 0}, clip]{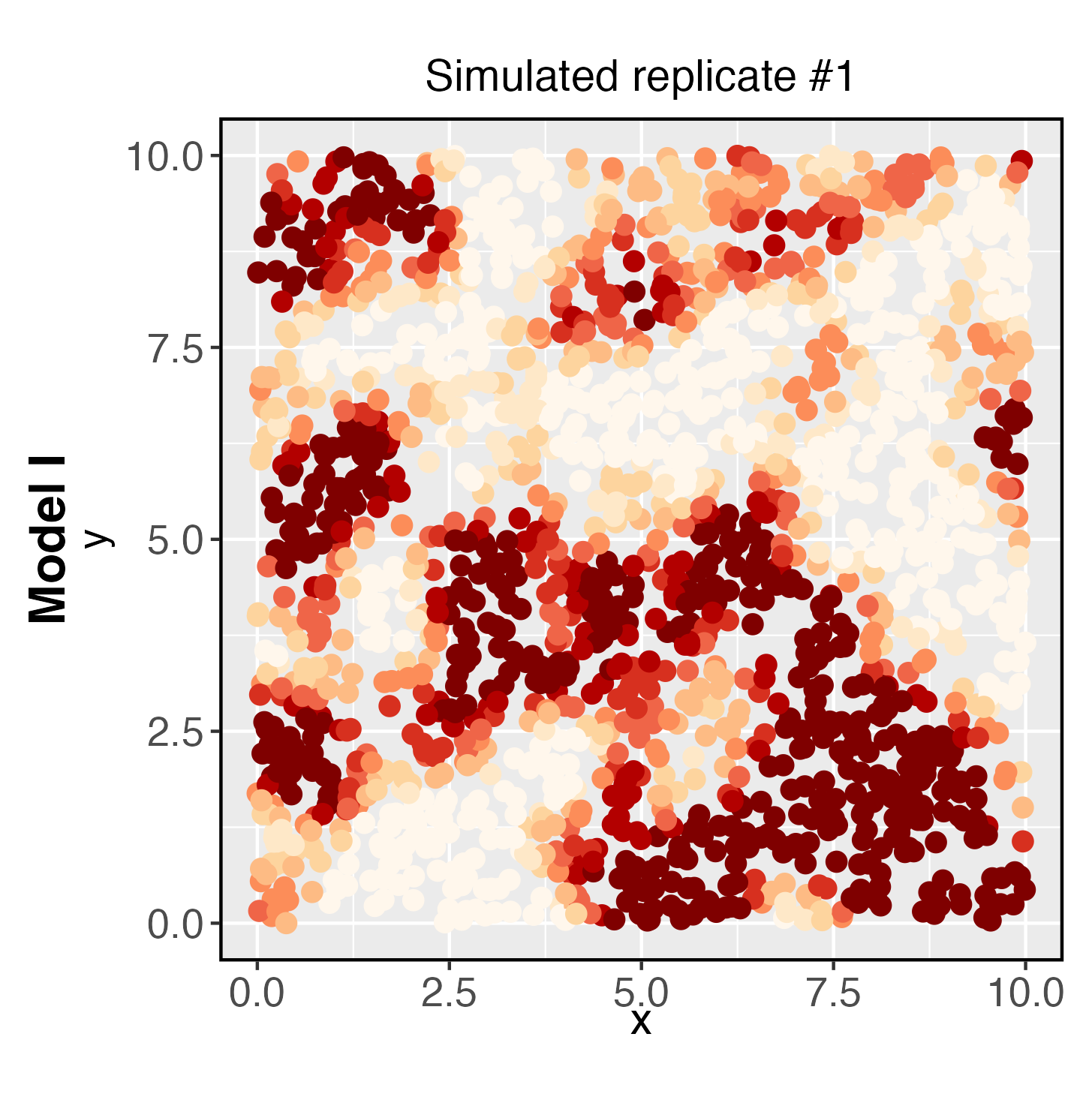}
    \includegraphics[height=0.245\linewidth, trim={0 1.6cm 0 0}, clip]{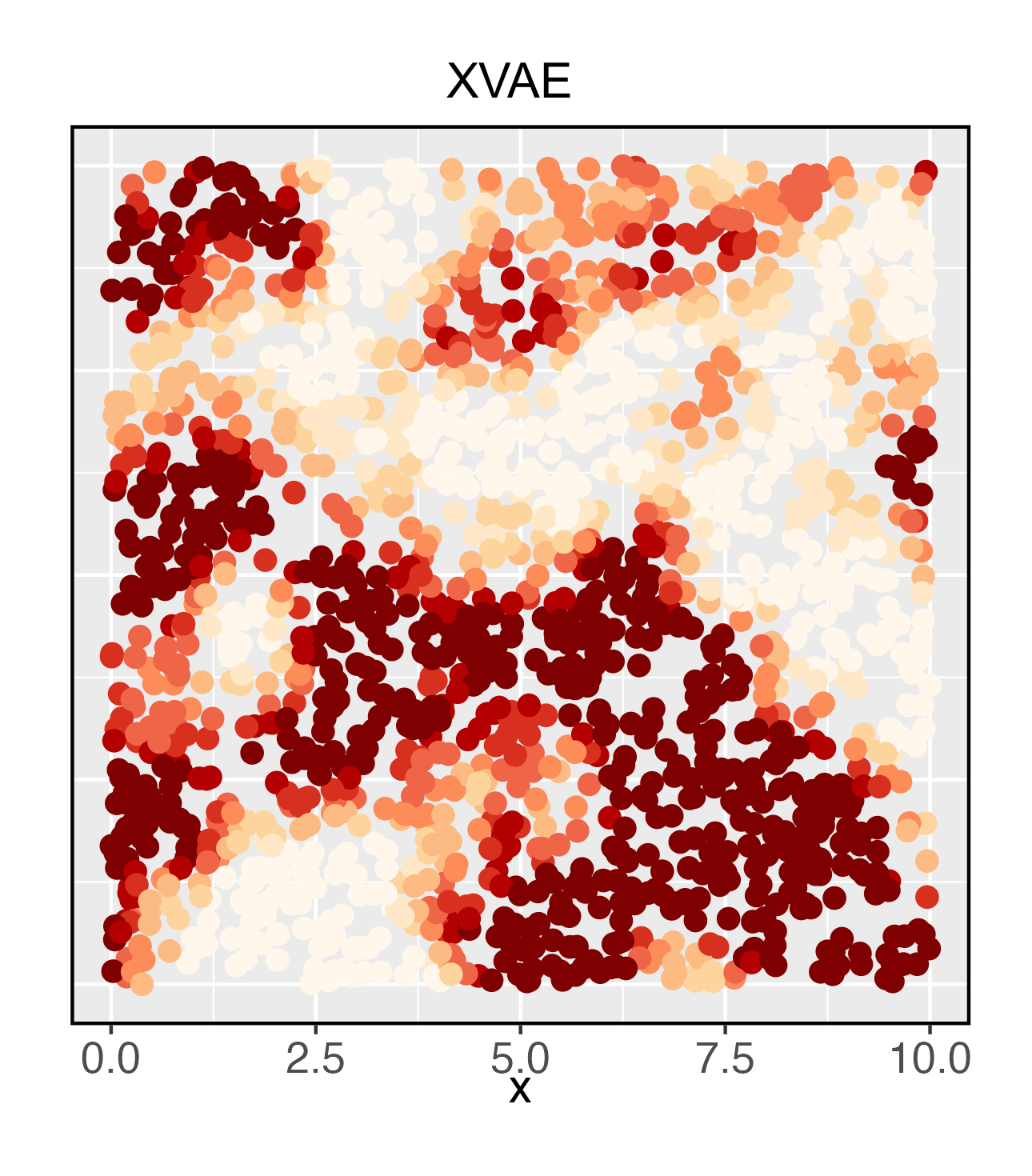}
    \includegraphics[height=0.245\linewidth, trim={0 1.6cm 0 0}, clip]{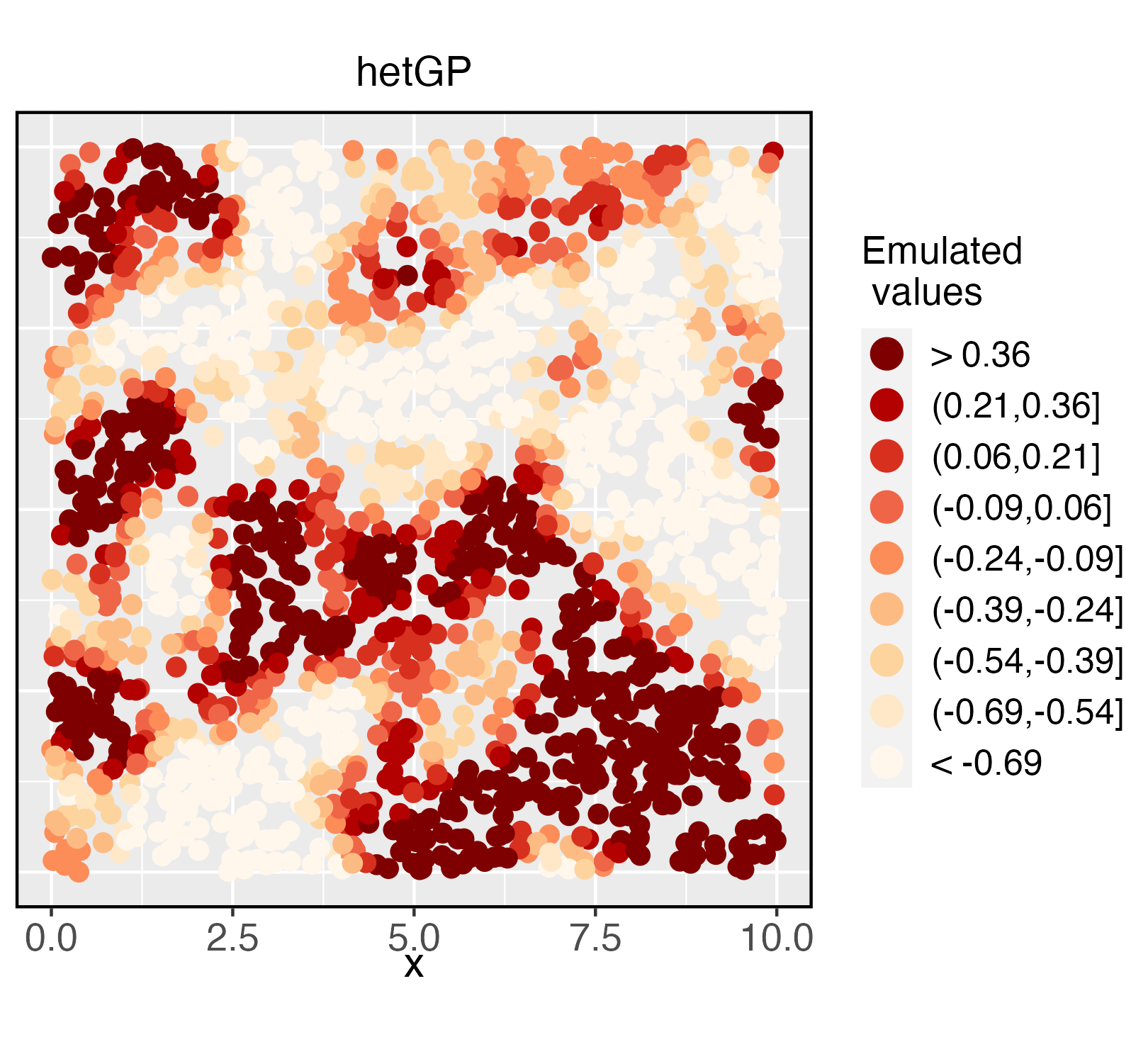}
    \vskip 0.3cm
    
    \includegraphics[height=0.22\linewidth, trim={0 1.6cm 0 1.12cm}, clip]{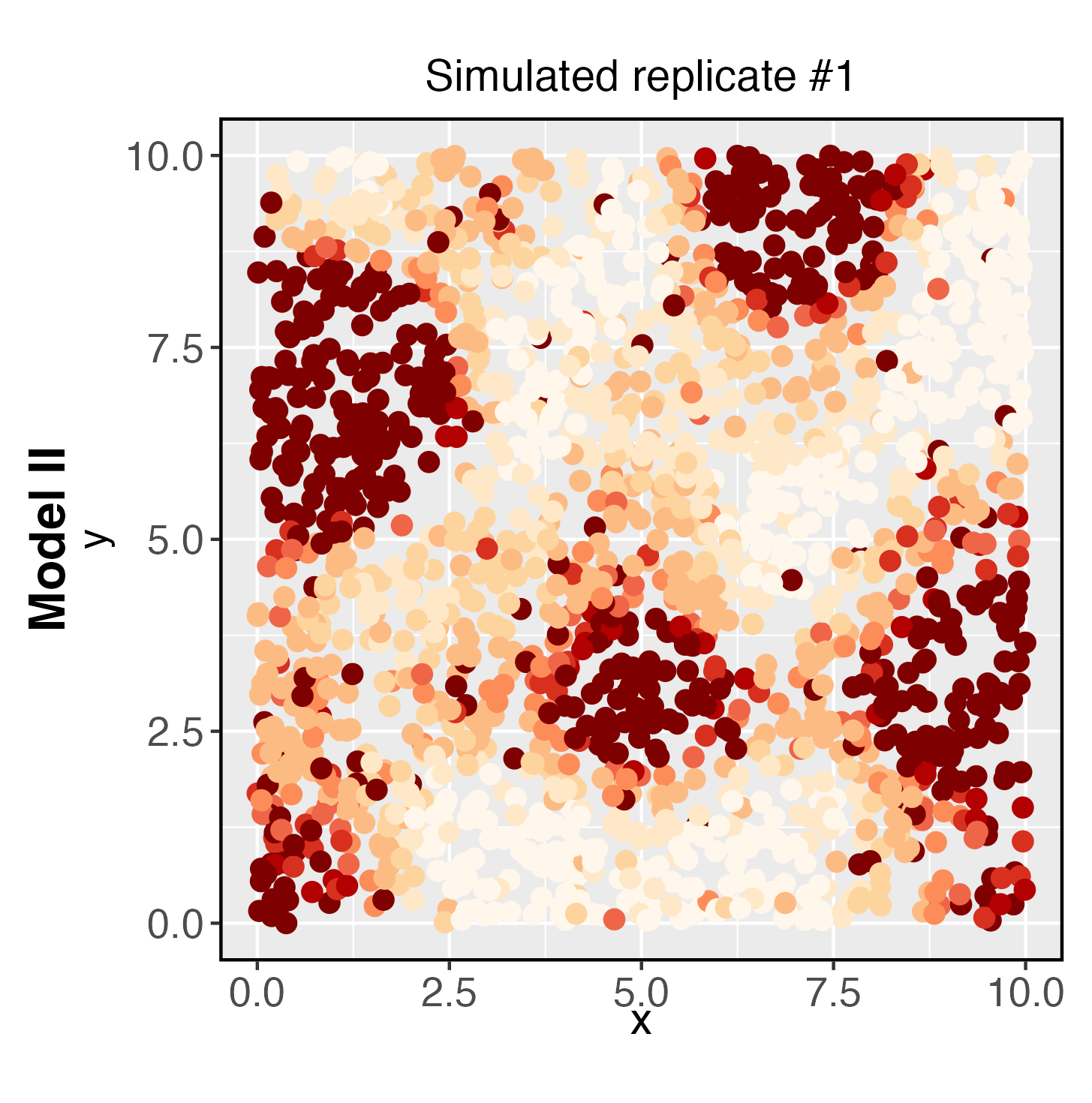}
    \includegraphics[height=0.22\linewidth, trim={0 1.6cm 0 1.12cm}, clip]{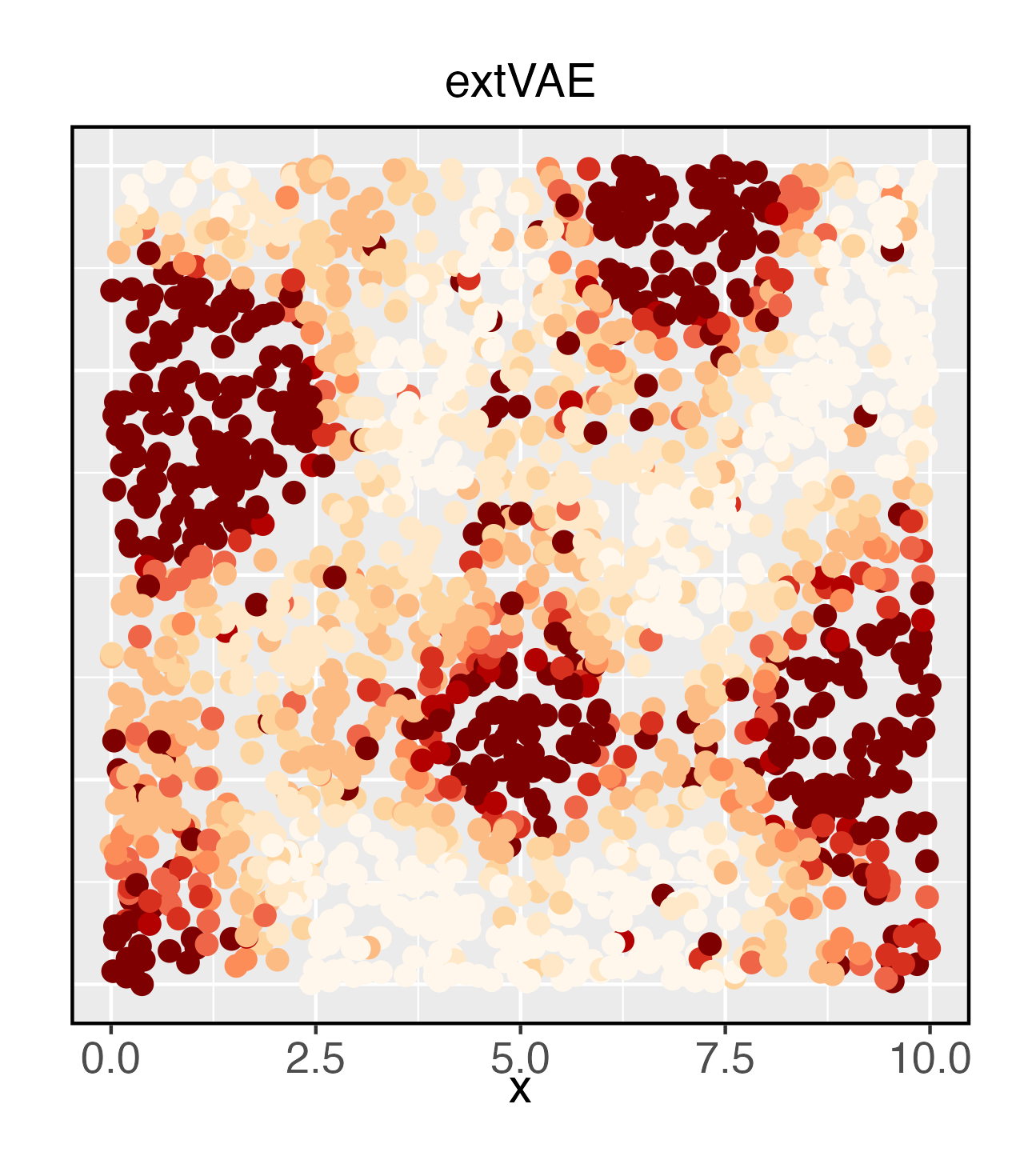}
    \includegraphics[height=0.22\linewidth, trim={0 1.6cm 0 1.12cm}, clip]{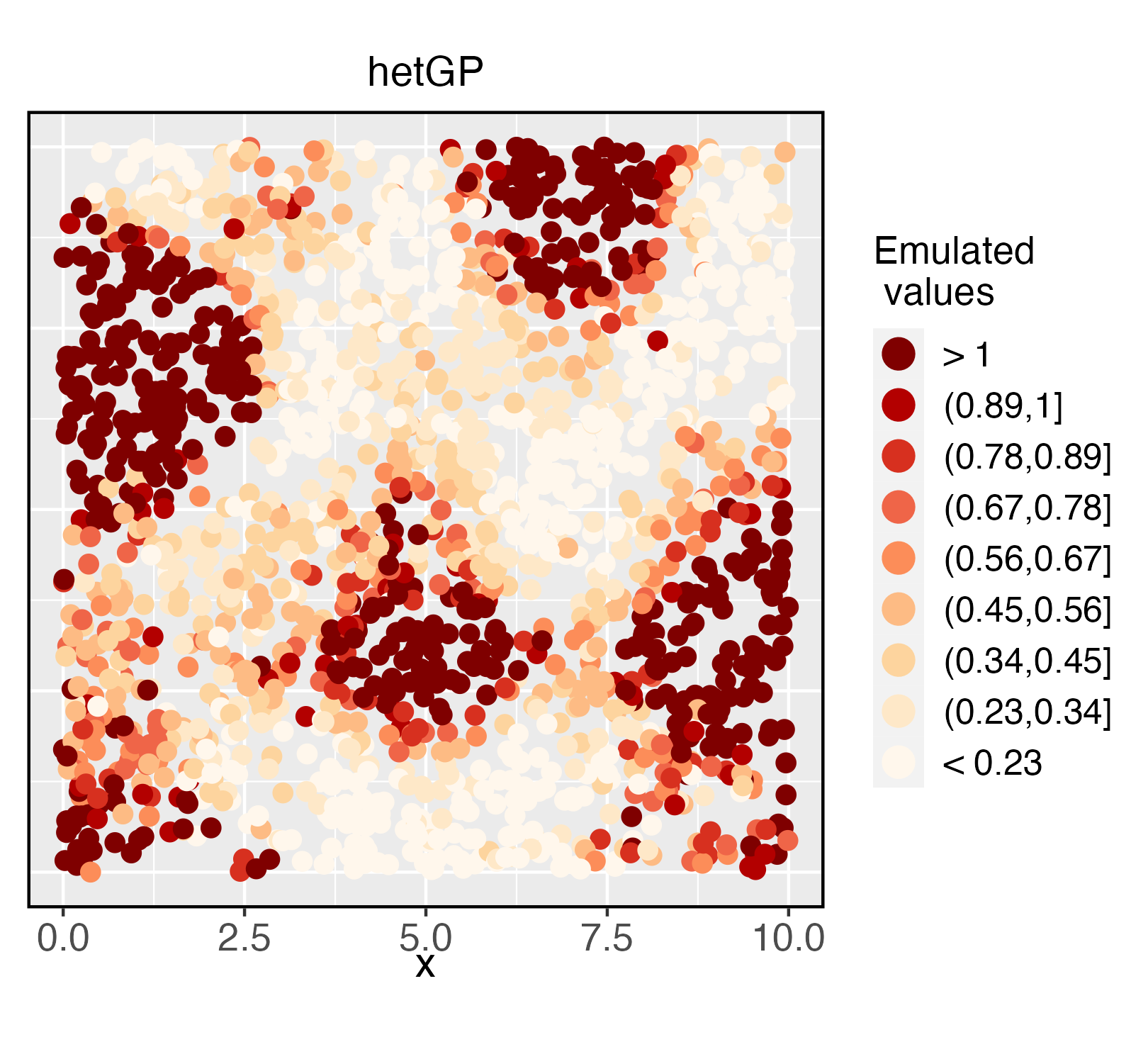}
    \vskip 0.3cm
    
    
    \includegraphics[height=0.22\linewidth, trim={0 1.6cm 0 1.12cm}, clip]{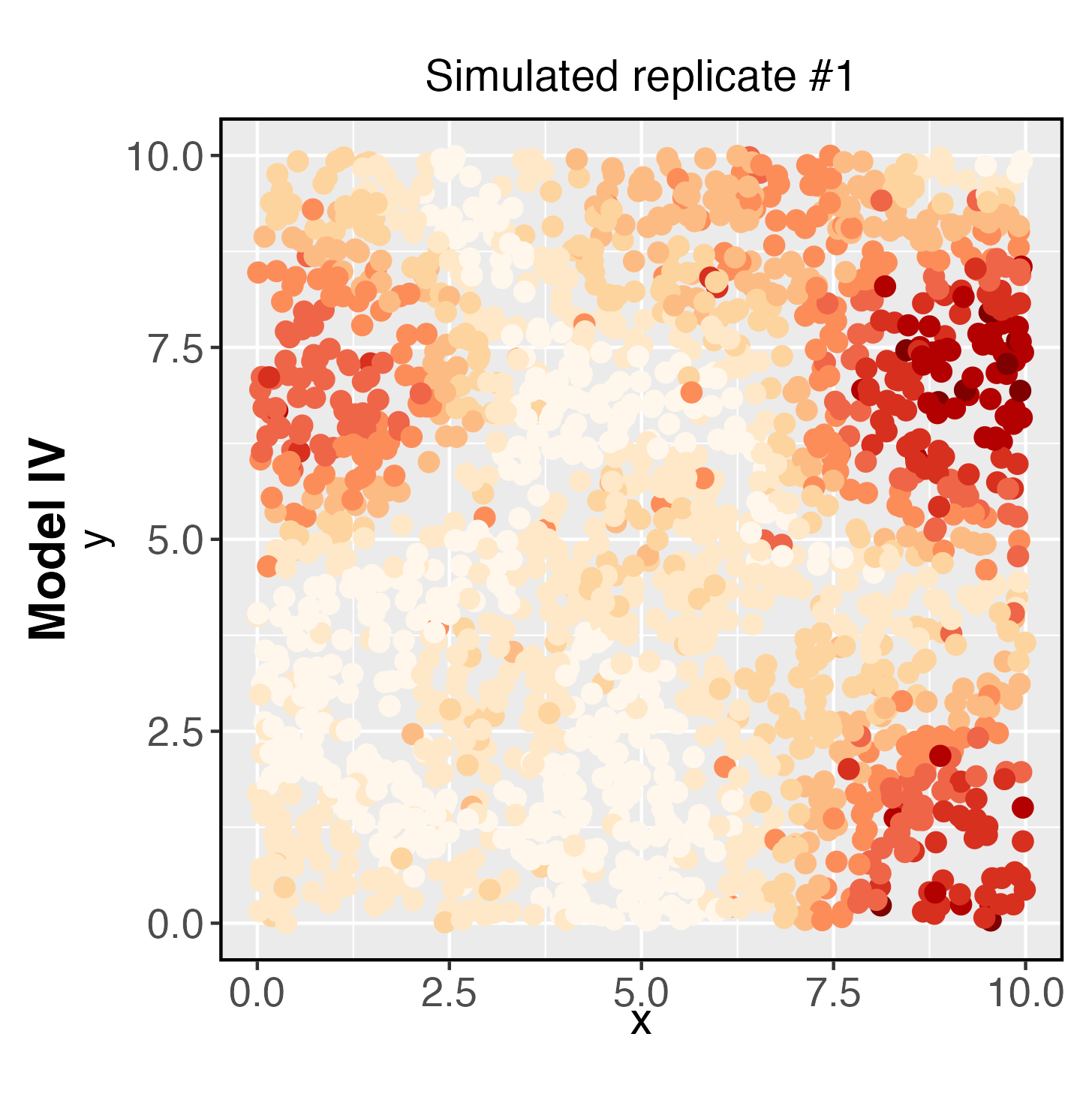}
    \includegraphics[height=0.22\linewidth, trim={0 1.6cm 0 1.12cm}, clip]{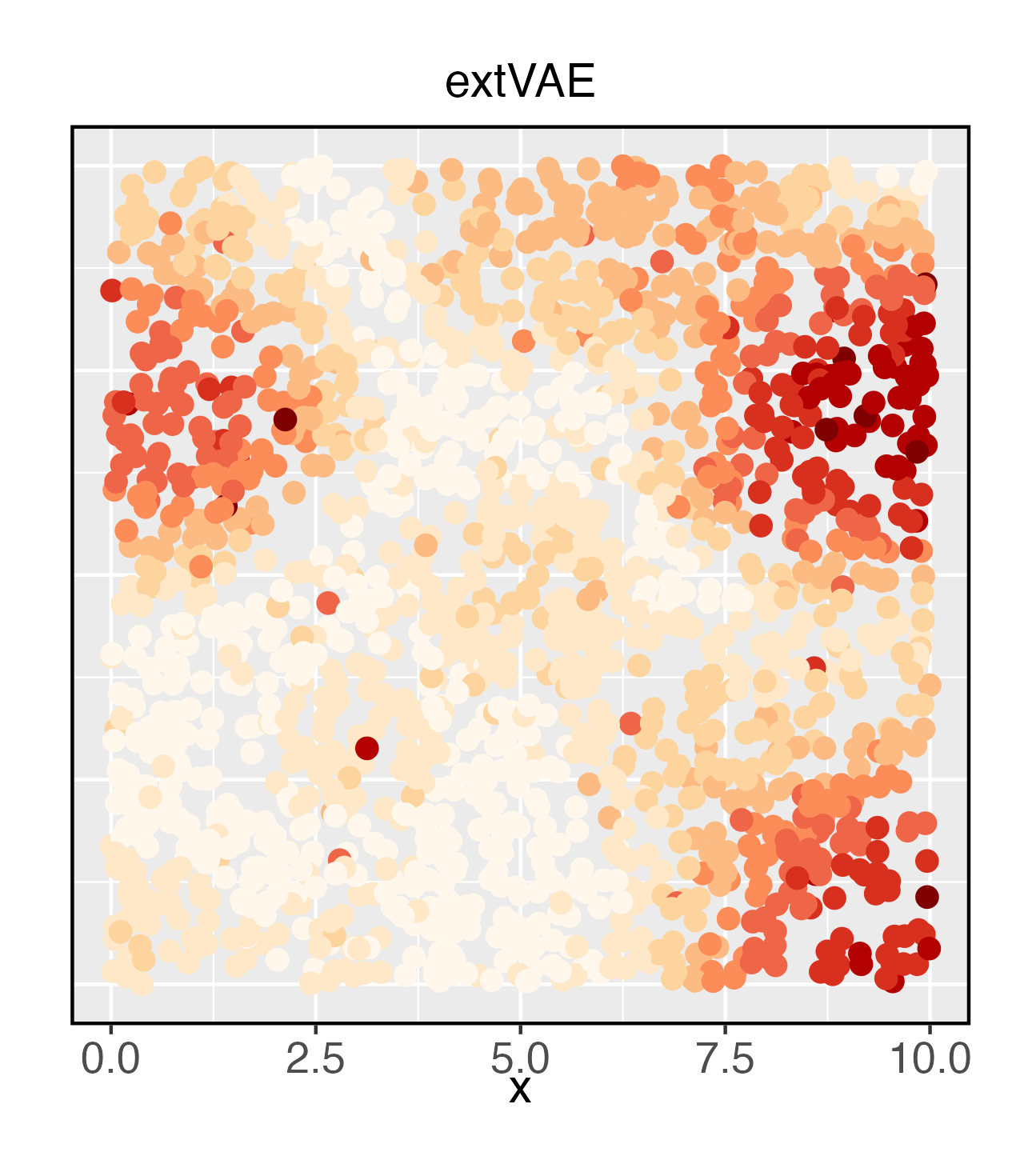}
    \includegraphics[height=0.22\linewidth, trim={0 1.6cm 0 1.12cm}, clip]{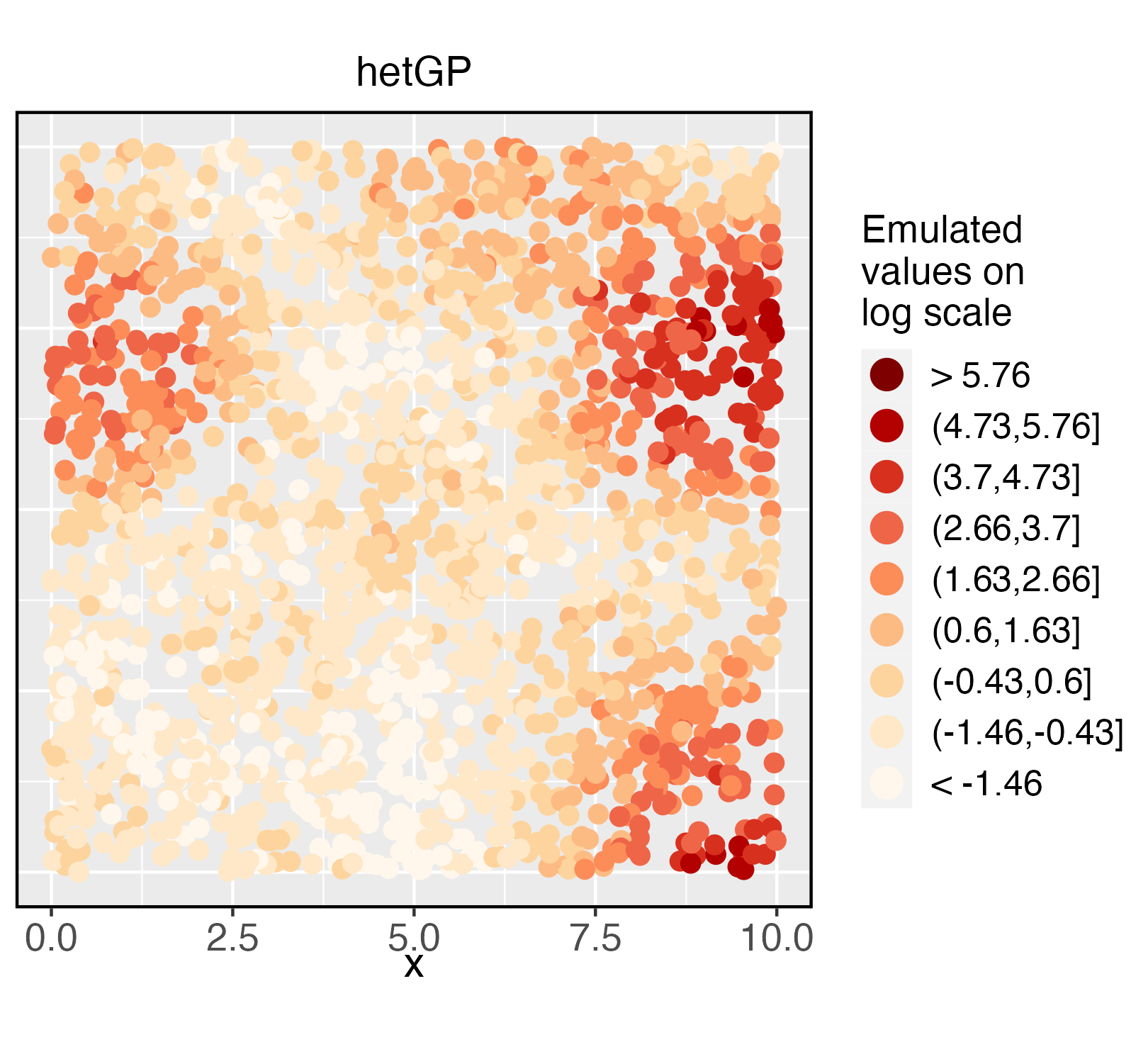}
    \vskip 0.3cm
    
    \includegraphics[height=0.256\linewidth, trim={0 0 0 1.12cm}, clip]{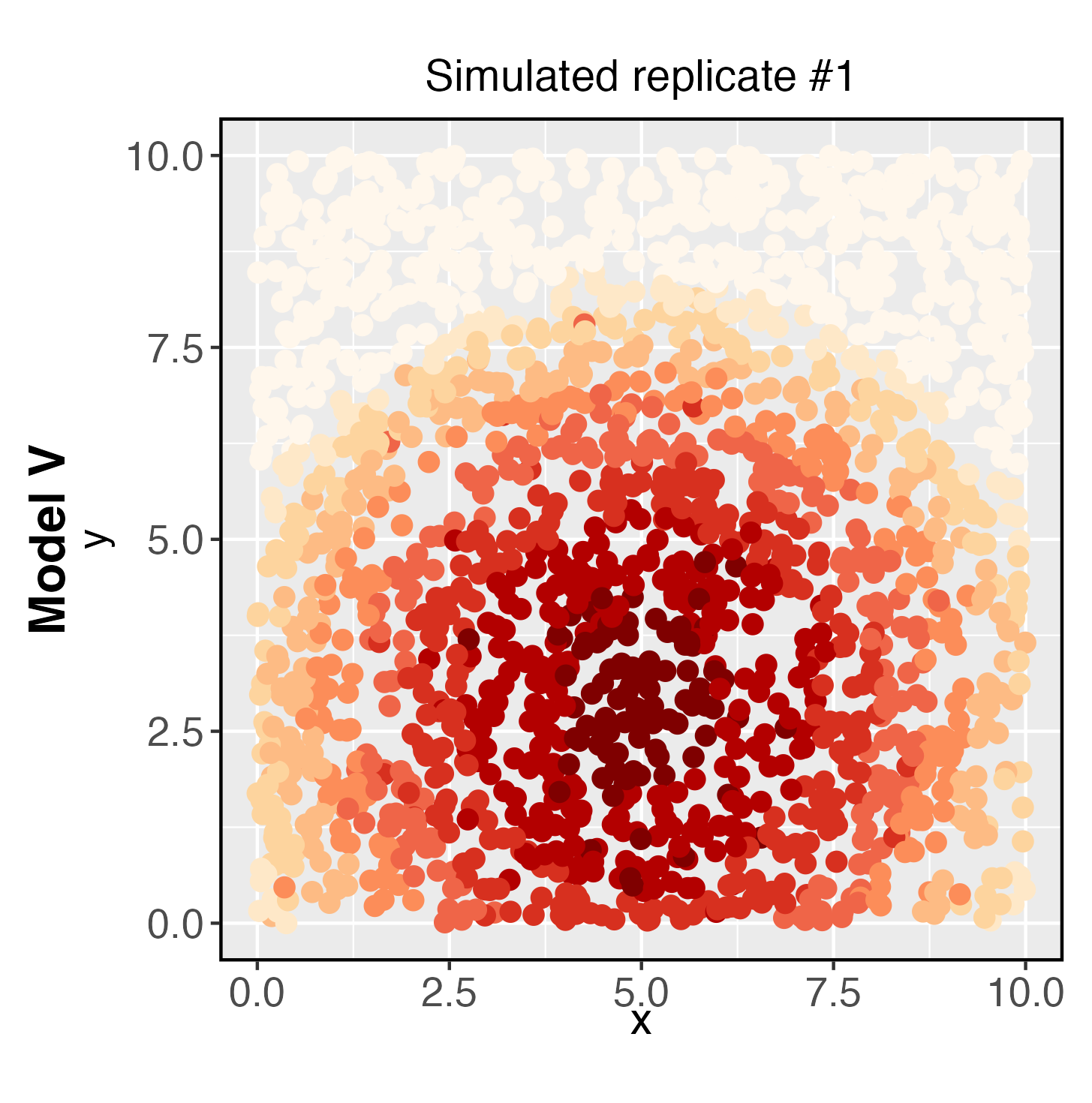}
    \includegraphics[height=0.256\linewidth, trim={0 0 0 1.12cm}, clip]{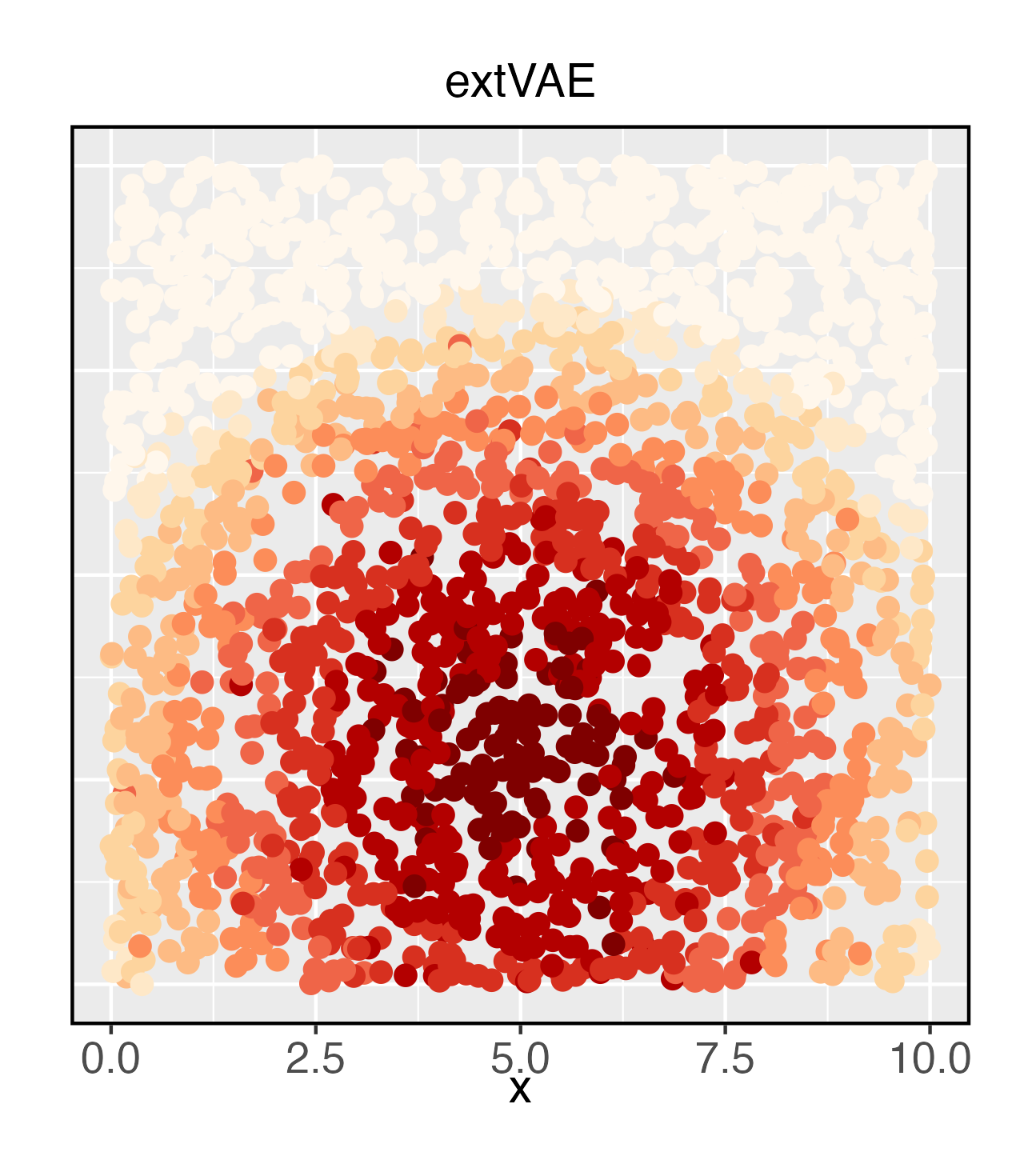}
    \includegraphics[height=0.256\linewidth, trim={0 0 0 1.12cm}, clip]{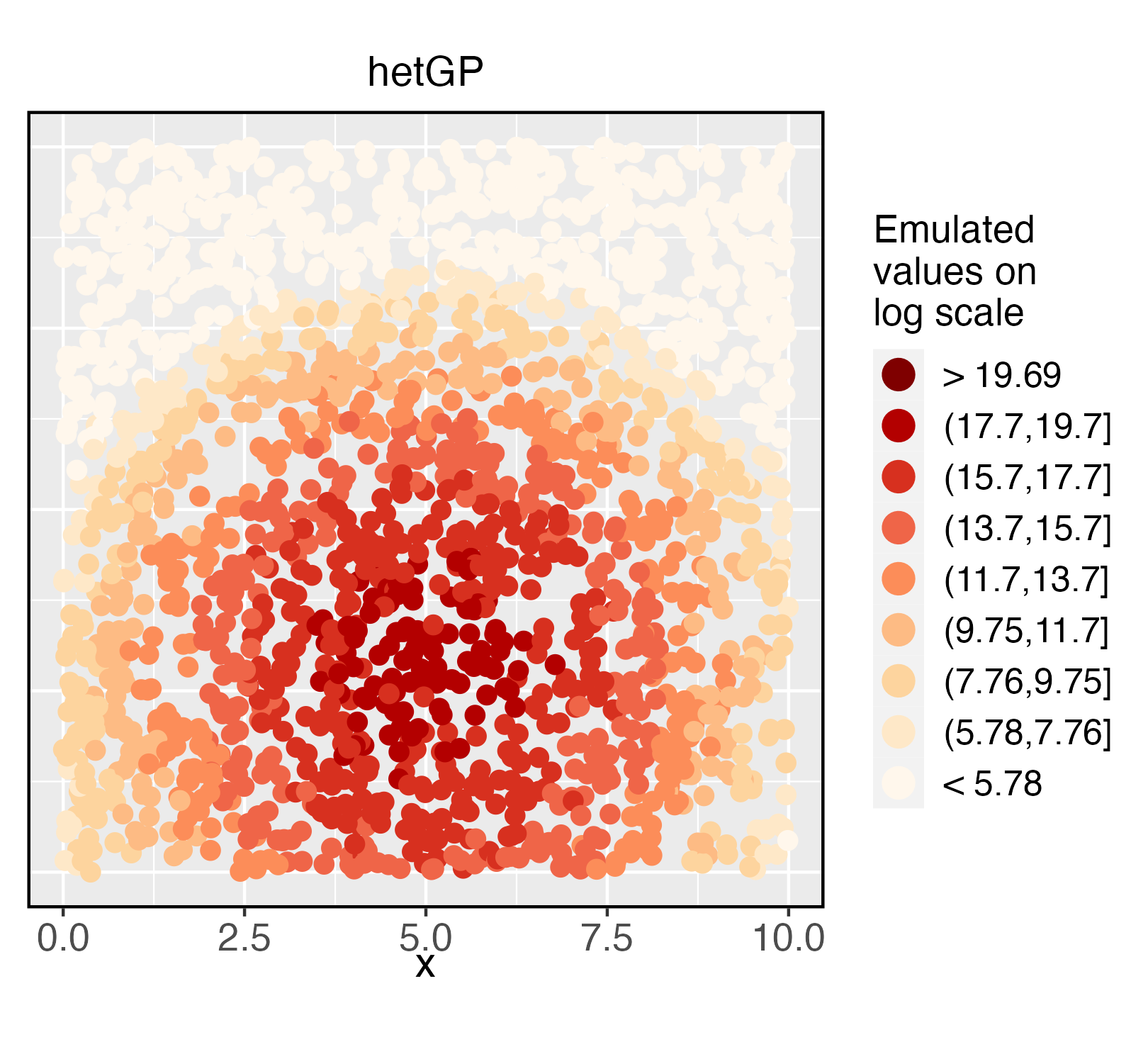}
    \vskip -0.5cm
    \caption{Simulated data sets (left column) and emulated fields (XVAE, middle column; \texttt{hetGP}, right column) from Models~\ref{modelGP}, \ref{modelAI}, \ref{modelAD} and \ref{modelMaxStable} (top to bottom). In all cases, we use data-driven knots for emulation using XVAE. See Figure~\ref{fig:comps_across_models} of the main paper for comparison for Model \ref{modelFlex}.}
    \label{fig:comps_across_models1}
\end{figure}

\begin{figure}
    \centering
    \includegraphics[width=0.27\linewidth]{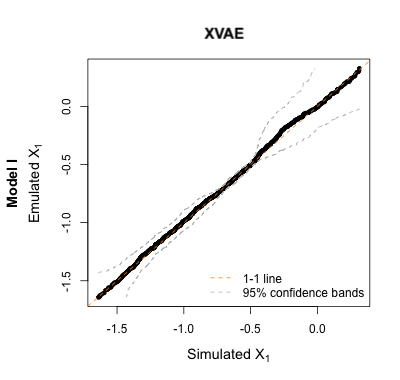}
    \includegraphics[width=0.27\linewidth]{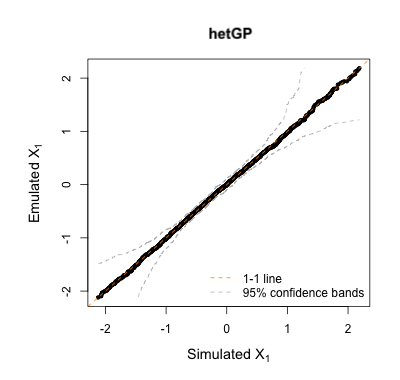}

   \includegraphics[width=0.27\linewidth, trim={0 0 0 2cm}, clip]{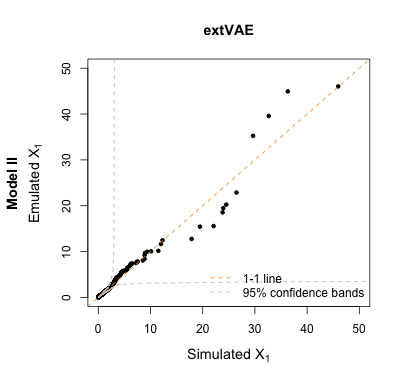}
    \includegraphics[width=0.27\linewidth, trim={0 0 0 2cm}, clip]{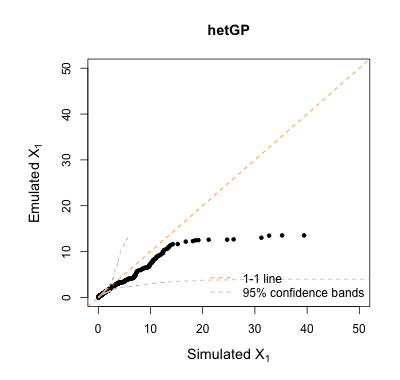}

    \includegraphics[width=0.27\linewidth, trim={0 0 0 2cm}, clip]{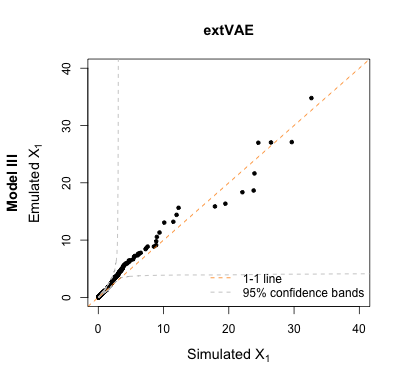}
    \includegraphics[width=0.27\linewidth, trim={0 0 0 2cm}, clip]{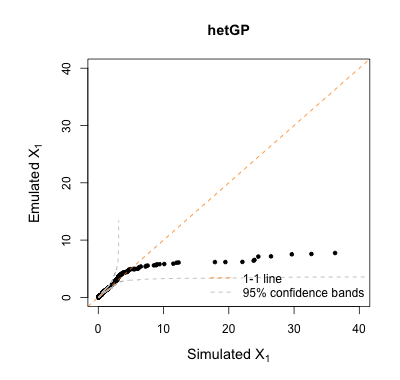}

    \includegraphics[width=0.27\linewidth, trim={0 0 0 2cm}, clip]{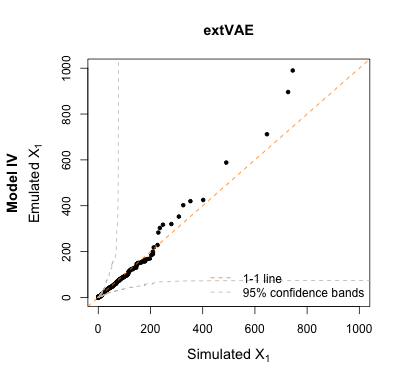}
    \includegraphics[width=0.27\linewidth, trim={0 0 0 2cm}, clip]{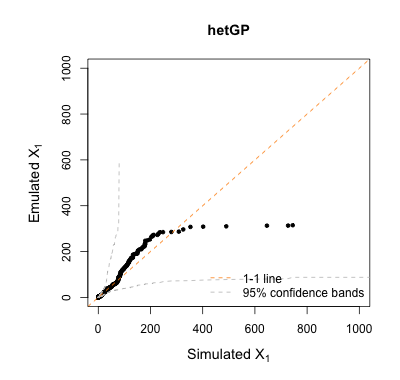}

    \includegraphics[width=0.27\linewidth, trim={0 0 0 2cm}, clip]{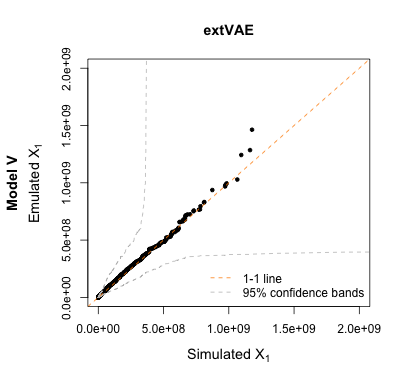}
    \includegraphics[width=0.27\linewidth, trim={0 0 0 2cm}, clip]{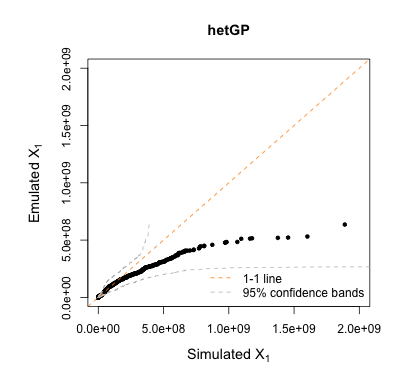}
    \caption{QQ-plots comparing simulated data sets and emulated fields from XVAE (left), and \texttt{hetGP} (right) based on Models~\ref{modelGP}--\ref{modelMaxStable} (top to bottom).}
    \label{fig:qqplots_across_models}
\end{figure}

\begin{figure}
    \centering
    \includegraphics[height=0.251\linewidth, trim={0 0.8cm 0 0}, clip]{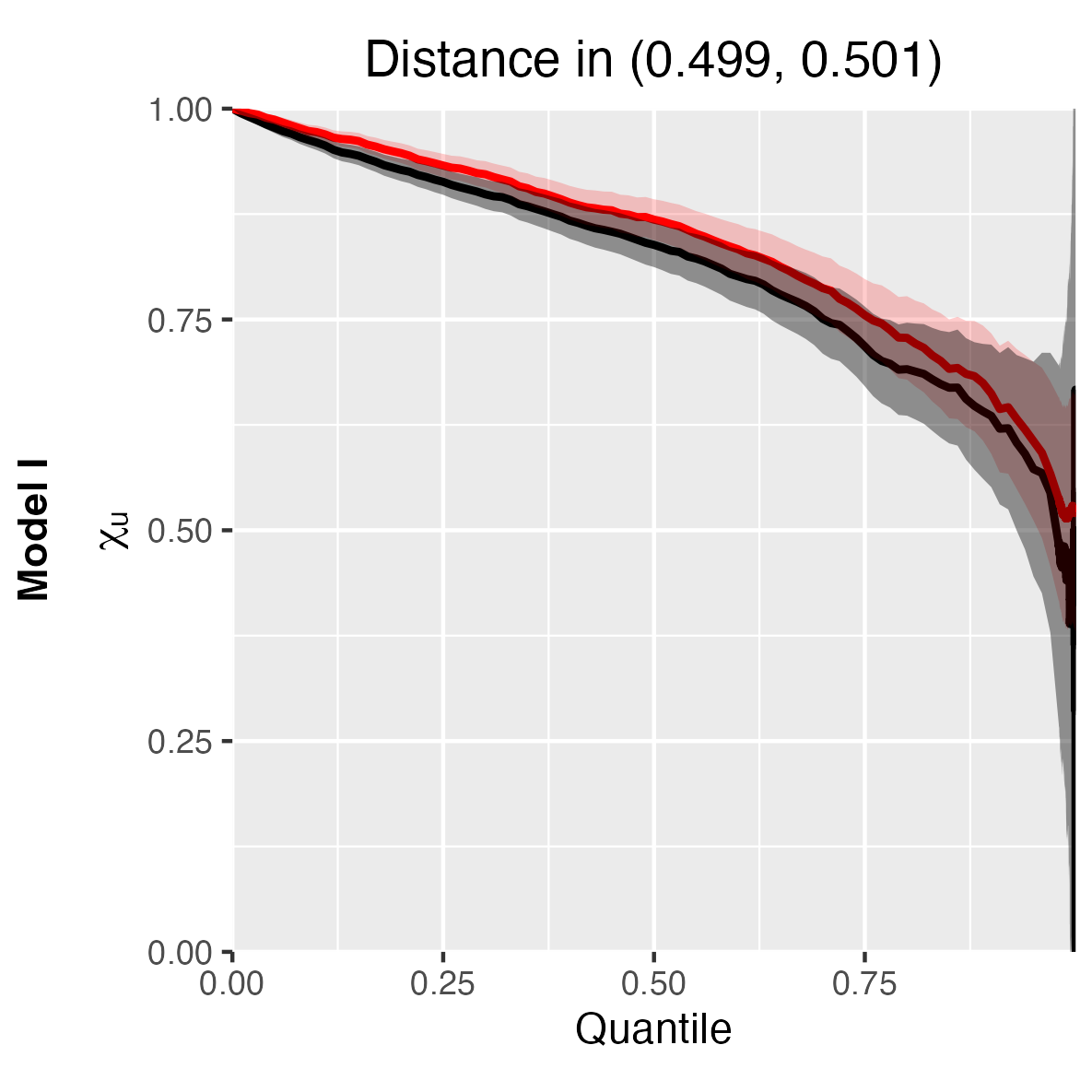}
    \includegraphics[height=0.251\linewidth, trim={0 0.8cm 0 0}, clip]{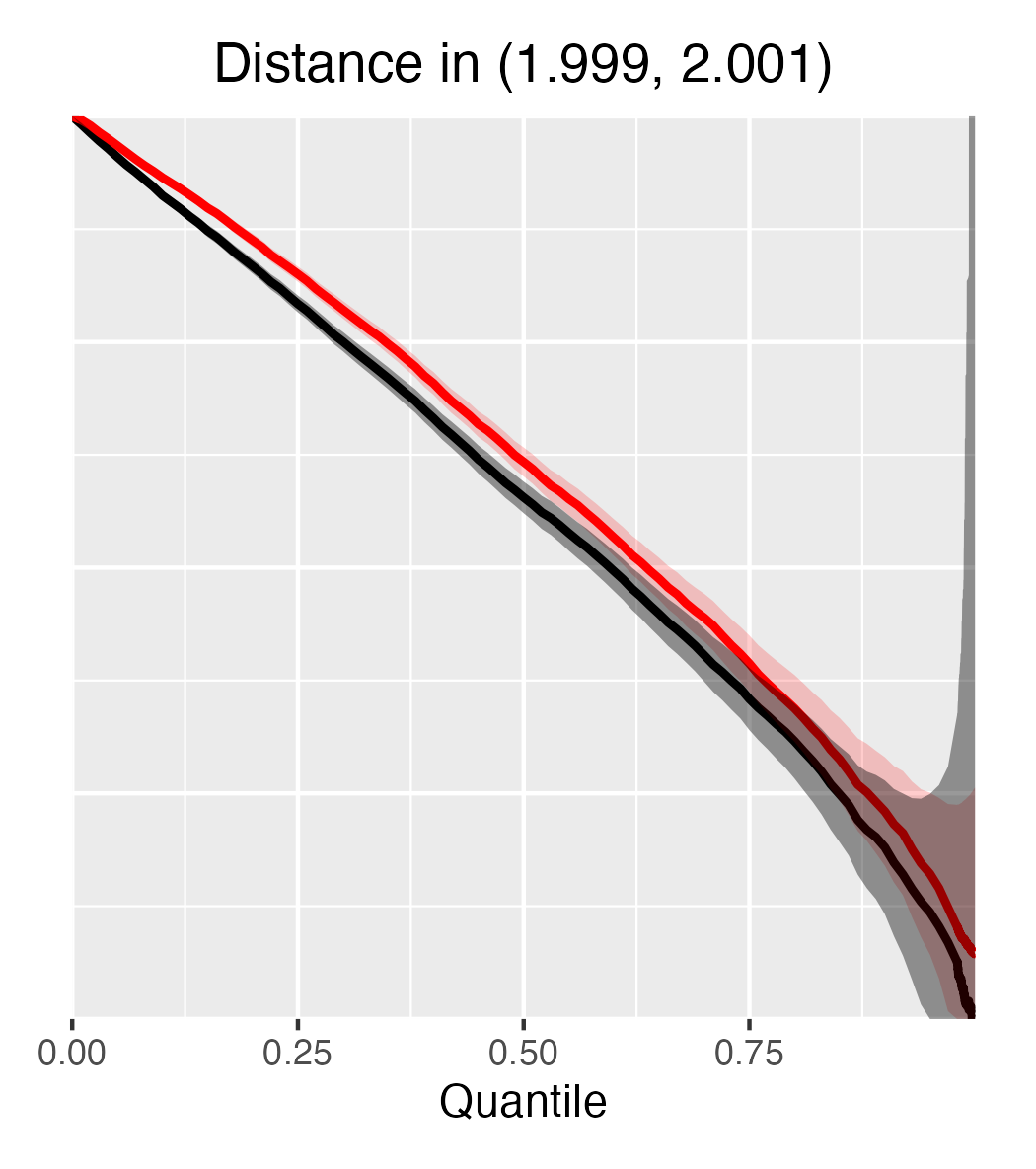}
    \includegraphics[height=0.251\linewidth, trim={0 0.8cm 0 0}, clip]{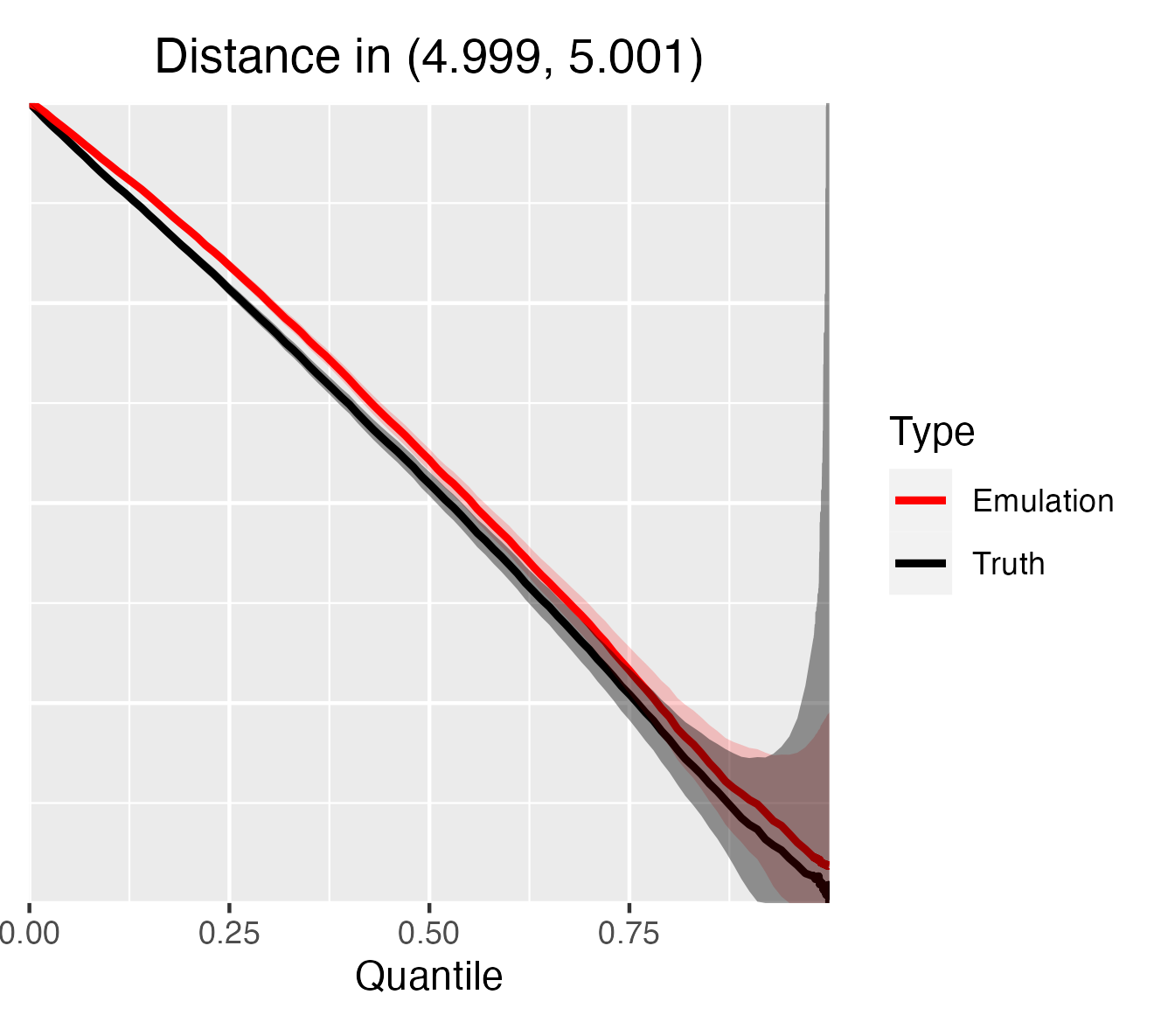}
    \vskip 0.15cm
    
    \includegraphics[height=0.23\linewidth, trim={0 0.8cm 0 0.8cm}, clip]{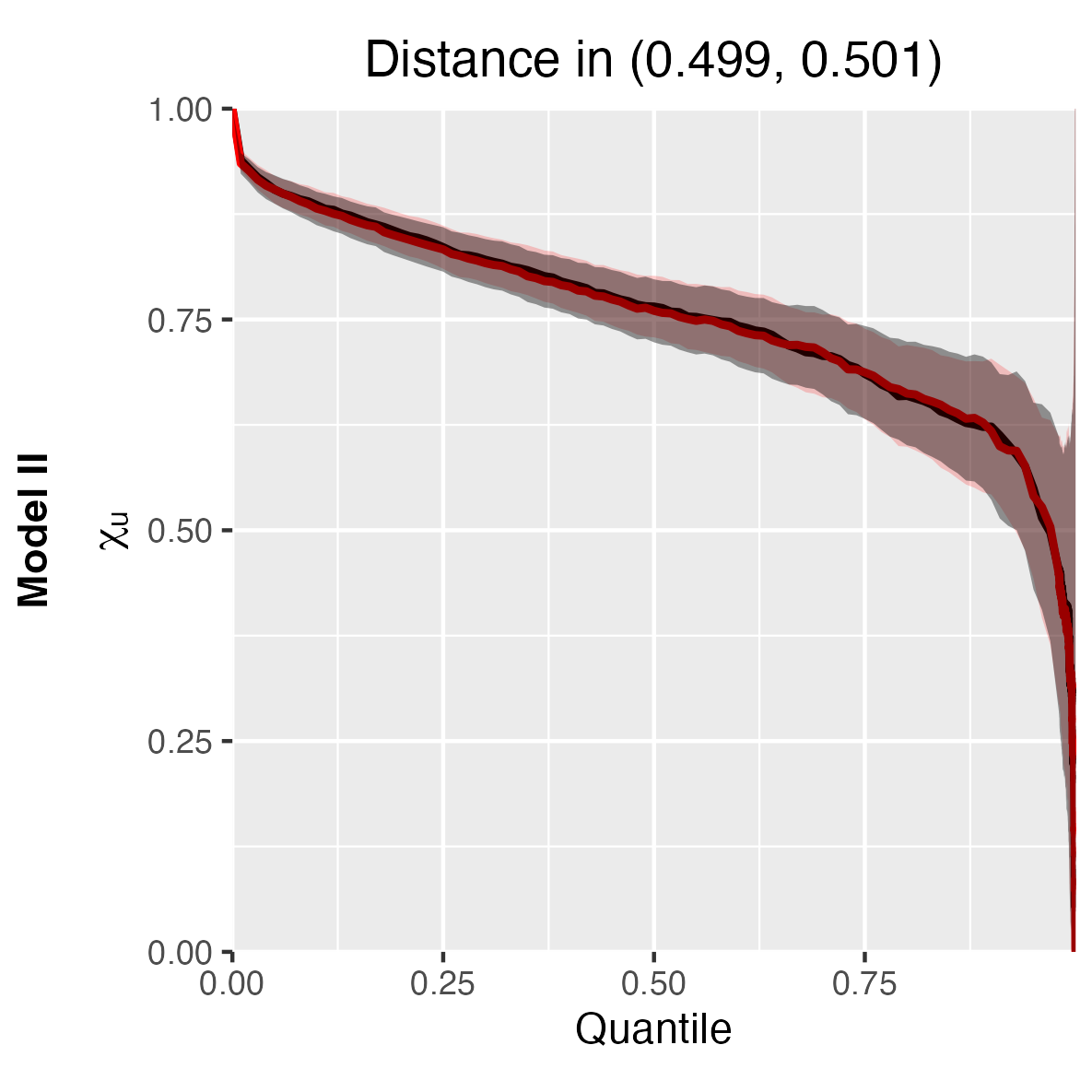}
    \includegraphics[height=0.23\linewidth, trim={0 0.8cm 0 0.8cm}, clip]{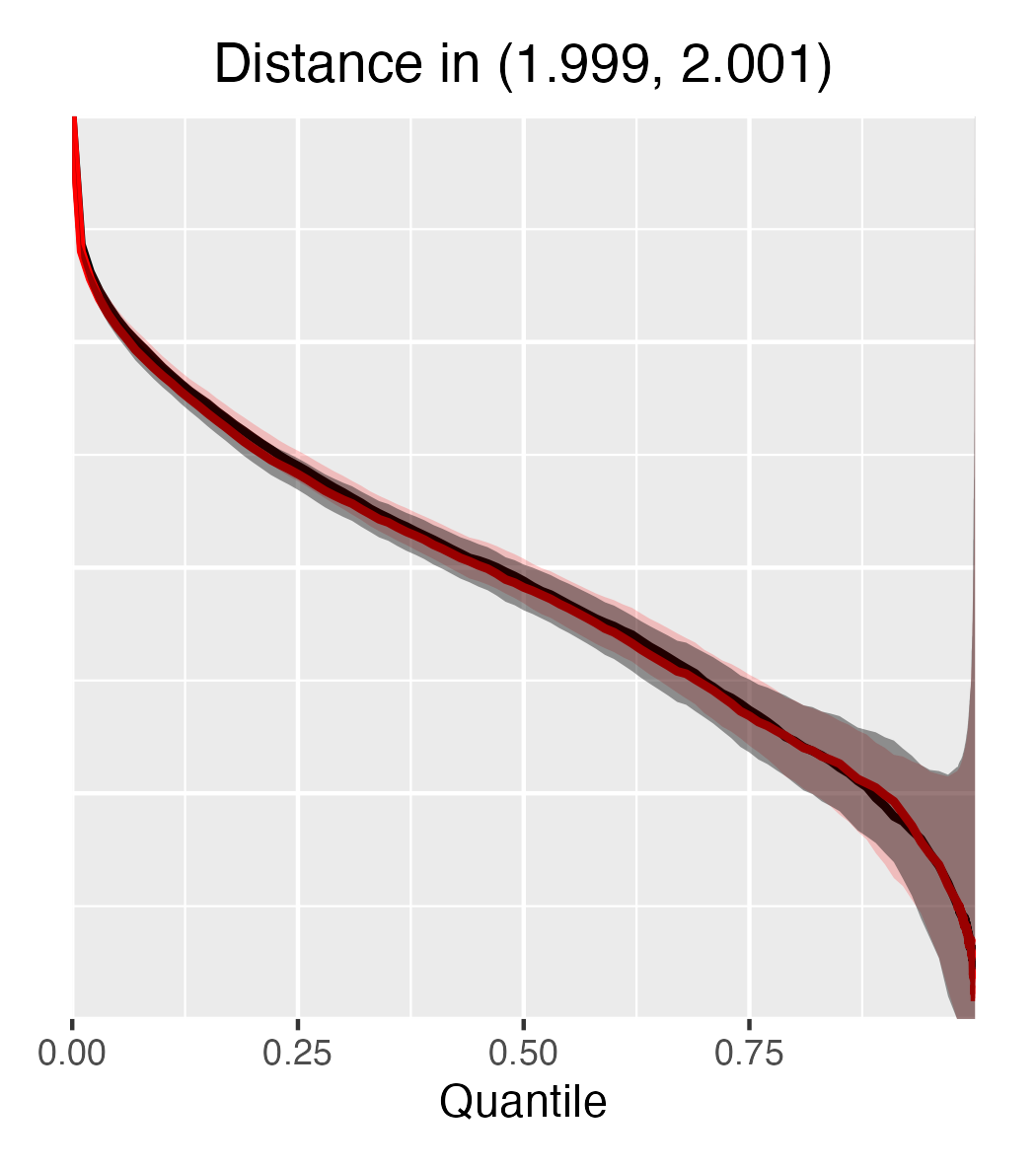}
    \includegraphics[height=0.23\linewidth, trim={0 0.8cm 0 0.8cm}, clip]{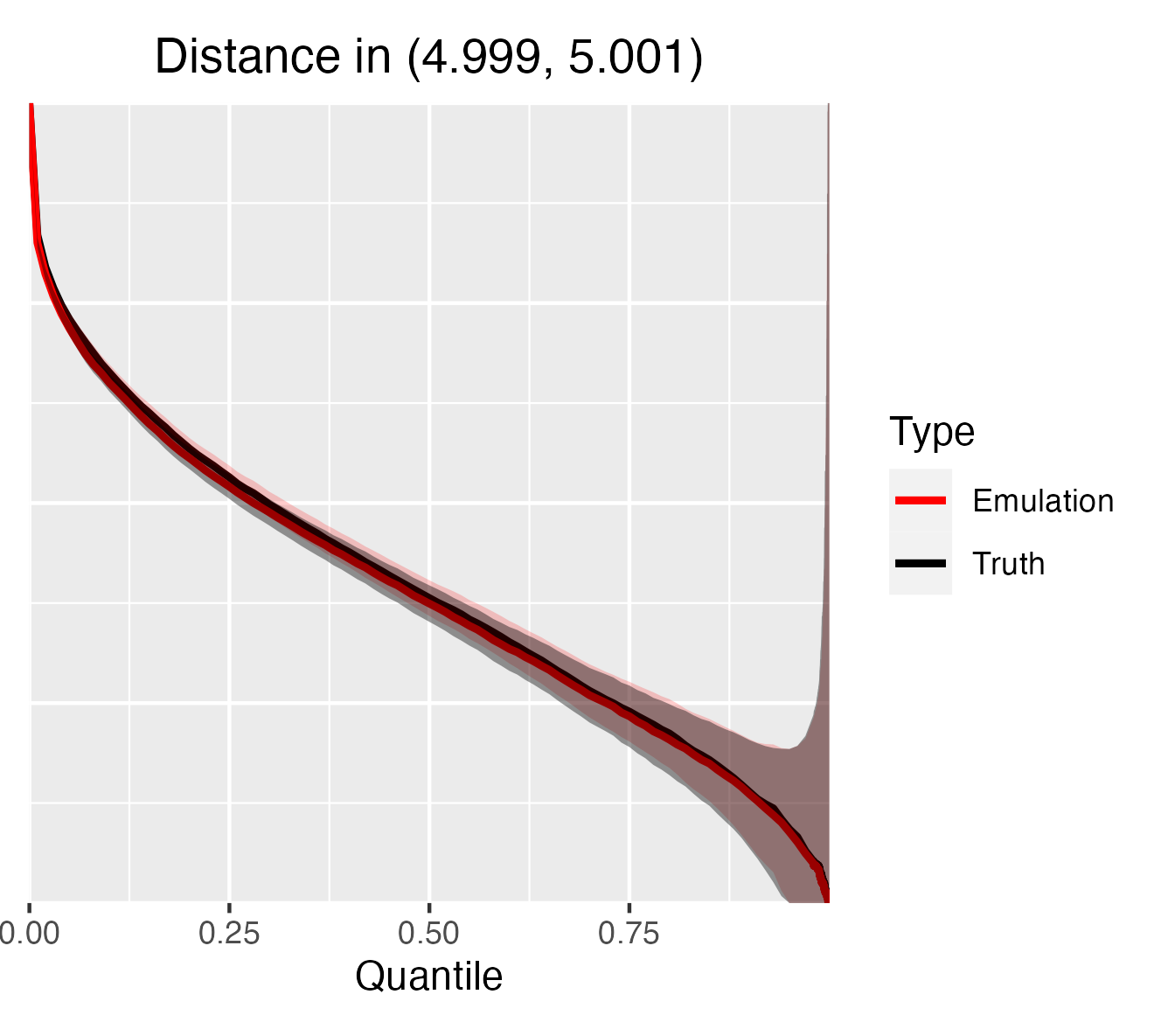}
    \vskip 0.15cm

    \includegraphics[height=0.23\linewidth, trim={0 0.8cm 0 0.8cm}, clip]{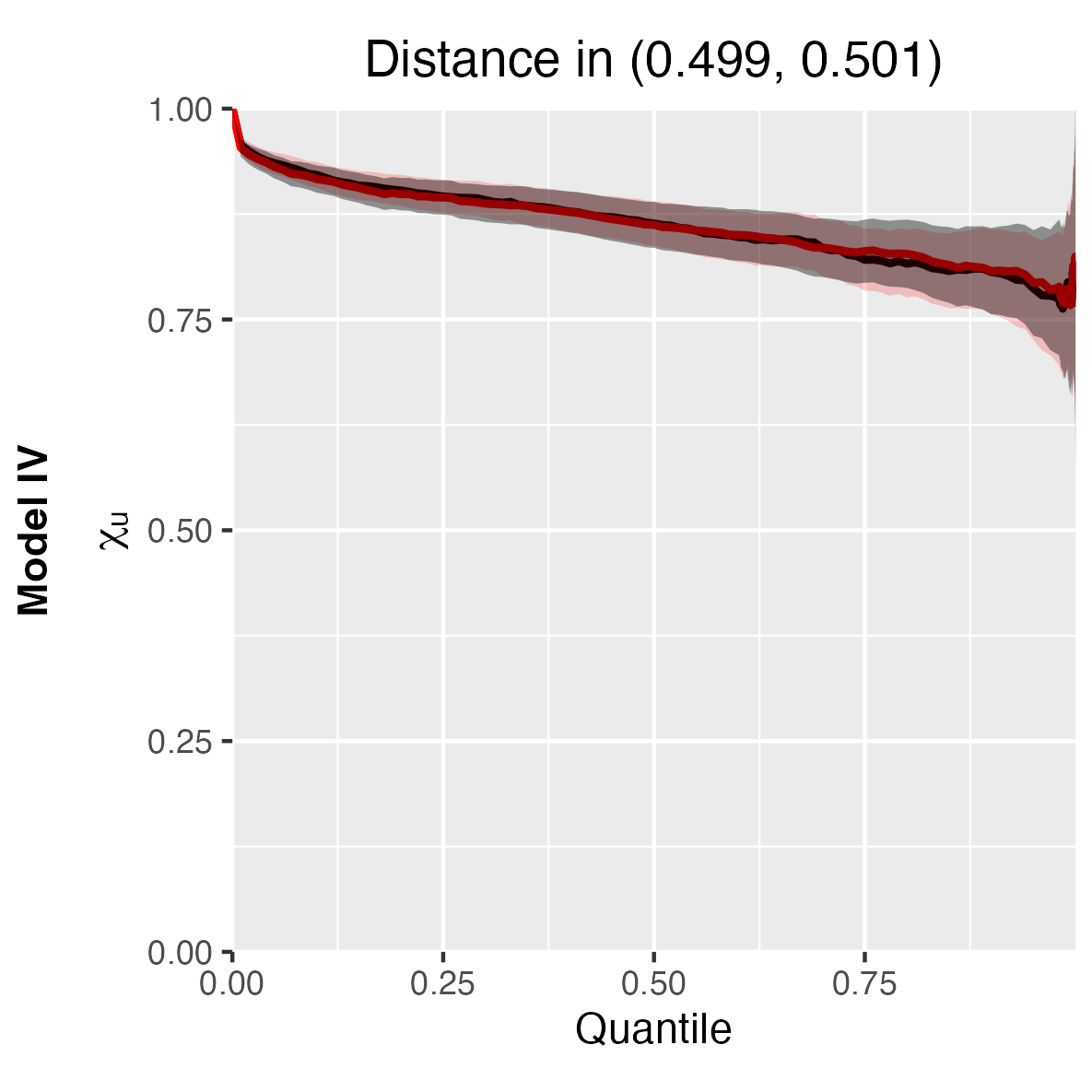}
    \includegraphics[height=0.23\linewidth, trim={0 0.8cm 0 0.8cm}, clip]{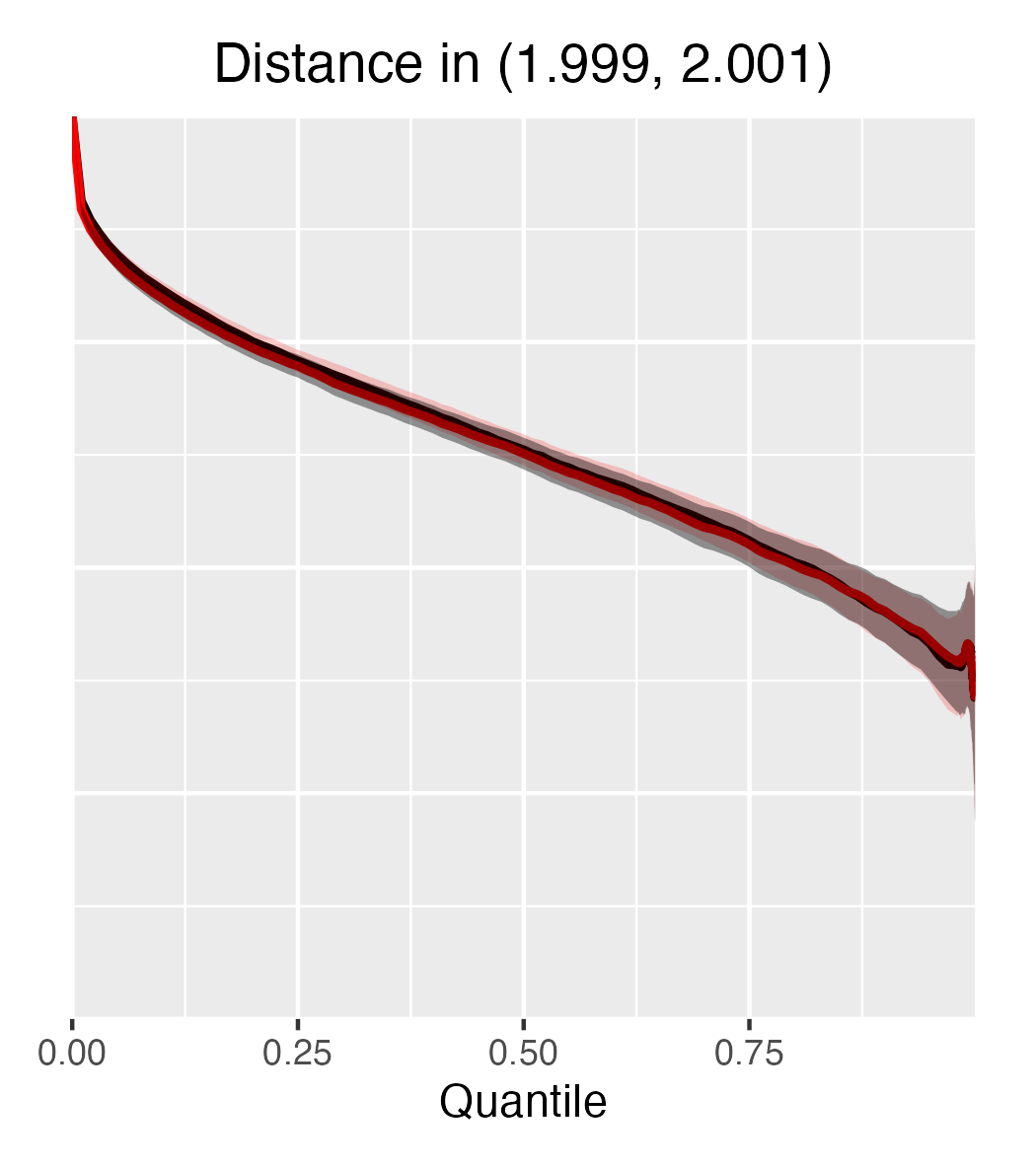}
    \includegraphics[height=0.23\linewidth, trim={0 0.8cm 0 0.8cm}, clip]{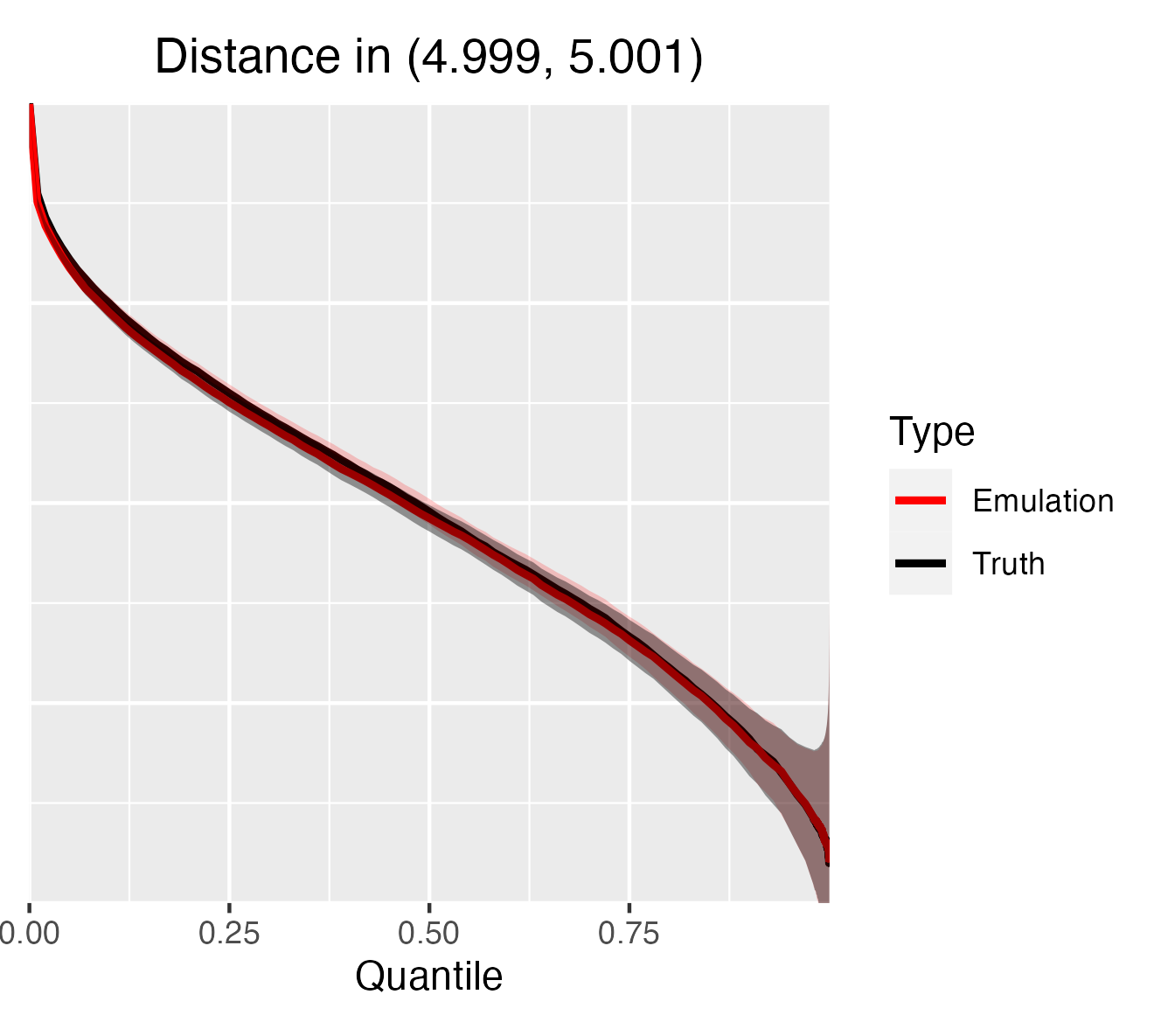}
    \vskip 0.15cm
    
    \includegraphics[height=0.253\linewidth, trim={0 0 0 0.8cm}, clip]{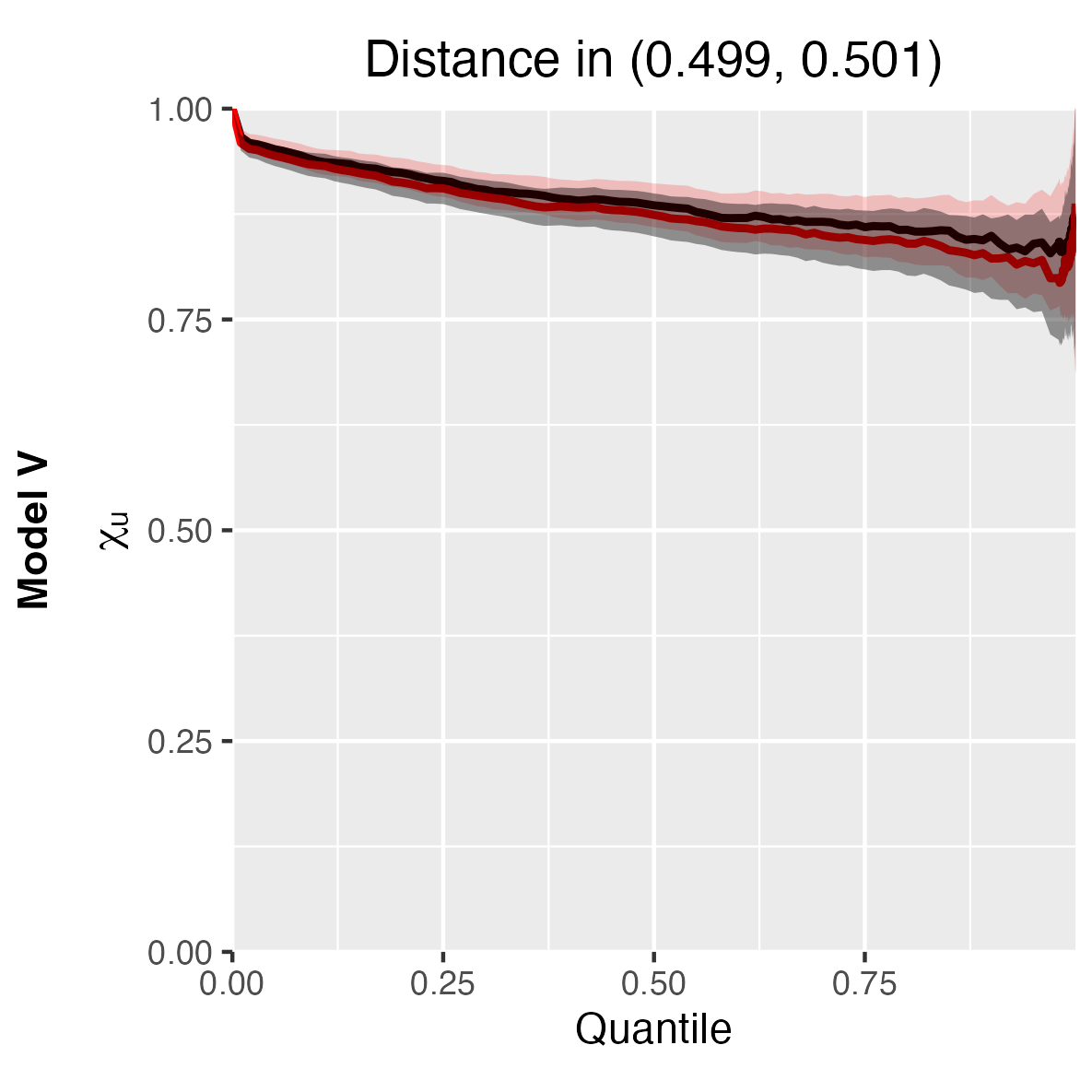}
    \includegraphics[height=0.253\linewidth, trim={0 0 0 0.8cm}, clip]{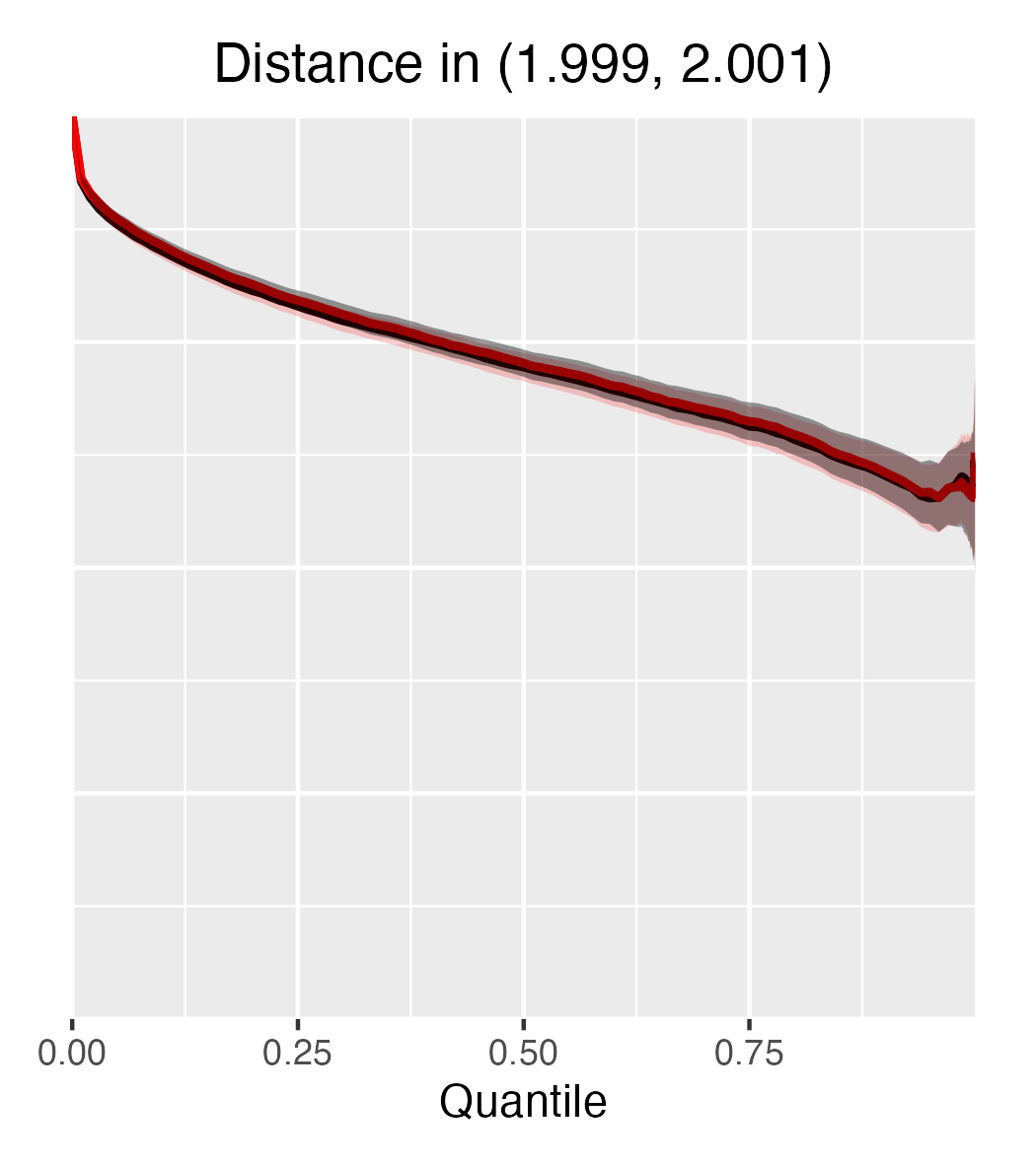}
    \includegraphics[height=0.253\linewidth, trim={0 0 0 0.8cm}, clip]{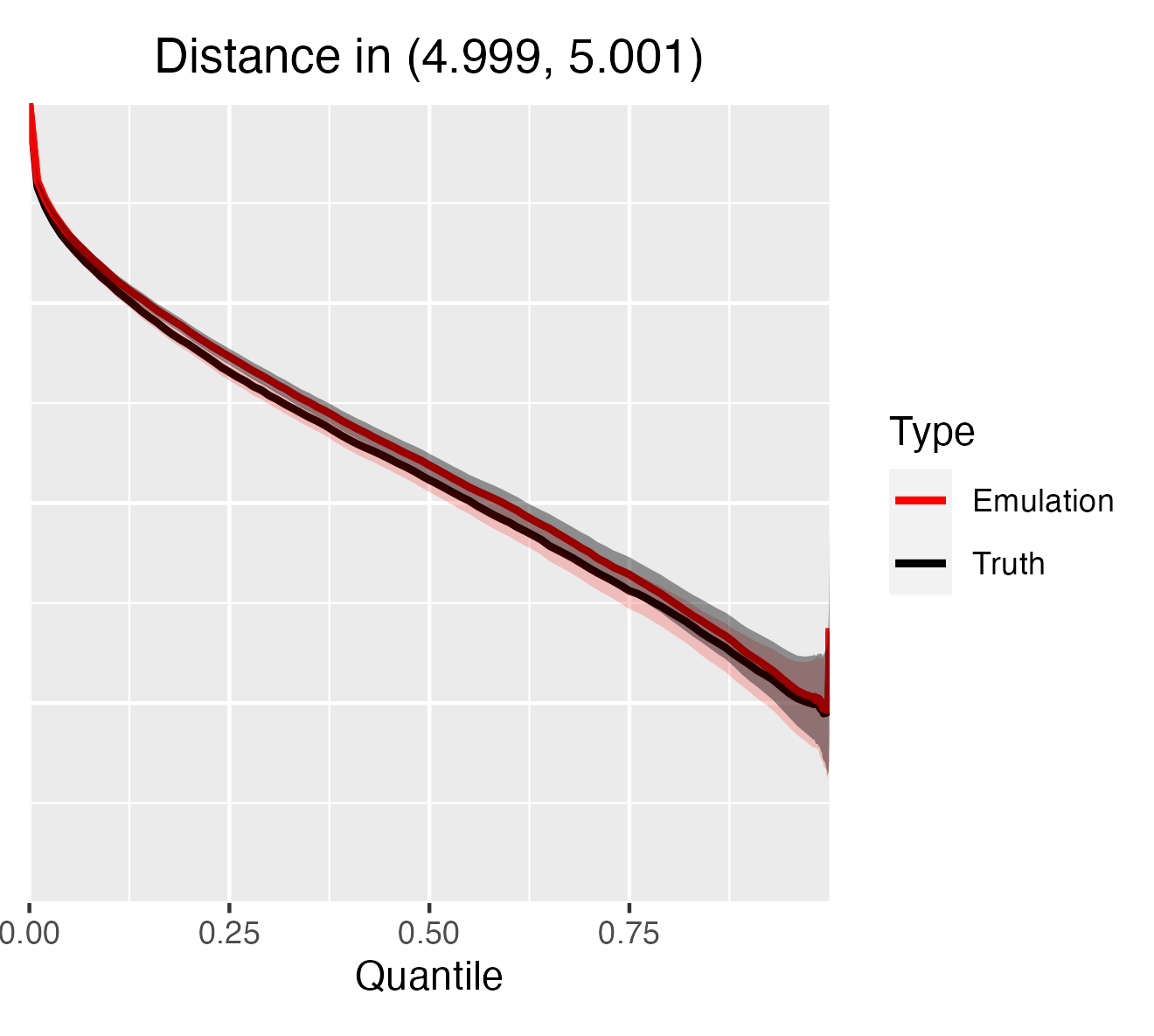}
    \vskip -0.3cm
    \caption{The empirically-estimated tail dependence measure $\chi_h(u)$ at $h=0.5$ (left),\;2 (middle),\;5 (right) for Models~\ref{modelGP}, \ref{modelAI}, \ref{modelAD} and \ref{modelMaxStable} (top to bottom), based on simulated (black) and XVAE emulated (red) data. See Figure~\ref{fig:chi_ests} of the main paper for $\chi_h(u)$ estimates for Model \ref{modelFlex}.}
    \label{fig:chi_ests2}
\end{figure}

\begin{figure}
    \centering
    \hspace*{-0.1cm}\includegraphics[height=0.28\linewidth]{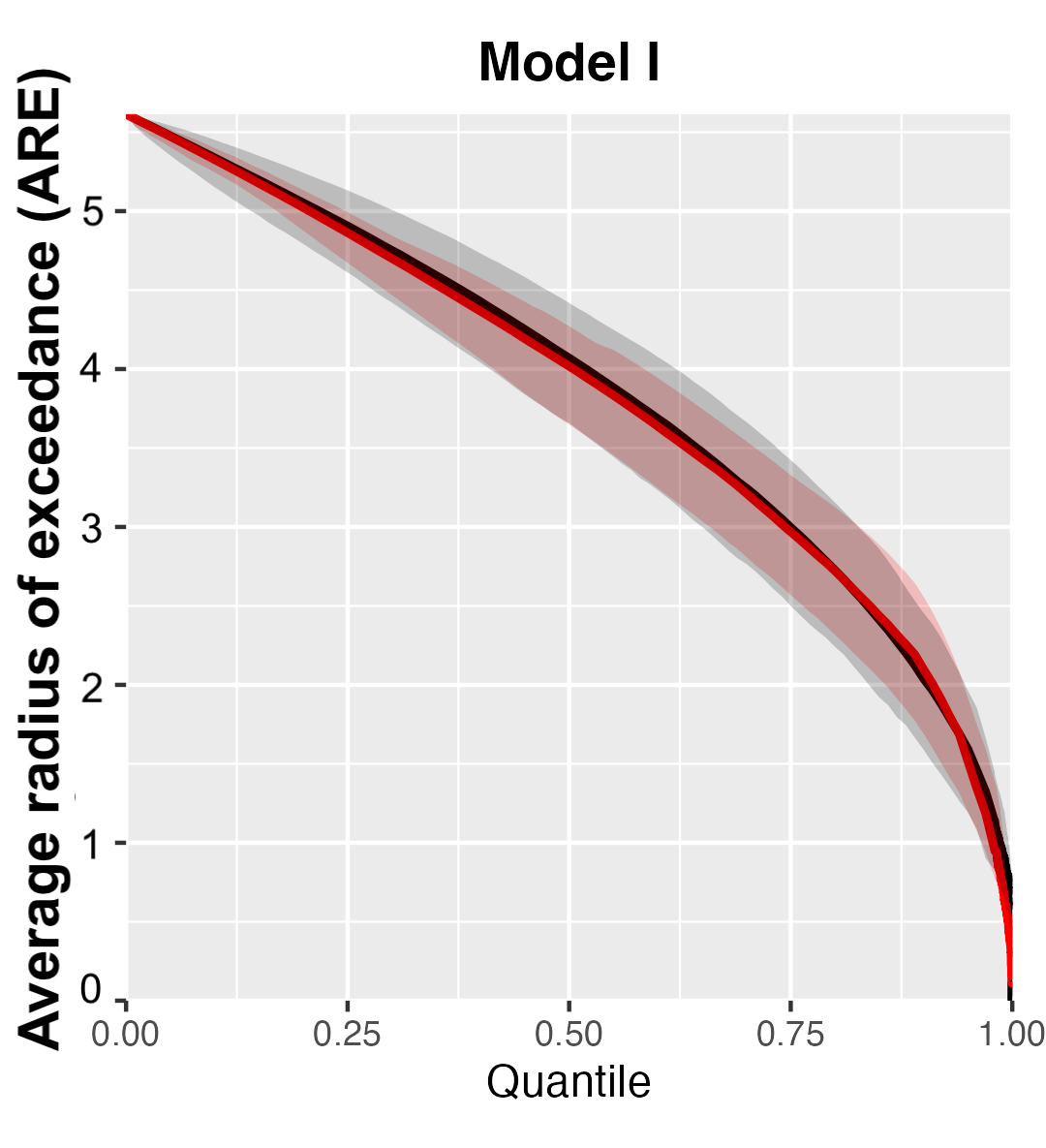}
    \hspace*{-0.35cm}\includegraphics[height=0.28\linewidth]{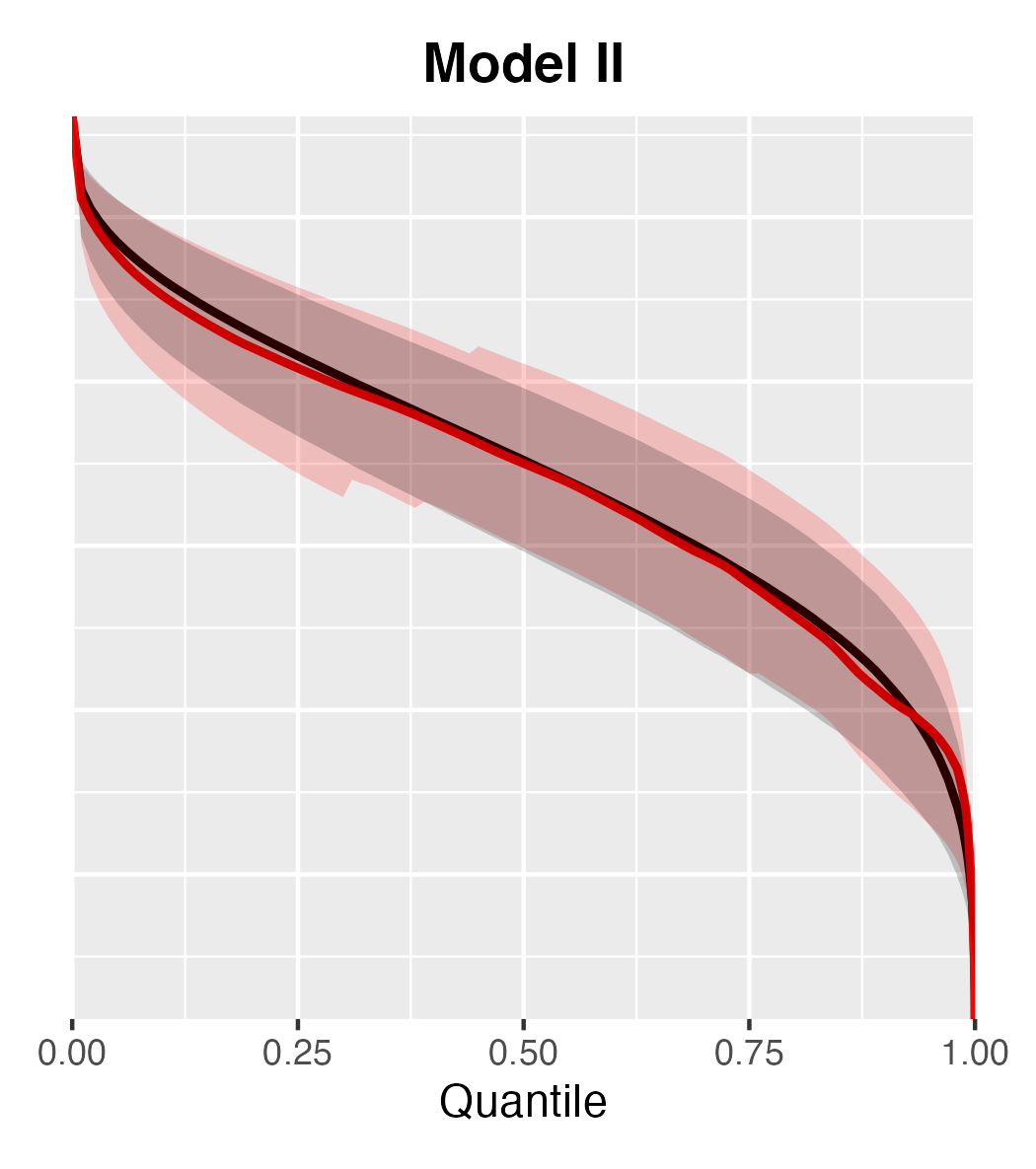}
    \hspace*{-0.15cm}\includegraphics[height=0.28\linewidth, trim={1.1cm 0 0 0}, clip]{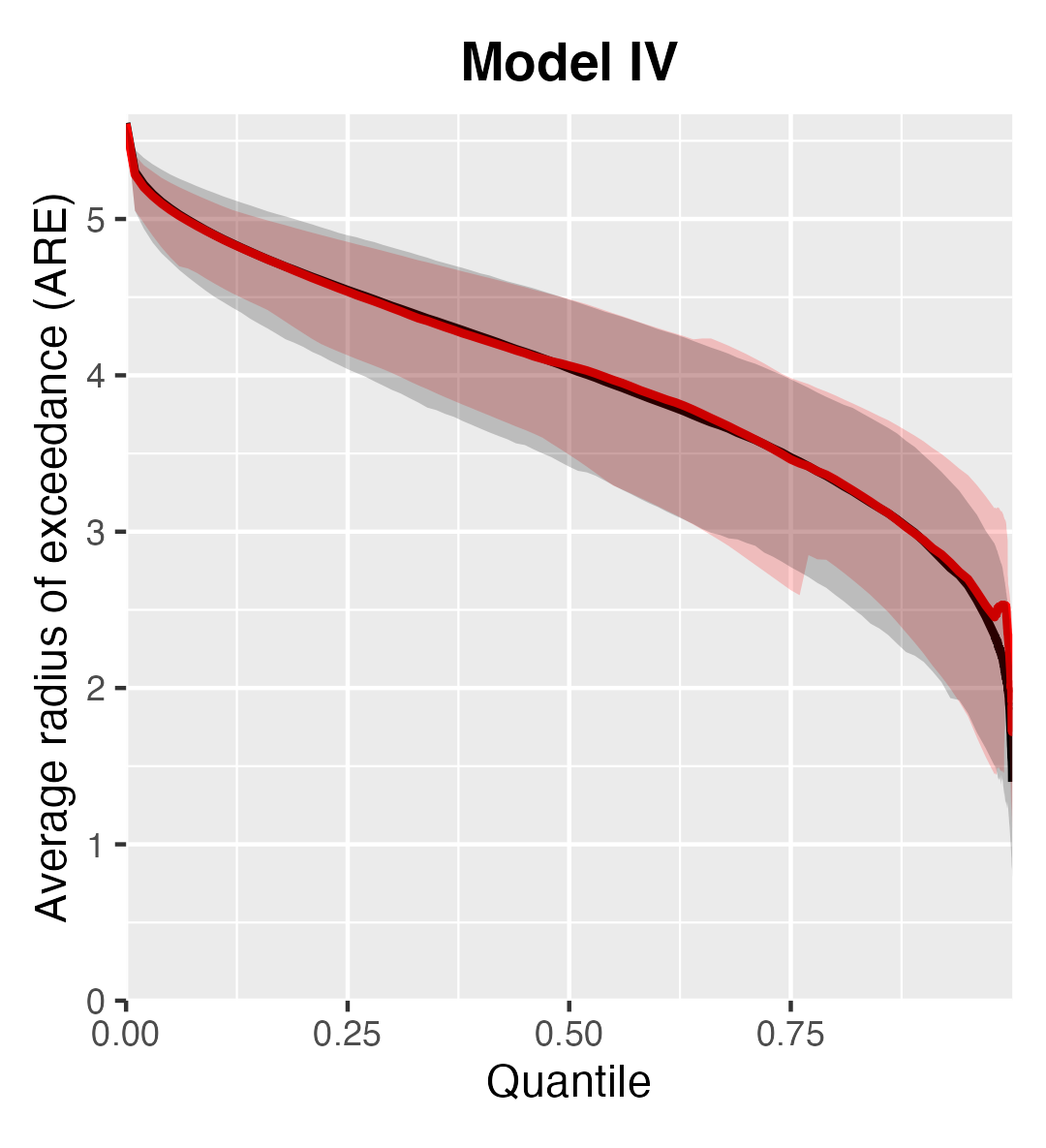}
    \hspace*{-0.2cm}\includegraphics[height=0.28\linewidth, trim={0 0 0.5cm 0}, clip]{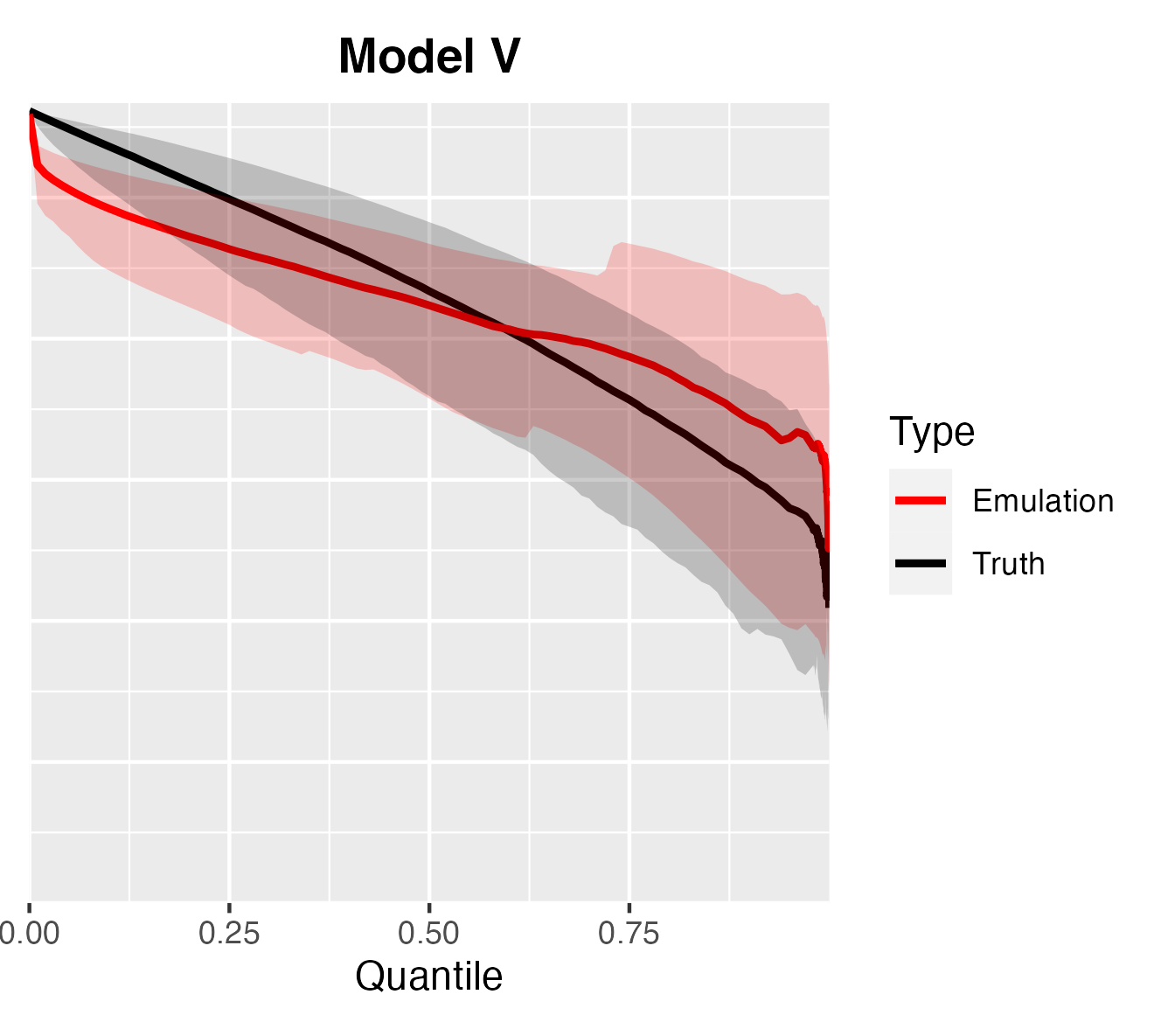}
    \caption{Estimates of $\mathrm{ARE}_\psi(u)$, $\psi=0.05$, for both simulations (black) and XVAE emulations (red) under Models~\ref{modelGP}, \ref{modelAI}, \ref{modelAD} and \ref{modelMaxStable} (left to right). See the right panel of Figure~\ref{fig:chi_ests} of the main paper for $\mathrm{ARE}_\psi(u)$ estimates for Model \ref{modelFlex}.}
    \label{fig:ARE_comps2}
\end{figure}

\begin{figure}
    \centering
    \includegraphics[width=0.4\linewidth]{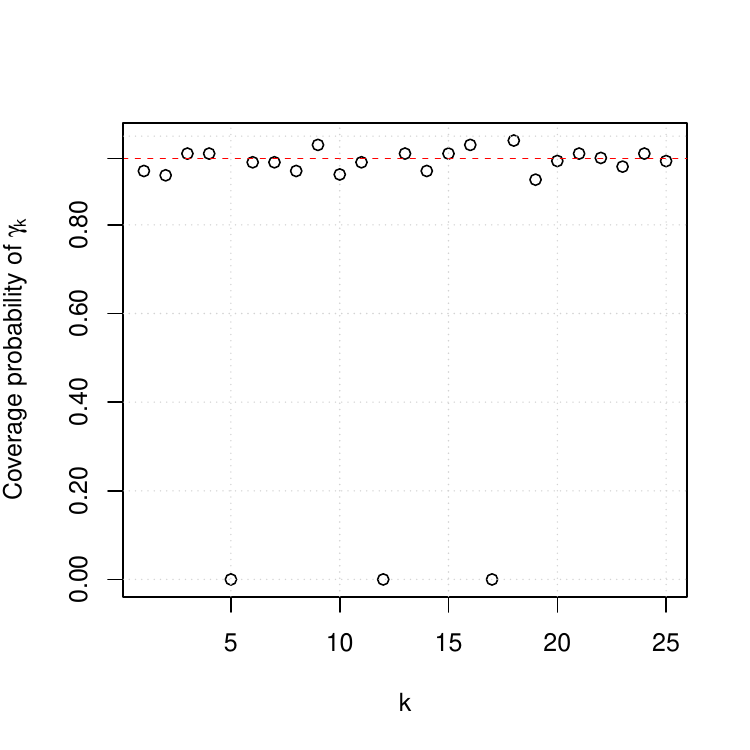}
    \caption{Coverage probabilities for each of the parameters $\gamma_k$, $k=1,\ldots,K=25$, from emulating 100 simulated data sets of Model \ref{modelFlex}, in which $n_s=2,000$ and $n_t=100$. The nominal levels of the credible intervals are $0.95$ (red dashed line). Zero probabilities correspond to $\gamma_k=0$, $k=5,12,17$.}
    \label{fig:coverage}
\end{figure}

\subsection{Nonstationary space-time dependence setting}\label{sec:nonstat_sim}
Here, we simulate data based on the model setting~\ref{modelFlex}, but we additionally impose a single linear time trend to all knots. We generate 100 time points ($n_t = 100$) at the same 2,000 spatial locations ($n_s = 2,000$) as described in Section~\ref{sec:sim_setting} of the main paper. To evaluate the model's ability to capture temporal nonstationarity, we train the XVAE on the true knots and use the trained decoder to generate 1,000 samples to estimate $(\alpha_t, \bgamma_t^{\rm T})^{\rm T}$. Figure~\ref{fig:nonstat_sim} presents the median estimates of $\{\gamma_{kt}: k = 1, \dots, K, t = 1, \dots, n_t\}$, averaged over time (left panel) and space (right panel).

The results show that our method effectively captures temporal nonstationarity, though some temporal stochastic fluctuations appear due to the working independence assumption of the max-id model and the natural variability of the estimator. This highlights an area where a conditional VAE could be particularly useful, as it could allow the encoder and decoder parameters to vary with time and other conditioning variables. This flexibility would introduce temporal change in dependence structure through the time-varying tilting parameters and enable the modeling of more complex, nonlinear temporal trends. On the spatial side, our method continues to emulate the variation in the tilting parameters well, although the estimates near the edges of the spatial domain are slightly less accurate. In comparison, \texttt{hetGP} does not naturally handle non-stationarity over space and time.
\begin{figure}
    \centering
    \includegraphics[width=0.4\linewidth]{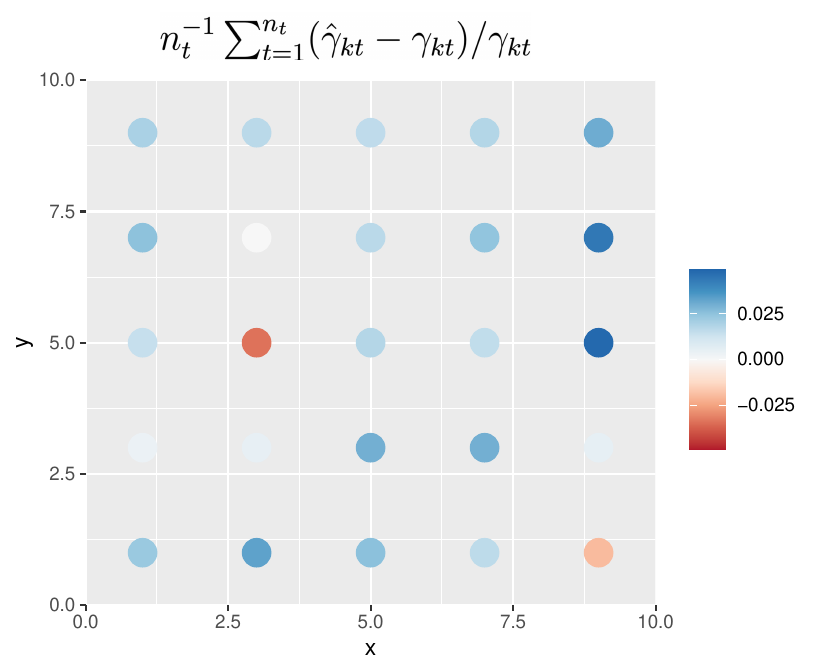}
    \includegraphics[width=0.43\linewidth]{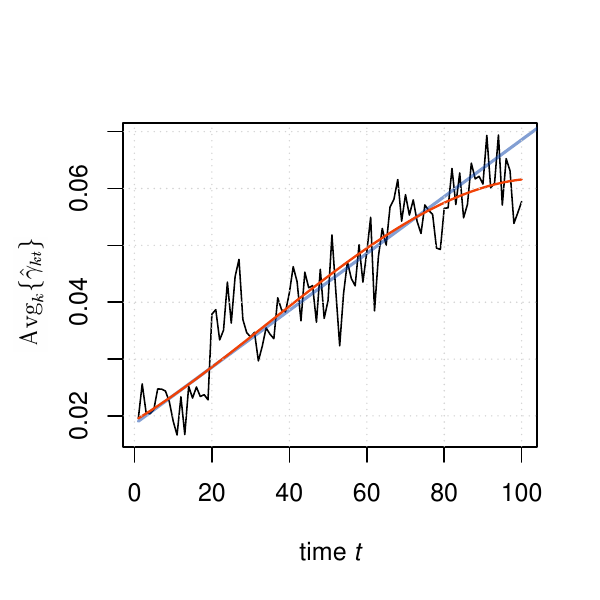}
    \vskip -0.3cm
    \caption{Left: Time-averaged relative differences between the estimated and true tilting parameters at $K=25$ true knots (i.e., $n_t^{-1}\sum_{t=1}^{n_t}(\hat{\gamma}_{kt}-\gamma_{kt})/\gamma_{kt}$, $k=1,\ldots, K$). Right: Estimated tilting parameters averaged over space (i.e., $K^{-1}\sumK\hat{\gamma}_{kt}$, $t=1,\ldots, n_t$), overlaid  with the true trend (blue line) and the best median regression fit (red line).}
    \label{fig:nonstat_sim}
\end{figure}

\subsection{Comparison to extGAN}\label{sec:extGAN_sim}
The extGAN proposed by \citet{boulaguiem2022modeling} uses convolutional neural networks (CNNs) in both its generator and discriminator, constraining the spatial input to a regular grid. Consequently, for a fair comparison of the emulation performance between our XVAE and extGAN, we must use their specific simulations setup, including the same grid size and number of spatial locations. Altering the number of locations would require a complete re-design and tuning of the neural network architecture to accommodate the new dimensions.  

\begin{figure}[!t]
    \centering
    \includegraphics[height=0.3\linewidth]{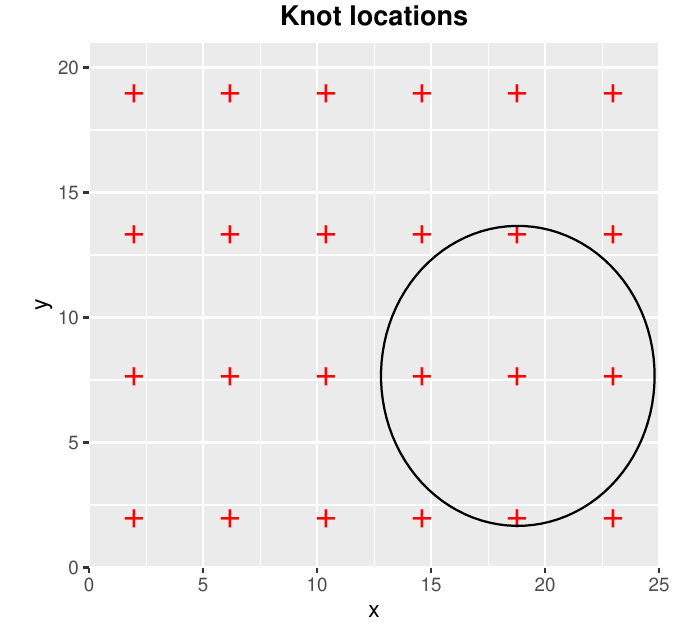}
    \includegraphics[height=0.3\linewidth]{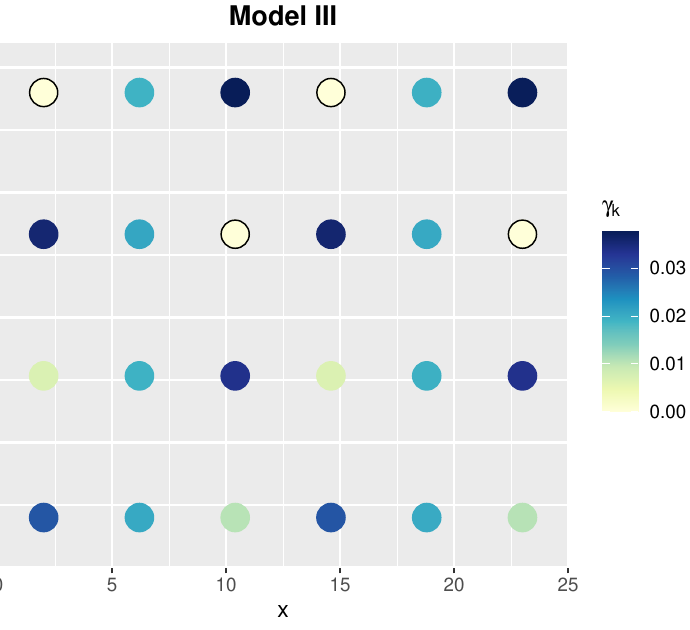}
    \vskip -0.3cm
    \caption{The left panel presents knot locations used for Model~\ref{modelFlex} but a different spatial setting to be consistent with extGAN, and we only show the support of the one Wendland basis function centered at knot in the middle of the domain. The right panels display the $\gamma_k$ values, $k=1,\ldots, K$, used in the expPS variables. The circled knots signify $\gamma_k=0$, which induces local AD.}
    \label{fig:sim_knots_extGAN}
\end{figure}

Here, we simulate from the max-id model setting (i.e., Model~\ref{modelFlex} from Section~\ref{sec:simulation}) with a non-stationary dependence structure. Specifically, we use an $18\times 22$ grid, with $K=24$ evenly-spaced knots, employing compactly supported Wendland basis functions centered at each knot with $r=6$. As in Section~\ref{sec:simulation}, we use a mix of positive and zero $\gamma_k$'s; see Figure~\ref{fig:sim_knots_extGAN} for an illustration of the knot placement and $\{\gamma_k\}$ values. Since extGAN assumes stationary marginal distributions at each grid site,  we only consider time-invariant dependence parameters $\alpha_t\equiv 1/2$ and $\bgamma_t\equiv \bgamma$. Following the simulation setup for  precipitation and temperature applications in \citet{boulaguiem2022modeling}, we simulate $n.t=500$ independent replicates.

The extGAN implementation by \citet{boulaguiem2022modeling} was coded in TensorFlow version 1.0, which is incompatible with Python versions $\geq$ 3.0 and modules such as \texttt{pandas} and \texttt{tensorflow\_probability}. Additionally, it includes inconsistencies with TensorFlow operations in \texttt{Keras} layers, hindering the execution of \texttt{tf.function} compilation. To address these limitations, we translated their GAN implementation to a \texttt{tf.keras.Model} class; this modified code is available in our GitHub repository at \href{https://github.com/likun-stat/XVAE}{https://github.com/likun-stat/XVAE}.

In general, GANs aim to generate diverse images that resemble the overall distribution of the training data rather than replicate any \textit{one specific} image. As outlined in Algorithm~1 of \citet{boulaguiem2022modeling}, the extGAN is indeed trained to generate random images from the empirical copula---the empirical joint distribution of the data after a rank transformation. Figure~\ref{copula-map} illustrates three rank-transformed inputs (from the empirical copula) used to train extGAN (top) and three generated images on the copula scale (bottom), showing that extGAN seeks to capture the copula's overall dependence structure without emulating single realizations. Although GAN inversion or generator overfitting could enable emulation of specific training images, this approach diverges from standard GAN usage. 

\begin{figure}
    \centering
    \includegraphics[width=0.9\linewidth]{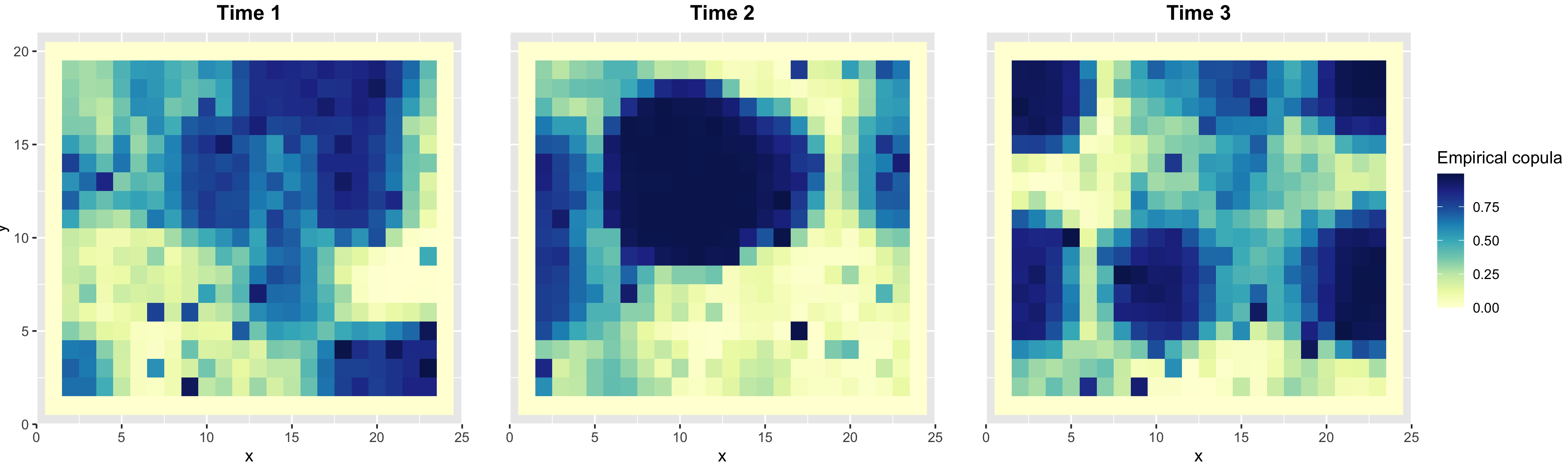}
    \includegraphics[width=0.9\linewidth]{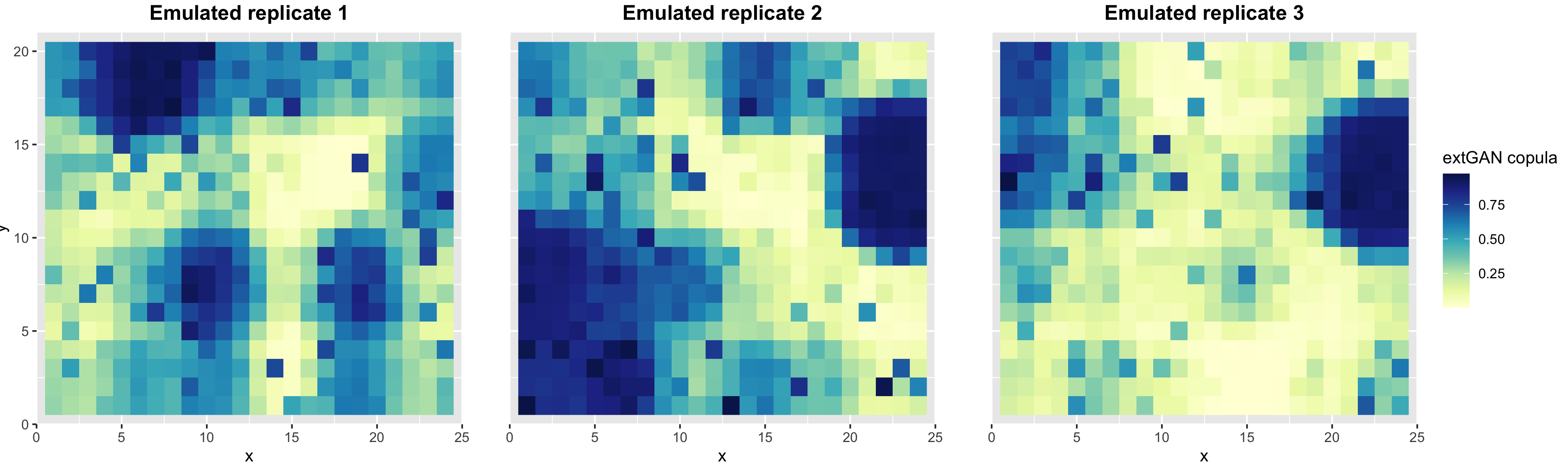}
    \caption{Top panels show the copulas (marginally transformed to standard uniform distributions using rank transformation) from the first three time replicates. Note \citet{boulaguiem2022modeling} padded the copulas with 0's at the periphery of the domain for the application of CNNs. Bottom panels show three copulas generated by extGAN, in which the GAN does not try to replicate a specific image but rather the overall patterns of the spatial distribution.}
    \label{copula-map}
\end{figure}

Also, due to the use of CNNs, extGAN is not able to handle missing locations across space, preventing out-of-sample emulation performances assessment via CRPS. Instead, we compare the overall dependence structures from both emulation methods to the original simulated data. Figure~\ref{fig:extGAN_comp} compares the overall marginal distribution obtained by pooling data across space and time. The emulated copula from extGAN is transformed to the data scale using the analytic marginal distribution function derived in Proposition \ref{prop:marg_distr} with the true parameters. Note that in real data applications, we have to design the marginal models carefully and estimate the model parameters well if we want to use extGAN. The right panel of Figure~\ref{fig:extGAN_comp} shows that, in this case, the marginal data quantiles are quite severely underestimated when using extGAN emulations even though we used the true parameters. This is largely due to the mis-characterization of the dependence structure in the copula. The left panel of Figure~\ref{fig:extGAN_comp2} shows the comparison of MSPE in log scale (see definition in Section~\ref{sec:predict}), confirming that the quality of emulation is a bit higher when a parametric constraint on the copula is imposed, as with our XVAE. Furthermore, we compare the ARE estimates from the simulated data and the emulations. The right panel of Figure~\ref{fig:extGAN_comp2} shows that the length scale of threshold exceedance is slightly overestimated for large quantile levels, although the corresponding uncertainty bands are a bit wider and contain the estimates from the simulated data (i.e., the ``truth'') and the XVAE emulations.

\begin{figure}
    \centering
    \hspace*{-1.7cm}\includegraphics[width=0.4\linewidth]{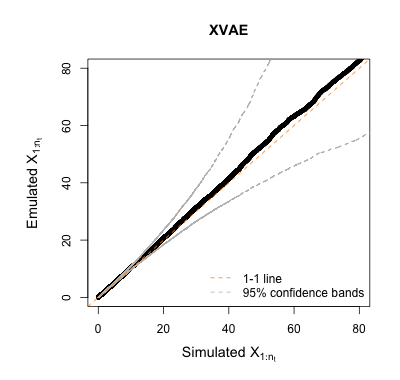}
    \includegraphics[width=0.4\linewidth]{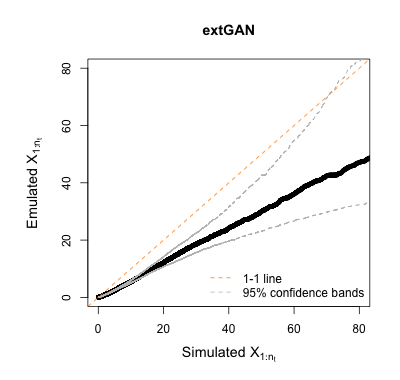}
    \caption{QQ-plots comparing simulated data sets and emulated fields from XVAE (left) and extGAN (right) based on Model~\ref{modelFlex}. Data are pooled across space and time.}
    \label{fig:extGAN_comp}
\end{figure}

For this simulated dataset, extGAN also takes longer to train ($\sim$3 hours) compared to our XVAE ($\sim$ 30 minutes). In addition, in real data applications where the marginal parameters are time-varying, generating arbitrary random images (even on the copula scale) may be not sensible, especially for the purpose of emulation.

\begin{figure}
    \centering
    \includegraphics[width=0.4\linewidth]{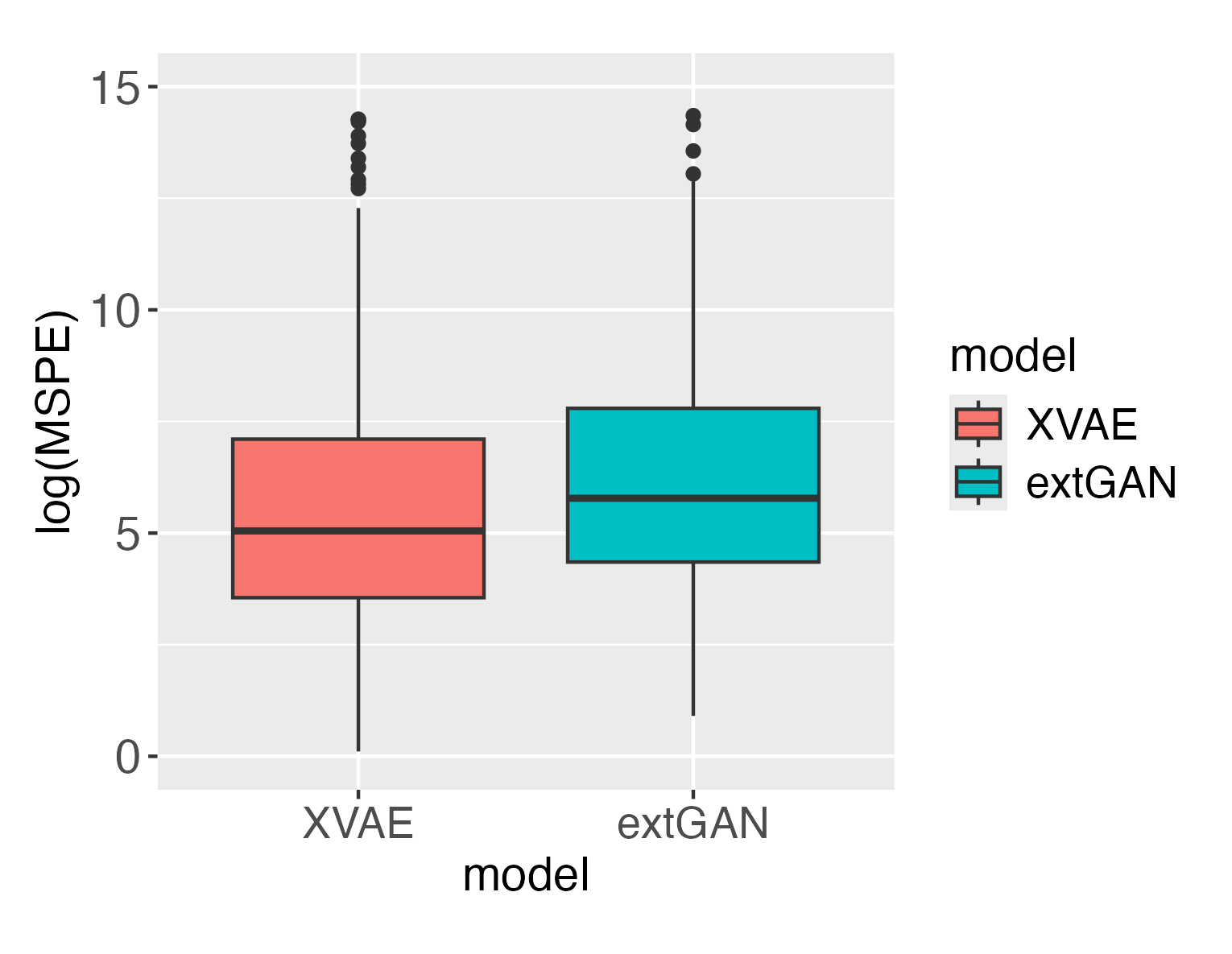}
    \includegraphics[width=0.4\linewidth]{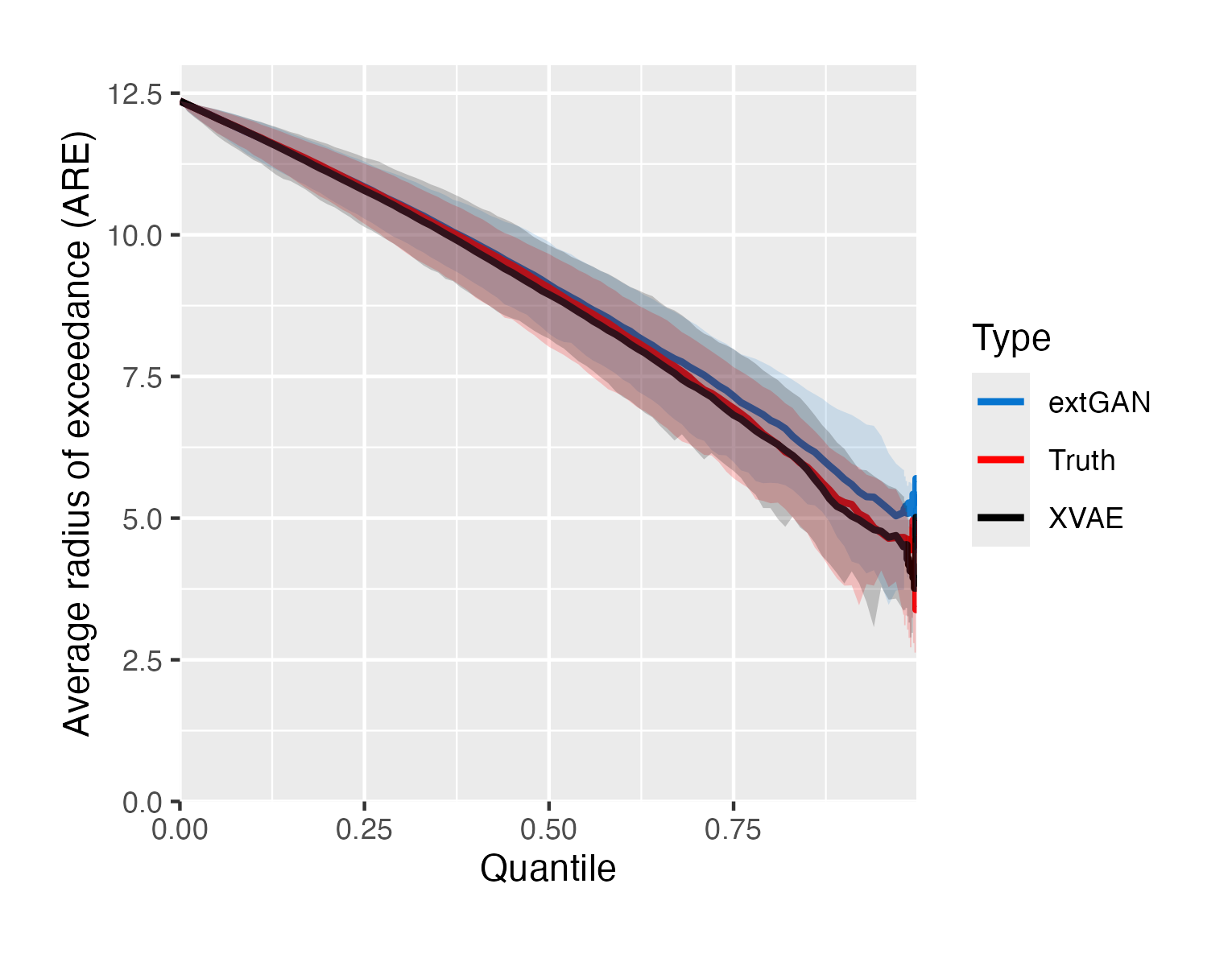}
    \caption{On the left, we show the MSPE values from two emulation approaches on the dataset simulated from Model~\ref{modelFlex}. Lower MSPE values indicate better emulation results. On the right, we show $\mathrm{ARE}_\psi(u)$ with $\psi=1$ based on data replicates (black), XVAE emulated data (red) and extGAN emulated data (blue).}
    \label{fig:extGAN_comp2}
\end{figure}

\section{Red Sea Dataset}
This dataset has previously been analyzed (sometimes partially) by \citet{hazra2021estimating}, \citet{simpson2021conditional}, \citet{simpson2020high}, \citet{oesting2022patterns}, and \citet{sainsbury-dale2022}. The latter three studies focused on a small portion of the Red Sea using the summer months only to eliminate the effects of seasonality. For example, \citet{sainsbury-dale2022} retained a dataset with only 678 spatial locations and 141 replicates. By contrast, \citet{hazra2021estimating} extensively studied weekly data over the entire spatial domain using a Dirichlet process mixture of low-rank spatial Student’s $t$ processes to account for 
spatial dependence. However, their model is AD across the entire domain (i.e., for any pair of locations), limiting its flexibility in capturing extreme behavior.

\subsection{Removing seasonality}\label{appendix:remove_season}
For any site $\bs_j$, we combine daily observations across all days as a vector and denote it by $\bv_j=(v_{j1}, \cdots, v_{jN})^\top$ where $N=11,315$ is the number of days between 1985/01/01 and 2015/12/31. Following \citet{huser2021eva}, we remove the seasonality from the Red Sea SST daily records at a fixed $\bs_j$ via subtracting the overall trend averaged within its neighborhood of radius $r=30$ km, and then we repeat the same procedure for every other location.

More specifically, denote the index set of all location with the neighborhood of $\bs_j$ by $\mathcal{N}_j= \{i: ||\bs_i-\bs_j|| < r, \;i=1,\ldots,n_s \}$.  To get rid of the seasonality in $\bv_j$, we first concatenate all records in the neighborhood $\{\bv_i:i\in \mathcal{N}_j\}$ to get a flattened response vector $\boldsymbol{V}_j$; that is, $\boldsymbol{V}_j=(\bv_{i_1}^\top,\bv_{i_2}^\top, \cdots, \bv_{i_{|\mathcal{N}_j|}}^\top)^\top$ where $\{i_1,\ldots,i_{|\mathcal{N}_j|}\}$ include all elements of $\mathcal{N}_j$. Thus, the length of the vector $\boldsymbol{V}_j$ is $|\mathcal{N}_j|\times N$. Second, we construct the matrix $\boldsymbol{M}=(\boldsymbol{1}_{N},\boldsymbol{t},\boldsymbol{B}_{N \times 12})$, where $\boldsymbol{t}=(1,\ldots,N)^\top$ is used to capture linear time trend and the columns of $\boldsymbol{B}$ are $12$ cyclic cubic spline bases defined over the continuous interval $[0,366]$ evaluated at $1, \ldots, N$ modulo 365 or 366 (i.e., the day in the corresponding year). These basis functions use equidistant knots over of $[0,366]$ that help capturing the monthly-varying features. Then, we vertically stack the matrix $\boldsymbol{M}$ for $|\mathcal{N}_j|$ times to build the design matrix $\boldsymbol{M}_j$. Through simple linear regression of $\boldsymbol{V}_j$ on $\boldsymbol{M}_j$, we get the fitted values $\hat{\boldsymbol{V}}_j=(\hat{\bv}_{i_1}^\top,\hat{\bv}_{i_2}^\top, \cdots, \hat{\bv}_{i_{|\mathcal{N}_j|}}^\top)^\top$. 

To model the residuals $\boldsymbol{V}_j - \hat{\boldsymbol{V}}_j$, we only use an intercept and a time trend which are the first two columns of $\boldsymbol{M}_j$ (denote as $\boldsymbol{M}^{\sigma}_j$). The model for the residuals is
\begin{align*}
    \boldsymbol{V}_j - \hat{\boldsymbol{V}}_j  &\sim N(\boldsymbol{0}, \mathrm{diag}(\bE^2_j)),\\
    \log\bE_j &= \boldsymbol{M}^{\sigma}_j \times (\beta_1,\beta_2)^\top.
\end{align*}
Hence we can estimate parameters $(\beta_1,\beta_2)^\top$ via optimizing the multivariate normal density function, i.e.,
\begin{align*}
    (\hat{\beta}_1,\hat{\beta}_2)^\top = \operatorname*{argmin}_{(\beta_1,\beta_2)^\top} \left\{-\frac{1}{2}\log \boldsymbol{1}^\top \bE^2_j- \frac{1}{2}(\boldsymbol{V}_j- \hat{\boldsymbol{V}}_j)^\top\mathrm{diag}(\bE^{-2}_j)(\boldsymbol{V}_j - \hat{\boldsymbol{V}}_j)\right\}.
\end{align*}

Let $\hat{\bE}_j = \exp\{\boldsymbol{M}^{\sigma}_j \times (\hat{\beta}_1,\hat{\beta}_2)^\top\}\equiv(\hat{\be}_{i_1}^\top,\hat{\be}_{i_2}^\top, \cdots, \hat{\be}_{i_{|\mathcal{N}_j|}}^\top)^\top$. Note that in defining the neighborhood of site $\bs_j$, we also include the $j$th site. By an abuse of notation, we denote the fitted values corresponding to the $j$th site by $\hat{\bv}_j$ and $\hat{\be}_j$, which correspond to the mean trend and residual standard deviations at site $\bs_j$, respectively. Finally, the daily records at $\bs_j$ can be de-trended by calculating 
\begin{equation}\label{eqn:detrend_obs}
    \bv_j^* = \frac{\bv_j - \hat{\bv}_j}{\hat{\be}_j},
\end{equation}
in which the subtraction and division are done on a elementwise basis. We repeat the procedure described above to remove the seasonal variability from all other locations. 

\subsection{Marginal distributions of the monthly maxima}\label{appendix:marg_distr}

After removing seasonality by normalization (see Eq.~\eqref{eqn:detrend_obs}), we extract monthly maxima from $\bv_j^*$ at site $\bs_j$ and denote them as $\bm_j=(m_{j1}, \ldots, m_{jn_t})^\top$, in which $n_t=372$ is the number of months between 1985/01/01 and 2015/12/31 and $j=1\ldots, n_s$. Before applying our proposed model, we need to find a distribution which fits the monthly maxima well so we can transform the data to the Pareto-like distribution shown in Eq.~\eqref{eqn:marg_cdf} of the main paper. Given prior experience in analyzing monthly maxima, we propose two candidate distributions: the generalized extreme value (GEV) distribution and the general non-central $t$ distribution. To choose between them, we choose to perform $\chi^2$ goodness-of-fit tests 
due to its flexibility in choosing the degrees of freedom as well as the size of intervals. 

The $\chi^2$ goodness-of-fit test at a site $\bs_j\in \mathcal{S}$ proceeds as follows. First, we calculate the equidistant cut points within the range of all monthly maxima at $\bs_j$ to get $n_I$ intervals. Second, we count the number of monthly maxima falling within each interval and denote them by $O_i$ $(i=1,\ldots,n_I)$. Third, we fit the GEV and $t$ distributions to the block maxima series at $\bs_j$ to get the parameter estimates. Then the expected frequencies $E_i$ $(i=1,\ldots,n_I)$ is calculated by multiplying the number of monthly maxima at each site (i.e., $n_t$) by the probability increment of the fitted GEV or $t$ distribution in each interval (denoted by $p_i$). Treating the frequencies as a multinomial distribution with $n_t$ trials and $n_I$ categories, we can derive the generalized likelihood-ratio test statistic for the null hypothesis $H_0$ that $(p_1, \cdots, p_{n_I})^\top$ are the true event probabilities. Specifically, under the null hypothesis $H_0$, Wilk's Theorem guarantees
\begin{align*}
    \sum_{i=1}^{n_I} O_i \log(O_i/E_i)\stackrel{d}{\rightarrow} \chi^2_{\nu}\text{ as }n_t\rightarrow\infty,
\end{align*}
in which $\nu=n_I-4$ when $H_0$ corresponds to the GEV model which has three parameters (i.e., location, scale, and shape) and $\nu=n_I-3$ when $H_0$ corresponds to the $t$ model which has two parameters (i.e., non-centrality parameter and degrees of freedom). Since $n_t=372$ in the Red Sea SST data, we can safely assume that the asymptotic distribution is a good approximation of the true distribution under $H_0$, which is then used to calculate the $p$-value to evaluate the goodness-of-fit. 

We repeat the procedure and obtain a $p$-value for each location.  Figure~\ref{fig:Chisq_pval} shows the spatial maps for $p$-values along with the binary maps signifying whether the null hypothesis is accepted or not with significance level $0.05$. In Figure~\ref{fig:B(1)}, the goodness-of-fit tests result in $p$-values greater than $0.05$ at all locations, indicating GEV distribution is a good fit for all monthly maxima time series. For the shaded locations in Figure~\ref{fig:B(4)} and \ref{fig:B(2)}, the \texttt{fitdistr($\cdot$, "t")} function from the \texttt{MASS} package in \texttt{R} failed to converge when optimizing joint $t$ likelihood, and we were not able to obtain parameter estimates of the $t$ distribution at these locations which were needed for the subsequent $\chi^2$ tests. For the locations that have valid fitted $t$ distributions in Figure~\ref{fig:B(4)}, the $p$ values are mostly less than those in Figure~\ref{fig:B(3)}. This indicates that the GEV distribution, the asymptotic distribution for univariate block maxima, is a better choice to describe the marginal distribution of the monthly maxima, as expected.

\begin{figure}
     \centering
     \begin{subfigure}[b]{0.24\textwidth}
         \centering
         \includegraphics[width = 1.8in]{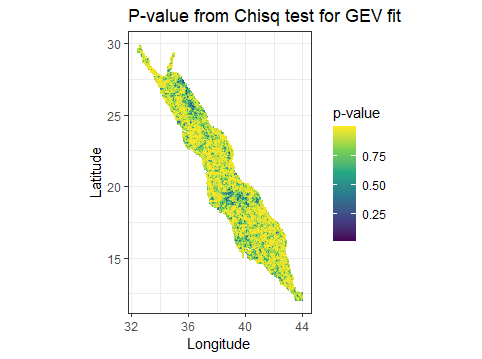}
         \caption{}
         \label{fig:B(3)}
     \end{subfigure}
     \hfill
     \begin{subfigure}[b]{0.24\textwidth}
         \centering
         \includegraphics[width = 1.8in]{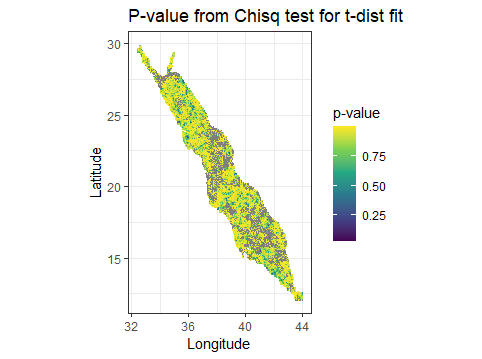}
         \caption{}
         \label{fig:B(4)}
     \end{subfigure}
     \hfill
     \begin{subfigure}[b]{0.24\textwidth}
         \centering
         \includegraphics[width = 1.8in]{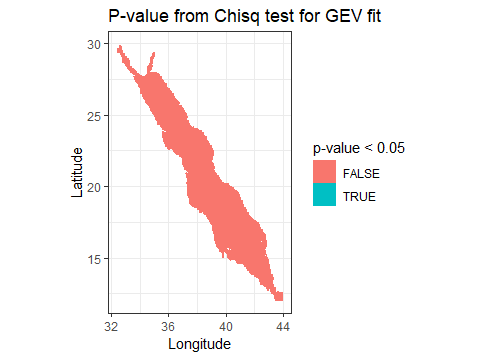}
         \caption{}
         \label{fig:B(1)}
     \end{subfigure}
     \hfill
     \begin{subfigure}[b]{0.24\textwidth}
         \centering
         \includegraphics[width = 1.8in]{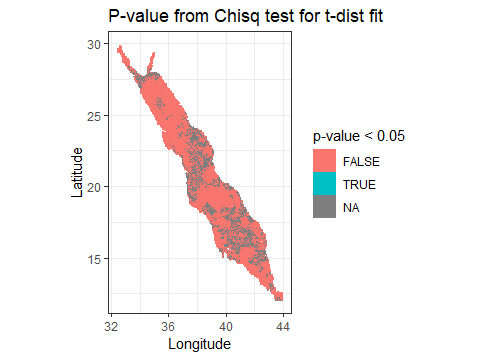}
         \caption{}
         \label{fig:B(2)}
     \end{subfigure}
        \caption{In the left two panels, we show heatmaps of $p$-values from $\chi^2$ goodness-of-fit tests under the GEV model in \subref{fig:B(3)} and the $t$ model in \subref{fig:B(4)}. In the right two panels, we show binary $p$-values maps from $\chi^2$ goodness-of-fit tests under the GEV model in \subref{fig:B(1)} and the $t$ model in \subref{fig:B(2)}.}
        \label{fig:Chisq_pval}
\end{figure}

\subsection{Marginal transformation}\label{appendix:margin_trans}
Before applying our model to monthly maxima, certain transformations need to be done to match our marginals in Section~\ref{sec:ext_dep}. When performing the goodness-of-fit tests, we already obtained the sitewise GEV parameters: $\mu_j$, $\sigma_j$, and $\xi_j$ for $j=1,\ldots,n_s$. Since monotonic transformations of the marginal distributions do not alter the dependence structure of the data input, we define $x_{jt}=F_{jt}^{-1}\{F_\mathrm{GEV}(m_{jt};\mu_j, \sigma_j, \xi_j)\},\; t=1,\ldots, n_t, \;j=1,\ldots, n_s$, in which $F_{jt}$ is the marginal distribution function of $X_t(\bs_j)$ displayed in Eq.~\eqref{eqn:marg_cdf} of the main paper, the function $F_\mathrm{GEV}(\cdot;\mu_j, \sigma_j, \xi_j)$ is the distribution function of $\mathrm{GEV}(\mu_j, \sigma_j, \xi_j)$, and $m_{jt}$ is the monthly maximum at site $\bs_j$ from $t^\mathrm{th}$ month. Further, we have $\bx_t=(x_{1t}, x_{2t}, \cdots, x_{n_st})^\top$, $t=1,\ldots, n_t$, which will be treated as the response in Algorithm \ref{alg:sim}. It should be noted that $F_{jt}$ is defined with the parameters $\alpha_t$, $\bgamma_t$ and $\bW$. Recall that the matrix $\bW$, defined in Eq.~\eqref{eqn:decoder}, contains the basis function evaluations at all locations.  After updating these parameters in each iteration of the stochastic gradient descent algorithm, we need to update the values of $\{x_{jt}: t=1,\ldots, n_t, \;j=1,\ldots, n_s\}$ before continuing the next iteration.



\subsection{Empirical \texorpdfstring{$\chi_h(u)$}{Lg} estimates}\label{appendix:chi_ests}
\begin{figure}[H]
    \centering
    \includegraphics[height=0.29\linewidth]{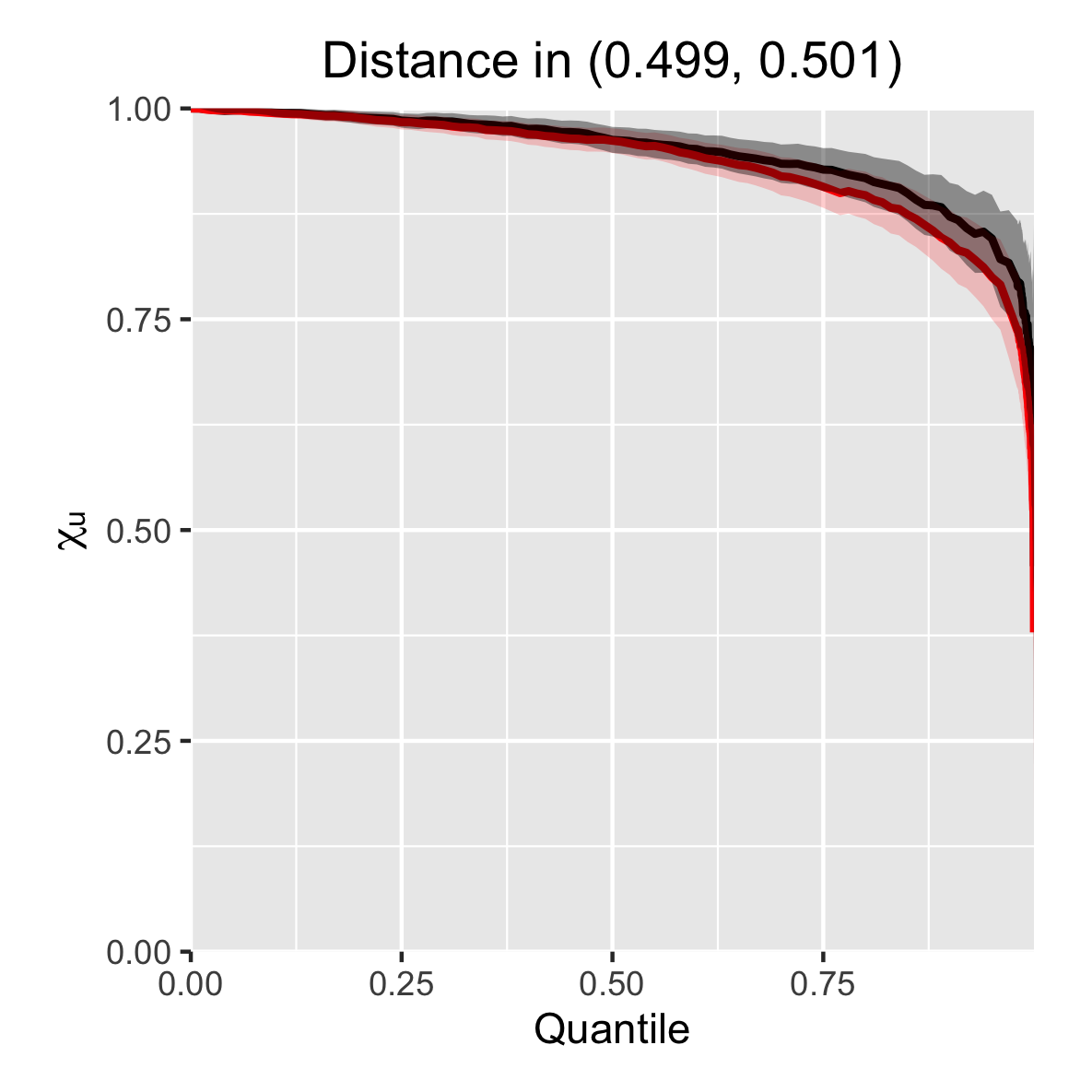}
    \includegraphics[height=0.29\linewidth]{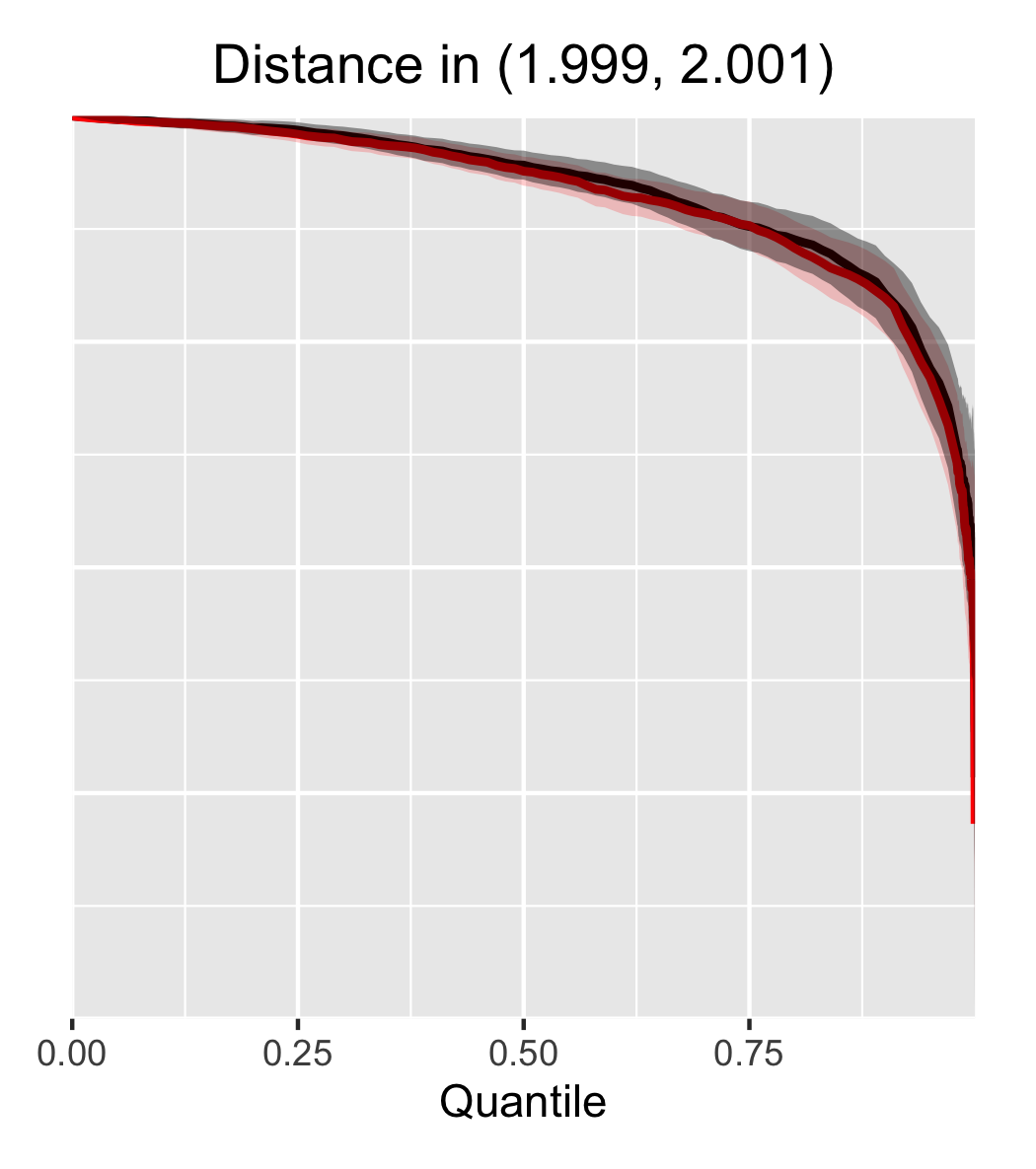}
    \includegraphics[height=0.29\linewidth]{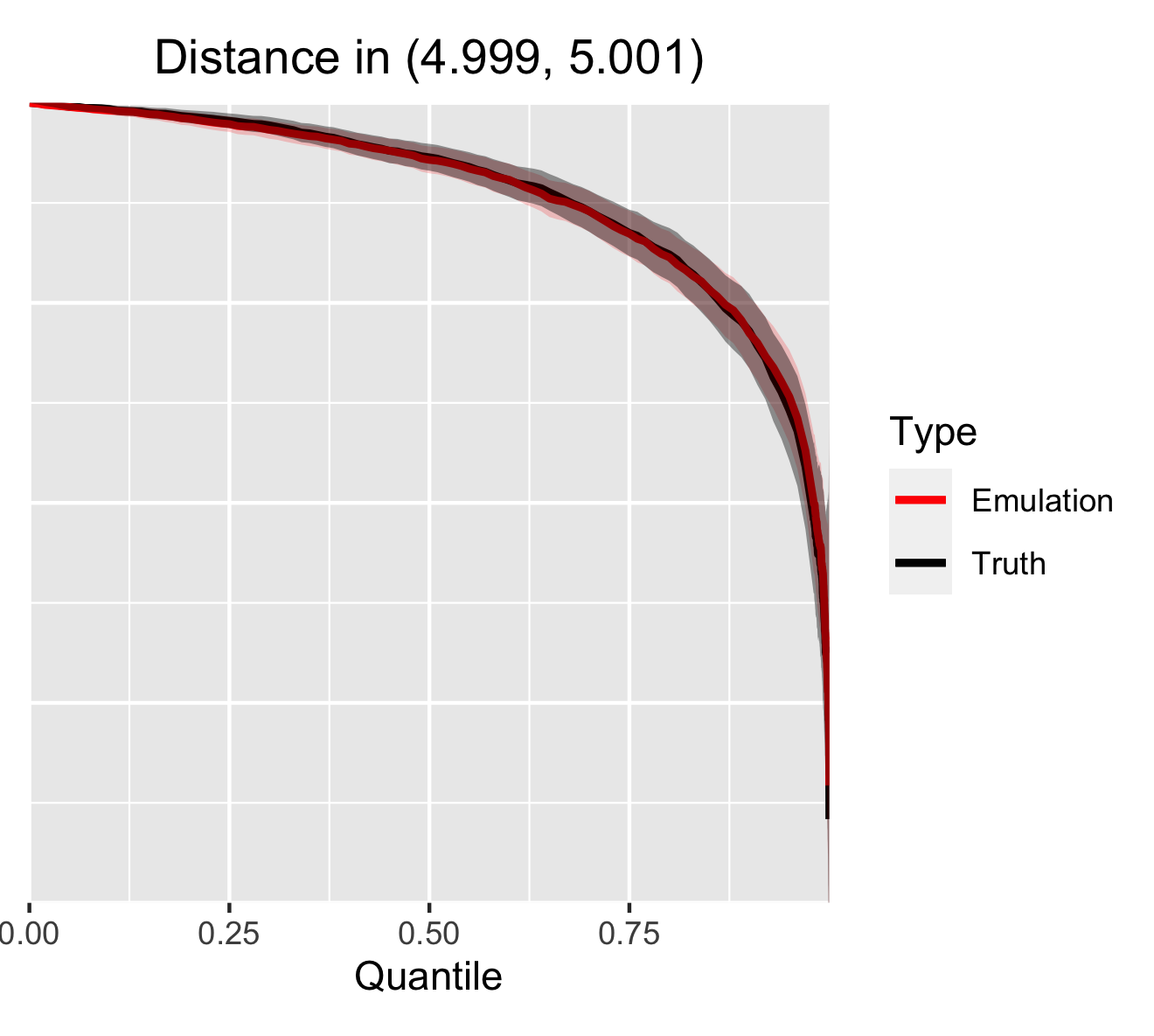}
    \vskip -0.3cm
    \caption{Empirically-estimated $\chi_h(u)$ for $h=0.5,2,5$ ($\approx 50 \text{km, } 200\text{km, } 500\text{km}$) for the Red Sea SST monthly maxima (black) and the XVAE emulations (red).}
    \label{fig:chi_SST}
\end{figure}

\subsection{Additional results}
Figure~\ref{fig:data_app_emulation} shows emulated replicates of the original monthly maxima field for the first and last months (1985/01 and 2015/12, respectively). Here, we convert the emulated values back to the original data scale using the estimated GEV parameters fitted from the previous step. Figure~\ref{fig:data_app_emulation} demonstrates that the XVAE is able to capture the detailed features of the temperature fields and to accurately characterize  spatial dependence, while the QQ-plot shows an almost perfect alignment with the 1-1 line.

\begin{figure}[!t]
    \centering
    {\includegraphics[width=0.29\linewidth, trim={0 1cm 0 0}, clip]{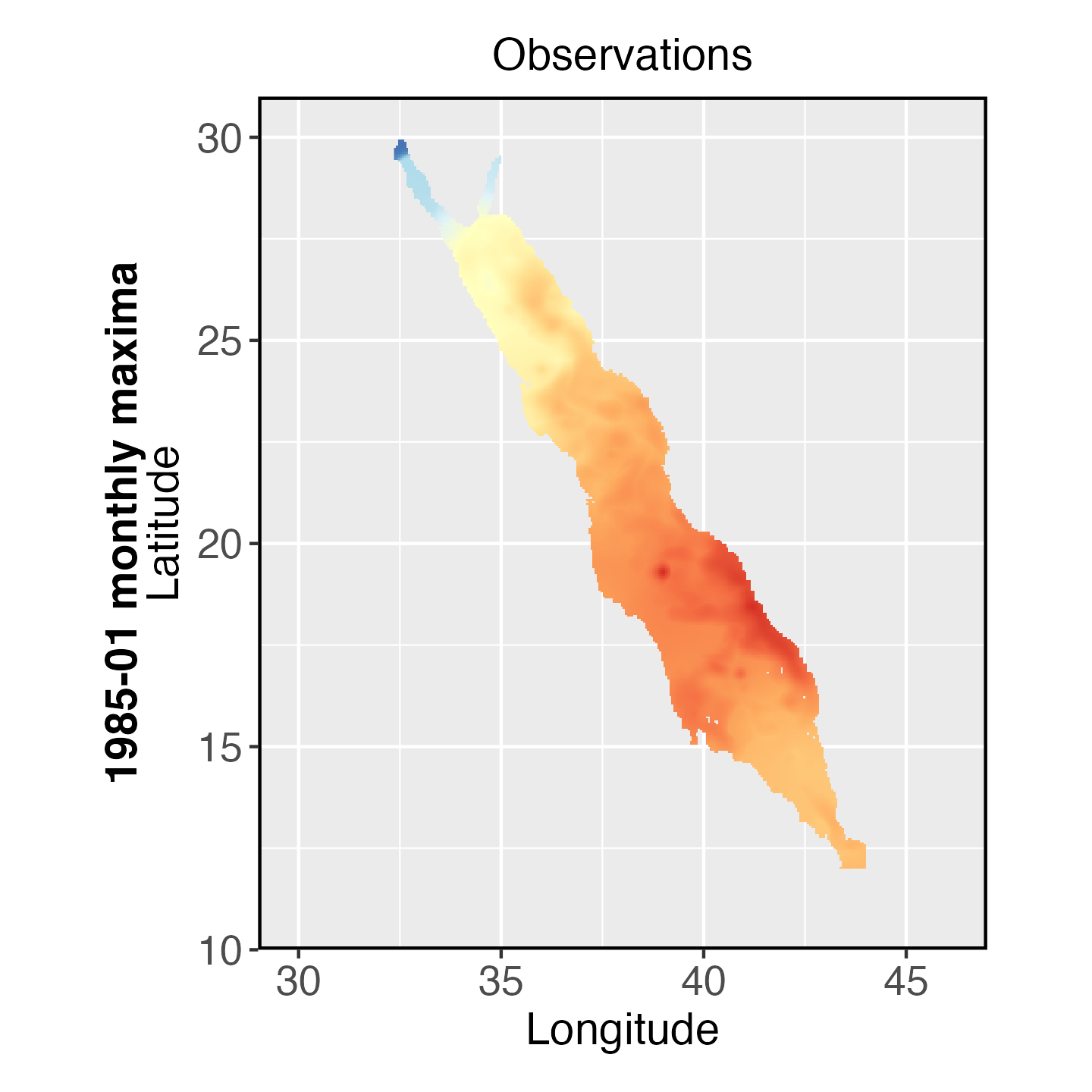}
    \includegraphics[width=0.29\linewidth, trim={0 1cm 0 0}, clip]{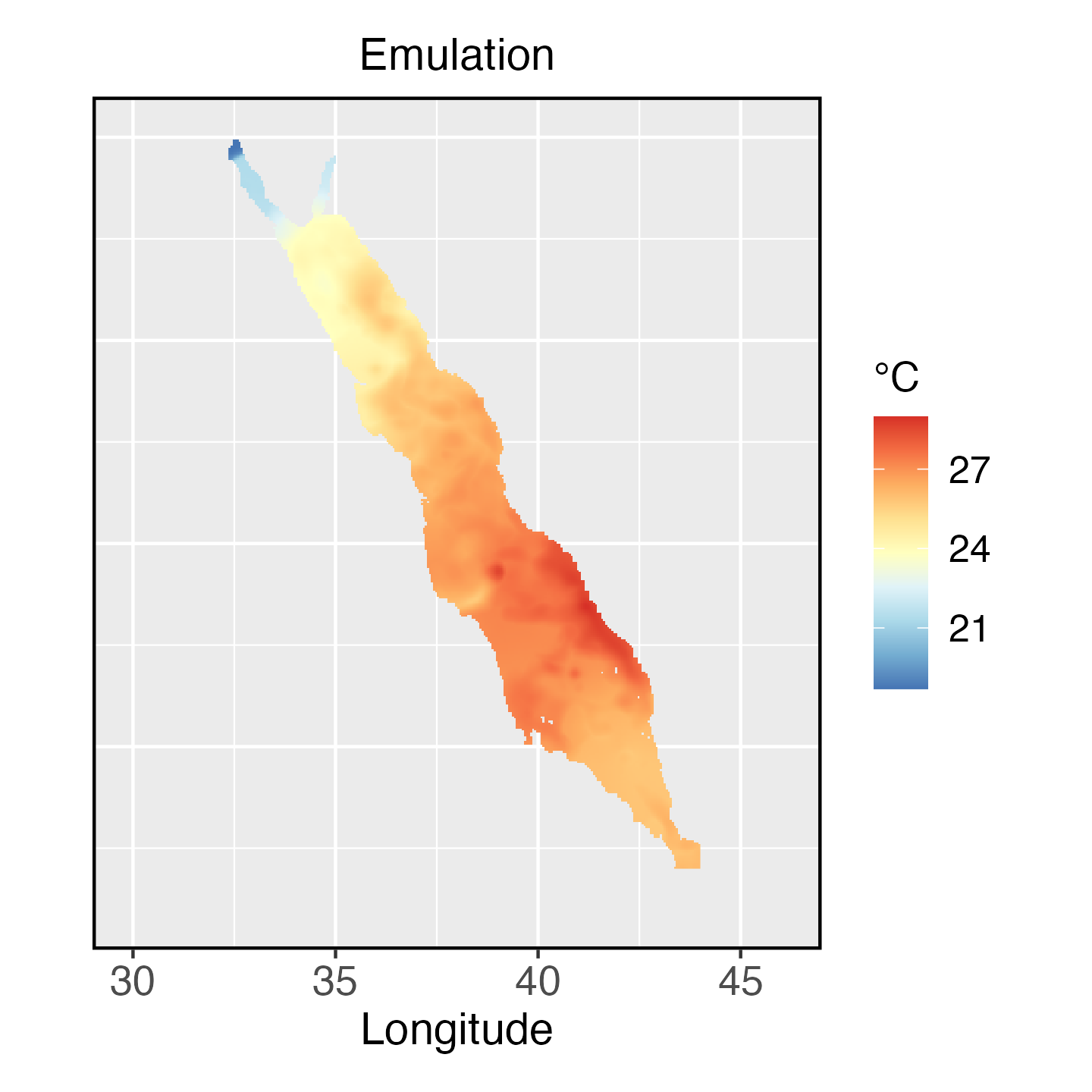}
    \includegraphics[width=0.29483\linewidth, trim={0 1cm 0 0cm}, clip]{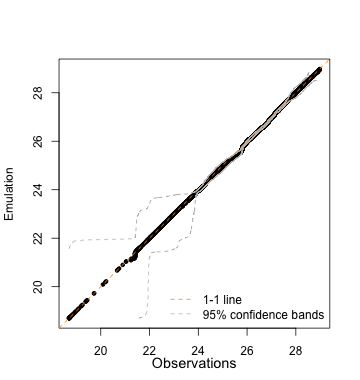}}
    \vskip 0.2cm
    
    {\raisebox{0.25\linewidth}{\includegraphics[width=0.29\linewidth, trim={0 0 0 1cm}, clip]{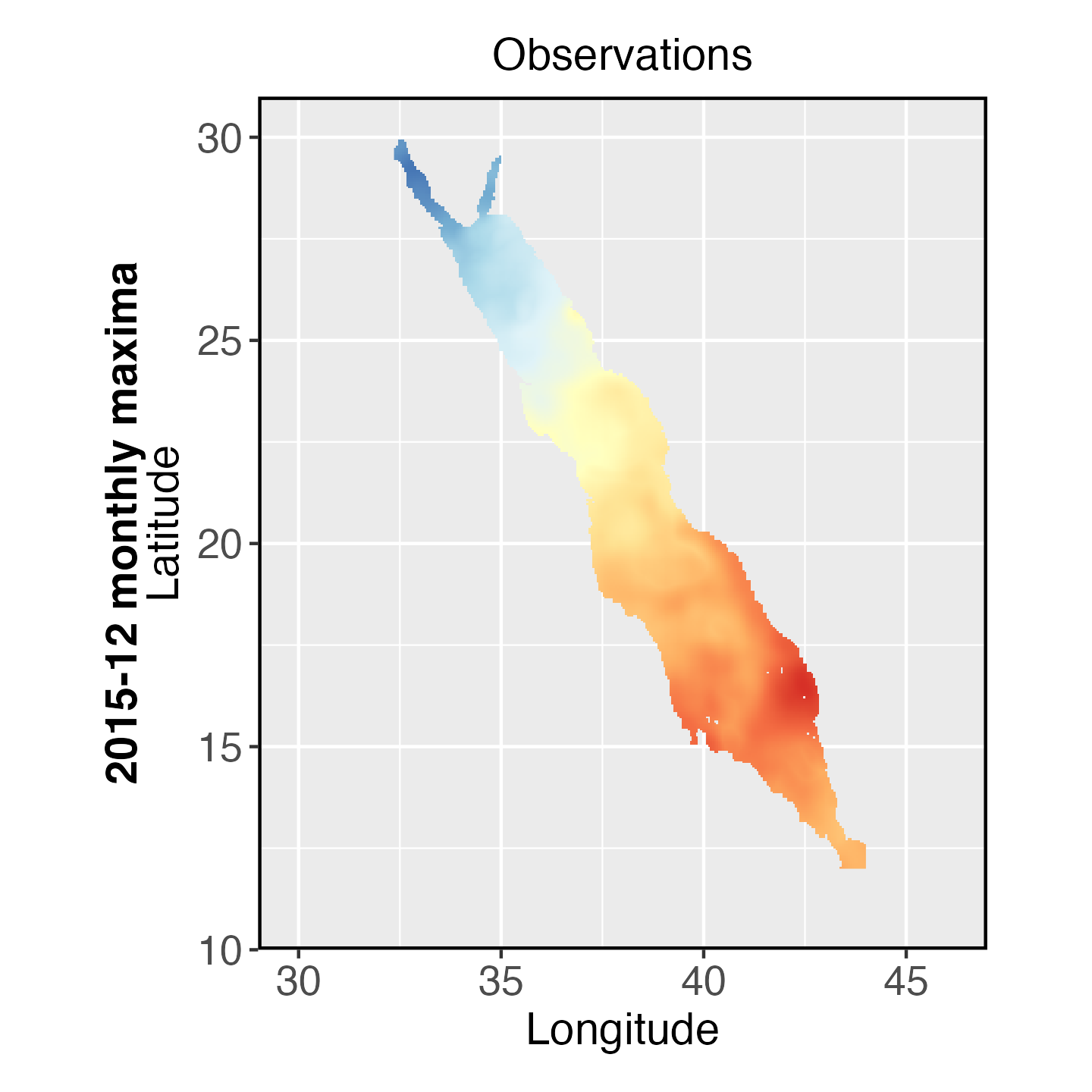}}
    \raisebox{0.25\linewidth}{\includegraphics[width=0.29\linewidth, trim={0 0 0 1cm}, clip]{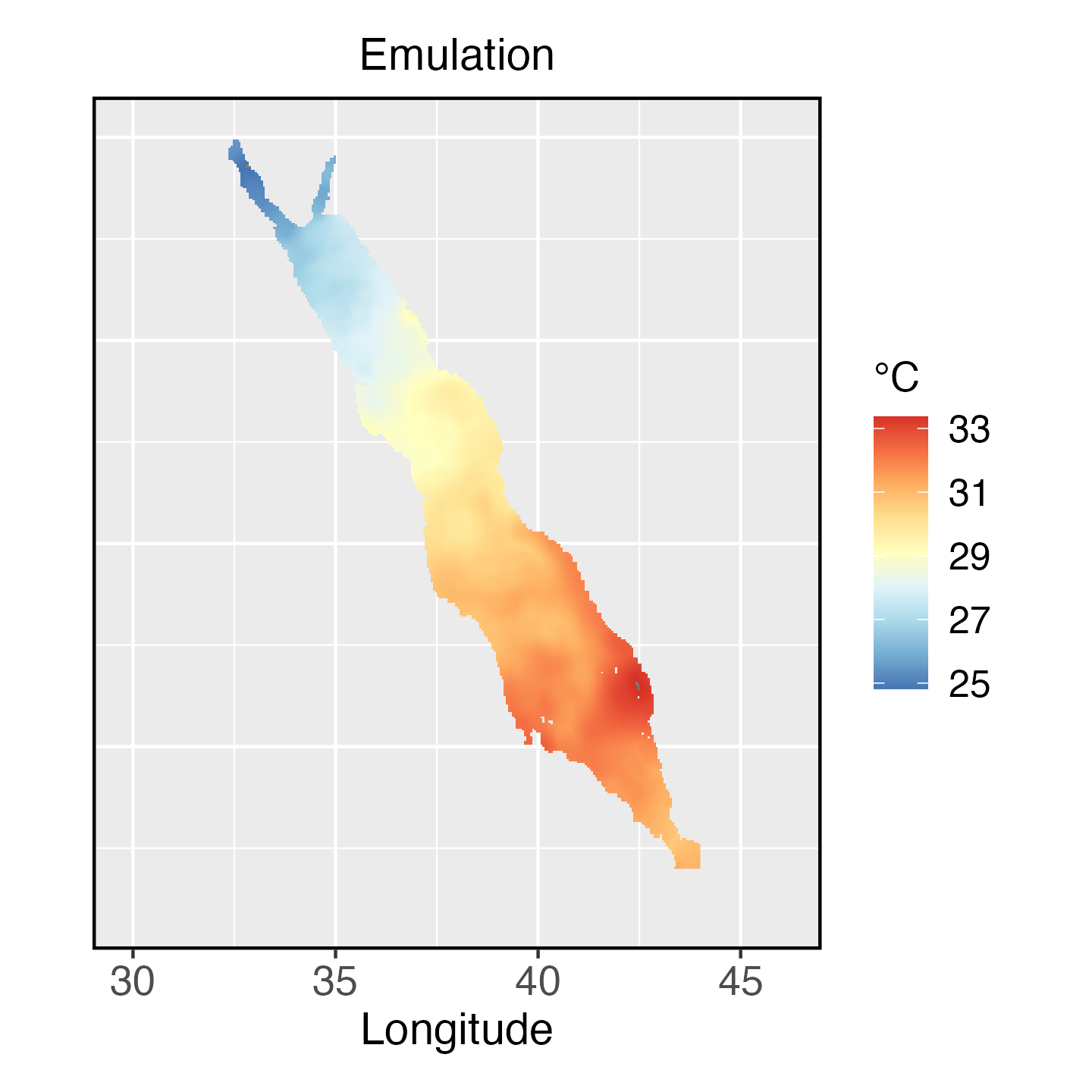}}
    \raisebox{0.25\linewidth}{\includegraphics[width=0.29483\linewidth, trim={0 0 0 2cm}, clip]{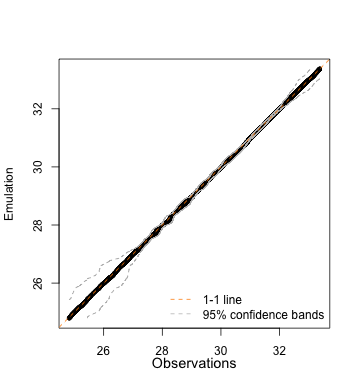}}}
    \vskip -0.27\linewidth
    \caption{Observed (left) and emulated (middle) Red Sea SST monthly maxima, for the 1985/01 (top) and 2015/12 (bottom) months. From the emulation maps and QQ plots (right), we see that the emulated fields from the XVAE match the observations very well.}
    \label{fig:data_app_emulation}
\end{figure}

\end{document}